\documentclass[11pt]{article}

\RequirePackage{amsthm,amsmath,amsfonts,amssymb}
\RequirePackage[numbers]{natbib}
\usepackage{mathtools}
\usepackage{bm}
\usepackage{bbm}
\usepackage{booktabs,multirow}
\usepackage{algorithm}
\usepackage{algpseudocode}
\usepackage{graphicx}
\usepackage{xcolor}
\usepackage{enumerate}
\usepackage{cases}
\usepackage{tikz}
\usepackage{array}
\usepackage{enumitem}
\usepackage{placeins}
\usepackage[T1]{fontenc}
\usepackage{microtype}
\usepackage[margin=1in]{geometry}
\usepackage{authblk}
\usepackage[colorlinks,linkcolor=blue,citecolor=blue,urlcolor=blue]{hyperref}

\mathtoolsset{showonlyrefs}
\hypersetup{
  pdftitle={Wedge Sampling: Efficient Tensor Completion with Nearly-Linear Sample Complexity},
  pdfauthor={Hengrui Luo, Anna Ma, Ludovic Stephan, Yizhe Zhu}
}

\newcommand{\R}{\mathbb{R}}

\newcommand{\E}{\mathbb{E}}

\newcommand{\rank}{\mathrm{rank}}

\newcommand{\col}{\operatorname{col}}

\newcommand{\1}{\mathbf{1}}

\newcommand{\diag}{\operatorname{diag}}

\DeclareMathOperator{\unfold}{unfold}
\DeclareMathOperator{\sgn}{sgn}
\DeclareMathOperator*{\argmin}{\arg\,\min}

\newcommand{\iid}{\stackrel{\mathrm{i.i.d.}}{\sim}}
\newcommand{\Bern}{\operatorname{Bernoulli}}

\newcommand{\cY}{\mathcal{Y}}

\newcommand{\edit}[1]{#1}

\theoremstyle{plain}
\newtheorem{theorem}{Theorem}
\newtheorem{lemma}[theorem]{Lemma}
\newtheorem{proposition}[theorem]{Proposition}
\newtheorem{corollary}[theorem]{Corollary}

\theoremstyle{definition}
\newtheorem{definition}[theorem]{Definition}
\newtheorem{remark}[theorem]{Remark}
\newtheorem*{remark*}{Remark}

\begin{document}

\title{Wedge Sampling: Efficient Tensor Completion with Nearly-Linear Sample Complexity}
\author[1]{Hengrui Luo}
\author[2]{Anna Ma}
\author[3]{Ludovic Stephan}
\author[4]{Yizhe Zhu}
\affil[1]{Department of Statistics, Rice University}
\affil[2]{Department of Mathematics, University of California, Irvine}
\affil[3]{CREST, UMR 9194, Univ Rennes, ENSAI, CNRS, Bruz Cedex, France}
\affil[4]{Department of Mathematics, University of Southern California}
\date{}
\maketitle

% Author contact metadata from the current FoCM edition:
% Hengrui Luo: 6100 Main Street, Maxfield Hall, Houston, TX 77005, USA;
% hrluo@rice.edu.
% Anna Ma: 340 Rowland Hall, Irvine, CA 92697-3875, USA;
% anna.ma@uci.edu.
% Ludovic Stephan: Campus de Ker Lann, 51 Rue Blaise Pascal, BP 37203,
% 35172 Bruz Cedex, France; ludovic.stephan@ensai.fr.
% Yizhe Zhu: 3620 South Vermont Avenue, KAP 104,
% Los Angeles, CA 90089-2532, USA; yizhezhu@usc.edu.

\begin{abstract}
We introduce wedge sampling, a non-adaptive sampling scheme for recovering an order-$k$ low-rank tensor of side length $n$ from structured observations of its entries.
The scheme samples pairs of entries in the same column of a tensor unfolding, forming length-two paths (wedges) in the associated bipartite sampling graph.
The products of the paired entries supply the row correlations needed for spectral initialization in regimes where uniform sampling yields too few wedges.
For fixed rank, incoherence, tensor order, and condition number, wedge sampling enables polynomial-time recovery with nearly linear sample complexity in $n$.
The spectral initializer can be combined with existing refinement procedures using only an additional $\tilde O(n)$ uniformly sampled entries.
Spectral refinement achieves weak recovery to any fixed relative Frobenius error.
For symmetric third-order tensors with incoherent canonical polyadic factors, gradient descent achieves exact recovery in the noiseless setting, with error decreasing geometrically.
The improvement over the $\tilde O(n^{k/2})$ sample complexity typical of efficient methods under uniform entry sampling shows how the conjectured computational barrier depends on the observation model.
A key ingredient for the gradient descent analysis is a concentration bound for sparse random tensors in an incoherent spectral norm.
We also give error bounds for both procedures under additive Gaussian noise and explicit signal-to-noise conditions.
Beyond tensor completion, the same sampling scheme improves the sample complexity of one-sided matrix completion, where the goal is to recover the left singular subspace of a matrix with many more columns than rows.
\end{abstract}

%\noindent\textbf{Keywords:} Tensor completion; statistical-to-computational gap; wedge sampling; random tensors.
%
%\noindent\textbf{Mathematics Subject Classification:} Primary 68W20; Secondary 15A69, 62H25, 65F55.

%%\clearpage
%\setcounter{tocdepth}{2}
%%\tableofcontents
%%\clearpage
%
%\setcounter{secnumdepth}{3}

\section{Introduction}

Matrix completion studies the problem of reconstructing a matrix from a (typically random) subset of its entries by exploiting prior structural assumptions such as low rank and incoherence. Roughly speaking, when the underlying $n \times n$ matrix has \edit{fixed rank} and its eigenvectors are sufficiently incoherent, observing $\Omega(n \log n)$ entries sampled uniformly at random is sufficient for exact recovery via efficient optimization methods \citep{keshavan.montanari.ea_2009_matrix,keshavan2010matrix,candes.tao_2010_power,candes.plan_2010_matrix,recht2011simpler,candes2012exact,jain2013low}. This sample complexity is nearly optimal \edit{for fixed $r$}, since specifying a rank-$r$ matrix requires only \edit{$O(nr)$} degrees of freedom.

Tensor completion generalizes this problem to higher-order {tensors}, aiming to recover a low-rank tensor from a limited set of observed entries, for example, under uniform random sampling. As a natural higher-order analogue of matrix completion, tensor completion has found broad applications in areas such as recommendation systems \citep{frolov2017tensor}, signal and image processing \citep{govindu2005tensor,liu2012tensor}, and data science \citep{song2019tensor}.

Despite this close analogy, tensor completion behaves fundamentally differently from its matrix counterpart. In contrast to the classical matrix setting, tensor completion exhibits a pronounced trade-off between \emph{computational} and \emph{statistical} complexity: while information-theoretic considerations suggest that relatively few samples suffice for recovery, all currently known polynomial-time algorithms require substantially more observations than this optimal limit.

\paragraph{Polynomial-time methods}
A widely used polynomial-time approach to tensor completion is to reduce the problem to matrix completion via matricization. Concretely, one flattens an order-$k$ tensor of rank $r$ into an $n^{\lfloor k/2\rfloor} \times n^{\lceil k/2\rceil}$ matrix (square up to the parity of $k$), and then applies standard matrix-completion techniques. Since the unfolded tensor has rank \edit{at most} $r$, this procedure yields an efficient algorithm with sample complexity $\tilde O(n^{\lceil k/2\rceil})$.
For example, when $k=3$, the unfolded tensor becomes an $n \times n^2$ matrix with a highly unbalanced aspect ratio. In this case, exact recovery of the original tensor only requires one-sided matrix recovery, specifically, estimation of the left singular subspace, rather than full matrix completion. By further exploiting this connection to one-sided recovery, the sample complexity can be improved to $\tilde O(n^{k/2})$ for tensors of arbitrary order $k$. This sample complexity can be achieved by spectral methods or iterative algorithms with spectral initialization \citep{montanari.sun_2018_spectral,jain2014provable,xia2019polynomial,liu2020tensor,xia2021statistically,cai2021subspace,cai2022nonconvex,tong2022scaling,wang2023implicit,stephan2024non}.
An alternative perspective formulates tensor completion as a polynomial optimization problem, for which the sum-of-squares (SoS) hierarchy in semidefinite programming is particularly well suited. Polynomial-time algorithms based on the SoS framework were studied in \citep{barak.moitra_2016_noisy,potechin2017exact}, and also achieve a sample complexity of $\tilde O(n^{k/2})$.

\paragraph{Non-polynomial-time methods}
In the matrix completion setting, many successful algorithms are based on convex optimization, including nuclear norm minimization \citep{candes2012exact,candes.tao_2010_power} and max-norm minimization \citep{srebro2005rank,linial2007complexity}. These ideas admit natural extensions to tensors through generalized nuclear and max norms. Indeed, several tensor norm minimization approaches \citep{yuan2016tensor,yuan2017incoherent,ghadermarzy2019near,harris2021,harris2023spectral} achieve near-optimal sample complexity $\tilde O(n)$ for tensor completion.
However, these methods are not computationally efficient: computing the tensor nuclear norm is NP-hard \citep{hillar2013most}, and no efficient algorithm is known for tensor max-norm minimization. More recently, \citep{hamaguchi2024sample} established a tight $\Theta(n \log n)$ sample complexity bound for exact tensor completion.

\paragraph{Statistical-to-computational gap in \edit{tensor estimation}}
Despite substantial progress, a significant gap remains between the sample complexity achievable by known polynomial-time algorithms and the statistical lower bound for tensor completion. In particular, \citep{barak.moitra_2016_noisy} conjectured that for noisy tensor completion of an order-$k$ tensor, any polynomial-time algorithm requires at least $\Omega(n^{k/2})$ samples, drawing connections to the hardness of refuting random $k$-SAT instances. This conjecture is consistent with a broader body of work on tensor estimation problems, where a pronounced statistical-to-computational gap is widely believed to exist \citep{montanari.richard_2014_statistical,wein2025kikuchi,ding2020estimating,jagannath2020statistical,arous2020algorithmic,auddy2021estimating,ding2022fast,dudeja2022statistical,arous2021long,luo2024tensor}.
It is worth noting that the lower bound for noisy tensor completion in \citep{barak.moitra_2016_noisy} does not apply to the noiseless rank-one case. In this regime, $\tilde O(n)$ samples are sufficient for partial or exact recovery under uniform sampling, using simple algorithms based on solving linear systems \citep{stephan2024non,gomez2024simple}. However, these techniques do not extend beyond the noiseless rank-one setting. Consequently, for tensors of rank greater than one, all existing results under uniform sampling support the conjectured statistical-to-computational gap.

\paragraph{Other sampling schemes and related settings}
Beyond tensor completion under uniform sampling, relatively few works achieve improved or nearly linear sample complexity through alternative sampling schemes. For instance, {Krishnamurthy and Singh~\citep{krishnamurthy2013low}} study adaptive sampling strategies that attain near-optimal $\tilde O(n)$ sample complexity. While powerful, these approaches require active control over the sampling process, which may be unavailable in many practical settings.

Zhang's Cross tensor completion method~\citep{zhang2019cross} achieves nearly linear sample complexity through a non-adaptive sampling scheme that observes a core subtensor and selected complete fibers. However, its recovery guarantees rely on full-rank conditions for the intersections between the core and fibers, whereas our  approach does not require such local rank conditions or observations of complete fibers.

Very recently, \citep{haselby2024tensor} proposed a non-adaptive sampling algorithm for tensor completion that achieves $\tilde O(n)$ sample complexity by using nonuniform sampling across tensor entries, together with nuclear norm minimization, {simultaneous diagonalization}, and censored least-squares procedures. {Their guarantee for order-three tensors also requires every $s$-row submatrix of the third-mode factor matrix to have Kruskal rank at least two when $r\ge2$, where $s$ is the number of densely sampled slices. This technical condition for simultaneous diagonalization is absent from our wedge-sampling subspace guarantees.} Moreover, their algorithmic pipeline is comparatively involved, while our refinement step is simple and can be readily combined with existing spectral and optimization-based methods \citep{montanari.sun_2018_spectral,cai2022nonconvex} under uniform entry sampling. In addition, while \citep{haselby2024tensor} focuses on the noiseless setting, we show that wedge sampling is robust to additive Gaussian noise under a suitable signal-to-noise condition; see Section~\ref{sec:noisy} for details. 

A separate line of work achieves nearly linear sample complexity by leveraging additional side information. In particular, \citep{yu2022tensor} studies tensor completion when a matrix obtained by tensor contraction is informative, and shows that nearly linear sample complexity is possible. This perspective is closely connected to community detection in hypergraphs, where the contracted adjacency matrix already contains sufficient information about the underlying low-rank tensor structure. Consequently, no diverging statistical-to-computational gap arises \citep{cole2020exact,Pal_2021,stephan2022sparse,gu2023weak,dumitriu2021partial,bresler2024thresholds}.

\paragraph{One-sided matrix completion}
One-sided matrix completion refers to the problem of recovering a low-rank matrix when the row and column dimensions of an $n \times m$ matrix are highly unbalanced ($m \gg n$), and the goal is only to estimate the left singular subspace. This setting naturally arises in tensor completion after matricization, as well as in bipartite graph community detection \citep{florescu2016spectral,ndaoud2021improved,braun2022minimax}.
Under uniform sampling, near-optimal sample complexity $\tilde O(\sqrt{mn})$ for left singular subspace estimation has been achieved in \citep{cai2021subspace,stephan2024non}.   In contrast, our spectral method based on a new sampling scheme achieves \edit{$\tilde O(n)$} sample complexity  (Theorem~\ref{thm:wedge_spectral_guarantees}). \edit{This improves the uniform-sampling scaling in the highly unbalanced regime \(m\gg n\), and shows that one-sided matrix completion can be carried out at the nearly linear scale under a non-uniform, non-adaptive design.} We expect this non-adaptive sampling strategy to be useful in other settings with limited observations, particularly when spectral estimation relies on products of entries sharing a common index.

\subsection{Our approach}
We introduce a new sampling scheme, termed \emph{Wedge Sampling}, which we argue provides a more natural way to sample low-rank tensors than uniform entry sampling and enables polynomial-time algorithms for tensor completion. \edit{The model is relevant when the data-collection mechanism can query or aggregate paired entries sharing a common mode index, for instance by collecting matched measurements along tensor fibers or two observations from a common context.}

\paragraph{Failure of uniform entry sampling below $n^{k/2}$}
Consider the case $k=3$. {We give a random-graph perspective explaining the $n^{3/2}$ sample scale for  spectral initialization of a low-rank order-$3$ tensor $T$ under uniform entry sampling.} After unfolding, we obtain a long matrix $A \in \mathbb{R}^{n \times n^2}$ of low rank. Many efficient approaches reduce tensor completion to estimating the left singular subspace of $A$, equivalently, the leading eigenspace of $AA^\top$
from a subsampled version $\tilde A$ in which each entry is observed independently with probability $p$. Define the hollowed matrix
$B=\tilde A \tilde A^\top-\mathrm{Diag}(\tilde A \tilde A^\top)$, then for $i\not=j$, \[B_{ij}=\sum_{\ell \in [n^2]} \tilde A_{i\ell}\,\tilde A_{j\ell}.\] 
When $p$ is small, the diagonal entries of $\tilde A\tilde A^\top$ can dominate its spectrum, and must therefore be shrunk or removed \citep{lounici2014high, montanari.sun_2018_spectral,cai2022nonconvex}.

{For $B_{ij}$ to be nonzero, there must be at least one index $\ell$ for which both $\tilde A_{i\ell}$ and $\tilde A_{j\ell}$ are observed. The converse can fail because entries may vanish or products may cancel.}
This has a simple interpretation in terms of the bipartite sampling graph $G$: the left vertices are rows $i\in[n]$, the right vertices are columns $\ell\in[n^2]$, and an observed entry corresponds to an edge $(i,\ell)$. {Thus a nonzero $B_{ij}$ requires a length-two path $i\to\ell\to j$ in $G$.} 

Such a length-two pattern is exactly a \emph{wedge}. Under uniform entry sampling, each potential wedge occurs with probability $p^2$, and since there are $n^2$ possible intermediates $\ell$, the probability that a fixed pair $(i,j)$ is connected by at least one wedge is on the order of $n^2 p^2$. Consequently, $B$ is extremely sparse when $p \ll n^{-3/2}$, with only $o(n)$ nonzero entries, and the resulting graph on the left vertices fails to become well connected. Related wedge-walk ideas for tensor completion first appear in \citep{stephan2024non}, which achieves $O(n^{k/2})$ sample complexity without any $\log n$ factors for weak recovery via a non-backtracking wedge operator.

A standard connectivity obstruction in this bipartite graph model \citep{johansson2012giant} then implies that one needs
$p\gtrsim n^{-3/2}$ \citep[Theorem 3]{cai2021subspace}
in order to reliably recover the left singular structure of $A$ from $B$. Equivalently, this corresponds to an $n^{3/2}$ sample-size barrier, {consistent with the computational threshold conjectured for noisy tensor completion in \citep{barak.moitra_2016_noisy}.} 

\paragraph{A remedy via wedge sampling}

\begin{figure}
    \centering
   \begin{tikzpicture}[scale=0.6, every node/.style={font=\small}]
  % Two ellipses (left: rows, right: columns)
  \draw (0,0) ellipse (1.2 and 2.0);
  \draw (6,0) ellipse (1.2 and 2.0);

  % Group labels
  \node[align=center] at (-3,1.5) {$V_1=[n]$};
  \node[align=center] at (9,1.5) {$V_2=[n^2]$};

  % Nodes (points)
  \node[circle, fill=blue!60, inner sep=2pt, label=above left:$i$] (i) at (-0.05, 0.85) {};
  \node[circle, fill=blue!60, inner sep=2pt, label=below left:$j$] (j) at (-0.05,-0.85) {};
  \node[circle, fill=blue!60, inner sep=2pt, label=above right:$\ell$] (ell) at (6.25, 0.00) {};

  % Edges
  \draw (i) -- (ell);
  \draw (j) -- (ell);
\end{tikzpicture}
\caption{Illustration of uniform entry sampling versus wedge sampling on the bipartite sampling graph for an order-3 tensor. Under uniform sampling, the edges $(i,\ell)$ and $(j,\ell)$ are observed independently; under wedge sampling, we sample the length-two path (wedge) $(i,\ell,j)$ uniformly from the wedge space $\{(i,\ell,j):  1\leq i\leq j\leq n, \ell\in [n^2] \}$, and $\{A_{i\ell}, A_{j\ell}\}$ are both observed for each wedge $(i,\ell,j)$.} 
    \label{fig:wedge_fig}
\end{figure}

The discussion above highlights the core issue: the statistic $B_{ij}$ is informative only when we observe \emph{enough} length-two paths $i\to \ell \to j$, that is, enough wedges linking pairs of left vertices. Uniform entry sampling spends samples on individual edges, but it creates wedges only indirectly, and far too sparsely in the regime $p \ll n^{-3/2}$.
Wedge sampling addresses this mismatch directly. Rather than sampling edges of the bipartite graph $G$ uniformly, it samples \emph{wedge walks} (length-two paths) uniformly. And $\{A_{i\ell}, A_{j\ell}\}$ are both observed for each sampled wedge $(i,\ell,j)$. See Figure~\ref{fig:wedge_fig} for an illustration. 

By allocating sampling budget in the wedge space, we directly promote the wedge connectivity that makes $B$ informative, thereby strengthening the graph structure relevant for spectral recovery. As a result, wedge sampling achieves accurate initialization with only \edit{$\tilde O(n)$ wedge samples}, compared to the $\tilde O(n^{3/2})$ samples under uniform entry sampling.

A precise description of our sampling scheme is provided in Section~\ref{sec:sampling}. For fixed rank, tensor order, incoherence, and condition number, our wedge-sampling procedures produce reliable spectral initializations with \edit{$O(n\log n)$ wedge samples}. Moreover, once we obtain a sufficiently accurate estimate of the left singular subspace of the unfolded tensor $T$, the subsequent refinement stage requires only $\tilde O(n)$ additional uniformly random samples for several state-of-the-art tensor completion algorithms \citep{montanari.sun_2018_spectral,cai2022nonconvex}.
Taken together, these results make our approach effectively \emph{plug-and-play}: a simple change to the sampling strategy in the initialization phase can be seamlessly combined with existing tensor completion methods to substantially reduce the overall sample complexity.

The central message of this work is: The conjectured statistical-to-computational gap in tensor completion appears to be \edit{model-dependent and closely}
tied to the uniform sampling paradigm; under alternative structured sampling schemes that enable
effective initialization (such as wedge sampling), \edit{the initialization barrier can be reduced}.  Similar measurement-design ideas have also been explored in sparse phase retrieval \citep{iwen2017robust}.

\edit{We summarize the main results informally as follows:}

    \begin{theorem}[Tensor completion with wedge sampling, informal]
For fixed rank, tensor order, incoherence, and condition number, the following two tensor completion algorithms succeed with high probability:
\begin{enumerate}
    \item \edit{{Spectral method (weak recovery)}: Wedge-sampling spectral initialization with $O(n\log n)$ sampled entries, followed by a spectral denoising step using an additional $O(\log n)$ uniformly sampled entries.}
    \item {Gradient descent (exact recovery)}: Wedge-sampling  spectral initialization with \edit{$\tilde O(n)$ sampled wedges}, followed by a gradient-descent refinement  with an additional $\tilde O(n)$ uniformly sampled entries.
\end{enumerate}
\end{theorem}

\subsection{Technical overview}

Although wedge sampling is intuitive from a random graph perspective, its analysis introduces new technical obstacles. Sampling in the wedge space induces dependence among entries of the resulting subsampled tensor, which complicates both matrix and tensor concentration arguments and the associated eigenvector perturbation analysis. Moreover, in the ultra-sparse regime with only $\tilde O(n)$ samples, existing random tensor concentration results are too coarse to control the gradient updates needed for a nonconvex analysis. To address these issues, we develop several new tools in random matrix/tensor concentration and matrix perturbation theory:
\begin{itemize} 
\item \textbf{Concentration for long matrices under wedge sampling.} We prove a concentration bound for an unfolded (highly rectangular) matrix formed under wedge sampling that remains valid with only $\tilde O(n)$ observations by exploiting the incoherence of the tensor unfolding (Theorem~\ref{thm:wedge_concentration}). This sparsity level lies below the regimes covered by standard matrix completion inequalities \citep{chen2015incoherence} and one-sided matrix completion results \citep{cai2021subspace}.
\item \textbf{Leave-one-out analysis for $\ell_{2,\infty}$ singular subspace recovery.} We develop a leave-one-out analysis tailored to wedge sampling to obtain fine-grained recovery of the left singular subspace in the $\ell_{2,\infty}$ norm (Theorem~\ref{thm:wedge_spectral_guarantees}). \edit{Compared with existing uniform-sampling guarantees, this gives substantially sharper scaling in the highly unbalanced regime,} and the resulting $\ell_{2,\infty}$ control is crucial for our spectral tensor completion algorithm introduced in Section~\ref{sec:spectral}.
\item \textbf{Improved concentration for sparse random tensors under an incoherent norm.} Under uniform sampling, standard spectral-norm bounds for order-3 random tensors encounter a barrier around \edit{$q = n^{-3/2}$} \citep{jain2014provable,yuan2016tensor,cai2022nonconvex}. To go beyond this regime, we work with {a variant of the incoherent tensor norm of \citep{yuan2017incoherent}} and prove {a concentration inequality} that remains valid down to \edit{$q = n^{-(k-1)}$} (Theorem~\ref{thm:symmetric_tensor_concentration}). \edit{This ultra-sparse incoherent-norm regime is not covered by the usual operator-norm concentration tools.} This estimate is a key ingredient in establishing local convexity of the gradient-descent landscape for tensor completion with $\tilde O(n)$ samples.
\end{itemize}

\paragraph{Organization of the paper}
The rest of the paper is organized as follows. Section~\ref{sec:prelim} introduces notation and background for tensor completion. In Section~\ref{sec:sampling}, we present the wedge-sampling scheme and establish its theoretical guarantees. Section~\ref{sec:spectral} develops a spectral method based on wedge sampling, including a two-stage algorithm adapted from \citep{montanari.sun_2018_spectral}. In Section~\ref{sec:GD}, we analyze gradient descent (GD) for tensor completion with a wedge-sampling spectral initialization, building on the framework of \citep{cai2022nonconvex}. \edit{In Section~\ref{sec:noisy}, we extend our analysis to tensor completion with noisy measurements.} Section~\ref{sec:numerical} reports numerical experiments. Additional proofs are deferred to the appendix.

\section{Preliminaries}\label{sec:prelim}

\paragraph{Tensor notation}
An order-$k$ tensor of size $n_1\times\cdots\times n_k$ is an element
$T\in\mathbb{R}^{n_1\times\cdots\times n_k}$.
We say that $T$ has canonical polyadic (CP) rank $r$ if there exist vectors
$\{x_i^{(j)}\in\mathbb{R}^{n_j}\}_{i\in[r],\,j\in[k]}$ such that
\begin{equation}\label{eq:low_rank_tensor}
    T \;=\; \sum_{i=1}^r x_i^{(1)}\otimes\cdots\otimes x_i^{(k)},
    \qquad
    (x^{(1)}\otimes\cdots\otimes x^{(k)})_{i_1,\dots,i_k}
    \;=\; x^{(1)}_{i_1}\cdots x^{(k)}_{i_k}.
\end{equation}
A tensor $T$ is \emph{symmetric} if $n_1=\cdots=n_k$ and
$T_{i_1,\dots,i_k}=T_{i_{\sigma(1)},\dots,i_{\sigma(k)}}$ for every permutation
$\sigma$ of $[k]$.
Given $u\in\mathbb{R}^{n_j}$, the mode-$j$ tensor-vector product is defined by
\[
(T\times_j u)_{i_1,\dots,i_{j-1},\,i_{j+1},\dots,i_k}
\;=\;
\sum_{i_j=1}^{n_j} T_{i_1,\dots,i_k}\,u_{i_j}.
\]
This operation is associative, and in particular $T\times_1 u_1 \times_2 \cdots \times_k u_k
\;=\;
\langle T,\, u_1\otimes\cdots\otimes u_k\rangle$.

Given $T\in\mathbb{R}^{n_1\times\cdots\times n_k}$, its mode-$j$ unfolding is
the matrix $\operatorname{unfold}_j(T)\in\mathbb{R}^{n_j\times m_j}$ with
$m_j=\prod_{\ell\neq j} n_\ell$, defined entrywise as
\[
[\operatorname{unfold}_j(T)]_{\,i_j,\,(i_1,\dots,i_{j-1},\,i_{j+1},\dots,i_k)}
\;=\;
T_{i_1,\dots,i_k}.
\]
Equivalently, the rows of $\operatorname{unfold}_j(T)$ correspond to the $j$th
mode of $T$, while the columns index the remaining modes.
If $T$ has CP rank $r$, then each unfolding
$\operatorname{unfold}_j(T)$ has matrix rank at most $r$.

%\paragraph{Tensor incoherence}
% \hl{This paragraph reads odd, I rewrote this, but feel free to revert. }
The low-rank property alone does not guarantee identifiability of each factor from a sparsely observed tensor: for
instance, the ``spiky'' tensor $T=\delta_{i_1\cdots i_k}$ can only be recovered
if $(i_1,\dots,i_k)\in\Omega$, which requires extremely dense sampling.
Accordingly, we restrict attention to \emph{incoherent} tensors, extending the
classical notion for matrices (see, e.g., \citep{chen2015incoherence}).

\begin{definition}[Incoherence]\label{def:mu12-incoh}
Let $M\in\mathbb{R}^{n_1\times n_2}$ have rank $r$ and singular value
decomposition $M=U\Sigma V^\top$ with
$U\in\mathbb{R}^{n_1\times r}$ and $V\in\mathbb{R}^{n_2\times r}$. Let $U_{i,\cdot}$ denote the $i$-th row of $U$.
We say that $M$ is \emph{$(\mu_1,\mu_2)$-incoherent} if
\[
\|U\|_{2,\infty}^2
:=\max_{i\in[n_1]}\|U_{i,\cdot}\|_2^2 \;\le\; \frac{\mu_1 r}{n_1},
\qquad
\|V\|_{2,\infty}^2
:=\max_{j\in[n_2]}\|V_{j,\cdot}\|_2^2 \;\le\; \frac{\mu_2 r}{n_2}.
\]
\end{definition}
We extend this definition to tensors by requiring incoherence of every mode
unfolding.
\begin{definition}[Tensor incoherence]\label{def:tensor_incoherence}
A tensor $T\in\mathbb{R}^{n_1\times\cdots\times n_k}$ is
\emph{$(\mu_1,\mu_2)$-incoherent} if for every $j\in[k]$, denoting the SVD of the
mode-$j$ unfolding by
$\operatorname{unfold}_j(T)=U^{(j)}\Sigma^{(j)}(V^{(j)})^\top$ with
$U^{(j)}\in\mathbb{R}^{n_j\times r_j}$ and $V^{(j)}\in\mathbb{R}^{m_j\times r_j}$,
we have
\[
\|U^{(j)}\|_{2,\infty}^2
:=\max_{i\in[n_j]}\|U^{(j)}_{i,\cdot}\|_2^2 \;\le\; \frac{\mu_1 r_j}{n_j},
\qquad
\|V^{(j)}\|_{2,\infty}^2
:=\max_{i\in[m_j]}\|V^{(j)}_{i,\cdot}\|_2^2 \;\le\; \frac{\mu_2 r_j}{m_j},
\]
where $r_j=\rank(\operatorname{unfold}_j(T))$ and $m_j=\prod_{\ell\neq j} n_\ell$.
\end{definition}
The following lemma shows that tensor incoherence can be verified  from
a CP decomposition via entrywise delocalization and limited cross-correlation.
 Its proof is deferred to {Appendix~\ref{sec:proof_lemma_unfold_incoherence}}. 
\begin{lemma}\label{lem:unfold_incoherence}
Let $T$ be the tensor that admits a decomposition as  in~\eqref{eq:low_rank_tensor}. Suppose that there
exists a parameter $\mu$ with $ r\mu \leq  \frac{1}{2}\min_j n_j$  and for every $j\in[k]$,
\begin{equation}\label{eq:cp_incoherent}
\sup_{i\in[r]}
\frac{\|x_i^{(j)}\|_\infty^2}{\|x_i^{(j)}\|_2^2}\le \frac{\mu}{n_j},
\qquad
\sup_{i\neq i'\in[r]}
\frac{|\langle x_i^{(j)},x_{i'}^{(j)}\rangle|}
{ \|x_i^{(j)}\|_2\,\|x_{i'}^{(j)}\|_2}\le \frac{\mu}{n_j}.
\end{equation}
Then $T$ is $(\mu_1,\mu_2)$-incoherent with $\mu_1\leq 2\mu,
\mu_2 \leq 2\mu^{\,k-1}$.  
\end{lemma}
\begin{definition}[$\mu$-CP incoherent tensor]
   We say  $T$ is \emph{$\mu$-CP incoherent} if \eqref{eq:cp_incoherent} holds. 
\end{definition}

\section{Wedge sampling}\label{sec:sampling}

A recurring bottleneck in tensor completion is not the refinement stage: once a reasonably accurate estimate of the left singular subspace of the unfolded tensor is available, existing nonconvex or spectral refinement methods can often achieve weak or exact recovery with only a modest number of additional samples \citep{montanari.sun_2018_spectral,cai2022nonconvex}. Rather, the main difficulty lies in obtaining a good \emph{initialization} in the first place.

In this section, we introduce \emph{Wedge Sampling}, a simple non-adaptive sampling scheme that produces informative initializations in the limited-sample regime. The key idea is to sample \emph{wedges} (pairs of observations sharing a common index) in order to construct an estimator whose expectation matches the desired second-moment matrix and whose spectrum remains stable with only \edit{$\tilde O(n)$ wedge samples}. We first present the sampling procedure in the long-matrix setting, which will serve as a basic primitive throughout the paper; extensions to the tensor setting then follow by applying the same scheme to appropriate unfoldings.

\begin{algorithm}
\caption{Wedge sampling\label{alg:wedge}}
\begin{algorithmic}[1]
\State \textbf{Input:} Unknown  matrix  $A \in \mathbb{R}^{n \times m}$, wedge sampling rate {$p\in(0,1]$}
\State \textbf{Wedge sampling indices:} Form the wedge index set as a subset in 
\[ \mathcal{W}=\{ (i,\ell,j): 1\leq i\leq j\leq n, \ell\in [m]\}.\] 
Sample each triple $(i,\ell,j)$ independently with probability $p$. Let 
$$\tilde{\mathcal W}=\{(i,\ell,j)\in \mathcal{W}: (i,\ell,j) \text{~is~sampled}. \}$$

\State \textbf{Wedge matrix:} Let $Z_{i\ell j}=p^{-1}A_{i\ell }A_{j\ell} E_{ij}\1_{(i,\ell,j)\in \tilde{\mathcal W}}$, where $E_{ij}=e_i e_j^\top$. Form 
\begin{align} \label{eq:defZ}
Z=\sum_{i<j, \ell\in [m]} [Z_{i\ell j} +Z_{i\ell j}^\top]+\sum_{i,\ell} Z_{i\ell i}.
\end{align}

\State \textbf{Spectral estimation:} Let $Z = \tilde U \Sigma\tilde U^\top$ be the eigendecomposition of $Z$. Return the submatrix $\hat U$ of the $r$ eigenvectors of $Z$ with the highest associated eigenvalues. 
\end{algorithmic}
\end{algorithm}
% \paragraph{The sampling scheme}
Consider a matrix  $A \in \mathbb{R}^{n \times m}$.
Define the \emph{wedge set}
$\mathcal{W}=\{ (i,\ell,j): 1\leq i\leq j\leq n, \ell\in [m]\}$.
We sample each entry of $\mathcal{W}$ independently with probability $p$, to obtain the sampled set $\tilde{\mathcal{W}}$. For each entry $(i, \ell, j) \in \tilde{\mathcal{W}}$, we reveal the entries $A_{i\ell}$ and $A_{j\ell}$ simultaneously, then  form an estimator $Z$ defined in \eqref{eq:defZ}.
 The left singular vectors $U$ are then estimated by the top-$r$ eigenvectors of $Z$. Such a procedure is summarized in Algorithm~\ref{alg:wedge}.
\edit{Throughout the paper, the sample complexity of a wedge-sampling step refers to the number of sampled wedges, namely \(|\widetilde{\mathcal W}|\).  Since each off-diagonal wedge reveals two entries and each diagonal wedge reveals one entry, this is equivalent up to a factor of two to counting entry observations with multiplicity.  We do not merge repeated appearances of the same entry unless explicitly stated.  For an order-\(k\) tensor unfolding with \(m=n^{k-1}\), the expected number of sampled wedges is \(p|\mathcal W|\asymp p n^{k+1}\), while an independent uniform refinement step at rate \(q\) uses \(q n^k\) expected tensor-entry observations.}
 
We next provide theoretical guarantees for the estimator $\hat U$ obtained by Algorithm~\ref{alg:wedge}.  Denote the singular value decomposition by $A = U\Sigma V^\top$, with $\Sigma = \operatorname{diag}(\sigma_1, \dots, \sigma_r)$, and its condition number by $\kappa = \sigma_1/\sigma_r$.
We begin with a concentration bound of $Z$ around $AA^\top$:
\begin{theorem}[Concentration of random matrices under wedge sampling]\label{thm:wedge_concentration} Assume $A\in \R^{n \times m}$ has rank $r$ and is $(\mu_1, \mu_2)$-incoherent.
For any $a>2$, with probability at least $1-O(n^{-a})$,
\begin{align}
    \|Z-AA^\top \| \lesssim\left(\sqrt{\frac{a\mu_1\mu_2 r^2 \log n}{pmn}}+\frac{a\mu_1\mu_2 r^2 \log n}{pmn}\right) \|A\|^2.
\end{align}
\end{theorem}

The proof of Theorem~\ref{thm:wedge_concentration} is provided in {Appendix~\ref{sec:app:wedge_proofs}}.
Theorem~\ref{thm:wedge_concentration} shows that, under incoherence assumptions on $A$, $Z$ concentrates around $AA^\top$ with \edit{$O(n\log n)$ wedge samples}. Consequently, the eigendecomposition of $Z$ provides an accurate approximation to the left singular subspace of $A$.

To provide reconstruction bounds on $\hat U$ obtained from Algorithm~\ref{alg:wedge}, we need to align the matrices $U$ and $\hat U$.  We define the optimal rotation matrix $R$ as
\begin{equation}\label{eq:def_procrustes}
    R =\argmin_{O\in \R^{r\times r}, O^\top O =I_r}\|\hat{U} O-U\|_F.
\end{equation} 
Then the following holds:
\begin{theorem}[Left singular subspace recovery]\label{thm:wedge_spectral_guarantees}
Assume  $A$ is a matrix of size $n \times m$ which has rank $r$ and is $(\mu_1, \mu_2)$-incoherent.
    \edit{Let \(a>2\) be a constant}, and assume that there exists an absolute constant $c_0 > 0$ such that
    $p \geq \frac{c_0 \kappa^4 a \mu_1 \mu_2 r^2 \log(n)}{mn}$.
    Then, with probability at least $1 - O(n^{-a})$,
    \begin{align}
        \|\hat UR - U\| &\lesssim  \sqrt{\frac{\kappa^4 a \mu_1 \mu_2 r^2 \log(n)}{mnp}}, \label{eq:wedge_spectral_op} 
        \end{align}
        {Moreover, for a sufficiently large universal constant \(C_0>0\), when
        $p\geq \frac{C_0\kappa^4 a \mu_1\mu_2 r^2\log n}{mn}$,
        \begin{align}
        \| \hat UR - U \|_{2, \infty} &\lesssim \sqrt{\frac{ \kappa^4 a \mu_1 \mu_2 r^2 \log(n)}{mnp}} \|U\|_{2, \infty}. \label{eq:wedge_spectral_loo}
    \end{align}}
\end{theorem}

Theorem~\ref{thm:wedge_spectral_guarantees} also provides a direct guarantee for one-sided matrix completion when $m\gg n$.
For fixed $r,\kappa,\mu_1,\mu_2$ and a fixed target accuracy, it requires  $O(n\log n)$ entry observations.
Under independent uniform sampling, \citep{cai2021subspace,stephan2024non} have sample size $\tilde O(\sqrt{mn})$  under their respective assumptions, so the dimension dependence improves by a factor of $\sqrt{m/n}$, up to logarithmic factors.
Related approaches recover the singular subspace by completing an estimated Gram matrix: \citep{cao2023one} uses two observations per column, while \citep{zhang2025one} studies a constant expected number of observations per column under a random factor model.
In these observation models, the total budget scales with $m$.
Wedge sampling directly collects the products needed to estimate $AA^\top$  with a sample size governed by the smaller dimension $n$.

The proof of Theorem~\ref{thm:wedge_spectral_guarantees} can be found in {Appendix~\ref{sec:app:wedge_proofs}}.
While~\eqref{eq:wedge_spectral_op} is a simple corollary of Theorem~\ref{thm:wedge_concentration} using the Davis-Kahan inequality, the proof of~\eqref{eq:wedge_spectral_loo} is significantly more involved and requires a careful leave-one-out analysis. And  \eqref{eq:wedge_spectral_loo} is crucial for the analysis of our spectral method for tensor completion in Section~\ref{sec:spectral}.

\section{Spectral method with wedge sampling}\label{sec:spectral}

In this section, we adapt the spectral method of \citep[Algorithm~1]{montanari.sun_2018_spectral}, with three modifications: 
(i) whereas \citep{montanari.sun_2018_spectral} use different unfoldings for odd and even \(k\), we employ a single unbalanced unfolding for all \(k\); 
(ii) We replace their spectral initialization with the wedge-sampling initialization in Algorithm~\ref{alg:wedge}; (iii) We do not use sample splitting in the denoising step; instead, we form \(Y\) using additional i.i.d.\ uniform samples of \(T\).
These changes simplify the procedure and yield a substantially improved sample complexity. 
The resulting algorithm is summarized in Algorithm~\ref{alg:montanari}.
Its recovery guarantee is provided in Theorem~\ref{thm:final_bound_montanari}.

\begin{algorithm}
\caption{Spectral Tensor Completion with Wedge Sampling}
\label{alg:montanari}
\begin{algorithmic}[1]
\State \textbf{Input:} Unknown order-$k$ symmetric tensor $T \in \mathbb{R}^{n \times \cdots \times n}$, \edit{wedge sampling rate} $p$, uniform sampling rate {$q\in(0,1]$}.  Let \edit{$A=\unfold_1(T)\in \R^{n\times n^{k-1}}$}. Assume $\rank(A)= r$  is known.
\State \textbf{Wedge sampling:} 
\State  Let $\tilde{\mathcal W}$ be the index set of wedges sampled with probability $p$ from
$\mathcal W_k:=\{(i,\ell,j):1\leq i\leq j\leq n,\ \ell\in[n^{k-1}]\}$.
\State \textbf{Spectral method:}
    \State \quad \edit{Form the matrix $Z$ as in Algorithm~\ref{alg:wedge} using $\tilde{\mathcal W}$.}
    \State \quad Compute the 
    leading $r$ unit eigenvectors of $Z$ denoted by $\hat U\in \R^{n\times r}$.  Let $Q=\hat U\hat U^\top$.
 \State \textbf{Denoising:} Let \edit{$Y=\unfold_1\left(\frac{1}{q} \tilde T\right) \in \R^{n\times n^{k-1}}$}, where $\tilde T$ is the subsampled tensor of $T$ with uniform sampling probability $q$. Let
 $\hat T= Q Y (Q \otimes \cdots \otimes Q) \in \R^{n\times n^{k-1}}$. 
 \State \textbf{Output:}
 \edit{Return the order-$k$ tensor obtained by folding $\hat T$.}
\end{algorithmic}
\end{algorithm}

\begingroup
\setlength{\emergencystretch}{2em}
\begin{theorem}[Spectral method, Frobenius-norm bound]
\label{thm:final_bound_montanari}
Let $T$ be a $(\mu_1,\mu_2)$-incoherent order-$k$ symmetric tensor.
Assume that $A=\unfold_1(T)\in\mathbb R^{n\times n^{k-1}}$ has rank $r$
and condition number $\kappa$, and let $\hat T$ be the estimator in
Algorithm~\ref{alg:montanari}, using an independent uniform refinement
sample. For any $a\ge2$, there are universal constants $C_0,C>0$ such
that if
\[
 p\ge\frac{C_0a\kappa^4\mu_1\mu_2r^2\log n}{n^k},
 \qquad
 q\ge\frac{ak\log n}{n^k},
\]
then, with probability at least $1-O(n^{-a})$,
\begin{equation}\label{eq:main-spectral-frobenius}
\begin{aligned}
 \frac{\|T-\hat T\|_F}{\|T\|_F}
 &\le C\sqrt{k}\,\kappa^2\sqrt{\frac{a\mu_1\mu_2r^2\log n}{pn^k}}
   + C2^k(\mu_1r)^{k/2}\sqrt{\frac{1-q}{qn^k}}\\
 &\quad+C\frac{\|A\|}{\|T\|_F}r\sqrt{\mu_1\mu_2}
  \left(\sqrt{\frac{(1-q)a\log n}{qn^k}}
       + 2^k(\mu_1r)^{k/2}\frac{a\log n}{qn^k}\right).
\end{aligned}
\end{equation}
Here $\|A\|/\|T\|_F\le\min\{1,\kappa/\sqrt r\}$.
\end{theorem}
The proof first controls the row norms of $Q$ by
Theorem~\ref{thm:wedge_spectral_guarantees}, and then bounds the
projected sampling error directly in Frobenius norm, conditionally on
the wedge sample. {Appendix~\ref{sec:app_final_montanri}} provides the proof
and extends the construction to asymmetric tensors.

For fixed $a,k,\kappa,\mu_1,\mu_2$ and $0<\varepsilon\le1$, relative Frobenius
error at most $\varepsilon$ follows with $O(\varepsilon^{-2}nr^2\log n)$
expected wedges and
\[
 O\!\left(\varepsilon^{-2}r^k+\varepsilon^{-2}r\log n
          +\varepsilon^{-1}r^{(k+1)/2}\log n\right)
\]
expected uniform observations. In particular, the term $\varepsilon^{-2}r^k$
has no logarithmic factor. At fixed rank the two budgets are
$O(\varepsilon^{-2}n\log n)$ and $O(\varepsilon^{-2}\log n)$, respectively.

\begin{remark*}[rank dependence]
The example in Appendix~\ref{sec:app:spectral_rank_example}
shows that, even when $\mu_1=\mu_2=\kappa=1$, Algorithm~\ref{alg:montanari}
can require order $nr^2$ wedges and order $r^k$ uniform observations
for constant relative Frobenius error. Each lower bound holds when the
other stage observes every entry. Thus the powers $r^2$ and $r^k$ in
Theorem~\ref{thm:final_bound_montanari} are sharp for this algorithm in
the worst case.
\end{remark*}
\endgroup

\section{Gradient descent with wedge sampling}\label{sec:GD}

\begingroup
We consider in this section the case of a $\mu$-CP incoherent tensor
\[
    T = \sum_{i=1}^r x_i^\star \otimes x_i^\star \otimes x_i^\star .
\]
Define
\[
    \lambda_{\min} = \min_i \|x_i^\star\|^3,
    \qquad
    \lambda_{\max} = \max_i \|x_i^\star\|^3,
\]
and denote the ratio of the largest to the smallest CP weight by
$\kappa_{\mathrm{CP}}=\lambda_{\max}/\lambda_{\min}$.
\endgroup
In this setting, Cai et al.~\cite{cai2022nonconvex} propose a three-step algorithm: (1) a spectral initialization step based on the off-diagonal entries of the matrix $AA^\top$, with $A = \unfold_1(T)$ in our notation; (2) an extraction step to refine $U$ into estimates for each CP-factor $x_i$ (see \citep[Algorithm 3]{cai2022nonconvex}); and (3) gradient descent on the objective 
\begin{equation}
      F(\hat X) = \frac1{6q}\left\|\mathcal{P}_{\Omega}\left(T - \sum_{i=1}^r \hat x_i \otimes \hat x_i \otimes \hat x_i\right) \right\|_F^2, 
    \label{eq:gd_objective}
\end{equation}
  \edit{where \(\hat X\) is the optimization variable.  In the proof, following the notation of Cai et al.~\cite{cai2022nonconvex}, we denote by \(X=[x_1,\ldots,x_r]\in\mathbb R^{n\times r}\) the ground-truth factor matrix, with \(x_i=x_i^\star\) from the display above.}
Under their incoherence and rank assumptions and for bounded $\kappa_{\mathrm{CP}}$,
Cai et al.~\citep[Theorem~2 and Corollary~2]{cai2022nonconvex} prove geometric
convergence in the noiseless setting under uniform entry sampling at rate
$q\gtrsim\mu^4r^4\log^4(n)/n^{3/2}$, where $\mu$ is their incoherence parameter.

\begin{algorithm}
\caption{Tensor Completion with Wedge Sampling via Gradient Descent}
\begin{algorithmic}[1]
\State \textbf{Input:} Unknown symmetric tensor $T\in\mathbb R^{n\times n\times n}$, rank $r$, sampling rates $p,q$, number of Gaussian draws for Algorithm~\ref{alg:factor_extraction}, step size $\eta$, and number of iterations $t$.
\State \textbf{Wedge sampling:} 
\State \quad Let $\tilde{\mathcal W}$ be the index set of wedges sampled with probability $p$ from
$\mathcal W_3:=\{(i,\ell,j):1\leq i\leq j\leq n,\ \ell\in[n^2]\}$.
\State \textbf{Initialization:}
\State \quad Form matrix $Z$ defined in Algorithm~\ref{alg:wedge} based on $\tilde{\mathcal W}$.
    \State \quad Compute the 
    leading $r$ unit eigenvectors of $Z$ denoted by $\hat U\in \R^{n\times r}$. 
    \State \quad {Draw a uniform initialization batch independently of the wedge sample, with entry probability $q$.}
    \State \quad {Apply Algorithm~\ref{alg:factor_extraction} in Appendix~\ref{sec:app:proof_cai} to $\hat U$ and this batch to obtain the initial factor matrix $X_0\in\mathbb R^{n\times r}$.}
    
 \State \textbf{Gradient Descent:}
 \State \quad {Independently draw a second uniform batch $\Omega$, also with entry probability $q$.}
 \State \quad {Perform $t$ steps of gradient descent on \eqref{eq:gd_objective}, using the second batch and initialization $X_0$, to obtain $\hat X^t$.}
 \State \textbf{Output:} \edit{$\hat T^t = \sum_{i=1}^r \hat{x}_i^t \otimes \hat{x}_i^t \otimes \hat{x}_i^t$.}
\end{algorithmic}\label{alg:wedge_GD}
\end{algorithm}

{\setlength{\emergencystretch}{2em}{We use spectral factor extraction followed by gradient descent, with two explicit adjustments to the framework of \citep{cai2022nonconvex}. The contracted extraction matrices are symmetrized, and the extraction and gradient-descent stages use independent uniform batches. Appendix~\ref{sec:app:proof_cai} proves the corresponding initialization and coupled gradient-descent estimates at the sampling rates below.}
\begingroup\setlength{\emergencystretch}{2em}
\begin{theorem}[Wedge sampling with gradient descent]\label{thm:cai_sparse}
Let $T$ be a $\mu$-CP incoherent tensor of order $3$ and rank $r$, and fix $\delta\in(0,1)$. Suppose that, for sufficiently large absolute constants $c_0,c_1$,
\[
 p\ge\frac{c_0\kappa_{\mathrm{CP}}^4\mu^3r^2\log n}{n^3},\qquad
 q\ge\frac{c_1\kappa_{\mathrm{CP}}^6\mu^5r^4\log^5 n}{n^2}.
\]
Use Algorithm~\ref{alg:wedge_GD} with step size $\eta=(8\lambda_{\max}^{4/3})^{-1}$
and prescribe
\[
 \left\lceil\min\!\left\{(r+1)^{C\kappa_{\mathrm{CP}}^2},\,
 C^r\max\!\left\{1,\frac{\kappa_{\mathrm{CP}}}{\sqrt r}\right\}^{r-1}\right\}
 \log(8r/\delta)\right\rceil
\]
Gaussian draws for Algorithm~\ref{alg:factor_extraction}, where $C$ is a sufficiently large universal constant.
For all sufficiently large $n$ depending only on $\delta$, with probability at least $1-\delta$,
simultaneously for every integer $t\ge0$,
\begin{align*}
 \|\hat T^t-T\|_F&\lesssim\rho^t\|T\|_F,\\
 \|\hat T^t-T\|_\infty&\lesssim\rho^t\|T\|_\infty,
 \qquad \rho=1-\frac{1}{32\kappa_{\mathrm{CP}}^{4/3}}.
\end{align*}
\end{theorem}
Choosing $p$ and $q$ at the thresholds in Theorem~\ref{thm:cai_sparse}
uses $O(\kappa_{\mathrm{CP}}^4\mu^3r^2n\log n)$ wedges and
$O(\kappa_{\mathrm{CP}}^6\mu^5r^4n\log^5 n)$ uniform entries in each
of the two batches, in expectation. The expected total across all three stages is
\[
 pn^4+2qn^3=O(\kappa_{\mathrm{CP}}^6\mu^5r^4n\log^5 n)
\]
tensor-entry observations. Algorithm~\ref{alg:factor_extraction} reuses
one batch for all its Gaussian draws. For fixed $\delta$, the draw count
is polynomial in $r$ for fixed $\kappa_{\mathrm{CP}}$ and polynomial in
$\kappa_{\mathrm{CP}}$ for fixed $r$.

Cai et al.~\citep[Theorem~2 and Corollary~2]{cai2022nonconvex} assume
$\kappa_{\mathrm{CP}}=O(1)$. They track the condition number in the number of
Gaussian draws and leave its effect on their sampling threshold and geometric
convergence rate implicit. Theorem~\ref{thm:cai_sparse} makes both dependencies
explicit. For fixed incoherence parameters and fixed $\kappa_{\mathrm{CP}}$,
in the rank range covered by both results, their sufficient uniform-sampling
budget is $O(r^4n^{3/2}\log^4 n)$, compared with $O(r^4n\log^5 n)$ here.
This improves the dimension dependence by a factor of order $\sqrt n/\log n$.
\endgroup\par}

At the core of the previous results lies a concentration property of $\hat T$ on \emph{delocalized} vectors.  Previous works (see \citep{jain2014provable,cai2022nonconvex}, among others) have made use of the fact that for a delocalized tensor $T$,
$\| q^{-1} \mathcal{P}_\Omega(T) - T \| \lesssim \frac{\mathrm{polylog}(n)}{n^{3/2}q} \|T\|$. 
 This implies that when $q \gg n^{-3/2}$,
the observed tensor $\hat T$ concentrates around $T$ in spectral norm. 
{At the sparser sampling rates considered here, ordinary spectral-norm fluctuations can remain too large relative to the signal for this argument. We therefore use a variant of the \emph{$\delta$-incoherent spectral norm} introduced in \citep{yuan2017incoherent}.}

For $\delta\in(0,\infty)^k$ and a tensor
$T\in\mathbb R^{n_1\times\cdots\times n_k}$, define
\[
 \|T\|_\delta=\sup_{U\in\mathcal U(\delta)}\langle T,U\rangle,
 \qquad
 \mathcal U(\delta)=\bigcup_{j_1\ne j_2}\mathcal U_{j_1,j_2}(\delta),
\]
where
\[
 \mathcal U_{j_1,j_2}(\delta)=
 \left\{u_1\otimes\cdots\otimes u_k:
 \begin{array}{l}
 \|u_{j_1}\|_2,\|u_{j_2}\|_2\le1,\\
 \|u_j\|_\infty\le\delta_j\quad(j\notin\{j_1,j_2\})
 \end{array}\right\}.
\]
Thus, two factors have Euclidean norm at most one, while all other
factors satisfy only infinity norm bounds.
We obtain the following incoherent norm concentration for sparse random tensors: 
\begingroup\setlength{\emergencystretch}{2em}

\begin{theorem}[Concentration of sparse random tensors]
\label{thm:symmetric_tensor_concentration}
Let $T\in\mathbb R^{n\times\cdots\times n}$ be a deterministic
order-$k$ tensor with $k\ge3$, and set $\delta_j=\sqrt{\mu/n}$ for
$j\in[k]$ and $\mu>0$. If $\Omega$ samples each entry independently
with probability $q\in(0,1]$, then, with probability at least
$1-Ck^2n^{-12}$,
\[
 \big\|q^{-1}\mathcal P_\Omega(T)-T\big\|_\delta
 \le C\sqrt{k}\left\{
 \mu^{(k-2)/2}\sqrt{\frac nq}\log^2(n)
 +\frac{\mu^{(k-2)/2}}{qn^{(k-2)/2}}\log^{5/2}(n)
 \right\}\|T\|_\infty.
\]
Here $C>0$ is a universal constant.
\end{theorem}
This is the equal-dimension specialization of Theorem~\ref{thm:delocalized_concentration} in {Appendix~\ref{sec:app:proof_tensor_concentration}}.
If $\|T\|_\infty\asymp n^{-k/2}\|T\|$ and $k,\mu$ are fixed,
the bound yields concentration at
$q=n^{-(k-1)}\operatorname{polylog}(n)$.
\endgroup

{\setlength{\emergencystretch}{2em}
The bound is in the spirit of Yuan and Zhang~\cite{yuan2017incoherent}.
Compared with their general incoherent-norm concentration bound for
i.i.d.\ uniform draws with replacement, our Bernoulli-sampling estimate
improves the dimension dependence of the second term by a factor of
$\sqrt n$, up to logarithmic factors, for fixed $k$ and $\mu$.
Theorem~\ref{thm:symmetric_tensor_concentration} and its rectangular
extension in {Appendix~\ref{sec:app:proof_tensor_concentration}} impose only coordinatewise bounds in the
restricted modes, which gives the uniform control of multilinear
terms needed in the Hessian analysis.
Related concentration results for the ordinary tensor spectral norm
include \cite{jain2014provable,cai2022nonconvex,yuan2016tensor,nguyen2015tensor,zhou2021sparse,boedihardjo2024injective,zhou2026sharp}.
\par}

\section{Extension to noisy tensor completion}\label{sec:noisy}

In this section, we extend our wedge-sampling framework to noisy tensor completion. Recovering low-rank structure from noisy observations is a central problem in high-dimensional statistics and machine learning, and has been extensively studied in the literature \citep{candes.plan_2010_matrix,candes2011tight,cai2022nonconvex,ma2023robust,stoger2024non}. Each sampling stage starts with a new noisy copy of $T$. To form it, we add an independent $N(0,\sigma^2)$ error to each entry. If two sampled wedges share an entry, for example, both use the same noisy value. Factor extraction and gradient descent each use a separate uniform sample and independently generated errors. All gradient-descent iterations use the same sampled entries and noisy values. Under this model, we prove that both the wedge-initialized spectral method and the wedge-initialized gradient descent method are stable under noise with strong recovery guarantees. 

We begin by specifying the noisy observation models for wedge sampling and uniform entry sampling.

\paragraph{Wedge sampling under noise}
Recall $A=\unfold_1(T)\in \R^{n\times n^{k-1}}$.
Let $G\in\R^{n\times n^{k-1}}$ have i.i.d. $N(0,\sigma^2)$ entries and define the noisy unfolding
\begin{align}\label{eq:noisy_wedge}
         \widetilde{A}=A+ G.
\end{align}
For every wedge $(i,\ell,j)$ with $1\le i\le j\le n$ and $\ell\in[m]$, draw an independent indicator
        $\delta_{i\ell j}\sim \mathrm{Ber}(p)$.
When $\delta_{i\ell j}=1$, the noisy entries $\widetilde{A}_{i\ell}$ and $\widetilde{A}_{j\ell}$ are revealed from $\widetilde A$ in \eqref{eq:noisy_wedge}.  Define the raw noisy wedge Gram matrix
\begin{equation}\label{eq:raw-wedge-Z}
\begin{aligned}
        \widetilde{Z}
        :=&\sum_{1\le i<j\le n}\sum_{\ell=1}^m
        \frac{\delta_{i\ell j}}{p}\widetilde{A}_{i\ell}\widetilde{A}_{j\ell}(e_ie_j^\top+e_je_i^\top) +\sum_{i=1}^n\sum_{\ell=1}^m
        \frac{\delta_{i\ell i}}{p}(\widetilde{A}_{i\ell})^2 e_ie_i^\top .
\end{aligned}
\end{equation}
For notational convenience, define \(\delta_{\{i,j\},\ell}:=\delta_{\min\{i,j\},\ell,\max\{i,j\}}\), so the wedge indicator is symmetric in its two row indices.
Conditionally on $\widetilde{A}$, $\E_\delta \widetilde{Z}=\widetilde{A}(\widetilde{A})^\top$, and after averaging over both the wedge sampling and the Gaussian noise,
\[
        \E_{\delta,G}\widetilde{Z}=AA^\top+m\sigma^2 I_n.
\]

\paragraph{Uniform entry sampling under noise}
\edit{Independently of the wedge-stage noise \(G\) and wedge indicators \(\delta\), draw}
\[
        \omega_{i_1,\ldots,i_k}\sim \Bern(q),
        \qquad
        H_{i_1,\ldots,i_k}\iid N(0,\sigma^2).
\]
The observed noisy tensor entries are
\begin{align}\label{eq:noisy_uniform}
        \omega_{\bm i}(T_{\bm i}+ H_{\bm i}),
        \qquad
        \bm i=(i_1,\ldots,i_k)\in[n]^k.
\end{align}
Define the debiased noisy tensor
\begin{align}\label{eq:noisy_Y}
        \widetilde\cY=q^{-1}\,\omega\odot(T+ H),
        \qquad
        \edit{\widetilde{Y}=\unfold_1(\widetilde\cY).}
\end{align}

\subsection{Spectral method for noisy tensor completion}
We will use the same Algorithm~\ref{alg:montanari} for noisy tensor completion, with two differences:
\begin{enumerate}
    \item In  Step 5, the $Z$ matrix  is replaced by the noisy version $\widetilde{Z}$ in \eqref{eq:raw-wedge-Z}.
    \item In Step 7, $Y$ is replaced by the noisy version $\widetilde{Y}$ in \eqref{eq:noisy_Y}.
\end{enumerate}

\begin{theorem}[Spectral method for noisy tensor completion]\label{thm:noisy_spectral}
Under the tensor assumptions of Theorem~\ref{thm:final_bound_montanari},
let $k\ge3$ and let $\hat T$ be the output of Algorithm~\ref{alg:montanari}
under the observation models \eqref{eq:noisy_wedge} and \eqref{eq:noisy_uniform}.
Define $\theta=\sigma/\|A\|$. Suppose $\mu_1r\le cn$ and
$\mu_2r\le cn^{k-1}$ and, for a sufficiently small universal $c>0$
and a sufficiently large universal $C_0$,
\begin{align}\label{eq:thm11-simple-noise}
 p n^k &\ge C_0\kappa^4\mu_1\mu_2 r^2\log n,\qquad
 q n^k \ge C_0 4^k\mu_1^k k(\log n)^2,\nonumber\\
        \theta
        &\le
        c\kappa^{-2}
        \min\Bigg\{
        \frac{(\mu_1r)^{1/4}p^{1/2}}
             {\sqrt k\,n^{k/2+1/4}(\log n)^{1/2}},
        \frac{(\mu_1\mu_2)^{1/4}r^{1/2}p^{1/4}}
             {\sqrt k\,n^{3k/4}(\log n)^{1/4}}
        \Bigg\}
\end{align}
Then, with probability at least $1-O(n^{-3})$, the relative Frobenius
error is at most the right-hand side of \eqref{eq:main-spectral-frobenius}
with $a=4$, plus
\[
 C\,\frac{\sigma}{\|T\|_F}\,
       \frac{r^{k/2}+\sqrt{\log n}}{\sqrt q},
\]
where $C$ is universal.
\end{theorem}
The wedge sampling requirement in Theorem~\ref{thm:noisy_spectral} has the same
$r^2$ dependence as in Theorem~\ref{thm:final_bound_montanari}.
The term $r^{k/2}$ in the Gaussian error comes from projecting the
uniform-sample noise onto the tensor product of $k$ rank-$r$ subspaces.
The proof in Appendix~\ref{sec:app:noisy_spectral} separates this term
from the $\sqrt{\log n}$ deviation. For fixed $k,r,\kappa,\mu_1,\mu_2$, the
sampling hypotheses permit $pn^k\asymp\log n$ and $qn^k\asymp(\log n)^2$
at the stated noise level.
\subsection{Gradient descent for noisy tensor completion}

We next give a noisy counterpart of Theorem~\ref{thm:cai_sparse}.  The proof is
deferred to {Appendix~\ref{sec:app:noisy_cai}} and follows the proof strategy of
Cai et al.~\cite{cai2022nonconvex}: first control the initializer, then run the local
leave-one-out gradient-descent induction, and finally convert factor error into
tensor error.
Throughout this subsection,
\[
        T=\sum_{i=1}^r x_i^\star\otimes x_i^\star\otimes x_i^\star,
        \qquad
        X_\star=[x_1^\star,\ldots,x_r^\star],
\]
and
\[
        \lambda_{\min}=\min_i\|x_i^\star\|^3,\qquad
        \lambda_{\max}=\max_i\|x_i^\star\|^3,\qquad
        \kappa_{\mathrm{CP}}=\frac{\lambda_{\max}}{\lambda_{\min}}.
\]
Let \(A=\unfold_1(T)=U\Sigma V^\top\).  In the wedge stage, we use the noisy
model from Section~\ref{sec:noisy}, namely \(\widetilde A=A+G\) with
\(G_{ij}\stackrel{\mathrm{i.i.d.}}{\sim}N(0,\sigma^2)\).  {The extraction and gradient-descent stages each use an independent noisy copy of the tensor.}  In the entrywise
stage, we observe
\[
        \omega_{abc}\bigl(T_{abc}+H_{abc}\bigr),
        \qquad
        H_{abc}\stackrel{\mathrm{i.i.d.}}{\sim}N(0,\sigma^2),
        \qquad
        \omega_{abc}\sim{\mathrm{Bernoulli}}(q),
\]
and use
\[
        F_\sigma(X)
        :=
        \frac{1}{6q}
        \left\|
        \mathcal P_\Omega
        \left(
        T+H-\sum_{i=1}^r x_i\otimes x_i\otimes x_i
        \right)
        \right\|_F^2 .
\]
{
Denote by \(F_0\) the same loss with \(H=0\), and set
\[
        \tau:=\frac{\sigma}{\lambda_{\min}} .
\]
}

\begin{theorem}[Noisy wedge sampling with gradient descent]\label{thm:noisy_cai_sparse}
Let $T$ be $\mu$-CP incoherent of rank $r$ under the observation model above.
Fix $\delta\in(0,1)$ and let
\(\widehat T^t\) be the output of Algorithm~\ref{alg:wedge_GD} with noisy
wedge matrix \(\widetilde Z\), independent noisy batches for extraction and
gradient descent, and the Gaussian draw count in Theorem~\ref{thm:cai_sparse}.
Use objective \(F_\sigma\) and step size \edit{\(0<\eta\le (4\lambda_{\max}^{4/3})^{-1}\)}.  {There exist sufficiently large absolute constants \(c_0,c_1\) and a sufficiently small absolute constant \(c\) such that, if}
{
\[
\begin{aligned}
        p &\ge c_0\frac{\kappa_{\mathrm{CP}}^4\mu^8 r^5\log^2 n}{n^3},
        \qquad
        q \ge c_1\frac{\kappa_{\mathrm{CP}}^6\mu^6 r^4\log^5 n}{n^2},\\
        \tau &\le \frac{c}{\mu^4r^3n^2},
\end{aligned}
\]
}
{then,} for all sufficiently large \(n\) depending only on \(\delta\), with
probability at least \(1-\delta\), simultaneously for all \(0\le t\le n^5\),
\[
\begin{aligned}
 \frac{\|\widehat T^t-T\|_F}{\|T\|_F}
 &\lesssim \rho^t+\kappa_{\mathrm{CP}}^{2/3}\tau\sqrt{\frac{n\log n}{q}},\\
 \frac{\|\widehat T^t-T\|_\infty}{\|T\|_\infty}
 &\lesssim \rho^t+\kappa_{\mathrm{CP}}^{5/3}\sqrt{\mu^3r}\,\tau\sqrt{\frac{n\log n}{q}},
 \qquad \rho=1-\eta\lambda_{\min}^{4/3}/4 .
\end{aligned}
\]
\end{theorem}
At the sampling thresholds, the expected number of tensor-entry observations is
\[
 pn^4+2qn^3
 =O\!\left(\kappa_{\mathrm{CP}}^4\mu^8r^5 n\log^2n
       +\kappa_{\mathrm{CP}}^6\mu^6r^4 n\log^5n\right).
\]
Thus for fixed $\mu,r,\kappa_{\mathrm{CP}}$, the expected sample size is
$O(n\log^5 n)$. In that regime, taking $q$ at its threshold bounds both
relative errors by $O(\rho^t+(\sigma/\lambda_{\min})n^{3/2}/\log^2 n)$
under the stated noise condition. When $\sigma=0$, the geometric rate agrees
with Theorem~\ref{thm:cai_sparse} at the same step size through the stated
time horizon. Under Gaussian noise, the geometric and statistical terms
have the same form as in the local analysis of \citep{cai2022nonconvex}.

%\hl{We plan to add 1 or 2 experiments?}

\section{Numerical experiments}\label{sec:numerical}

Wedge sampling improves initialization, and thus downstream completion, by allocating budget to wedges that induce informative correlations in the unfolded matrices. The experiments follow this pipeline:
(i) subspace recovery for the unfolding (Algorithm~\ref{alg:wedge}; cf.\ Theorem~\ref{thm:wedge_spectral_guarantees}),
(ii) spectral completion (Algorithm~\ref{alg:montanari}; cf.\ Theorem~\ref{thm:final_bound_montanari}),
(iii) refinement by gradient descent from the spectral initializer (Section~5; cf.\ Theorem~\ref{thm:cai_sparse}).
Unless otherwise noted, the unknown tensor \(T\in\R^{n\times n\times n}\) is symmetric, and the ground-truth factor \(X\in\R^{n\times r}\) has i.i.d.\ Gaussian entries with column normalization. Due to storage constraints, we evaluate the approximate relative reconstruction error on a subset of \(n^2\) entries of \(T\), and report the median over {20 runs} with either wedge or uniform sampling. \edit{Figure~\ref{fig:alg1_factor_recovery} presents empirical results for factor recovery in the noiseless setting. Figure~\ref{fig:alg1_factor_recovery_noise}, Figure~\ref{fig:alg2_spectral_TC} and Figure~\ref{fig:alg3_GD_TC} present empirical results for factor and tensor recovery in the noisy setting.}

\begin{figure}[H]
    \centering
    \includegraphics[width=0.7\textwidth]{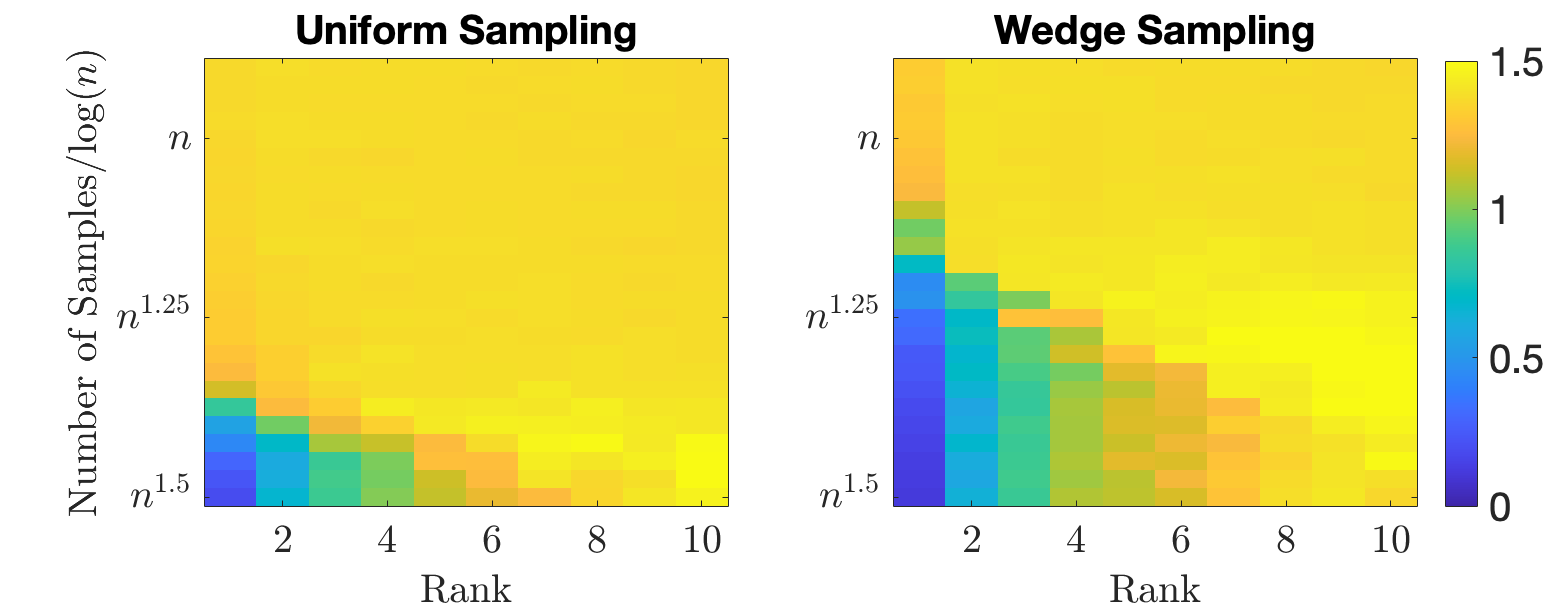}
    \caption{The subspace estimate $\hat{U}$ performance of  Algorithm~\ref{alg:wedge}.}
    \label{fig:alg1_factor_recovery}
\end{figure}

Figure~\ref{fig:alg1_factor_recovery} supports the view that uniform sampling is bottlenecked by spectral initialization. 
We compare wedge and uniform sampling for subspace recovery in factor matrix estimation with \(n=1000\). Each heatmap cell reports the median relative error \(\|\hat U R - U\|/\|U\|\), where $R$ is as defined in \eqref{eq:def_procrustes}, over ranks \(r\in\{1,\dots,10\}\). We use \(p_{\text{wedge}}=O(\log n/n^{s+1})\) and \(p_{\text{unif}}=O(\log n/n^{s})\) with \(s\in\{1.5,1.525,\dots,2.10\}\), and plot the vertical axis in terms of average sample complexity for direct comparison. The results show that wedge sampling achieves accurate recovery at budgets where uniform sampling remains stuck, with the gap widening as \(r\) increases.

\noindent
\begin{minipage}[t]{0.48\textwidth}
    \begin{figure}[H]
        \centering
        \includegraphics[width=\linewidth]{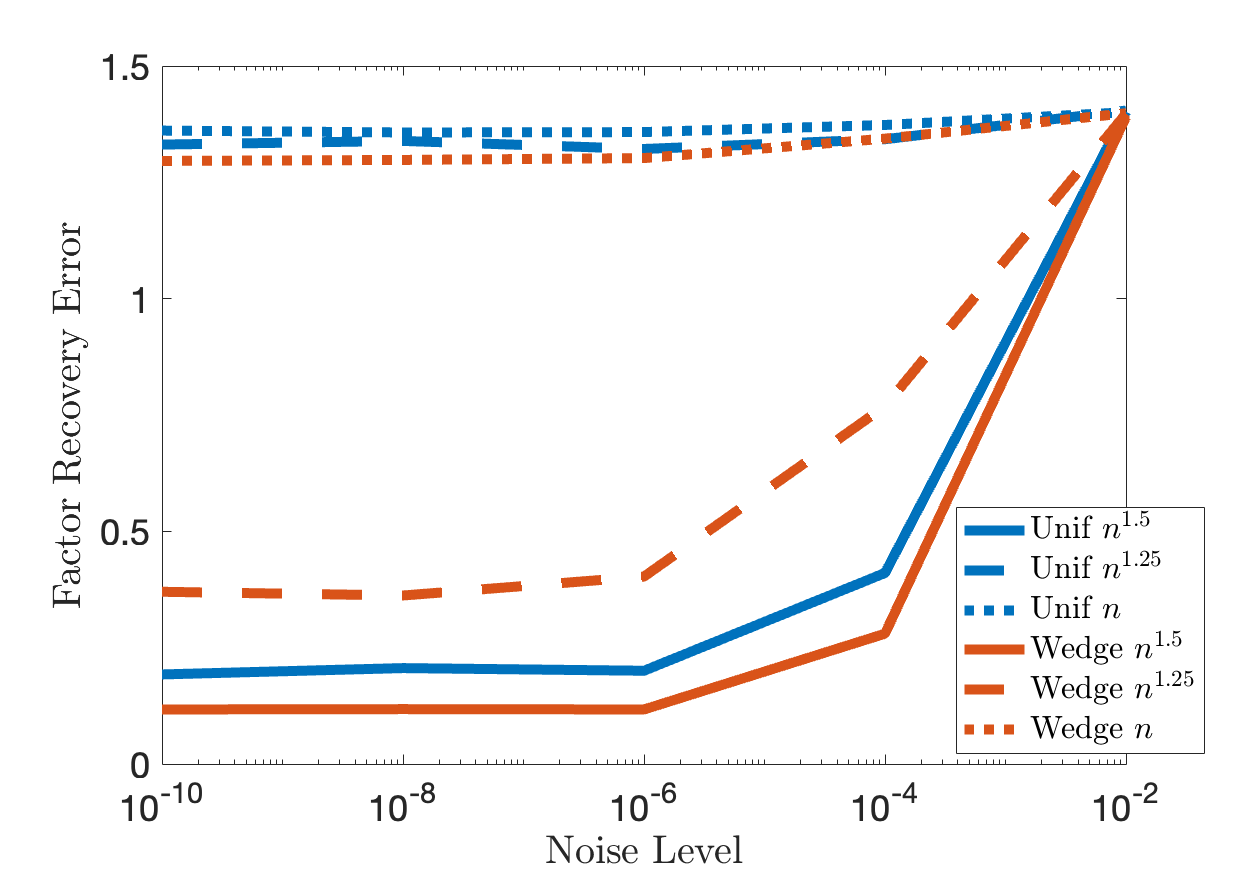}
        \caption{Comparison of subspace recovery error as a function of noise level ($\sigma$). }
        \label{fig:alg1_factor_recovery_noise}
    \end{figure}
\end{minipage}
\hfill
\begin{minipage}[t]{0.48\textwidth}
    \begin{figure}[H]
        \centering
        \includegraphics[width=\linewidth]{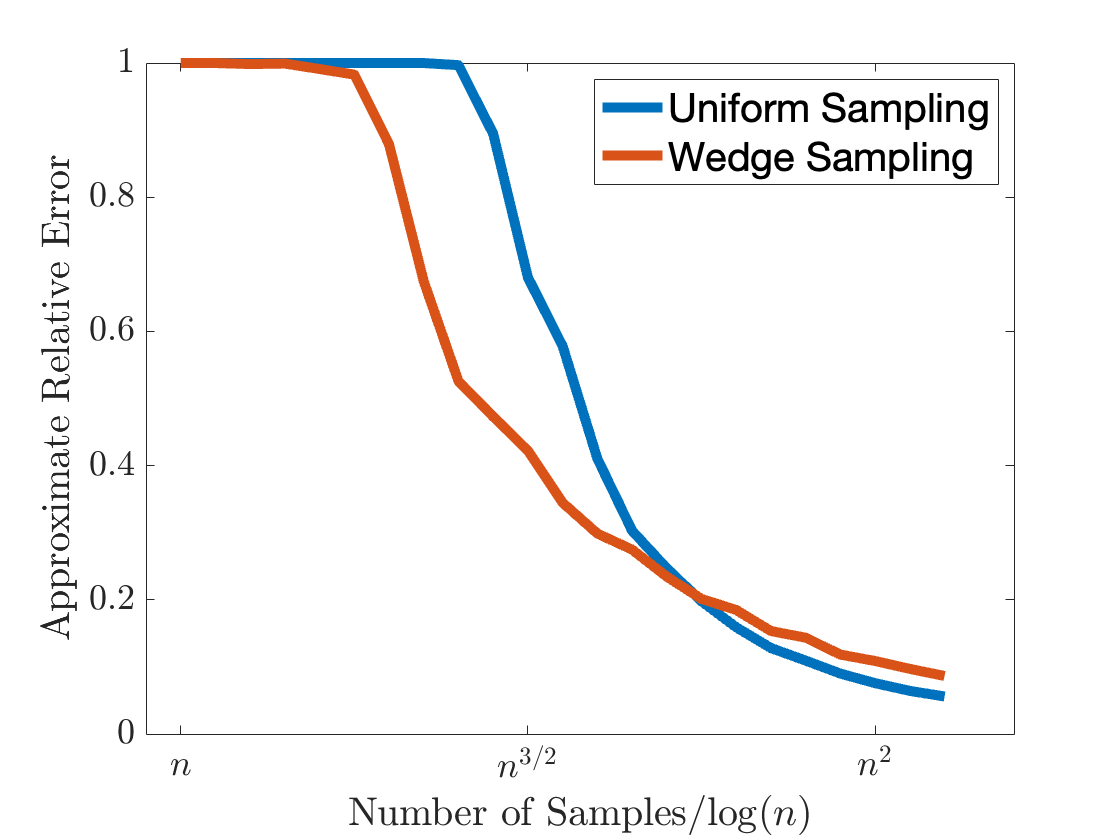}
        \caption{The relative reconstruction error of Algorithm~\ref{alg:montanari}.}
        \label{fig:alg2_spectral_TC}
    \end{figure}
\end{minipage}

We also evaluate wedge sampling in a noisy factor-recovery experiment, \edit{shown in Figure~\ref{fig:alg1_factor_recovery_noise}}. For a symmetric third-order tensor \(\mathcal{T}=\sum_{j=1}^{r}u_j^{\otimes 3}\), \edit{we form the noisy unfolding \(A+G\), where \(A=\unfold_1(\mathcal T)\) and \(G\) has i.i.d.\ \(N(0,\sigma^2)\) entries. We vary \(\sigma\) across experiments and use the same noisy value whenever the sampling scheme selects an entry more than once. } All parameter choices are the same as in Figure~\ref{fig:alg1_factor_recovery}, except that we fix the rank $r =1$ and evaluate fewer choices of sample complexity. In particular, we use \(n=1000\) and  rates \(p_{\text{wedge}}=O(\log n/n^{s+1})\) and \(p_{\text{unif}}=O(\log n/n^{s})\) with \(s\in\{1.5, 1.75, 2\}\). In Figure~\ref{fig:alg1_factor_recovery_noise}, the red lines represent uniform sampling over various approximate sample complexities, and the blue lines represent wedge sampling. The solid, dashed, and dotted lines correspond to $\tilde{O}(n^{1.5})$, $\tilde{O}(n^{1.25})$, and $\tilde{O}(n)$ sample complexities. \edit{In Figure~\ref{fig:alg1_factor_recovery_noise}, the horizontal} axis represents the noise level and the vertical axis is the median, over {20 runs}, relative error  \(\|\hat U R - U\|/\|U\|\), where $R$ is as defined in \eqref{eq:def_procrustes}. We can see that as the noise level increases, the subspace recovery error gets larger, as expected. We also observe that, fixing the sample rate, wedge sampling performs better than uniform sampling.

\edit{Figure~\ref{fig:alg2_spectral_TC} and Figure~\ref{fig:alg3_GD_TC} present empirical results for tensor completion under noisy measurements. The noise generation is as described for Figure~\ref{fig:alg1_factor_recovery_noise} and we fix $\sigma = 10^{-4}$.}
Figure~\ref{fig:alg2_spectral_TC} shows that the improved subspace estimate yields an end-to-end gain for Algorithm~\ref{alg:montanari}: wedge-initialized completion reaches the low-error regime at substantially smaller budgets, while uniform sampling improves only once the sample size is sufficiently large. 
Figure~\ref{fig:alg3_GD_TC} reports subspace-estimation budgets for Algorithm~\ref{alg:wedge_GD} with \(n=100\) and \(r=1\). 
We set \(p_{\text{wedge}}=O(\log n / n^{s+1})\) for wedge sampling, and compare against uniform sampling subspace estimation with \(p_{\text{unif}}=O(\log n / n^{s})\) for \(s\in\{1,1.5,1.75\}\). 
\edit{In both pipelines, the refinement samples used after the subspace-initialization step are drawn uniformly at rate \(p_{\text{unif}}\).} 
The panel label \(n^s\) corresponds to an effective budget \(\tilde O(n^s)\).
Overall, the experiment suggests that wedge sampling primarily improves the spectral initializer: at \(\tilde O(n^2)\) both methods converge quickly, at \(\tilde O(n^{1.5})\) wedge sampling converges faster to a lower error, and at \(\tilde O(n^{1.25})\) uniform sampling typically stalls while wedge sampling continues to make progress.

\begin{figure}
    \centering
    \includegraphics[width=1\linewidth]{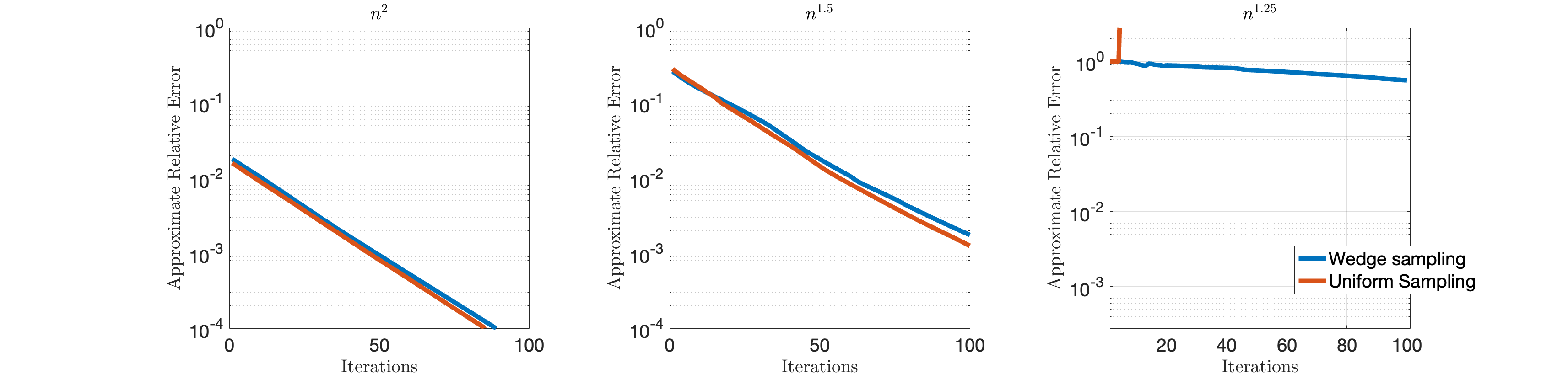}
    \caption{Comparison of tensor completion via GD when using Wedge Sampling and Uniform Sampling for subspace estimation. Each panel corresponds to a different number of samples, ignoring constants and log factors, (left) $n^2$ (middle) $n^{1.5}$, and (right) $n^{1.25}$.  }
    \label{fig:alg3_GD_TC}
\end{figure}

\section*{Acknowledgments}
A preliminary version of this work was presented at the Conference on Learning Theory (COLT) 2026.
H.L. was partially supported by the National Science Foundation Grant NSF-DMS 2412403. A.M. was partially supported by the Rose Hills Innovator Fellowship, Simons Grant MPS-TSM-00014026 and NSF DMS-2606335. Y.Z. was partially supported by the Simons Grant MPS-TSM-00013944 and NSF DMS-2606337. Part of the work was done while  Y.Z. was visiting the Simons Institute for the Theory of Computing during the Spectral Theory Beyond Graphs program in Fall 2026. 

\textbf{AI disclosure.}
The authors developed the mathematical ideas and proofs, designed the algorithms,
and carried out the simulations. During revision, ChatGPT helped the authors
sharpen the rank and condition-number dependence of the sample complexity
bounds and correct minor proof issues. It also assisted with routine
calculations in the proof extending the results to additive Gaussian noise.
The authors reviewed all AI-assisted revisions and assume responsibility for
all content.

\begingroup
\setlength{\emergencystretch}{2em}
\bibliographystyle{plainnat}
\IfFileExists{arxiv/wedge.bib}{\bibliography{arxiv/wedge}}{\bibliography{wedge}}
\endgroup

\appendix
\section[Proof of Lemma~\ref{lem:unfold_incoherence}]{Proof of {Lemma~\ref{lem:unfold_incoherence}}}\label{sec:proof_lemma_unfold_incoherence}

\begin{proof}
\begingroup
Fix $j\in[k]$. Normalize the CP factors by
$a_i^{(\ell)}=x_i^{(\ell)}/\|x_i^{(\ell)}\|_2$ and put
$\lambda_i=\prod_{\ell=1}^k\|x_i^{(\ell)}\|_2$, so
\[
 T=\sum_{i=1}^r\lambda_i\,a_i^{(1)}\otimes\cdots\otimes a_i^{(k)}.
\]
Using the factor matrices
\[
 A^{(\ell)}=[a_1^{(\ell)},\ldots,a_r^{(\ell)}],\qquad
 \Lambda=\diag(\lambda_1,\ldots,\lambda_r),\qquad
 m_j=\prod_{\ell\ne j}n_\ell,
\]
\[
 B^{(j)}=[b_1^{(j)},\ldots,b_r^{(j)}],\qquad
 b_i^{(j)}=\bigotimes_{\ell\ne j}a_i^{(\ell)},
\]
the mode-$j$ unfolding is
\begin{equation}\label{eq:unfold_factorization}
 \unfold_j(T)=A^{(j)}\Lambda(B^{(j)})^\top.
\end{equation}

The two Gram matrices $G^{(j)}=(A^{(j)})^\top A^{(j)}$ and
$H^{(j)}=(B^{(j)})^\top B^{(j)}$ have unit diagonal.
For $i\ne i'$, \eqref{eq:cp_incoherent} gives
\[
 |G^{(j)}_{ii'}|\le\frac\mu{n_j},\qquad
 |H^{(j)}_{ii'}|
 =\prod_{\ell\ne j}|\langle a_i^{(\ell)},a_{i'}^{(\ell)}\rangle|
 \le\frac{\mu^{k-1}}{m_j}.
\]
The rank condition is
\begin{equation}\label{eq:mu_condition}
 r\mu\le\frac12\min_j n_j.
\end{equation}
Thus $m_j\ge(2r\mu)^{k-1}$ and
$(r-1)\mu^{k-1}/m_j\le(r-1)/(2r)^{k-1}\le1/2$.
Gershgorin's theorem now gives both bounds at once:
\begin{equation}\label{eq:g_min}
\begin{aligned}
 \lambda_{\min}(G^{(j)})&\ge1-\frac{(r-1)\mu}{n_j}\ge\frac12,\\
 \lambda_{\min}(H^{(j)})&\ge1-\frac{(r-1)\mu^{k-1}}{m_j}\ge\frac12.
\end{aligned}
\end{equation}
Both factor matrices therefore have full column rank. Since
$\Lambda$ is invertible, \eqref{eq:unfold_factorization} has rank $r$,
with column space $\col(A^{(j)})$ and row space $\col(B^{(j)})$.
Consequently, for the SVD
$\unfold_j(T)=U^{(j)}\Sigma^{(j)}(V^{(j)})^\top$, there are
orthogonal $R,S\in\mathbb R^{r\times r}$ such that
\begin{equation}\label{eq:U_from_A}
 U^{(j)}=A^{(j)}(G^{(j)})^{-1/2}R,\qquad
 V^{(j)}=B^{(j)}(H^{(j)})^{-1/2}S.
\end{equation}

The entrywise part of \eqref{eq:cp_incoherent} gives, for
$t\in[n_j]$ and $\mathbf t=(t_\ell)_{\ell\ne j}$,
\[
 \|A^{(j)}_{t,:}\|_2^2
 =\sum_i|a_i^{(j)}(t)|^2\le\frac{\mu r}{n_j},\qquad
 \|B^{(j)}_{\mathbf t,:}\|_2^2
 =\sum_i\prod_{\ell\ne j}|a_i^{(\ell)}(t_\ell)|^2
 \le\frac{\mu^{k-1}r}{m_j}.
\]
Combining these row bounds with \eqref{eq:g_min} and
\eqref{eq:U_from_A} yields
\[
\begin{aligned}
 \|U^{(j)}\|_{2,\infty}^2
 &\le\|A^{(j)}\|_{2,\infty}^2\|(G^{(j)})^{-1/2}\|^2
 \le\frac{2\mu r}{n_j},\\
 \|V^{(j)}\|_{2,\infty}^2
 &\le\|B^{(j)}\|_{2,\infty}^2\|(H^{(j)})^{-1/2}\|^2
 \le\frac{2\mu^{k-1}r}{m_j}.
\end{aligned}
\]
Since $j$ was arbitrary and every unfolding has rank $r$, these
inequalities verify Definition~\ref{def:tensor_incoherence} with
$\mu_1\le2\mu$ and $\mu_2\le2\mu^{k-1}$.
\endgroup
\end{proof}

\section[Proofs for Theorems~\ref{thm:wedge_concentration} and~\ref{thm:wedge_spectral_guarantees}]{Proofs for {Theorems~\ref{thm:wedge_concentration} and~\ref{thm:wedge_spectral_guarantees}}}\label{sec:app:wedge_proofs}

\subsection{Preliminaries}

We first collect a few consequences of $(\mu_1, \mu_2)$-incoherence on a matrix $A$.

\begin{lemma}\label{lem:norm_bound}
    For any $A \in \R^{d_1\times d_2}, B\in \R^{d_2\times d_3}$, 
    \[\|AB\|_{2,\infty}\leq \|A\|_{2,\infty} \|B\|.\]
\end{lemma}
\begin{proof}
\begingroup
For each row,
\[
 \|e_i^\top AB\|_2\le\|e_i^\top A\|_2\|B\|.
\]
Taking the maximum over $i$ proves the claim.
\endgroup
\end{proof}

\begin{lemma}\label{lem:matrix_norm}
    Assume that $A\in\mathbb{R}^{n\times m}$ is $(\mu_1,\mu_2)$-incoherent, we have the following holds for $A$:
    \begin{align}
     \|A\|_{2,\infty}&\leq \sqrt{\frac{\mu_1r}{n}} \|A\|, \quad 
     \|A\|_{\infty,2}\leq \sqrt{\frac{\mu_2r}{m}} \|A\|, \label{eq:2-infty_A}\\
     \|A\|_{\max}&\leq \sqrt{\frac{\mu_1\mu_2}{mn}} r\|A\| \label{eq:infty_norm}, \\
        \|A\circ A\|_{\max} &\leq \frac{\mu_1\mu_2 r^2}{mn} \|A\|^2,
    \end{align}
    where  $\circ$ denotes the Hadamard, or element-wise, product.
\end{lemma}
\begin{proof}
\begingroup
Denote the singular value decomposition by $A=U\Sigma V^\top$. Orthonormality and
Definition~\ref{def:mu12-incoh} give
\[
 \|A\|_{2,\infty}\le\|U\|_{2,\infty}\|\Sigma\|
 \le\sqrt{\frac{\mu_1r}{n}}\|A\|,
 \qquad
 \|A\|_{\infty,2}\le\|V\|_{2,\infty}\|\Sigma\|
 \le\sqrt{\frac{\mu_2r}{m}}\|A\|.
\]
This proves \eqref{eq:2-infty_A}. For each entry,
\[
 |A_{ij}|=|e_i^\top U\Sigma V^\top e_j|
 \le\|U\|_{2,\infty}\|\Sigma\|\|V\|_{2,\infty}
 \le r\sqrt{\frac{\mu_1\mu_2}{mn}}\|A\|,
\]
which is \eqref{eq:infty_norm}. Squaring gives the bound for
$\|A\circ A\|_{\max}=\|A\|_{\max}^2$.
\endgroup
\end{proof}

We will restate and use the following version of Matrix Bernstein's inequality, stated in  \citep[Lemma 11]{chen2014robust} as a consequence of  \citep[Theorem 1.6]{tropp2012user}:
\begin{lemma}[Matrix Bernstein inequality]
\label{thm:(Matrix-Bernstein-inequality)}
Consider ${m}$ independent random matrices $M_\ell$ $(1\le \ell\le m)$ of dimension $d_1\times d_2$, each satisfying 
\[
\mathbb{E}[M_\ell]=0
\quad\text{and}\quad
\|M_\ell\|\le B \ \text{almost surely}.
\]
Define
\[
\sigma^2 :=
\max\left\{
\left\|\sum_{\ell=1}^m \mathbb{E}[M_\ell M_\ell^\top]\right\|
\;,\;
\left\|\sum_{\ell=1}^m \mathbb{E}[M_\ell^\top M_\ell]\right\|
\right\}.
\]
Then there exists a universal constant $C>0$ such that for any   $a\ge2$,
\[
\left\|\sum_{\ell=1}^m M_\ell\right\|
\le
C\left(
\sqrt{a\sigma^2\log(d_1+d_2)}
+
aB\log(d_1+d_2)
\right)
\]
with probability at least $1-(d_1+d_2)^{-a}$.
\end{lemma}

\paragraph{Matrix perturbation} We summarize a few matrix perturbation concepts for subspace reconstruction. Given orthogonal matrices $U, \hat U \in \R^{n \times r}$, we define $H = \hat U^\top U$. Since $\|H\| \leq 1$, its SVD has the form \[H = U_H \cos(\Theta) V_H^\top, \] where \[\cos(\Theta) = \diag(\cos(\theta_1), \dots, \cos(\theta_r)).\] The $\theta_i$ are known as the \emph{principal angles} between the subspaces spanned by $U$ and $\hat U$.

{It is known \citep{schonemann1966generalized} that the solution to the Procrustes problem \eqref{eq:def_procrustes} is the \emph{sign matrix} of $H$, defined as} 
\begin{align}\label{def:sgn_matrix}
   \sgn(H) = U_H V_H^\top. 
\end{align}
We collect some interesting facts about these quantities in the following lemma:
\begin{lemma}\label{lem:sintheta_bounds}
Let \(H=\hat U^\top U\) and \(R=\sgn(H)\). In operator norm, the following identities hold:
    \begin{align}
        \|\hat U \hat U^\top - UU^\top \| &= \|\sin(\Theta)\|, \\
        \|H - \sgn(H) \| &\leq \|\sin(\Theta)\|^2, \\
        \|\hat U\sgn(H) - U\| &\leq 2\|\sin(\Theta)\|.
    \end{align}
    {In Frobenius norm, the first identity becomes
    \[
    \|\hat U\hat U^\top-UU^\top\|_F
    =\sqrt{2}\,\|\sin(\Theta)\|_F.
    \]
    The other two inequalities also hold in Frobenius norm.}
\end{lemma}
\begin{proof}
\begingroup
The principal-angle decomposition $H=U_H\cos(\Theta)V_H^\top$
gives the operator and Frobenius projection identities
\[
\begin{aligned}
 \|\hat U\hat U^\top-UU^\top\|^2
 &=\|I-HH^\top\|=\|\sin\Theta\|^2,\\
 \|\hat U\hat U^\top-UU^\top\|_F^2
 &=2r-2\operatorname{tr}(HH^\top)=2\|\sin\Theta\|_F^2.
\end{aligned}
\]
The singular values of $H-R$ are $1-\cos\theta_i$.
Since $1-\cos\theta_i\le\sin^2\theta_i$,
\[
 \|H-R\|\le\|\sin\Theta\|^2,\qquad
 \|H-R\|_F^2\le\sum_i\sin^4\theta_i
 \le\left(\sum_i\sin^2\theta_i\right)^2.
\]
Finally, $R=U_HV_H^\top$ implies
\[
 (\hat UR-U)^\top(\hat UR-U)
 =2I-R^\top H-H^\top R
 =2V_H(I-\cos\Theta)V_H^\top.
\]
Taking the operator norm and trace, respectively, yields
\[
 \|\hat UR-U\|\le\sqrt2\|\sin\Theta\|,
 \qquad
 \|\hat UR-U\|_F\le\sqrt2\|\sin\Theta\|_F,
\]
which are stronger than the two asserted Procrustes bounds.
\endgroup
\end{proof}

{The Procrustes inequality in Lemma~\ref{lem:sintheta_bounds} reduces comparison of $\hat U$ and $U$ to a bound on $\sin(\Theta)$. We use the Davis--Kahan inequality in Lemma~\ref{lem:DK} (see, e.g., \citep[Corollary 2.8]{chen2021spectral}).}

\begin{lemma}[Davis-Kahan inequality]\label{lem:DK}
Let $M, \hat M$ be two symmetric matrices, with \edit{$\Delta = \hat M - M$}. Denote the eigendecompositions of $M$ and $\hat M$ by
\begin{align}
\hat M=
\hat U
\hat \Lambda
\hat U^{\top} + \hat U_{\perp}
\hat \Lambda_{\perp}
\hat U_{\perp}^{\top}, \quad M=U\Lambda U^\top +U_{\perp} \Lambda_{\perp} U_{\perp}^\top,
\end{align}
where $\Lambda$ (resp. $\hat\Lambda$) contains the top $r$ eigenvalues of $M$ (resp $\hat M$).
{Assume a spectral gap}
\(\displaystyle
\delta\;:=\lambda_r(M)-\lambda_{r+1}(M) > 0, 
\)
and a perturbation bound \(\|\Delta\|\le\delta/2\).
Then
\begin{align}
\|\sin(\Theta)\| \leq \frac{2\|\Delta U\|}{\delta} \leq 
  \frac{2\,\|\Delta\|}{\delta}.
\end{align}

\end{lemma}

\subsection[Proof of Theorem~\ref{thm:wedge_concentration}]{Proof of {Theorem~\ref{thm:wedge_concentration}}}

\begin{proof}
\begingroup
For $i<j$, the centered symmetric summands are
\begin{equation}\label{eq:def_Silj}
 S_{i\ell j}=p^{-1}
 \bigl(\mathbf1\{(i,\ell,j)\in\widetilde{\mathcal W}\}-p\bigr)
 A_{i\ell}A_{j\ell}(E_{ij}+E_{ji}),\qquad E_{ij}=e_ie_j^\top.
\end{equation}
For $i=j$, use the single diagonal term
$S_{i\ell i}=p^{-1}(\mathbf1\{(i,\ell,i)\in\widetilde{\mathcal W}\}-p)
 A_{i\ell}^2E_{ii}$.
The matrices are independent and centered, and
$Z-AA^\top=\sum_{i\le j,\ell}S_{i\ell j}$.
Lemma~\ref{lem:matrix_norm} bounds their norms by
\begin{equation}\label{eq:Sikj_bound}
 \|S_{i\ell j}\|\le p^{-1}\|A\|_{\max}^2
 \le\frac{\mu_1\mu_2r^2}{pmn}\|A\|^2.
\end{equation}
Since $(E_{ij}+E_{ji})^2=E_{ii}+E_{jj}$ for $i\ne j$,
\begin{equation}\label{eq:E_sum}
 \sum_{i\le j,\ell}\mathbb E S_{i\ell j}^2
 =\frac{1-p}{p}\sum_{i,j,\ell}A_{i\ell}^2A_{j\ell}^2E_{ii}.
\end{equation}
Thus \eqref{eq:2-infty_A} gives the variance bound
\[
 \left\|\sum_{i\le j,\ell}\mathbb E S_{i\ell j}^2\right\|
 \le\frac1p\max_i\sum_\ell A_{i\ell}^2\sum_jA_{j\ell}^2
 \le\frac{\|A\|_{2,\infty}^2\|A\|_{\infty,2}^2}{p}
 \le\frac{\mu_1\mu_2r^2}{pmn}\|A\|^4.
\]
Lemma~\ref{thm:(Matrix-Bernstein-inequality)} now gives, with failure
probability at most $(2n)^{-a}\le n^{-a}$,
\[
 \|Z-AA^\top\|
 \le C\left(\sqrt{\frac{a\mu_1\mu_2r^2\log n}{pmn}}
       +\frac{a\mu_1\mu_2r^2\log n}{pmn}\right)\|A\|^2,
\]
which proves Theorem~\ref{thm:wedge_concentration}.
\endgroup
\end{proof}

\subsection[Proof of Theorem~\ref{thm:wedge_spectral_guarantees}: \texorpdfstring{$\ell_2$}{l2} bound]{Proof of {Theorem~\ref{thm:wedge_spectral_guarantees}}: \texorpdfstring{$\ell_2$}{l2} bound}

{If $r=n$, then $\hat U\hat U^\top=UU^\top=I$ and the Procrustes rotation is $R=\hat U^\top U$, so $\hat UR=U$ exactly. Both conclusions of {Theorem~\ref{thm:wedge_spectral_guarantees}} are then immediate. In the proofs below we therefore assume $r<n$.}

Recall $\kappa=\frac{\sigma_{1}(A)}{\sigma_{r}(A)}$ and   denote in decreasing order \[\lambda_i=\lambda_i(AA^\top)=\sigma_i^2(A), \quad 1\leq i\leq r.\] As a corollary of {Theorem~\ref{thm:wedge_concentration}} by using the condition $p\geq \frac{C_0\kappa^4 a \mu_1\mu_2 r^2\log n}{mn}$, we obtain the following:
\begin{lemma}\label{lemma:spectral_gap}
       There is a universal constant $C_0>0$ such that  when $p\geq \frac{C_0\kappa^4 a \mu_1\mu_2 r^2\log n}{mn}$, we have 
    \begin{align}\label{eq:lambda_r}
        \lambda_r \geq 4 \|Z-AA^\top \|
    \end{align}
    with probability $1-O(n^{-a})$.
\end{lemma}
\begin{proof}
\begingroup
On the concentration event of Theorem~\ref{thm:wedge_concentration},
\begin{equation}\label{eq:app:spectral_gap_bernstein}
 \|Z-AA^\top\|
 \le C\kappa^2\left(\sqrt{\frac{a\mu_1\mu_2r^2\log n}{pmn}}
       +\frac{a\mu_1\mu_2r^2\log n}{pmn}\right)\lambda_r,
\end{equation}
since $\|A\|^2=\kappa^2\lambda_r$.
The sampling condition in Lemma~\ref{lemma:spectral_gap} gives
\begin{equation}\label{eq:app:spectral_gap_coefficients}
 \kappa^2\sqrt{\frac{a\mu_1\mu_2r^2\log n}{pmn}}
 \le C_0^{-1/2},\qquad
 \kappa^2\frac{a\mu_1\mu_2r^2\log n}{pmn}
 \le C_0^{-1}\kappa^{-2}\le C_0^{-1}.
\end{equation}
Choose $C_0$ so that $C(C_0^{-1/2}+C_0^{-1})\le1/4$.
Then \eqref{eq:app:spectral_gap_bernstein}--\eqref{eq:app:spectral_gap_coefficients}
imply \eqref{eq:lambda_r}, with the same $O(n^{-a})$ failure probability.
\endgroup
\end{proof}

\begingroup
On \eqref{eq:lambda_r}, Lemmas~\ref{lem:DK} and~\ref{lem:sintheta_bounds}
give
\[
 \|\hat U\sgn(H)-U\|
 \le\frac{4\|Z-AA^\top\|}{\lambda_r}
 \le C\kappa^2\left(\sqrt{\frac{a\mu_1\mu_2r^2\log n}{pmn}}
                  +\frac{a\mu_1\mu_2r^2\log n}{pmn}\right).
\]
The sampling condition of Theorem~\ref{thm:wedge_spectral_guarantees}
makes $a\mu_1\mu_2r^2\log n/(pmn)\le1$, so the linear term is
bounded by the square-root term. Hence
\[
 \|\hat U\sgn(H)-U\|
 \le C\sqrt{\frac{\kappa^4a\mu_1\mu_2r^2\log n}{pmn}},
\]
proving \eqref{eq:wedge_spectral_op} with failure probability $O(n^{-a})$.
\endgroup

{\setlength{\emergencystretch}{2em}
\subsection[Proof of Theorem~\ref{thm:wedge_spectral_guarantees}: \texorpdfstring{$\ell_{2, \infty}$}{l2,infinity} bound]{Proof of {Theorem~\ref{thm:wedge_spectral_guarantees}}: \texorpdfstring{$\ell_{2, \infty}$}{l2,infinity} bound}\label{subsec:app:wedge_spectral}

\paragraph{Preliminary lemmas}
Define $\Delta=Z-AA^\top$ and
\[
 Z^{(s)}=AA^\top+\Delta^{(s)},
 \qquad
 \eta=\sqrt{\frac{a\mu_1\mu_2r^2\log n}{pmn}},
\]
where $\Delta^{(s)}$ is obtained by setting row and column $s$ of $\Delta$ to zero. Let $\hat U^{(s)}$ contain the leading $r$ eigenvectors of $Z^{(s)}$, and put $H^{(s)}=(\hat U^{(s)})^\top U$ and $H=\hat U^\top U$.

\begin{lemma}\label{lemma:submatrix}
With probability $1-O(n^{-a})$,
\[
 \max_{s\in[n]}\|Z^{(s)}-AA^\top\|
 \le C(\eta+\eta^2)\|A\|^2.
\]
\end{lemma}
\begin{proof}
{Theorem~\ref{thm:wedge_concentration}} gives $\|\Delta\|\le C(\eta+\eta^2)\|A\|^2$. Since $\Delta^{(s)}$ is a compression of $\Delta$, its operator norm is no larger than $\|\Delta\|$, simultaneously for all $s$.
\end{proof}

\begin{lemma}\label{lemma:Davis_Kahan_loo}
For a sufficiently large universal constant $C_0$, assume
\[
 p\ge\frac{C_0\kappa^4a\mu_1\mu_2r^2\log n}{mn}.
\]
Then, with probability $1-O(n^{-a})$, simultaneously for all $s\in[n]$,
\begin{align}
 \|H^{-1}\|&\le2,\label{eq:415claim2}\\
 \|(H^{(s)})^{-1}\|&\le2,\label{eq:415claim3}\\
 \|\hat U\hat U^\top-\hat U^{(s)}(\hat U^{(s)})^\top\|
 &\le \frac{C\|(Z-Z^{(s)})\hat U^{(s)}\|}{\lambda_r}.
 \label{eq:415claim4}
\end{align}
\end{lemma}
\begin{proof}
\begingroup
Theorem~\ref{thm:wedge_concentration} and a sufficiently large $C_0$
give the event
\begin{equation}\label{eq:app:loo_small_perturbation}
 \lambda_r\ge8\|\Delta\|,\qquad \lambda_r=\sigma_r(A)^2,
\end{equation}
with probability $1-O(n^{-a})$. Lemmas~\ref{lem:DK} and
\ref{lem:sintheta_bounds} imply
\begin{equation}\label{eq:app:loo_cross_gram_bounds}
\begin{aligned}
 \|H-\sgn(H)\|&\le\left(\frac{2\|\Delta\|}{\lambda_r}\right)^2
 \le\frac1{16},\\
 \sigma_{\min}(H)&\ge1-\|H-\sgn(H)\|\ge\frac{15}{16}.
\end{aligned}
\end{equation}
Since $\|\Delta^{(s)}\|\le\|\Delta\|$ for all $s$, the same
principal-angle bounds applied to $Z^{(s)}$ give
\begin{equation}\label{eq:app:loo_aux_inverse_bound}
 \|H^{(s)}-\sgn(H^{(s)})\|
 \le\left(\frac{2\|\Delta^{(s)}\|}{\lambda_r}\right)^2
 \le\frac1{16},\qquad
 \|(H^{(s)})^{-1}\|\le\frac{16}{15}<2.
\end{equation}
Equations~\eqref{eq:app:loo_cross_gram_bounds} and
\eqref{eq:app:loo_aux_inverse_bound} prove
\eqref{eq:415claim2}--\eqref{eq:415claim3}.
For the subspace comparison, Weyl's inequality gives
\[
 \lambda_r(Z^{(s)})-\lambda_{r+1}(Z^{(s)})
 \ge\lambda_r-2\|\Delta^{(s)}\|\ge\frac34\lambda_r,
 \qquad \|Z-Z^{(s)}\|\le2\|\Delta\|\le\frac14\lambda_r.
\]
The perturbation is at most half this gap. Apply Lemma~\ref{lem:DK}
with reference matrix $Z^{(s)}$, and use the projection identity in
Lemma~\ref{lem:sintheta_bounds}, to obtain
\[
\begin{aligned}
 \|\hat U\hat U^\top-\hat U^{(s)}(\hat U^{(s)})^\top\|
 &\le\frac{2\|(Z-Z^{(s)})\hat U^{(s)}\|}
              {\lambda_r(Z^{(s)})-\lambda_{r+1}(Z^{(s)})}\\
 &\le\frac{8\|(Z-Z^{(s)})\hat U^{(s)}\|}{3\lambda_r}.
\end{aligned}
\]
This proves \eqref{eq:415claim4} simultaneously for all $s$ on
\eqref{eq:app:loo_small_perturbation}.
\endgroup
\end{proof}

\begin{lemma}\label{lem:FixF}
For any fixed $F\in\mathbb R^{n\times r}$, with probability $1-O(n^{-a})$,
\[
 \|\Delta F\|_{2,\infty}
 \le C(\eta+\eta^2)\|A\|^2\|F\|_{2,\infty}.
\]
In particular,
\begin{equation}\label{eq:E44}
 \|\Delta U\|_{2,\infty}
 \le C(\eta+\eta^2)\|A\|^2\|U\|_{2,\infty}.
\end{equation}
\end{lemma}
\begin{proof}
\begingroup
For each fixed row $s$,
\[
 e_s^\top\Delta F
 =\sum_{j=1}^n\sum_{\ell=1}^m
 \frac{\mathbf1\{(\min\{s,j\},\ell,\max\{s,j\})
          \in\widetilde{\mathcal W}\}-p}{p}
 A_{s\ell}A_{j\ell}F_{j,\cdot}
\]
is a sum of independent centered row vectors. Their norms satisfy
\begin{equation}\label{eq:norm_split}
 \frac{\|A\|_{\max}^2}{p}\|F\|_{2,\infty}
 \le\frac{\mu_1\mu_2r^2}{pmn}\|A\|^2\|F\|_{2,\infty}.
\end{equation}
Both rectangular Bernstein variance parameters are bounded by the
sum of expected squared Euclidean norms:
\begin{equation}\label{eq:app:fixed-row-variance}
\begin{aligned}
 \sigma^2
 &\le\frac{1-p}{p}\sum_{j,\ell}
 A_{s\ell}^2A_{j\ell}^2\|F_{j,\cdot}\|_2^2
 \le\frac{\|F\|_{2,\infty}^2}{p}
       \sum_\ell A_{s\ell}^2\sum_j A_{j\ell}^2\\
 &\le\frac{\|F\|_{2,\infty}^2}{p}
       \|A\|_{2,\infty}^2\|A\|_{\infty,2}^2
 \le\frac{\mu_1\mu_2r^2}{pmn}\|A\|^4\|F\|_{2,\infty}^2.
\end{aligned}
\end{equation}
Apply Lemma~\ref{thm:(Matrix-Bernstein-inequality)} to this
$1\times r$ sum with tail parameter $(a+2)\log n/\log(1+r)$.
Equations~\eqref{eq:norm_split} and \eqref{eq:app:fixed-row-variance},
with $a+2\le2a$, give the threshold
$C(\eta+\eta^2)\|A\|^2\|F\|_{2,\infty}$ and failure probability
$n^{-a-2}$ for each $s$. A union bound over the rows gives failure
at most $n^{-a-1}$, proving the lemma. Taking $F=U$ gives
\eqref{eq:E44}.
\endgroup
\end{proof}

\begin{lemma}
With probability $1-O(n^{1-a})$, simultaneously for all $s\in[n]$,
\begin{equation}\label{eq:step2}
 \|\Delta_{s,\cdot}(\hat U^{(s)}H^{(s)}-U)\|
 \le C(\eta+\eta^2)\|A\|^2
             \|\hat U^{(s)}H^{(s)}-U\|_{2,\infty}.
\end{equation}
\end{lemma}
\begin{proof}
\begingroup
Fix $s$ and condition on the wedge indicators whose two row indices
avoid $s$. Then $\hat U^{(s)}H^{(s)}-U$ is fixed, while the indicators
in $\Delta_{s,\cdot}$ remain independent Bernoulli variables with
mean $p$. Apply the fixed-row calculation in the proof of
Lemma~\ref{lem:FixF} with $F=\hat U^{(s)}H^{(s)}-U$.
Its increment bound \eqref{eq:norm_split} and variance bound
\eqref{eq:app:fixed-row-variance} hold for every conditioning value,
and its Bernstein tail parameter gives conditional failure
probability at most $n^{-a-2}$. Thus \eqref{eq:step2} fails for
at least one $s$ with probability at most $n\cdot n^{-a-2}
=n^{-a-1}$, which is stronger than the stated $O(n^{1-a})$ bound.
\endgroup
\end{proof}

\begin{lemma}\label{lemma:LOO_concentration}
For every fixed $F\in\mathbb R^{n\times r}$, with probability $1-O(n^{1-a})$, simultaneously for all $s\in[n]$,
\begin{equation}\label{eq:E550}
 \|(Z-Z^{(s)})F\|
 \le C(\eta+\eta^2)\|A\|^2\|F\|_{2,\infty}.
\end{equation}
{The $1-O(n^{1-a})$ probability bound in \eqref{eq:E550} also holds with $F=\hat U^{(s)}$, which is independent of $Z-Z^{(s)}$ for each $s$, giving}
\begin{equation}\label{eq:E55}
 \|(Z-Z^{(s)})\hat U^{(s)}\|
 \le C(\eta+\eta^2)\|A\|^2\|\hat U^{(s)}\|_{2,\infty}
 \qquad(s\in[n]).
\end{equation}
\end{lemma}
\begin{proof}
\begingroup
Use the symmetric indexing convention
\[
 w_{i\ell j}=p^{-1}\bigl(\mathbf1\{(\min\{i,j\},\ell,\max\{i,j\})
                 \in\widetilde{\mathcal W}\}-p\bigr).
\]
For fixed $s$, the expansion
\[
 (Z-Z^{(s)})F
 =\sum_{i\ne s,\ell}w_{i\ell s}A_{i\ell}A_{s\ell}
             (e_iF_{s,\cdot}+e_sF_{i,\cdot})
  +\sum_\ell w_{s\ell s}A_{s\ell}^2e_sF_{s,\cdot}
\]
has independent centered summands, each with norm at most
\begin{equation}\label{eq:app:loo_residual_increment}
 2p^{-1}\|A\|_{\max}^2\|F\|_{2,\infty}.
\end{equation}
For an off-diagonal summand, its left second moment is at most
\[
 \frac2p A_{i\ell}^2A_{s\ell}^2
 \bigl(\|F_{s,\cdot}\|_2^2E_{ii}
       +\|F_{i,\cdot}\|_2^2E_{ss}\bigr)
\]
in the semidefinite order, and its right second moment is
\[
 \frac{1-p}{p}A_{i\ell}^2A_{s\ell}^2
 \bigl(F_{s,\cdot}^\top F_{s,\cdot}
       +F_{i,\cdot}^\top F_{i,\cdot}\bigr).
\]
After including the diagonal terms, both variance norms are bounded by
\begin{equation}\label{eq:app:loo_residual_variance}
 \frac4p\|F\|_{2,\infty}^2\sum_\ell A_{s\ell}^2\sum_i A_{i\ell}^2
 \le\frac{4\mu_1\mu_2r^2}{pmn}\|A\|^4\|F\|_{2,\infty}^2.
\end{equation}
Use \eqref{eq:app:loo_residual_increment} and
\eqref{eq:app:loo_residual_variance} in
Lemma~\ref{thm:(Matrix-Bernstein-inequality)} with tail parameter
$(a+2)\log n/\log(n+r)$. This parameter is at least two, since
$r\le n$ and $a>2$. The resulting threshold is
$C(\eta+\eta^2)\|A\|^2\|F\|_{2,\infty}$, with failure probability
$n^{-a-2}$ for each $s$. A union bound proves \eqref{eq:E550}.
For \eqref{eq:E55}, condition on the indicators whose two row indices
avoid $s$. The matrix $\hat U^{(s)}$ is fixed and the indicators in
$Z-Z^{(s)}$ remain independent. Equations~\eqref{eq:app:loo_residual_increment} and
\eqref{eq:app:loo_residual_variance} hold with $F=\hat U^{(s)}$ for
every conditioning value, so \eqref{eq:E55} fails with probability at most
$n^{-a-2}$ for each $s$, and a union bound over $s$ gives at most
$n^{-a-1}$.
\endgroup
\end{proof}

\begin{lemma}\label{lemma:E6}
For a sufficiently large universal $C_0$, suppose
\[
 p\ge\frac{C_0\kappa^4a\mu_1\mu_2r^2\log n}{mn}.
\]
Then, with probability $1-O(n^{1-a})$,
\[
 \|\Delta(\hat UH-U)\|_{2,\infty}
 \le C\|A\|^2\eta
       \bigl(\|\hat UH-U\|_{2,\infty}
                         +\|U\|_{2,\infty}\bigr).
\]
\end{lemma}
\begin{proof}
\begingroup
The sampling condition gives $\eta\le c\kappa^{-2}$, so
$\eta+\eta^2\le2\eta$. Work on the intersection of
Theorem~\ref{thm:wedge_concentration},
\eqref{eq:415claim2}--\eqref{eq:415claim4}, \eqref{eq:step2}, and
\eqref{eq:E55}, which has probability $1-O(n^{1-a})$.
Split each row using its auxiliary estimate:
\begin{equation}\label{eq:step1}
 \|\Delta(\hat UH-U)\|_{2,\infty}
 \le\max_s\|\Delta_{s,\cdot}(\hat UH-\hat U^{(s)}H^{(s)})\|
   +\max_s\|\Delta_{s,\cdot}(\hat U^{(s)}H^{(s)}-U)\|.
\end{equation}
For the first term,
\begin{equation}\label{eq:term1_LOO}
 \max_s\|\Delta_{s,\cdot}(\hat UH-\hat U^{(s)}H^{(s)})\|
 \le\|\Delta\|\max_s\|\hat UH-\hat U^{(s)}H^{(s)}\|.
\end{equation}
The identities $H=\hat U^\top U$ and
$H^{(s)}=(\hat U^{(s)})^\top U$, followed by
\eqref{eq:415claim4} and \eqref{eq:E55}, give
\begin{equation}\label{eq:app:loo_aligned_comparison}
\begin{aligned}
 \|\hat UH-\hat U^{(s)}H^{(s)}\|
 &\le\|\hat U\hat U^\top-\hat U^{(s)}(\hat U^{(s)})^\top\|\\
 &\le\frac{C\|(Z-Z^{(s)})\hat U^{(s)}\|}{\lambda_r}
 \le C\kappa^2\eta\|\hat U^{(s)}\|_{2,\infty}.
\end{aligned}
\end{equation}
By \eqref{eq:415claim3},
\begin{equation}\label{eq:app:loo_basis_row_bound}
 \|\hat U^{(s)}\|_{2,\infty}
 \le2\|\hat U^{(s)}H^{(s)}\|_{2,\infty}
 \le2\|\hat UH\|_{2,\infty}
       +2\|\hat UH-\hat U^{(s)}H^{(s)}\|.
\end{equation}
Combining \eqref{eq:app:loo_aligned_comparison} and
\eqref{eq:app:loo_basis_row_bound} yields
\begin{equation}\label{eq:self_bounding}
 \|\hat UH-\hat U^{(s)}H^{(s)}\|
 \le C\kappa^2\eta\bigl(\|\hat UH\|_{2,\infty}
                    +\|\hat UH-\hat U^{(s)}H^{(s)}\|\bigr).
\end{equation}
Choose the sampling constant so that
\begin{equation}\label{eq:gap}
 C\kappa^2\eta\le\frac14.
\end{equation}
Absorbing the last term of \eqref{eq:self_bounding}, then using
$\|\hat UH\|_{2,\infty}\le\|\hat UH-U\|_{2,\infty}+\|U\|_{2,\infty}$,
gives
\begin{equation}\label{eq:gap2}
 \max_s\|\hat UH-\hat U^{(s)}H^{(s)}\|
 \le C\kappa^2\eta
       \bigl(\|\hat UH-U\|_{2,\infty}+\|U\|_{2,\infty}\bigr).
\end{equation}
Insert \eqref{eq:gap2} and the operator bound from
Theorem~\ref{thm:wedge_concentration} into \eqref{eq:term1_LOO}:
\begin{equation}\label{eq:hatU}
 \max_s\|\Delta_{s,\cdot}(\hat UH-\hat U^{(s)}H^{(s)})\|
 \le C\|A\|^2\kappa^2\eta^2
       \bigl(\|\hat UH-U\|_{2,\infty}+\|U\|_{2,\infty}\bigr).
\end{equation}
For the second term of \eqref{eq:step1}, use the conditional estimate
\eqref{eq:step2}, the triangle inequality, and \eqref{eq:gap2}:
\begin{equation}\label{eq:hatU2}
\begin{aligned}
 &\max_s\|\Delta_{s,\cdot}(\hat U^{(s)}H^{(s)}-U)\|
 \le C\|A\|^2\eta\bigl(\|\hat UH-U\|_{2,\infty}
                  +\max_s\|\hat UH-\hat U^{(s)}H^{(s)}\|\bigr)\\
 &\qquad\le C\|A\|^2\eta\|\hat UH-U\|_{2,\infty}
       +C\|A\|^2\kappa^2\eta^2
         \bigl(\|\hat UH-U\|_{2,\infty}+\|U\|_{2,\infty}\bigr).
\end{aligned}
\end{equation}
Equations~\eqref{eq:hatU}--\eqref{eq:hatU2} and \eqref{eq:gap}
therefore prove
\begin{equation}\label{eq:step1new}
 \|\Delta(\hat UH-U)\|_{2,\infty}
 \le C\|A\|^2\eta
        \bigl(\|\hat UH-U\|_{2,\infty}+\|U\|_{2,\infty}\bigr).
\end{equation}

\endgroup
\end{proof}

\begin{proof}[Completion of the proof of Theorem~\ref{thm:wedge_spectral_guarantees}]
\begingroup
Take $b=a+1$ and $\eta=\sqrt{b\mu_1\mu_2r^2\log(n)/(pmn)}$.
Since $b\le3a/2$, a sufficiently large universal sampling constant
ensures $\kappa^2\eta\le c$. Theorem~\ref{thm:wedge_concentration},
\eqref{eq:app:loo_small_perturbation}, \eqref{eq:415claim2},
\eqref{eq:E44}, and Lemma~\ref{lemma:E6} with parameter $b$
hold together with probability $1-O(n^{1-b})=1-O(n^{-a})$.
On this event $\|\Delta\|\le C\eta\|A\|^2$, $\|H^{-1}\|\le2$,
and $\|H\|\le1$ since $H=\hat U^\top U$. The identity
$\Delta\hat U=[\Delta(\hat UH-U)+\Delta U]H^{-1}$ and
\eqref{eq:E44}, \eqref{eq:step1new} give
\begin{equation}\label{eq:app:gd-wedge-noise-times-basis}
 \|\Delta\hat U\|_{2,\infty}
 \le C\eta\|A\|^2
       (\|\hat U\|_{2,\infty}+\|U\|_{2,\infty}).
\end{equation}
Let $\hat\Lambda$ contain the leading $r$ eigenvalues of $Z$.
By Weyl's inequality and \eqref{eq:app:loo_small_perturbation},
$\lambda_r(Z)\ge\lambda_r-\|\Delta\|\ge\lambda_r/2$, so
$\|\hat\Lambda^{-1}\|\le2/\lambda_r$. Projecting the eigenvector
equation $Z\hat U=\hat U\hat\Lambda$ perpendicular to the column space of $U$ gives
\[
 (I-UU^\top)\hat U=(I-UU^\top)\Delta\hat U\hat\Lambda^{-1}.
\]
Since $\|UU^\top\Delta\hat U\|_{2,\infty}\le\|U\|_{2,\infty}\|\Delta\|$,
\eqref{eq:app:gd-wedge-noise-times-basis} gives
\[
 \|(I-UU^\top)\hat U\|_{2,\infty}
 \le C\kappa^2\eta(\|\hat U\|_{2,\infty}+\|U\|_{2,\infty}).
\]
For the component in the column space, $R=\sgn(H)$ implies
$U^\top\hat U R=(H^\top H)^{1/2}$. Its eigenvalues are the cosines
of the principal angles. By Lemmas~\ref{lem:DK} and~\ref{lem:sintheta_bounds},
\[
 \|U^\top\hat U R-I\|\le\|\sin\Theta\|^2
 \le 4\|\Delta\|^2/\lambda_r^2\le C(\kappa^2\eta)^2 .
\]
Decompose $\hat U R-U=(I-UU^\top)\hat U R+U(U^\top\hat U R-I)$.
Right multiplication by $R$ preserves each Euclidean row norm.
Using $\|\hat U\|_{2,\infty}\le\|U\|_{2,\infty}+\|\hat U R-U\|_{2,\infty}$
in the perpendicular component, and choosing $C\kappa^2\eta\le1/2$,
we can move its contribution involving $\|\hat U R-U\|_{2,\infty}$ to the left.
Together, the two estimates give
\begin{equation}\label{eq:app:gd-wedge-aligned-rows}
 \|\hat U R-U\|_{2,\infty}\le C\kappa^2\eta\|U\|_{2,\infty}.
\end{equation}
Since $b\le3a/2$, this proves \eqref{eq:wedge_spectral_loo}.
\endgroup
\end{proof}
}

\section[Proof of Theorem~\ref{thm:final_bound_montanari}]{Proof of {Theorem~\ref{thm:final_bound_montanari}}}\label{sec:app_final_montanri}
\begingroup
\setlength{\emergencystretch}{2em}
We apply Theorem~\ref{thm:wedge_spectral_guarantees} to
$A=\unfold_1(T)\in\mathbb R^{n\times n^{k-1}}$.

\subsection{Concentration of \texorpdfstring{$Y$}{Y} after a delocalized projection}

\begin{lemma}\label{lem:delocalization_Q}
Assume that
\[
 p\ge\frac{C_0a\kappa^4\mu_1\mu_2r^2\log n}{n^k}
\]
for a sufficiently large universal constant $C_0$.
Then, with probability at least $1-O(n^{-a-1})$,
\[
 \|Q\|_{2,\infty}\le2\sqrt{\frac{\mu_1r}{n}}.
\]
\end{lemma}
\begin{proof}
\begingroup
Apply Theorem~\ref{thm:wedge_spectral_guarantees} with parameter
$a+1>2$. Since $a+1\le3a/2$, increasing $C_0$ ensures its sampling
conditions. The orthogonal-projector identity gives
$\|Q\|_{2,\infty}=\|\hat UR\|_{2,\infty}$, so on the event
giving \eqref{eq:wedge_spectral_loo},
\[
\begin{aligned}
 \|Q\|_{2,\infty}
 &\le\|U\|_{2,\infty}+\|\hat UR-U\|_{2,\infty}\\
 &\le\left(1+C\kappa^2
   \sqrt{\frac{a\mu_1\mu_2r^2\log n}{pn^k}}\right)
   \|U\|_{2,\infty}
 \le(1+C/\sqrt{C_0})\sqrt{\frac{\mu_1r}{n}}.
\end{aligned}
\]
Choose $C_0$ so that $C/\sqrt{C_0}\le1$.
The failure probability of \eqref{eq:wedge_spectral_loo} with
parameter $a+1$ is $O(n^{-a-1})$.
\endgroup
\end{proof}

Recall that
\[
 \mathcal Q(Y)=QY(Q\otimes\cdots\otimes Q),
 \qquad
 \mathcal Q(A)=QA(Q\otimes\cdots\otimes Q).
\]
{Lemma~\ref{lemma:Projection_Y} bounds $\|\mathcal Q(Y-A)\|_F$
conditionally on the wedge sample, with probability taken over the
independent uniform refinement sample.}

\begin{lemma}\label{lemma:Projection_Y}
Let $A=\unfold_1(T)$, and condition on the wedge sample, so that $Q$ is fixed.
For every $a\ge2$, with conditional probability at least $1-n^{-a}$ over the
independent uniform sample,
\begin{align*}
 \|\mathcal Q(Y-A)\|_F
 &\le C\|Q\|_{2,\infty}^k\|T\|_F\sqrt{\frac{1-q}{q}}
 + C\|T\|_{\max}\sqrt{\frac{(1-q)a\log n}{q}}\\
 &\quad+C\|Q\|_{2,\infty}^k\|T\|_{\max}\frac{a\log n}{q}.
\end{align*}
\end{lemma}
\begin{proof}
\begingroup
If $q=1$, the projected sampling error is zero. For $0<q<1$, conditional
on the wedge sample, the vectors
\begin{equation}\label{eq:def_M}
 M_{ij}=(Y_{ij}-A_{ij})\operatorname{vec}\!\left(
 (Qe_i)\big((Q^{\otimes(k-1)})e_j\big)^\top\right),\qquad
 \operatorname{vec}(\mathcal Q(Y-A))=\sum_{i,j}M_{ij}
\end{equation}
are independent and centered. Here $i\in[n]$ and $j\in[n^{k-1}]$.
The projected coordinate vector in \eqref{eq:def_M} is a tensor product of
$k$ vectors $Qe_\ell$. Its Euclidean norm is at most $\|Q\|_{2,\infty}^k$, so
\begin{equation}\label{eq:app:projected_vector_increment}
 \|M_{ij}\|_2\le q^{-1}|A_{ij}|\,\|Q\|_{2,\infty}^k.
\end{equation}
Since $\mathbb E[(Y_{ij}-A_{ij})^2\mid Q]=(1-q)A_{ij}^2/q$,
independence gives
\begin{equation}\label{eq:app:projected_vector_variance}
\begin{aligned}
 \mathbb E\big[\|\mathcal Q(Y-A)\|_F^2\mid Q\big]
 &=\frac{1-q}{q}\sum_{i,j} A_{ij}^2\|Qe_i\|_2^2\|(Q^{\otimes(k-1)})e_j\|_2^2\\
 &\le\frac{1-q}{q}\,\|Q\|_{2,\infty}^{2k}\|A\|_F^2.
\end{aligned}
\end{equation}
The square root of this bound controls the conditional mean. To bound
fluctuations about the mean, fix a matrix $D$ with $\|D\|_F=1$.
The $(i,j)$ entry of $\mathcal Q(D)$ is the inner product of $D$ with
the projected coordinate matrix in \eqref{eq:def_M}. Thus
\[
 \sum_{i,j}\mathbb E\!\left[\langle\operatorname{vec}(D),M_{ij}\rangle^2\mid Q\right]
 =\frac{1-q}{q}\sum_{i,j}A_{ij}^2\mathcal Q(D)_{ij}^2
 \le\frac{1-q}{q}\,\|A\|_{\max}^2,
\]
because $\|\mathcal Q(D)\|_F\le\|D\|_F=1$.

The Frobenius norm in \eqref{eq:def_M} is the supremum of these centered
scalar sums over a countable dense subset of the Frobenius unit sphere.
Theorem~1.1(c) of \citet{klein2005concentration}, applied conditionally on
$Q$ with tail probability $e^{-a\log n}=n^{-a}$, yields
\[
\begin{aligned}
 \|\mathcal Q(Y-A)\|_F
 &\le C\,\mathbb E[\|\mathcal Q(Y-A)\|_F\mid Q]
 +C\sqrt{a\log n\sup_{\|D\|_F=1}
   \sum_{i,j}\mathbb E[\langle\operatorname{vec}(D),M_{ij}\rangle^2\mid Q]}\\
 &\quad+Cq^{-1}\|A\|_{\max}\|Q\|_{2,\infty}^k a\log n.
\end{aligned}
\]
Indeed, \eqref{eq:app:projected_vector_increment} supplies the bound on
each scalar summand. The mixed term in that theorem is bounded by the sum
of the first and last terms above using $2\sqrt{xy}\le x+y$.
Inserting the directional variance estimate and the mean bound from
\eqref{eq:app:projected_vector_variance} gives
\begin{equation}\label{eq:app:projected_vector_bernstein}
\begin{aligned}
 \|\mathcal Q(Y-A)\|_F
 &\le C\|Q\|_{2,\infty}^k\|A\|_F\sqrt{\frac{1-q}{q}}
 +C\|A\|_{\max}\sqrt{\frac{(1-q)a\log n}{q}}\\
 &\quad+C\|Q\|_{2,\infty}^k\|A\|_{\max}\frac{a\log n}{q}.
\end{aligned}
\end{equation}
Finally, $\|A\|_F=\|T\|_F$ and $\|A\|_{\max}=\|T\|_{\max}$.
\endgroup
\end{proof}

\subsection[Proof of Theorem~\ref{thm:final_bound_montanari}]{Proof of {Theorem~\ref{thm:final_bound_montanari}}}
\begin{proof}
\begingroup
Let $T^*=\mathcal Q(A)$ and $\hat T=\mathcal Q(Y)$, identifying
unfoldings with tensors when taking Frobenius norms.
Work first on the wedge event giving
\eqref{eq:wedge_spectral_op} and \eqref{eq:wedge_spectral_loo}
with parameter $a+1$. Its failure probability is $O(n^{-a-1})$,
and Lemma~\ref{lem:delocalization_Q} gives
$\|Q\|_{2,\infty}\le2\sqrt{\mu_1r/n}$ whenever
\eqref{eq:wedge_spectral_loo} holds with this parameter.

\paragraph{Projection error.}
The identity
\begin{equation}\label{eq:app:symmetric_projection_identity}
 I^{\otimes k}-Q^{\otimes k}
 =\sum_{\ell=1}^k
 Q^{\otimes(\ell-1)}\otimes(I-Q)\otimes I^{\otimes(k-\ell)}
\end{equation}
is a sum of mutually orthogonal projections: the product of two
distinct summands has a factor $(I-Q)Q=0$.
Dropping the other projections bounds the $\ell$th summand applied
to $T$ by $\|(I-Q)\unfold_\ell(T)\|_F$.
Symmetry makes each of these norms equal to $\|(I-Q)A\|_F$, so
\begin{equation}\label{eq:TTstar}
 \|T-T^*\|_F^2
 =\sum_{\ell=1}^k
 \left\|\big(Q^{\otimes(\ell-1)}\otimes(I-Q)
       \otimes I^{\otimes(k-\ell)}\big)\operatorname{vec}(T)\right\|_2^2
 \le k\|(I-Q)A\|_F^2.
\end{equation}
Since $(I-Q)\hat UR=0$, the SVD $A=U\Sigma V^\top$ and
\eqref{eq:wedge_spectral_op} give
\begin{equation}\label{eq:app:symmetric_projection_residual}
 \|(I-Q)A\|_F
 \le\|(I-Q)U\|\,\|\Sigma\|_F
 \le\|U-\hat UR\|\,\|T\|_F
 \le C\kappa^2\sqrt{\frac{a\mu_1\mu_2r^2\log n}{pn^k}}\,\|T\|_F.
\end{equation}
Combining \eqref{eq:TTstar} and
\eqref{eq:app:symmetric_projection_residual} yields
\begin{equation}\label{eq:app:symmetric_projection_error}
 \|T-T^*\|_F
 \le C\sqrt{k}\,\kappa^2
 \sqrt{\frac{a\mu_1\mu_2r^2\log n}{pn^k}}\,\|T\|_F.
\end{equation}

\paragraph{Refinement error.}
Apply Lemma~\ref{lemma:Projection_Y} conditionally on a wedge
sample satisfying \eqref{eq:wedge_spectral_op} and
\eqref{eq:wedge_spectral_loo}. Substitute the row bound from
Lemma~\ref{lem:delocalization_Q} and the entry bound
$\|T\|_{\max}\le r\sqrt{\mu_1\mu_2/n^k}\,\|A\|$ from
\eqref{eq:infty_norm} to obtain
\begin{equation}\label{eq:app:symmetric_refinement_error}
\begin{aligned}
 \|T^*-\hat T\|_F
 &\le C2^k(\mu_1r)^{k/2}\sqrt{\frac{1-q}{qn^k}}\,\|T\|_F
 +Cr\sqrt{\mu_1\mu_2}\sqrt{\frac{(1-q)a\log n}{qn^k}}\,\|A\|\\
 &\quad+C2^k(\mu_1r)^{k/2}r\sqrt{\mu_1\mu_2}\,
          \frac{a\log n}{qn^k}\,\|A\|.
\end{aligned}
\end{equation}
The conditional failure is at most $n^{-a}$, so the wedge and
refinement estimates hold together with failure
$n^{-a}+O(n^{-a-1})=O(n^{-a})$.
The triangle inequality combines
\eqref{eq:app:symmetric_projection_error} and
\eqref{eq:app:symmetric_refinement_error} to prove the theorem.
For its alternative bound, the singular values give
\begin{equation}\label{eq:app:operator_frobenius_comparison}
 \|A\|_F^2=\sum_{i=1}^r\sigma_i(A)^2
 \ge r\sigma_r(A)^2=\frac r{\kappa^2}\|A\|^2,
 \qquad \|A\|\le\|A\|_F=\|T\|_F.
\end{equation}
Consequently,
\[
 \|A\|\le\min\{1,\kappa/\sqrt r\}\|T\|_F,
\]
which bounds both terms involving $\|A\|$ in
\eqref{eq:app:symmetric_refinement_error}.
\endgroup
\end{proof}
\endgroup

\begingroup
\subsection{Rank dependence of the two sampling stages}\label{sec:app:spectral_rank_example}
Suppose $r$ divides $n$ and partition $[n]$ into $r$ groups of size $n/r$,
labeled by $0,\ldots,r-1$. Set $T_{i_1\ldots i_k}=1$ if the sum of
the $k$ group labels is zero modulo $r$, and set it to zero otherwise.
For each row group, the unfolding $A$ has ones on exactly one column group
of size $n^{k-1}/r$. These column groups are disjoint. The normalized
indicators of the row and column groups thus give singular vectors, with
all $r$ singular values equal to $n^{k/2}/r$. In particular,
$\mu_1=\mu_2=\kappa=1$ and $\|T\|_F^2=n^k/r$.

First take $p=1$. Then $Z=AA^\top$ and $Q$ averages within each row
group. Applying $Q$ in all $k$ modes averages the rescaled observations
on each product of $k$ groups. Exactly $r^{k-1}$ of these products have
value one. Each contains $(n/r)^k$ entries. If $qn^k\le r^k/8$, the
expected number of such products receiving any observations is at most
the expected number of sampled ones, $qn^k/r\le r^{k-1}/8$.
By Markov's inequality, with probability at least $3/4$, at least half
the products on which $T$ equals one receive no observations. The
estimate equals zero on those products, so
$\|\hat T-T\|_F^2\ge \|T\|_F^2/2$.

Now take $q=1$ and assume
\[
 \frac{2r^3\log(8r)}{n^{k+1}}\le p\le\frac{r^2}{32n^k}.
\]
This interval is nonempty when $n/r\ge64\log(8r)$.
A wedge contributes to $Z$ only when its two rows belong to the same group
and its column belongs to the corresponding column group. Thus $Z$ is
block diagonal. For each block, the number of possible contributing wedges is
\[
 \frac{n^{k-1}}{r}\,\frac{(n/r)(n/r+1)}2
 \ge \frac{n^{k+1}}{2r^3}.
\]
The probability that some block receives none is at most
$r\exp[-pn^{k+1}/(2r^3)]\le1/8$.
Every nonzero block has a positive eigenvalue, since it is symmetric with
nonnegative entries and its quadratic form at the all-ones vector is positive.
Moreover, each contributing wedge makes at most two rows nonzero.
The expected number of nonzero rows is therefore at most
$p n^{k-1}(n/r)(n/r+1)\le 2pn^{k+1}/r^2\le n/16$.
By Markov's inequality and the preceding block estimate, with probability
at least $3/4$, $Z$ has at least $r$ positive eigenvalues and at least
$n/2$ zero rows. Every eigenvector with a positive eigenvalue vanishes on
those rows. So does the first-mode unfolding of $\hat T$, since
Algorithm~\ref{alg:montanari} multiplies by $Q$ on the left.
Every row of $A$ has squared norm $n^{k-1}/r$, giving
$\|\hat T-T\|_F^2\ge \|T\|_F^2/2$ in this case as well.
\endgroup

\subsection{Generalization to asymmetric tensor completion}
\label{sec:asym_thm_and_proof}

In this section, we generalize Algorithm~2 to \edit{asymmetric} low-rank tensors. 

\begin{algorithm}[ht!]
\caption{Spectral Tensor Completion with Wedge Sampling (Asymmetric; reduces to Algorithm~2 when $n_1=\cdots=n_k$ and $Q^{(1)}=\cdots=Q^{(k)}$)\label{alg:asym_tensor}}
\begin{algorithmic}[1]
\State \textbf{Input:} Order-$k$ tensor $T\in\mathbb{R}^{n_1\times\cdots\times n_k}$, wedge sampling rate $p$, uniform sampling rate $q$, rank $r$.
\State \textbf{Debiased uniform subsample:} Let $\Omega\subseteq [n_1]\times\cdots\times[n_k]$ include each entry independently with probability $q$ (independently of all wedge samples), and set
$\widetilde T := \mathcal P_\Omega(T)$ and $Y := \frac{1}{q}\,\widetilde T$. Then $\mathbb{E}[Y]=T$ and $Y^{(j)}:=\unfold_j(Y)$.
\State \textbf{Mode-wise wedge spectral step:} For each $j=1,\dots,k$:
\begin{enumerate}
\item Let $A^{(j)}:=\unfold_j(T)\in\mathbb{R}^{n_j\times N_j}$ with $N_j:=\prod_{t\neq j}n_t$.
Construct the (debiased) wedge matrix $Z^{(j)}\in\mathbb{R}^{n_j\times n_j}$ from wedge samples on $A^{(j)}$ at rate $p$ so that $\mathbb{E}[Z^{(j)}]=A^{(j)}(A^{(j)})^\top$.
\item Compute the top-$r$ eigenvectors $\widehat U^{(j)}\in\mathbb{R}^{n_j\times r}$ of $Z^{(j)}$ and set $Q^{(j)}:=\widehat U^{(j)}(\widehat U^{(j)})^\top$.
\end{enumerate}
\State \textbf{Denoising (multilinear projection):} Define
\[
\widehat T \ :=\ Y \times_1 Q^{(1)} \times_2 Q^{(2)} \cdots \times_k Q^{(k)} .
\]
Equivalently, for any fixed mode $j$,
\[
\unfold_j(\widehat T) \;=\; Q^{(j)}\,Y^{(j)} \Big(\bigotimes_{t\neq j} Q^{(t)}\Big).
\]
\State \textbf{Output:} $\widehat T$.
\end{algorithmic}
\end{algorithm}

\paragraph{Asymmetric model and notation.}
Let $T\in\mathbb{R}^{n_1\times\cdots\times n_k}$ be an order-$k$ tensor and define, for each mode $j\in[k]$,
\[
A^{(j)}:=\unfold_j(T)\in\mathbb{R}^{n_j\times N_j},\qquad N_j:=\prod_{t\neq j}n_t,
\]
and assume $\mathrm{rank}(A^{(j)})=r$ for all $j$.
Denote the SVD by
\begin{align}\label{eq:svd_j}
  A^{(j)} \;=\; U^{(j)} \Sigma^{(j)} (V^{(j)})^\top,
  \qquad U^{(j)}\in\mathbb{R}^{n_j\times r},\quad V^{(j)}\in\mathbb{R}^{N_j\times r}.
\end{align}
Denote $\lambda^{(j)}_i:=\sigma_i^2(A^{(j)})$, $\lambda^{(j)}_r>0$, and the mode-wise condition number
$\kappa_j:=\sigma_1(A^{(j)})/\sigma_r(A^{(j)})$; set $\kappa:=\max_j \kappa_j$.
Let $Z^{(j)}:=A^{(j)}(A^{(j)})^\top+\Delta^{(j)}$ be the (debiased) wedge matrix and
$\widehat U^{(j)}$ its top-$r$ eigenspace with projector $Q^{(j)}:=\widehat U^{(j)}(\widehat U^{(j)})^\top$.
For each mode, define the cross-Gram matrix and its Procrustes factor
\[
H_U^{(j)}:=(\widehat U^{(j)})^\top U^{(j)},
\qquad R^{(j)}:=\operatorname{sgn}(H_U^{(j)}).
\]
Let $Y^{(j)}:=\unfold_j(Y)$ where $Y=\frac{1}{q}\widetilde T$ is the debiased \emph{uniform} subsample, so that $\mathbb{E}[Y^{(j)}]=A^{(j)}$
(where this expectation is over the uniform subsampling, conditionally on the wedge samples used to form $Z^{(j)}$).

We use the same incoherence parameters $(\mu_1,\mu_2)$ as in the symmetric case, but assume they hold \emph{uniformly} for all $j$:
\begin{align}\label{eq:asym_incoh}
\|U^{(j)}\|_{2,\infty}^2 \ \le\ \mu_1\,\frac{r}{n_j},
\qquad
\|V^{(j)}\|_{2,\infty}^2 \ \le\ \mu_2\,\frac{r}{N_j},
\qquad j=1,\dots,k.
\end{align}
Note that $n_jN_j=\prod_{t=1}^kn_t=:n_{\mathrm{tot}}$ for every $j$.
\edit{Fix \(a\ge2\), and let \(C_0>0\) be a sufficiently large universal constant.}
Assume wedge sampling at rate $p$ and debiased uniform sampling at rate $q$ satisfy
\begin{equation}\label{eq:asym_sampling_levels}
p\geq \frac{C_0 a\kappa^8 \mu_1^2\mu_2 r^3\log n_{\max}}{n_{\mathrm{tot}}},
\qquad
q \ \ge\ \frac{ak\log n_{\max}}{\,n_{\mathrm{tot}}\,},
\end{equation}
where \(n_{\max}:=\max_{1\le j\le k} n_j\).
\begingroup
If every dimension is one, the probability bound in
Theorem~\ref{thm:asym_main} is vacuous. Otherwise, for each mode
with $n_j>1$, use the probability parameter
\begin{equation}\label{eq:app:asym_probability_parameter}
 b=\frac{(a+2)\log n_{\max}}{\log n_j}.
\end{equation}
The Bernstein estimates proving Theorems~\ref{thm:wedge_concentration}
and~\ref{thm:wedge_spectral_guarantees} have universal constants
for $b>2$. Since $b\log n_j=(a+2)\log n_{\max}$ and $a+2\le2a$,
their sampling conditions follow from \eqref{eq:asym_sampling_levels}
after increasing $C_0$, with failure $O(n_{\max}^{-a-2})$ per mode.
In particular,
\begin{equation}\label{eq:app:asym_wedge_concentration}
 \|\Delta^{(j)}\|
 \le C\left(
 \sqrt{\frac{a\mu_1\mu_2r^2\log n_{\max}}{p n_{\mathrm{tot}}}}
 +\frac{a\mu_1\mu_2r^2\log n_{\max}}{p n_{\mathrm{tot}}}
 \right)\|A^{(j)}\|^2.
\end{equation}
When $n_j=1$, we have $r=1$, $Q^{(j)}=I$, and zero subspace error.
For the wedge fluctuation in this scalar mode, the variance and
maximal summand are bounded respectively by
\[
 p^{-1}\|A^{(j)}\|_{\max}^2\|A^{(j)}\|^2,
 \qquad p^{-1}\|A^{(j)}\|_{\max}^2.
\]
Scalar Bernstein with logarithmic parameter $(a+2)\log n_{\max}$
and \eqref{eq:asym_incoh} gives
\eqref{eq:app:asym_wedge_concentration} with the same failure
probability. For every mode, since
$\lambda_r^{(j)}=\|A^{(j)}\|^2/\kappa_j^2$, the sampling bound gives
\[
 \frac{\|\Delta^{(j)}\|}{\lambda_r^{(j)}}
 \le C\kappa^2\left(
 \frac{1}{\sqrt{C_0}\kappa^4\sqrt{\mu_1r}}
 +\frac{1}{C_0\kappa^8\mu_1r}\right)\le\frac14,
\]
using $\mu_1r\ge1$ and sufficiently large $C_0$.
A union bound therefore gives, with probability
$1-O(k n_{\max}^{-a})$, simultaneously for all $j$,
\begin{equation}\label{eq:asym_gap}
\lambda^{(j)}_r \ \ge\ 4\,\|Z^{(j)}-A^{(j)}(A^{(j)})^\top\|.
\end{equation}
\endgroup

\begin{theorem}[Asymmetric main error bound]\label{thm:asym_main}
\edit{{Under the rank assumption $\mathrm{rank}(A^{(j)})=r$ for every mode, the incoherence bounds \eqref{eq:asym_incoh}, and the sampling conditions \eqref{eq:asym_sampling_levels},}} let $\widehat T$ be the output of Algorithm~\ref{alg:asym_tensor}. Then with probability at least $1-O(k n_{\max}^{-a})$,
\begin{align}
&\|T-\widehat T\|_{\mathrm{F}}\lesssim \Bigg(
 k\,\sqrt r\,\kappa^2\,
\sqrt{\frac{a\,\mu_1\mu_2\,r^2\log n_{\max}}{\,p\,n_{\mathrm{tot}}\,}}
\ +\ 
2^{k}\,r^{1.5+k/2}
\sqrt{\frac{ak\,\mu_1^{k+1}\mu_2\,\log n_{\max}}{\,q\,n_{\mathrm{tot}}\,}}
\Bigg)\,\|T\|_{\mathrm{F}} .
\label{eq:asym_final_bound}
\end{align}
\end{theorem}
{When $n_1=\cdots=n_k=n$ and $Q^{(1)}=\cdots=Q^{(k)}$, the estimator is the one in {Theorem~\ref{thm:final_bound_montanari}}. {The rank dependence in \eqref{eq:asym_final_bound} comes from the operator-norm estimate \eqref{eq:QYA_asym_bound} and the Frobenius conversion \eqref{eq:app:asym_refinement_frobenius}.}}
\begin{remark}
To reuse each symmetric lemma for mode $j$, substitute
\[
  A\mapsto A^{(j)},\quad
  Z\mapsto Z^{(j)},\quad
  \Delta\mapsto \Delta^{(j)},\quad
  \lambda_r\mapsto \lambda^{(j)}_r,\quad
  Q\mapsto Q^{(j)},
\]
and use the uniform incoherence parameters \eqref{eq:asym_incoh}.
The only structural change is that the right projection $Q^{\otimes (k-1)}$ becomes $\bigotimes_{t\neq j} Q^{(t)}$; we handle this via the multilinear projection bound in Lemma~\ref{lem:QYA_asym}.
\end{remark}
Before the formal proof of Theorem \ref{thm:asym_main}, we need a series of auxiliary lemmas.

\begin{theorem}[Asymmetric $\ell_{2,\infty}$ subspace recovery]
\label{thm:asym_l2inf}
Assume \eqref{eq:asym_sampling_levels} holds. \edit{With probability at least $1-O(k n_{\max}^{-a})$, simultaneously for all $j\in[k]$,}
\[
\|\widehat U^{(j)} R^{(j)}-U^{(j)}\|_{2,\infty}
\;\le\;
C\,\kappa^{4}\,\|U^{(j)}\|_{2,\infty}\,
\sqrt{\frac{a\,\mu_{1}^2\mu_{2}\,r^{3}\log n_{\max}}{p\,n_{\mathrm{tot}}}}.
\]
\end{theorem}
\begin{proof}
\begingroup
For $r<n_j$, apply Theorem~\ref{thm:wedge_spectral_guarantees}
with $n=n_j$, $m=N_j$, the incoherence bounds
\eqref{eq:asym_incoh}, and the parameter
\eqref{eq:app:asym_probability_parameter}. Its row estimate gives
\[
\begin{aligned}
 \|\widehat U^{(j)}R^{(j)}-U^{(j)}\|_{2,\infty}
 &\le C\kappa_j^4\|U^{(j)}\|_{2,\infty}
 \sqrt{\frac{b\mu_1\mu_2r^2\log n_j}{p n_jN_j}}\\
 &\le C\kappa^4\|U^{(j)}\|_{2,\infty}
 \sqrt{\frac{a\mu_1^2\mu_2r^3\log n_{\max}}{p n_{\mathrm{tot}}}}.
\end{aligned}
\]
The second line uses $b\log n_j=(a+2)\log n_{\max}$,
$a+2\le2a$, and $\mu_1r\ge1$.
The sampling condition \eqref{eq:asym_sampling_levels} is stronger
than required because $\kappa\ge1$ and $\mu_1r\ge1$.
If $r=n_j$, both bases span $\mathbb R^{n_j}$ and their Procrustes
error is zero. The nontrivial modes each fail with probability
$O(n_{\max}^{-a-2})$; their union has probability
$O(k n_{\max}^{-a})$, as asserted.
\endgroup
\end{proof}

\begin{lemma}[Debiased projection concentration, mode $j$]\label{lem:QYA_asym}
Fix $j\in[k]$ and let \[Y^{(j)}=\unfold_j(Y)\] with $Y=\tfrac{1}{q}\,\widetilde T$ as in Algorithm~\ref{alg:asym_tensor}. {Conditional on all wedge observations, with probability at least $1-O(n_{\max}^{-a})$ over the independent uniform sample,}
{
\begin{align}
&\big\|\,Q^{(j)}(Y^{(j)}-A^{(j)})\big(\textstyle\bigotimes_{t\neq j} Q^{(t)}\big)\,\big\|\\
&\le C\,r\,\|A^{(j)}\|\,
\Big(\,\prod_{t=1}^{k}\|Q^{(t)}\|_{2,\infty}\Big)\cdot
\sqrt{\mu_1\mu_2}\,
\left(\sqrt{\frac{ak\log n_{\max}}{q}}
\ +\ ak\frac{\log n_{\max}}{q}\cdot \frac{1}{\sqrt{n_{\mathrm{tot}}}}\right).
\label{eq:QYA_asym_bound}
\end{align}
}
\end{lemma}

\begin{proof}
\begingroup
Condition on all wedge observations, so that the projectors are
fixed and the uniform indicators remain independent. Set $d_1=n_j$
and $d_2=N_j$. The projected matrix is the sum
\[
 Q^{(j)}(Y^{(j)}-A^{(j)})\Big(\bigotimes_{t\ne j}Q^{(t)}\Big)
 =\sum_{i,\alpha}(Y^{(j)}_{i\alpha}-A^{(j)}_{i\alpha})
 (Q^{(j)}e_i)\Big(\big(\bigotimes_{t\ne j}Q^{(t)}\big)e_\alpha\Big)^\top.
\]
For the tuple $\alpha=(\alpha_t)_{t\ne j}$, the right-hand vector
is $\bigotimes_{t\ne j}(Q^{(t)}e_{\alpha_t})$.
Its norm is bounded by $\prod_{t\ne j}\|Q^{(t)}\|_{2,\infty}$.
Also, \eqref{eq:asym_incoh} and the SVD give
\begin{equation}\label{eq:app:asym_entry_bound}
 \|A^{(j)}\|_{\max}
 \le\max_{i,\alpha}\|e_i^\top U^{(j)}\|_2\|\Sigma^{(j)}\|
       \|e_\alpha^\top V^{(j)}\|_2
 \le r\|A^{(j)}\|\sqrt{\frac{\mu_1\mu_2}{n_{\mathrm{tot}}}}.
\end{equation}
Since $|Y^{(j)}_{i\alpha}-A^{(j)}_{i\alpha}|
\le q^{-1}|A^{(j)}_{i\alpha}|$, each summand has norm at most
\begin{equation}\label{eq:app:asym_projected_increment}
 R\lesssim q^{-1}r\|A^{(j)}\|
 \sqrt{\frac{\mu_1\mu_2}{n_{\mathrm{tot}}}}
 \prod_{t=1}^k\|Q^{(t)}\|_{2,\infty}.
\end{equation}
Both covariance matrices are positive semidefinite and have the
same trace, equal to the sum of the expected squared Frobenius
norms of the summands. Their operator norms are therefore bounded by
\[
\begin{aligned}
 \sigma^2
 &\le\frac{1-q}{q}\sum_{i,\alpha}(A^{(j)}_{i\alpha})^2
 \|Q^{(j)}e_i\|_2^2
 \left\|\left(\bigotimes_{t\ne j}Q^{(t)}\right)e_\alpha\right\|_2^2\\
 &\le\frac{n_{\mathrm{tot}}}{q}\|A^{(j)}\|_{\max}^2
 \left(\prod_{t=1}^k\|Q^{(t)}\|_{2,\infty}\right)^2.
\end{aligned}
\]
Inserting \eqref{eq:app:asym_entry_bound} gives
\begin{equation}\label{eq:app:asym_projected_variance}
 \sigma^2\le\frac{\mu_1\mu_2r^2}{q}\|A^{(j)}\|^2
 \left(\prod_{t=1}^k\|Q^{(t)}\|_{2,\infty}\right)^2.
\end{equation}
For $n_{\max}>1$,
\[
 \log(d_1+d_2)=\log(n_j+N_j)
 \le\log(2n_{\max}^{k-1})\le k\log n_{\max}.
\]
Lemma~\ref{thm:(Matrix-Bernstein-inequality)}, with
\eqref{eq:app:asym_projected_increment} and
\eqref{eq:app:asym_projected_variance}, now gives
\eqref{eq:QYA_asym_bound}. Its conditional failure is at most
$(d_1+d_2)^{-a}\le n_{\max}^{-a}$ for every wedge realization,
which also proves the same failure bound under the joint distribution.
\endgroup
\end{proof}
\begin{lemma} \label{lem:Q_2inf_asym}
Under \eqref{eq:asym_sampling_levels}, \edit{with probability $1-O(k n_{\max}^{-a})$, simultaneously for all $j\in[k]$,}
\[
\|Q^{(j)}\|_{2,\infty}\ \le\ 2\sqrt{\frac{\mu_1 r}{n_j}},\qquad
Q^{(j)}:=\widehat U^{(j)}(\widehat U^{(j)})^\top .
\]
\end{lemma}

\begin{proof}
\begingroup
Work where Theorem~\ref{thm:asym_l2inf} gives its row estimates
simultaneously for all modes. As in the proof of
Lemma~\ref{lem:delocalization_Q}, orthonormality gives
$\|Q^{(j)}\|_{2,\infty}=\|\widehat U^{(j)}\|_{2,\infty}$.
The triangle inequality yields
\begin{equation}\label{eq:app:asym_basis_row_norm}
\begin{aligned}
 \|\widehat U^{(j)}\|_{2,\infty}
 &\le\|U^{(j)}\|_{2,\infty}
       +\|\widehat U^{(j)}R^{(j)}-U^{(j)}\|_{2,\infty}\\
 &\le\left(1+C\kappa^4
 \sqrt{\frac{a\mu_1^2\mu_2r^3\log n_{\max}}{p n_{\mathrm{tot}}}}\right)
 \|U^{(j)}\|_{2,\infty}\le2\sqrt{\frac{\mu_1r}{n_j}}.
\end{aligned}
\end{equation}
The last inequality uses \eqref{eq:asym_sampling_levels}, which
bounds the additional term in parentheses by $C/\sqrt{C_0}\le1$.
No probability beyond the simultaneous event of
Theorem~\ref{thm:asym_l2inf} is lost.
\endgroup
\end{proof}
Using Theorem~\ref{thm:asym_l2inf} and Lemmas~\ref{lem:QYA_asym} and~\ref{lem:Q_2inf_asym}, we can show that the projection error is dominated by the mode that carries the largest mode-wise projection error.

\begin{proof}[Proof of Theorem~\ref{thm:asym_main}]
\begingroup
Define the projection of the true tensor by
\[
 T^*:=T\times_1Q^{(1)}\times_2Q^{(2)}\cdots\times_kQ^{(k)}.
\]
We bound the two terms in
\[
 \|T-\widehat T\|_F\le\|T-T^*\|_F+\|T^*-\widehat T\|_F.
\]

\paragraph{Projection error.}
Successive mode projections give
\begin{equation}\label{eq:app:asym_projection_identity}
 T-T^*=\sum_{j=1}^k
 T\times_1Q^{(1)}\cdots\times_{j-1}Q^{(j-1)}
       \times_j(I-Q^{(j)}).
\end{equation}
The other projections contract the Frobenius norm, and
$(I-Q^{(j)})A^{(j)}$ has rank at most $r$. Thus
\begin{equation}\label{eq:app:asym_projection_sum}
 \|T-T^*\|_F\le\sum_{j=1}^k\|(I-Q^{(j)})A^{(j)}\|_F
 \le k\sqrt r\max_j\|(I-Q^{(j)})A^{(j)}\|.
\end{equation}
The residual is zero if $r=n_j$.
Otherwise, the population Gram matrix has gap $\lambda_r^{(j)}$
above zero. On the event giving \eqref{eq:asym_gap},
Lemma~\ref{lem:DK} and \eqref{eq:app:asym_wedge_concentration} give
\begin{equation}\label{eq:app:asym_projection_residual}
 \|(I-Q^{(j)})A^{(j)}\|
 \le\frac{2\|\Delta^{(j)}\|}{\lambda_r^{(j)}}\|A^{(j)}\|
 \lesssim\kappa^2
 \sqrt{\frac{a\mu_1\mu_2r^2\log n_{\max}}{p n_{\mathrm{tot}}}}\,
 \|A^{(j)}\|.
\end{equation}
Here $\|A^{(j)}\|^2/\lambda_r^{(j)}=\kappa_j^2\le\kappa^2$,
and the linear concentration term is absorbed because
\[
 \frac{a\mu_1\mu_2r^2\log n_{\max}}{p n_{\mathrm{tot}}}
 \le\frac1{C_0\kappa^8\mu_1r}\le1.
\]
Substitute \eqref{eq:app:asym_projection_residual} into
\eqref{eq:app:asym_projection_sum} and use
$\|A^{(j)}\|\le\|T\|_F$ to obtain the first term of
\eqref{eq:asym_final_bound}.

\paragraph{Refinement error.}
Fix one mode $j$ independently of the observations. Unfolding
preserves the Frobenius norm, and the projected matrix has rank
at most $r$, so
\begin{equation}\label{eq:app:asym_refinement_frobenius}
\begin{aligned}
 \|T^*-\widehat T\|_F
 &=\left\|Q^{(j)}(A^{(j)}-Y^{(j)})
       \Big(\bigotimes_{t\ne j}Q^{(t)}\Big)\right\|_F\\
 &\le\sqrt r\left\|Q^{(j)}(Y^{(j)}-A^{(j)})
       \Big(\bigotimes_{t\ne j}Q^{(t)}\Big)\right\|.
\end{aligned}
\end{equation}
Lemma~\ref{lem:Q_2inf_asym} gives
\begin{equation}\label{eq:app:asym_projector_product}
 \prod_{t=1}^k\|Q^{(t)}\|_{2,\infty}
 \le\prod_{t=1}^k2\sqrt{\frac{\mu_1r}{n_t}}
 =\frac{2^k(\mu_1r)^{k/2}}{\sqrt{n_{\mathrm{tot}}}}.
\end{equation}
In \eqref{eq:QYA_asym_bound}, the ratio of the linear term to the
square-root term is
\[
 \frac{ak\log n_{\max}/(q\sqrt{n_{\mathrm{tot}}})}
      {\sqrt{ak\log n_{\max}/q}}
 =\sqrt{\frac{ak\log n_{\max}}{q n_{\mathrm{tot}}}}\le1
\]
by \eqref{eq:asym_sampling_levels}.
Insert \eqref{eq:QYA_asym_bound} and
\eqref{eq:app:asym_projector_product} into
\eqref{eq:app:asym_refinement_frobenius}, and use
$\|A^{(j)}\|\le\|T\|_F$. The factors
$\sqrt r\,r(\mu_1r)^{k/2}\sqrt{\mu_1\mu_2}$ give
\begin{equation}\label{eq:app:asym_refinement_error}
 \|T^*-\widehat T\|_F
 \le C2^k r^{1.5+k/2}
 \sqrt{\frac{ak\mu_1^{k+1}\mu_2\log n_{\max}}
                  {q n_{\mathrm{tot}}}}\,\|T\|_F.
\end{equation}

The wedge estimates \eqref{eq:app:asym_wedge_concentration} and
\eqref{eq:app:asym_basis_row_norm} hold for all modes with failure
$O(k n_{\max}^{-a})$. They imply \eqref{eq:asym_gap} and the
projector bounds in Lemma~\ref{lem:Q_2inf_asym}.
Conditional on any wedge realization satisfying
\eqref{eq:app:asym_wedge_concentration} and
\eqref{eq:app:asym_basis_row_norm},
Lemma~\ref{lem:QYA_asym} supplies \eqref{eq:QYA_asym_bound} for
the fixed mode $j$ with failure at most $n_{\max}^{-a}$.
Thus \eqref{eq:app:asym_projection_residual} and
\eqref{eq:app:asym_refinement_error} hold together with failure
$O(k n_{\max}^{-a})$. Their sum proves \eqref{eq:asym_final_bound}.
\endgroup
\end{proof}

\section[Proof of Theorem~\ref{thm:cai_sparse}]{Proof of {Theorem~\ref{thm:cai_sparse}}}\label{sec:app:proof_cai}
Throughout this appendix, define  $\lambda_i=\|x_i^\star\|_2^3$ and $\bar x_i=x_i^\star/\|x_i^\star\|_2$, so that 
\[T=\sum_{i=1}^r\lambda_i\bar x_i^{\otimes3}.\]
The uniform observations used for factor extraction and those used for gradient descent are independent batches, both sampled with probability $q$. In the extraction argument, $\tilde T$ denotes the first batch. At the beginning of the gradient-descent argument, $\Omega$ denotes the freshly drawn second batch. We condition on the wedge observations and, when appropriate, on the Gaussian vectors used for extraction. Thus all auxiliary sequences for the second batch may start at the same initializer.

\subsection{Subspace estimation}
Let $A=\unfold_1(T)$. Lemma~\ref{lem:unfold_incoherence} gives rank $r$, incoherence parameters bounded by $C\mu$ and $C\mu^2$, and condition number bounded by a universal multiple of $\kappa_{\mathrm{CP}}$.  Theorem~\ref{thm:wedge_spectral_guarantees} gives the following statement.
\begin{proposition}\label{prop:app::subspace:main}
Suppose $T$ satisfies \eqref{eq:cp_incoherent} and $r\mu\le n/2$.
If $p\ge C\kappa_{\mathrm{CP}}^4\mu^3r^2\log(n)/n^3$, then, with probability $1-O(n^{-12})$,
\[
  \|\hat UR-U\|+\sqrt{\frac n{\mu r}}\|\hat UR-U\|_{2,\infty}
 \lesssim\kappa_{\mathrm{CP}}^2\mathcal E_{\mathrm{se}},\qquad
 \mathcal E_{\mathrm{se}}=\sqrt{\frac{\mu^3r^2\log n}{pn^3}}.
\]
In particular, $\|\hat U\|_{2,\infty}\lesssim\sqrt{\mu r/n}$.
\end{proposition}

\begin{proof}
\begingroup
By Lemma~\ref{lem:unfold_incoherence}, $A=\unfold_1(T)$ has
$m=n^2$, $\mu_1\lesssim\mu$ and $\mu_2\lesssim\mu^2$.
In the factorization \eqref{eq:unfold_factorization} with $j=1$,
the two matrices with columns $\bar x_i$ and $\bar x_i\otimes\bar x_i$,
respectively, have Gram eigenvalues in $[1/2,3/2]$.
Indeed, their off-diagonal row sums are bounded by
$(r-1)\mu/n$ and $(r-1)(\mu/n)^2$, respectively.
Thus $\sigma_r(A)\ge\lambda_{\min}/2$, $\|A\|\le 3\lambda_{\max}/2$
and $\kappa\le3\kappa_{\mathrm{CP}}$.
Take $a=12$ in Theorem~\ref{thm:wedge_spectral_guarantees}.
Since $m=n^2$ and $\mu_1\mu_2\lesssim\mu^3$, its operator bound
\eqref{eq:wedge_spectral_op} and the aligned row estimate
\eqref{eq:app:gd-wedge-aligned-rows} give
\[
 \|\hat UR-U\|\lesssim\kappa_{\mathrm{CP}}^2\mathcal E_{\mathrm{se}},
 \qquad
 \|\hat UR-U\|_{2,\infty}
 \lesssim\kappa_{\mathrm{CP}}^2\mathcal E_{\mathrm{se}}\|U\|_{2,\infty}.
\]
Use $\|U\|_{2,\infty}\le\sqrt{\mu_1r/n}$ and
$\kappa_{\mathrm{CP}}^2\mathcal E_{\mathrm{se}}\lesssim1$ to obtain
both conclusions. All bounds hold on the same event of probability $1-O(n^{-12})$.
\endgroup
\end{proof}

\subsection{Factor extraction}\label{subsec:app:factor_extraction}
Algorithm~\ref{alg:factor_extraction} constructs the initial factor matrix
$X_0=\hat X^0$ for Algorithm~\ref{alg:wedge_GD}. It adapts the
factor-extraction procedure of \citet{cai2022nonconvex} to the subspace
$\hat U$ estimated from wedge observations. Each Gaussian draw gives a
direction in this subspace. We contract the same independent uniform sample
$\tilde T$ along each direction and symmetrize the resulting matrix.

\begin{algorithm}[ht]
\caption{Factor Extraction}
\label{alg:factor_extraction}
\begingroup
\begin{algorithmic}[1]
\State \textbf{Input:} Estimated subspace $\hat U\in\mathbb R^{n\times r}$, independent uniform batch $\tilde T$, sampling probability $q$, rank $r$, and number of Gaussian draws.
\For{each prescribed draw $\tau$}
    \State Draw $g^\tau\sim N(0,I_n)$ independently of the observations and the other Gaussian vectors, and set $\theta^\tau=\hat U\hat U^\top g^\tau$.
    \State Form the symmetric contraction
    \[
      \tilde M^\tau=\frac1{2q}\bigl(\tilde T\times_3\theta^\tau
        +(\tilde T\times_3\theta^\tau)^\top\bigr).
    \]
    \State Compute a unit eigenvector $\hat u^\tau$ associated with an eigenvalue of largest absolute value, and record the singular value gap $\sigma_1(\tilde M^\tau)-\sigma_2(\tilde M^\tau)$.
\EndFor
\State Start with all candidates.
\For{$i=1,\ldots,r$}
    \State If no candidate remains, terminate and report failure.
    \State Choose the remaining candidate $\tau$ with the largest gap and set $\hat u_i=\hat u^\tau$.
    \State If $\langle q^{-1}\tilde T,\hat u_i^{\otimes3}\rangle<0$, replace $\hat u_i$ by $-\hat u_i$. Set $\hat\lambda_i=\langle q^{-1}\tilde T,\hat u_i^{\otimes3}\rangle$.
    \State Delete every candidate $\tau$ with $|\langle\hat u^\tau,\hat u_i\rangle|\ge1/2$.
\EndFor
\State \textbf{Output:}
\[
 X_0=\hat X^0=
 \bigl[\hat\lambda_1^{1/3}\hat u_1,\ldots,
       \hat\lambda_r^{1/3}\hat u_r\bigr]\in\mathbb R^{n\times r}.
\]
\end{algorithmic}
\endgroup
\end{algorithm}
\FloatBarrier

For target failure probability $\delta\in(0,1)$, use the number of draws in
\eqref{eq:app:extraction-trial-count}. Proposition~\ref{prop:app:refinement:main}
shows that every true factor has a candidate with a large singular value
gap whose eigenvector approximates that factor. Among these candidates,
the threshold
$|\langle\hat u^\tau,\hat u_i\rangle|\ge 1/2$ removes precisely those
associated with the selected factor.

%The symmetric contraction in Algorithm~\ref{alg:factor_extraction}
%uses only the entries in $\tilde T$ and permits the symmetric eigenvector
%arguments for the independent, ordered-entry sampling model.
%For each Gaussian draw, the full and leave-one-out problems use the
%same contraction vector $\theta^\tau$. All eigenvector and candidate selection
%ties are resolved by fixed deterministic rules. 

\paragraph{Separation of the Gaussian coefficients.}
\begingroup
For draw $\tau$, define $\gamma_j^\tau=\lambda_j\langle\theta^\tau,\bar x_j\rangle$.
These are the coefficients in the population contraction
$T\times_3\theta^\tau=\sum_j\gamma_j^\tau\bar x_j\bar x_j^\top$.
For each $i$, we seek a draw in which $\gamma_i^\tau$ exceeds every
competing magnitude $|\gamma_j^\tau|$ by the gap in
\eqref{eq:app:gaussian-trial-gap}.

The Gram matrix of the unit vectors $\bar x_j$ has ones on its diagonal.
By \eqref{eq:cp_incoherent}, each off-diagonal entry has magnitude at most
$\mu/n$. The maximum absolute row sum therefore bounds the Gram error as follows.
\begin{equation}\label{eq:app:trial-gram-square-root}
\begin{aligned}
 \left\|\bigl(\langle\bar x_a,\bar x_b\rangle\bigr)_{a,b=1}^r-I_r\right\|
 &\le\max_a\sum_{b\ne a}|\langle\bar x_a,\bar x_b\rangle|
 \le\frac{(r-1)\mu}{n}\le\frac{r\mu}{n}<\frac12,\\
 \left\|\left[\bigl(\langle\bar x_a,\bar x_b\rangle\bigr)_{a,b=1}^r\right]^{1/2}-I_r\right\|
 &\le C\frac{r\mu}{n}.
\end{aligned}
\end{equation}
Here $r\mu/n<1/2$ follows from the lower bound on $q$ in
Proposition~\ref{prop:app:refinement:main}. Since $q\le1$,
$n\ge\sqrt C\,\kappa_{\mathrm{CP}}^3\mu^2r^{3/2}\log^{5/2}n$.
For the second estimate in \eqref{eq:app:trial-gram-square-root}, apply
$|\sqrt{1+t}-1|=|t|/(\sqrt{1+t}+1)$ to each Gram eigenvalue $1+t$.

Condition on wedge observations for which the subspace bounds in
Proposition~\ref{prop:app::subspace:main} hold. With the constant in the lower
bound on $p$ sufficiently large, the operator bound gives
\begin{equation}\label{eq:app:extraction-projection-close}
 \|\hat U\hat U^\top-UU^\top\|
 \le2\|\hat UR-U\|\le\frac12.
\end{equation}
Since $g^\tau\sim N(0,I_n)$ is independent of $\hat U$, the projected
Gaussian vector $\theta^\tau=\hat U\hat U^\top g^\tau$ has conditional
covariance $\hat U\hat U^\top$. To compare this projection with the
identity on the span of the factors, use \eqref{eq:app:extraction-projection-close}
and $\hat U\hat U^\top\preceq I_n$ to get
\[
 UU^\top-\tfrac12 I_n\preceq\hat U\hat U^\top\preceq I_n.
\]
Multiply on the right by $[\bar x_1,\ldots,\bar x_r]$ and on the left by its
transpose. Each $\bar x_i$ belongs to $\operatorname{col}(U)$ by
\eqref{eq:unfold_factorization}, so $UU^\top\bar x_i=\bar x_i$. It follows that
\[
 \tfrac12\bigl(\langle\bar x_i,\bar x_j\rangle\bigr)_{i,j=1}^r
 \preceq\bigl(\bar x_i^\top\hat U\hat U^\top\bar x_j\bigr)_{i,j=1}^r
 \preceq\bigl(\langle\bar x_i,\bar x_j\rangle\bigr)_{i,j=1}^r .
\]
By \eqref{eq:app:trial-gram-square-root}, the Gram matrix on the right
lies between $I_r/2$ and $3I_r/2$. The projected Gram matrix therefore
lies between $I_r/4$ and $3I_r/2$. Multiplying it on the left and right
by $\diag(\lambda_1,\ldots,\lambda_r)$ gives the conditional covariance of
the contraction coefficients and the following bounds on its eigenvalues.
\begin{equation}\label{eq:app:gaussian-coefficient-covariance}
\begin{aligned}
 \operatorname{Cov}\!\left((\gamma_i^\tau)_{i=1}^r\mid\hat U\right)
 &=\bigl(\lambda_i\lambda_j\bar x_i^\top
       \hat U\hat U^\top\bar x_j\bigr)_{i,j=1}^r,\\
 \frac{\lambda_{\min}^2}{4}I_r
 &\preceq\operatorname{Cov}\!\left((\gamma_i^\tau)_{i=1}^r\mid\hat U\right)
 \preceq\frac{3\lambda_{\max}^2}{2}I_r.
\end{aligned}
\end{equation}

\smallskip\noindent\emph{Probability of a gap for a given factor.}
Conditional on $\hat U$, the coefficients are linear functions of $g^\tau$,
so their joint law is centered Gaussian with covariance
\eqref{eq:app:gaussian-coefficient-covariance}. Fix $i$ and condition
further on $(\gamma_j^\tau)_{j\ne i}$. Substituting $\gamma_i^\tau=x$ into
the joint density, we find that the conditional density of $\gamma_i^\tau$ at $x$ is proportional to
\[
 \exp\!\left\{
 -\frac{x^2}{2}
   \left[\operatorname{Cov}\!\left((\gamma_\ell^\tau)_\ell\mid\hat U\right)^{-1}\right]_{ii}
 -x\sum_{j\ne i}
   \left[\operatorname{Cov}\!\left((\gamma_\ell^\tau)_\ell\mid\hat U\right)^{-1}\right]_{ij}
   \gamma_j^\tau
 \right\}.
\]
Comparing the coefficients of $x^2$ and $x$ with those of a scalar
Gaussian log density gives
\[
\begin{aligned}
 \frac{1}{\operatorname{Var}(\gamma_i^\tau\mid\hat U,(\gamma_j^\tau)_{j\ne i})}
 &=\left[\operatorname{Cov}\!\left((\gamma_\ell^\tau)_\ell\mid\hat U\right)^{-1}\right]_{ii}
 \le \frac{4}{\lambda_{\min}^2},\\
 \frac{\mathbb E[\gamma_i^\tau\mid\hat U,(\gamma_j^\tau)_{j\ne i}]}
      {\operatorname{Var}(\gamma_i^\tau\mid\hat U,(\gamma_j^\tau)_{j\ne i})}
 &=-\sum_{j\ne i}
 \left[\operatorname{Cov}\!\left((\gamma_\ell^\tau)_\ell\mid\hat U\right)^{-1}\right]_{ij}\gamma_j^\tau .
\end{aligned}
\]
The upper bound on the inverse covariance follows from the lower bound in
\eqref{eq:app:gaussian-coefficient-covariance}. Thus the conditional variance
is at least $\lambda_{\min}^2/4$ and depends only on $\hat U$.
The conditional mean is linear in $(\gamma_j^\tau)_{j\ne i}$, so it too
is centered Gaussian given $\hat U$. By the law of total variance,
\[
 \operatorname{Var}\!\left(
 \mathbb E[\gamma_i^\tau\mid\hat U,(\gamma_j^\tau)_{j\ne i}]
 \,\middle|\,\hat U\right)
 \le\operatorname{Var}(\gamma_i^\tau\mid\hat U)
 =\lambda_i^2\|\hat U^\top\bar x_i\|_2^2\le\lambda_{\max}^2 .
\]
The same diagonal covariance bound applies to every $\gamma_j^\tau$.
Thus, conditional on $\hat U$, a Gaussian tail bound and a union bound
show that the conditional mean and the $r-1$ competing coefficients all
have magnitude at most $\lambda_{\max}\sqrt{2\log(4r)}$ with probability
at least $1-2r\exp\{-\log(4r)\}=1/2$. For any such realization of the
competing coefficients, the standardized threshold for a gap satisfies
\[
\begin{aligned}
 &\frac{\max_{j\ne i}|\gamma_j^\tau|+2\lambda_{\min}\min\{\kappa_{\mathrm{CP}},\sqrt r\}
             -\mathbb E[\gamma_i^\tau\mid\hat U,(\gamma_j^\tau)_{j\ne i}]}
      {\sqrt{\operatorname{Var}(\gamma_i^\tau\mid\hat U,(\gamma_j^\tau)_{j\ne i})}}\\
 &\qquad\le \frac{2\lambda_{\max}\sqrt{2\log(4r)}+2\lambda_{\max}}{\lambda_{\min}/2}
 \le C\kappa_{\mathrm{CP}}\sqrt{\log(r+1)} .
\end{aligned}
\]
For $r=1$, the maximum of the competing magnitudes is zero.
For $t\ge0$, the standard Gaussian density on $[t,t+1]$ is at least
$(2\pi)^{-1/2}\exp(-(t+1)^2/2)$. Apply this tail lower bound at
$t=C\kappa_{\mathrm{CP}}\sqrt{\log(r+1)}$ and average over the realizations
just described, which have conditional probability at least $1/2$. This gives
\begin{equation}\label{eq:app:gaussian-trial-gap}
 \mathbb P\!\left(\gamma_i^\tau
 \ge \max_{j\ne i}|\gamma_j^\tau|+2\lambda_{\min}\min\{\kappa_{\mathrm{CP}},\sqrt r\}
 \ \middle|\ \hat U\right)\ge (r+1)^{-C\kappa_{\mathrm{CP}}^2}.
\end{equation}
A second estimate gives a polynomial draw count in $\kappa_{\mathrm{CP}}$
for fixed $r$. Conditional on $\hat U$, the covariance of the normalized
coefficients $(\gamma_j^\tau/\lambda_j)_{j=1}^r$ is the projected Gram
matrix used in \eqref{eq:app:gaussian-coefficient-covariance}. Its
eigenvalues lie in $[1/4,3/2]$, so the joint density at $z\in\mathbb R^r$
is at least $(3\pi)^{-r/2}\exp(-2\|z\|_2^2)$. On the rectangle
\[
 3\min\{\kappa_{\mathrm{CP}},\sqrt r\}\le z_i\le3\min\{\kappa_{\mathrm{CP}},\sqrt r\}+1,\qquad
 |z_j|\le\min\{1,\sqrt r/\kappa_{\mathrm{CP}}\}\quad(j\ne i),
\]
the dominant coordinate satisfies $\lambda_i z_i\ge3\lambda_{\min}\min\{\kappa_{\mathrm{CP}},\sqrt r\}$
and the competing magnitudes are at most $\lambda_{\min}\min\{\kappa_{\mathrm{CP}},\sqrt r\}$.
Thus the rectangle lies inside the gap event in \eqref{eq:app:gaussian-trial-gap}.
Its measure is $(2\min\{1,\sqrt r/\kappa_{\mathrm{CP}}\})^{r-1}$, and
$\|z\|_2^2\le(3\sqrt r+1)^2+r-1\le17r$ throughout it. Multiplying its
measure by the density lower bound gives a second lower bound on the
probability of the gap,
\[
 C^{-r}\min\{1,\sqrt r/\kappa_{\mathrm{CP}}\}^{r-1}.
\]
Fix a sufficiently large universal constant $C$ in the draw count
\eqref{eq:app:extraction-trial-count}. The draws are independent
conditional on $\hat U$. Using the larger of this bound and \eqref{eq:app:gaussian-trial-gap}, the probability that
all the draws prescribed in \eqref{eq:app:extraction-trial-count} miss the gap for a given $i$
is at most $\exp(-\log(4r/\delta))=\delta/(4r)$. A union bound over
$i$ gives failure probability at most $\delta/4$.

\smallskip\noindent\emph{Bounds uniform over the draws.}
We next control every coordinate of every contraction vector used in
Algorithm~\ref{alg:factor_extraction}. The lower bound on $q$ in
Proposition~\ref{prop:app:refinement:main} and $q\le1$ imply
$\kappa_{\mathrm{CP}}\le n$. Since $r\le n$, the number of draws satisfies
\begin{equation}\label{eq:app:extraction-trial-count}
 \left\lceil \min\{(r+1)^{C\kappa_{\mathrm{CP}}^2},\,C^r\max\{1,\kappa_{\mathrm{CP}}/\sqrt r\}^{r-1}\}\log(4r/\delta)\right\rceil
 \le n^{a-14},\qquad a=2C\min\{\kappa_{\mathrm{CP}}^2,r\}+16 .
\end{equation}
Indeed, $(r+1)^{C\kappa_{\mathrm{CP}}^2}\le n^{2C\kappa_{\mathrm{CP}}^2}$ and
$C^r\max\{1,\kappa_{\mathrm{CP}}/\sqrt r\}^{r-1}\le n^{2r}$ once $n\ge C$.
For sufficiently large $n$ depending on $\delta$, the remaining factor
$\log(4r/\delta)$ and the rounding are absorbed by $n^2$.
For each draw, the coordinate variances satisfy
\[
 \operatorname{Var}(\theta_s^\tau\mid\hat U)
 =\|e_s^\top\hat U\|_2^2\le\|\hat U\|_{2,\infty}^2,\qquad
 \operatorname{Var}(\gamma_i^\tau\mid\hat U)
 =\lambda_i^2\|\hat U^\top\bar x_i\|_2^2\le\lambda_{\max}^2.
\]
Applying the scalar Gaussian tail bound with exponent $a\log n$ gives the thresholds
\begin{equation}\label{eq:app:extraction-trial-maxima-thresholds}
\begin{aligned}
 \|\theta^\tau\|_\infty
 &\le\sqrt{2a\log n}\,\|\hat U\|_{2,\infty},\\
 \max_i|\lambda_i\langle\theta^\tau,\bar x_i\rangle|
 &\le\lambda_{\max}\sqrt{2a\log n}.
\end{aligned}
\end{equation}
A union bound over the at most $n^{a-14}$ draws in
\eqref{eq:app:extraction-trial-count} and the $n$ or $r$ coordinates gives
\begin{equation}\label{eq:app:extraction-trial-maxima-tail}
\begin{aligned}
 &\mathbb P\!\left\{\max_\tau\|\theta^\tau\|_\infty>
 \sqrt{2a\log n}\,\|\hat U\|_{2,\infty}
 \ \middle|\ \hat U\right\}\le2n^{-13},\\
 &\mathbb P\!\left\{\max_{\tau,i}|\gamma_i^\tau|>
 \lambda_{\max}\sqrt{2a\log n}
 \ \middle|\ \hat U\right\}\le2n^{-13}.
\end{aligned}
\end{equation}
Indeed, the two failure probabilities are bounded by $2n\,n^{a-14}n^{-a}$ and
$2r\,n^{a-14}n^{-a}$, respectively. Their sum is at most $4n^{-13}$.
Substituting $\|\hat U\|_{2,\infty}\lesssim\sqrt{\mu r/n}$ from
Proposition~\ref{prop:app::subspace:main} into
\eqref{eq:app:extraction-trial-maxima-thresholds} gives, simultaneously
for all draws,
\begin{equation}\label{eq:app:gaussian-trial-maxima}
 \|\theta^\tau\|_\infty\lesssim\sqrt{\frac{a\mu r\log n}{n}},
 \qquad
 \max_i|\lambda_i\langle\theta^\tau,\bar x_i\rangle|
 \lesssim\lambda_{\max}\sqrt{a\log n}.
\end{equation}

\endgroup

\begin{proposition}\label{prop:app:refinement:main}
Fix $\delta\in(0,1)$ and use the number of Gaussian draws in
\eqref{eq:app:extraction-trial-count}. Suppose $T$ satisfies
\eqref{eq:cp_incoherent} and
\[
 p\ge C\kappa_{\mathrm{CP}}^4\frac{\mu^3r^2\log n}{n^3},\qquad
 C\kappa_{\mathrm{CP}}^6\frac{\mu^4r^3\log^5 n}{n^2}\le q\le1
\]
for a sufficiently large universal sampling constant $C$, chosen after
the constant in the draw count. For all sufficiently large
$n$ depending on $\delta$, the following holds with probability at least
$1-\delta$. The outputs of Algorithm~\ref{alg:factor_extraction} can be
ordered so that, simultaneously for all $i\in[r]$,
\begin{align}
 \|\hat u_i-\bar x_i\|_2&\lesssim\kappa_{\mathrm{CP}}\min\{\kappa_{\mathrm{CP}},\sqrt r\}\mathcal E_{\mathrm{proj}},\label{eq:retrieval_l2}\\
 \|\hat u_i-\bar x_i\|_\infty&\lesssim\kappa_{\mathrm{CP}}\min\{\kappa_{\mathrm{CP}},\sqrt r\}\mathcal E_{\mathrm{op}}\sqrt{\frac\mu n},\label{eq:retrieval_linf}\\
 |\hat\lambda_i-\lambda_i|/\lambda_i&\lesssim\kappa_{\mathrm{CP}}\min\{\kappa_{\mathrm{CP}},\sqrt r\}\mathcal E_{\mathrm{proj}},\label{eq:retrieval_lbda}
\end{align}
where
\[
 \mathcal E_{\mathrm{proj}}=\frac{\mu r\log n}{n\sqrt q},\qquad
 \mathcal E_{\mathrm{op}}=\sqrt r\,\mathcal E_{\mathrm{proj}}.
\]
\end{proposition}

\begin{proof}
\begingroup
We first analyze a draw $\tau$ for which one population coefficient
dominates the others. Fix $i$ and suppose that
\begin{equation}\label{eq:app:extraction-dominant-coefficient}
 |\gamma_i^\tau|\ge\max_{j\ne i}|\gamma_j^\tau|+\lambda_{\min}\min\{\kappa_{\mathrm{CP}},\sqrt r\}.
\end{equation}
The lower bound on $q$ gives
\begin{equation}\label{eq:app:extraction-error-smallness}
 \kappa_{\mathrm{CP}}^{3}\mathcal E_{\mathrm{op}}
 \le\frac1{\sqrt C\,\mu\log^{3/2}n},\qquad
 \mathcal E_{\mathrm{proj}}=\mathcal E_{\mathrm{op}}/\sqrt r.
\end{equation}
In the error estimates we orient each eigenvector toward $\bar x_i$.
This fixes the sign of each auxiliary eigenvector using only its own
matrix and the deterministic factor. At the end of the proof we show
that the observed gaps select one draw satisfying
\eqref{eq:app:extraction-dominant-coefficient} for each $i$ and that
the sign chosen by Algorithm~\ref{alg:factor_extraction} agrees with
this orientation.

\smallskip\noindent\emph{Euclidean error.}
Condition on the wedge observations and Gaussian draws for which
\eqref{eq:app:gaussian-trial-maxima} holds. Let $\omega_{j\ell c}$
indicate whether entry $(j,\ell,c)$ belongs to the uniform extraction
sample. Before symmetrization, the error in the contraction is a sum of
independent centered matrices
\[
 (q^{-1}\omega_{j\ell c}-1)T_{j\ell c}\theta_c^\tau e_je_\ell^\top .
\]
In the unfolding \eqref{eq:unfold_factorization}, the Gram matrix of
the vectors $\bar x_i\otimes\bar x_i$ has off-diagonal row sums at most
$(r-1)(\mu/n)^2$ by \eqref{eq:cp_incoherent}. Consequently,
\[
 \max_j\|T_{j,:,:}\|_F^2
 \le 2\max_j\sum_i\lambda_i^2\bar x_i(j)^2
 \le 2\lambda_{\max}^2\mu r/n,\qquad
 \|T\|_\infty\le\lambda_{\max}\mu^{3/2}r/n^{3/2}.
\]
For both matrix products in Lemma~\ref{thm:(Matrix-Bernstein-inequality)},
the sum of the expected products is diagonal. Its entries are bounded by
$q^{-1}\sum_{\ell,c}T_{j\ell c}^2(\theta_c^\tau)^2$ when the transpose
is on the right. Symmetry of $T$ gives the same bound when the transpose
is on the left. Thus \eqref{eq:app:gaussian-trial-maxima} bounds the
variance and the largest summand norm, respectively, by
\[
\begin{aligned}
 \frac1q\max_j\sum_{\ell,c}T_{j\ell c}^2(\theta_c^\tau)^2
 &\le\frac{\|\theta^\tau\|_\infty^2}{q}\max_j\|T_{j,:,:}\|_F^2
 \lesssim\frac{a\lambda_{\max}^2\mu^2r^2\log n}{qn^2},\\
 q^{-1}\|T\|_\infty\|\theta^\tau\|_\infty
 &\lesssim\frac{\sqrt a\,\lambda_{\max}\mu^2r^{3/2}\sqrt{\log n}}{qn^2}.
\end{aligned}
\]
Apply Lemma~\ref{thm:(Matrix-Bernstein-inequality)} with $a$ as in
\eqref{eq:app:extraction-trial-count}. Taking the symmetric part of a
matrix cannot increase its operator norm, so
\begin{equation}\label{eq:app:extraction-matrix-concentration}
 \|\tilde M^\tau-T\times_3\theta^\tau\|
 \lesssim a\lambda_{\max}\sqrt{\frac{\mu^2r^2\log^2 n}{n^2q}}.
\end{equation}
Here the ratio of the linear Bernstein term to the square-root term is at most
\begin{equation}\label{eq:app:extraction-bernstein-absorption}
 C'\frac{\mu\sqrt{ar\log n}}{n\sqrt q}
 \lesssim\frac{1}{\sqrt{C}\,\kappa_{\mathrm{CP}}^{2}\mu r\log^2 n},
\end{equation}
which is at most one under the lower bound on $q$.

To isolate the direction $\bar x_i$, project the other components onto
its orthogonal complement and define
\begin{equation}\label{eq:app:def_M_tau}
 M^\tau=\gamma_i^\tau\bar x_i\bar x_i^\top+
 (I-\bar x_i\bar x_i^\top)
 \left(\sum_{j\ne i}\gamma_j^\tau\bar x_j\bar x_j^\top\right)
 (I-\bar x_i\bar x_i^\top).
\end{equation}
By construction, $M^\tau\bar x_i=\gamma_i^\tau\bar x_i$.
On $\bar x_i^\perp$, \eqref{eq:app:trial-gram-square-root} bounds the
operator norm of $M^\tau$ by $(1+Cr\mu/n)\max_{j\ne i}|\gamma_j^\tau|$.
Thus \eqref{eq:app:extraction-dominant-coefficient} gives
\[
 |\gamma_i^\tau|-\big\|M^\tau-\gamma_i^\tau\bar x_i\bar x_i^\top\big\|
 \ge \lambda_{\min}\min\{\kappa_{\mathrm{CP}},\sqrt r\}-C\frac{r\mu}{n}\max_{j\ne i}|\gamma_j^\tau|
 \ge\tfrac12\lambda_{\min}\min\{\kappa_{\mathrm{CP}},\sqrt r\}.
\]
For the last inequality, \eqref{eq:app:gaussian-trial-maxima} bounds the
correction by $C\lambda_{\max}\sqrt a\,\mathcal E_{\mathrm{proj}}$.
By \eqref{eq:app:extraction-error-smallness} and
$\sqrt a\lesssim\kappa_{\mathrm{CP}}$, this is at most $\lambda_{\min}/2$
for sufficiently large $C$ in the lower bound on $q$.

To compare $M^\tau$ with the population contraction, expand the two
projections in \eqref{eq:app:def_M_tau}. The resulting difference is
bounded in operator norm by
\[
 2\left\|\sum_{j\ne i}\gamma_j^\tau
       \langle\bar x_j,\bar x_i\rangle\,\bar x_j\right\|_2
 \lesssim \max_{j\ne i}|\gamma_j^\tau|
          \left(\sum_{j\ne i}\langle\bar x_j,\bar x_i\rangle^2\right)^{1/2},
\]
where \eqref{eq:app:trial-gram-square-root} bounds the norm of the
matrix with columns $\bar x_j$. Combining \eqref{eq:cp_incoherent} with
\eqref{eq:app:gaussian-trial-maxima} gives
\begin{equation}\label{eq:app:extraction-population-perturbation}
 \|T\times_3\theta^\tau-M^\tau\|
 \lesssim\lambda_{\max}\frac{\mu\sqrt{ar\log n}}{n}
 \le C\lambda_{\max}\sqrt a\,\mathcal E_{\mathrm{proj}} .
\end{equation}
By \eqref{eq:app:extraction-error-smallness}, the errors in
\eqref{eq:app:extraction-matrix-concentration} and
\eqref{eq:app:extraction-population-perturbation} together are smaller
than half the gap of $M^\tau$. Thus Weyl's inequality identifies the
eigenvalue of $\tilde M^\tau$ of largest magnitude with the perturbation
of $\gamma_i^\tau$. Multiplying both matrices by
$\operatorname{sgn}(\gamma_i^\tau)$ makes these the largest algebraic
eigenvalues. Lemma~\ref{lem:DK} now gives
\[
 \|\hat u^\tau-\bar x_i\|_2
 \lesssim \frac{\|\tilde M^\tau-T\times_3\theta^\tau\|
                   +\|T\times_3\theta^\tau-M^\tau\|}
                  {\lambda_{\min}\min\{\kappa_{\mathrm{CP}},\sqrt r\}}
 \lesssim \sqrt a\,\kappa_{\mathrm{CP}}\mathcal E_{\mathrm{proj}}.
\]
Since $\sqrt a\asymp\min\{\kappa_{\mathrm{CP}},\sqrt r\}$ by
\eqref{eq:app:extraction-trial-count}, this proves \eqref{eq:retrieval_l2}.

\smallskip\noindent\emph{Entrywise error.}
For each coordinate $s$, replace $\omega_{abc}$ by $q$ whenever
$s\in\{a,b,c\}$, using the same contraction vector $\theta^\tau$.
Denote the resulting symmetric matrix by $\tilde M^{(s),\tau}$.
Choose a unit eigenvector $\hat u^{(s),\tau}$ for the eigenvalue of largest
absolute value using a fixed rule, and orient it toward $\bar x_i$. We first bound
the error of $\hat u^{(s),\tau}$ at coordinate $s$, then compare this
auxiliary eigenvector with $\hat u^\tau$.

The centered summands of $\tilde M^{(s),\tau}$ are precisely those for
which $s\notin\{a,b,c\}$. Each expected matrix product in
Lemma~\ref{thm:(Matrix-Bernstein-inequality)} is positive semidefinite.
Omitting summands can therefore only decrease the variance bound and
the maximum summand norm. The argument of \eqref{eq:app:extraction-matrix-concentration}
and a union bound over $s$ give
\begin{equation}\label{eq:app:extraction-auxiliary-concentration}
 \max_s\|\tilde M^{(s),\tau}-T\times_3\theta^\tau\|
 \lesssim a\lambda_{\max}
 \sqrt{\frac{\mu^2r^2\log^2 n}{n^2q}}.
\end{equation}
By \eqref{eq:app:extraction-matrix-concentration},
\eqref{eq:app:extraction-population-perturbation}, and
\eqref{eq:app:extraction-auxiliary-concentration},
both $\tilde M^\tau$ and $\tilde M^{(s),\tau}$ lie within
$Ca\lambda_{\max}\mathcal E_{\mathrm{proj}}$ of $M^\tau$ in operator norm.
By \eqref{eq:app:extraction-error-smallness}, this is at most
$\lambda_{\min}/16$ for a sufficiently large sampling constant.
The gap in magnitude for $\gamma_i^\tau$ is at least
$\tfrac12\lambda_{\min}\min\{\kappa_{\mathrm{CP}},\sqrt r\}$.
Weyl's inequality therefore places the eigenvalue of largest magnitude
of each perturbed matrix within $\lambda_{\min}/16$ of $\gamma_i^\tau$.
In particular, both have the sign of $\gamma_i^\tau$ and magnitude at least $|\gamma_i^\tau|/2$.

Every entry in row $s$ of the auxiliary contraction uses the mean
sampling weight, so that
$e_s^\top\tilde M^{(s),\tau}=e_s^\top(T\times_3\theta^\tau)$.
Subtracting the $s$th coordinates of the eigenvector equations for
$\tilde M^{(s),\tau}$ and $M^\tau$, and applying Weyl's inequality to
their eigenvalues, gives
\begin{equation}\label{eq:app:extraction-own-coordinate-calculation}
\begin{aligned}
 |(\hat u^{(s),\tau}-\bar x_i)_s|
 &\lesssim\frac1{|\gamma_i^\tau|}\Big[
 \|e_s^\top M^\tau\|_2\|\hat u^{(s),\tau}-\bar x_i\|_2
 +\|e_s^\top(T\times_3\theta^\tau-M^\tau)\|_2\\
 &\hspace{8em} +\|\tilde M^{(s),\tau}-M^\tau\|\,|\bar x_i(s)|\Big].
\end{aligned}
\end{equation}
To bound the first two terms, expand \eqref{eq:app:def_M_tau}.
The inequalities $\sum_j\bar x_j(s)^2\le\mu r/n$ and
$\sum_{j\ne i}\langle\bar x_i,\bar x_j\rangle^2\le r(\mu/n)^2$,
together with $\max_j|\gamma_j^\tau|=|\gamma_i^\tau|$, give
\[
 \|e_s^\top M^\tau\|_2
 \lesssim|\gamma_i^\tau|\sqrt{\frac{\mu r}{n}},\qquad
 \|e_s^\top(T\times_3\theta^\tau-M^\tau)\|_2
 \lesssim\frac{|\gamma_i^\tau|\mu^{3/2}r}{n^{3/2}}.
\]
Lemma~\ref{lem:DK} applied to $M^\tau$ and $\tilde M^{(s),\tau}$
gives the same Euclidean error as in \eqref{eq:retrieval_l2},
namely $\|\hat u^{(s),\tau}-\bar x_i\|_2
\lesssim\sqrt a\,\kappa_{\mathrm{CP}}\mathcal E_{\mathrm{proj}}$.
Also, $\mu^{3/2}r/n^{3/2}\le\mathcal E_{\mathrm{proj}}\sqrt{\mu r/n}$
because $q\le1$. For the last term of
\eqref{eq:app:extraction-own-coordinate-calculation}, use
$|\bar x_i(s)|\le\sqrt{\mu/n}$ and divide the perturbation bound
$Ca\lambda_{\max}\mathcal E_{\mathrm{proj}}$ by
$|\gamma_i^\tau|\ge\lambda_{\min}\min\{\kappa_{\mathrm{CP}},\sqrt r\}$
from \eqref{eq:app:extraction-dominant-coefficient}. Combining the three
terms gives
\begin{equation}\label{eq:app:extraction-own-coordinate}
 |(\hat u^{(s),\tau}-\bar x_i)_s|
 \lesssim \sqrt a\,\kappa_{\mathrm{CP}}\mathcal E_{\mathrm{proj}}\sqrt{\frac{\mu r}{n}}
 =\sqrt a\,\kappa_{\mathrm{CP}}\mathcal E_{\mathrm{op}}\sqrt{\frac\mu n}.
\end{equation}

To compare the two eigenvectors, keep the wedge observations and Gaussian
vectors fixed and condition also on the extraction indicators with
$s\notin\{a,b,c\}$. This fixes $\hat u^{(s),\tau}$ and leaves the
indicators with $s\in\{a,b,c\}$ independent Bernoulli variables of mean $q$. Before symmetrization the residual
$(\tilde M^\tau-\tilde M^{(s),\tau})\hat u^{(s),\tau}$ is a sum of vectors
of the form
\[
 (q^{-1}\omega_{abc}-1)T_{abc}\theta_c^\tau
       \hat u_b^{(s),\tau}e_a,\qquad s\in\{a,b,c\}.
\]
Partition these triples by the first occurrence of $s$, giving the three
cases $a=s$, $a\ne s$ and $b=s$, and $a,b\ne s$ and $c=s$. The sum of the
squared norms of the vector coefficients is bounded by
\[
 \|\theta^\tau\|_\infty^2\|\hat u^{(s),\tau}\|_\infty^2
 \bigl(\|T_{s,:,:}\|_F^2+\|T_{:,s,:}\|_F^2+\|T_{:,:,s}\|_F^2\bigr).
\]
Since $\operatorname{Var}(q^{-1}\omega_{abc}-1)\le q^{-1}$, the variance
term is at most
$Ca\lambda_{\max}^2\mu^2r^2\log(n)\|\hat u^{(s),\tau}\|_\infty^2/(qn^2)$.
Each summand has norm at most
$q^{-1}\|T\|_\infty\|\theta^\tau\|_\infty\|\hat u^{(s),\tau}\|_\infty$.
Interchanging the first two indices gives the same bounds for the
transposed contraction. Apply vector Bernstein to both sums and use
\eqref{eq:app:extraction-bernstein-absorption} to bound the linear term.
Symmetrization then gives
\begin{equation}\label{eq:app:extraction-slice-residual}
 \|(\tilde M^\tau-\tilde M^{(s),\tau})\hat u^{(s),\tau}\|_2
 \lesssim a\lambda_{\max}\sqrt{\frac{\mu^2r^2\log^2 n}{n^2q}}
 \max\!\left\{\sqrt{\frac\mu n},\|\hat u^{(s),\tau}\|_\infty\right\}.
\end{equation}
After multiplying both matrices by $\operatorname{sgn}(\gamma_i^\tau)$,
the leading eigenvalue of $\tilde M^{(s),\tau}$ has a gap of order
$\lambda_{\min}\min\{\kappa_{\mathrm{CP}},\sqrt r\}$.
Apply Lemma~\ref{lem:DK} with this auxiliary matrix as the reference
and the residual in \eqref{eq:app:extraction-slice-residual}. The Euclidean bounds relative to $\bar x_i$ ensure
that the two eigenvectors have positive inner product. Their distance is
therefore at most a constant times the residual divided by the gap. Use
$\|\hat u^{(s),\tau}\|_\infty\le\|\hat u^\tau\|_\infty+
\|\hat u^{(s),\tau}-\hat u^\tau\|_2$.
By \eqref{eq:app:extraction-error-smallness}, the last term has coefficient
at most $1/2$ and can be moved to the left. Taking the maximum over $s$ gives
\begin{equation}\label{eq:app:extraction-eigenvector-difference}
 \max_s\|\hat u^{(s),\tau}-\hat u^\tau\|_2
 \lesssim \sqrt a\,\kappa_{\mathrm{CP}}\mathcal E_{\mathrm{proj}}
 \max\{\sqrt{\mu/n},\|\hat u^\tau\|_\infty\}.
\end{equation}
Choose $s$ with $|\hat u_s^\tau|=\|\hat u^\tau\|_\infty$. Then
\[
 |\hat u_s^\tau|\le|\bar x_i(s)|
 +|(\hat u^{(s),\tau}-\bar x_i)_s|+\|\hat u^\tau-\hat u^{(s),\tau}\|_2.
\]
Use \eqref{eq:app:extraction-own-coordinate} for the second term and
\eqref{eq:app:extraction-eigenvector-difference} for the third.
Increasing the sampling constant makes the coefficient of the maximum in
\eqref{eq:app:extraction-eigenvector-difference} at most $1/4$. Thus
\begin{equation}\label{eq:app:extraction-delocalization-closure}
 \|\hat u^\tau\|_\infty
 \le (1+C\sqrt a\,\kappa_{\mathrm{CP}}\mathcal E_{\mathrm{op}})\sqrt{\mu/n}
       +\tfrac14\max\{\sqrt{\mu/n},\|\hat u^\tau\|_\infty\}.
\end{equation}
If the maximum equals $\|\hat u^\tau\|_\infty$, subtract
$\|\hat u^\tau\|_\infty/4$ from both sides. In either case
\eqref{eq:app:extraction-error-smallness} gives $\|\hat u^\tau\|_\infty\lesssim\sqrt{\mu/n}$.
By \eqref{eq:app:extraction-eigenvector-difference}, the two eigenvectors
then differ in Euclidean norm by at most $C\sqrt a\,\kappa_{\mathrm{CP}}
\mathcal E_{\mathrm{proj}}\sqrt{\mu/n}$. Adding this to
\eqref{eq:app:extraction-own-coordinate} proves \eqref{eq:retrieval_linf}.

\smallskip\noindent\emph{Coefficient error.}
To estimate $\lambda_i$, Algorithm~\ref{alg:factor_extraction} evaluates
the sampled tensor at the estimated direction $\hat u^\tau$. The same
sample determines $\hat u^\tau$, so we first bound the sampling error
uniformly over the possible directions. Apply
Theorem~\ref{thm:symmetric_tensor_concentration} with $k=3$ and coordinate bound $C\sqrt{\mu/n}$,
and use $\|T\|_\infty\le\lambda_{\max}\mu^{3/2}r/n^{3/2}$ to obtain
\begin{equation}\label{eq:app:extraction-cubic-noise-norm}
\begin{aligned}
 &\max_{j\in\{1,2,3\}}
 \sup_{\substack{u_1,u_2,u_3\in\mathbb R^n\\
                 \|u_j\|_\infty\le C\sqrt{\mu/n}\\
                 \|u_\ell\|_2\le1\ (\ell\ne j)}}
 \left|\left\langle q^{-1}\tilde T-T,
       u_1\otimes u_2\otimes u_3\right\rangle\right|\\
 &\qquad\le C\lambda_{\max}\left(
 \frac{\mu^2r\log^2(n)}{n\sqrt q}
 +\frac{\mu^2r\log^{5/2}(n)}{qn^2}\right)
 \le C\lambda_{\min}.
\end{aligned}
\end{equation}
The lower bound on $q$ in Proposition~\ref{prop:app:refinement:main}
implies $q\ge C\kappa_{\mathrm{CP}}^2\mu^4r^3\log^5 n/n^2$,
so each term in parentheses is at most $C/\kappa_{\mathrm{CP}}$.
To compare the cubic evaluations at $\hat u^\tau$ and $\bar x_i$,
expand their difference one mode at a time.
\[
\begin{aligned}
 (\hat u^\tau)^{\otimes3}-\bar x_i^{\otimes3}
 &= (\hat u^\tau-\bar x_i)\otimes(\hat u^\tau)^{\otimes2}
    +\bar x_i\otimes(\hat u^\tau-\bar x_i)\otimes\hat u^\tau\\
 &\quad+\bar x_i^{\otimes2}\otimes(\hat u^\tau-\bar x_i).
\end{aligned}
\]
In each term, one unchanged unit vector supplies the coordinate bound
$C\sqrt{\mu/n}$ by \eqref{eq:cp_incoherent} or
\eqref{eq:app:extraction-delocalization-closure}. The other unchanged
vector supplies one Euclidean bound. Homogeneity in the difference
vector gives the remaining factor $\|\hat u^\tau-\bar x_i\|_2$.
Thus \eqref{eq:app:extraction-cubic-noise-norm} and \eqref{eq:retrieval_l2} give
\begin{equation}\label{eq:app:extraction-cubic-telescoping}
 \left|\left\langle q^{-1}\tilde T-T,
 (\hat u^\tau)^{\otimes3}-\bar x_i^{\otimes3}\right\rangle\right|
 \le C\lambda_{\min}\|\hat u^\tau-\bar x_i\|_2
 \lesssim \sqrt a\,\kappa_{\mathrm{CP}}\lambda_{\min}\mathcal E_{\mathrm{proj}}.
\end{equation}
At the deterministic direction $\bar x_i$, the summands of
$\langle q^{-1}\tilde T-T,\bar x_i^{\otimes3}\rangle$ are independent
and centered. Their variances sum to at most $q^{-1}\|T\|_\infty^2$
because $\|\bar x_i^{\otimes3}\|_F=1$, and each has absolute value at most
$q^{-1}\|T\|_\infty\|\bar x_i\|_\infty^3$. Scalar Bernstein's
inequality with parameter $16\log n$ gives
\begin{equation}\label{eq:app:extraction-fixed-cubic}
 \bigl|\langle q^{-1}\tilde T-T,\bar x_i^{\otimes3}\rangle\bigr|
 \lesssim\lambda_{\max}\left(
 \frac{\mu^{3/2}r\sqrt{\log n}}{n^{3/2}\sqrt q}
 +\frac{\mu^3r\log n}{qn^3}\right)
 \lesssim\lambda_{\max}\mathcal E_{\mathrm{proj}}.
\end{equation}
Dividing the two terms in parentheses by $\mathcal E_{\mathrm{proj}}$
gives $\sqrt{\mu/(n\log n)}$ and $\mu^2/(n^2\sqrt q)$, respectively.
Both are bounded under the sampling assumptions of
Proposition~\ref{prop:app:refinement:main}.

It remains to compare the population cubic
$\sum_j\lambda_j\langle\hat u^\tau,\bar x_j\rangle^3$ with $\lambda_i$.
For $j=i$, the unit-vector identity
$1-\langle\hat u^\tau,\bar x_i\rangle=\|\hat u^\tau-\bar x_i\|_2^2/2$
bounds the relative error in this term by
$3\|\hat u^\tau-\bar x_i\|_2^2/2$.
For $j\ne i$, incoherence gives
$|\langle\hat u^\tau,\bar x_j\rangle|^3\le
4(\mu/n)^3+4|\langle\hat u^\tau-\bar x_i,\bar x_j\rangle|^3$.
The sum of the latter cubes is bounded by
\[
 \sum_{j\ne i}|\langle\hat u^\tau-\bar x_i,\bar x_j\rangle|^3
 \le \left(\sum_{j\ne i}|\langle\hat u^\tau-\bar x_i,\bar x_j\rangle|^2\right)^{3/2}
 \lesssim \|\hat u^\tau-\bar x_i\|_2^3,
\]
where the last inequality follows from \eqref{eq:app:trial-gram-square-root}.
Since $\lambda_j/\lambda_i\le\kappa_{\mathrm{CP}}$, we obtain
\begin{equation}\label{eq:app:extraction-population-cubic}
 \frac{|\langle T,(\hat u^\tau)^{\otimes3}\rangle-\lambda_i|}{\lambda_i}
 \lesssim\|\hat u^\tau-\bar x_i\|_2^2+
 \kappa_{\mathrm{CP}}\left(\|\hat u^\tau-\bar x_i\|_2^3+r(\mu/n)^3\right).
\end{equation}
Equations~\eqref{eq:retrieval_l2} and \eqref{eq:app:extraction-error-smallness}
give $\sqrt{\kappa_{\mathrm{CP}}}\,\|\hat u^\tau-\bar x_i\|_2\lesssim
\kappa_{\mathrm{CP}}^{3}\mathcal E_{\mathrm{proj}}\le c'$, for a sufficiently small $c'$.
Also $r(\mu/n)^3/\mathcal E_{\mathrm{proj}}=(\mu/n)^2\sqrt q/\log n\le1$.
The right-hand side of \eqref{eq:app:extraction-population-cubic} is therefore
at most $C\sqrt a\,\kappa_{\mathrm{CP}}\mathcal E_{\mathrm{proj}}$.
Add the sampling errors in \eqref{eq:app:extraction-cubic-telescoping}
and \eqref{eq:app:extraction-fixed-cubic}, divided by $\lambda_i$,
to the population error in \eqref{eq:app:extraction-population-cubic}.
This proves \eqref{eq:retrieval_lbda} for the direction oriented toward
$\bar x_i$. Its relative coefficient error is less than $1/2$ by
\eqref{eq:app:extraction-error-smallness}.

\smallskip\noindent\emph{Simultaneous probability bounds.}
Condition on the wedge observations and Gaussian draws satisfying
\eqref{eq:app:gaussian-trial-maxima}, and take $a$ from
\eqref{eq:app:extraction-trial-count}. For each draw $\tau$,
Lemma~\ref{thm:(Matrix-Bernstein-inequality)} bounds the failure probability
by $(2n)^{-a}$ for the full matrix in
\eqref{eq:app:extraction-matrix-concentration} and for each of the
$n$ auxiliary matrices in \eqref{eq:app:extraction-auxiliary-concentration}.
For each $s$, the vector Bernstein bounds for the contraction and its
transpose in \eqref{eq:app:extraction-slice-residual} each fail with
probability at most $(n+1)^{-a}$. This bound holds conditionally on
the indicators with $s\notin\{a,b,c\}$, so it also holds after
averaging over those indicators. A union bound over the $1+n+2n$
estimates per draw and at most $n^{a-14}$ draws gives a conditional
failure probability of at most
\begin{equation}\label{eq:app:extraction-simultaneous-sampling-tail}
 (3n+1)n^{a-14}n^{-a}\le4n^{-13}.
\end{equation}
The $r$ fixed-direction estimates in \eqref{eq:app:extraction-fixed-cubic}
have total failure probability at most $2rn^{-16}\le2n^{-15}$.
The uniform estimate \eqref{eq:app:extraction-cubic-noise-norm} and
Proposition~\ref{prop:app::subspace:main} each fail with probability
$O(n^{-12})$. The Gaussian bounds in
\eqref{eq:app:extraction-trial-maxima-tail} and the sampling bounds in
\eqref{eq:app:extraction-simultaneous-sampling-tail} each contribute
at most $4n^{-13}$. Finally, the draw count in
\eqref{eq:app:extraction-trial-count} ensures that every factor has
a draw with the gap in \eqref{eq:app:gaussian-trial-gap}, except with
probability at most $\delta/4$. Thus all these estimates hold together
with failure probability at most $\delta/4+O(n^{-12})\le\delta$ for
sufficiently large $n$ depending on $\delta$. On this event,
\eqref{eq:retrieval_l2}--\eqref{eq:retrieval_lbda} hold for every
draw satisfying \eqref{eq:app:extraction-dominant-coefficient}.

\smallskip\noindent\emph{Selection by the observed gap.}
Algorithm~\ref{alg:factor_extraction} ranks candidates by
$\sigma_1(\tilde M^\tau)-\sigma_2(\tilde M^\tau)$. To relate this
quantity to the gap in \eqref{eq:app:extraction-dominant-coefficient},
first consider the population contraction
$T\times_3\theta^\tau=\sum_j\gamma_j^\tau\bar x_j\bar x_j^\top$.
By the polar decomposition of $[\bar x_1,\ldots,\bar x_r]$,
the nonzero singular values of this contraction are those of
\[
 \left[\bigl(\langle\bar x_a,\bar x_b\rangle\bigr)_{a,b=1}^r\right]^{1/2}
 \operatorname{diag}(\gamma_1^\tau,\ldots,\gamma_r^\tau)\,
 \left[\bigl(\langle\bar x_a,\bar x_b\rangle\bigr)_{a,b=1}^r\right]^{1/2}.
\]
Replacing the two Gram square roots by identities, one at a time, changes
this matrix in operator norm by at most
\[
 \left(\left\|\left[\bigl(\langle\bar x_a,\bar x_b\rangle\bigr)_{a,b=1}^r\right]^{1/2}\right\|+1\right)
 \left\|\left[\bigl(\langle\bar x_a,\bar x_b\rangle\bigr)_{a,b=1}^r\right]^{1/2}-I_r\right\|
 \max_j|\gamma_j^\tau|
 \le C\frac{r\mu}{n}\max_j|\gamma_j^\tau|,
\]
where the inequality follows from \eqref{eq:app:trial-gram-square-root}.
Weyl's inequality then gives, for $a=1,2$,
\begin{equation}\label{eq:app:trial-singular-comparison}
 \left|\sigma_a(\tilde M^\tau)
 -\sigma_a\bigl(\operatorname{diag}
       (\gamma_1^\tau,\ldots,\gamma_r^\tau,0,\ldots,0)\bigr)\right|
 \le \|\tilde M^\tau-T\times_3\theta^\tau\|
       +C\frac{r\mu}{n}\max_j|\gamma_j^\tau|\le\frac{\lambda_{\min}}8.
\end{equation}
The diagonal matrix includes zeros up to dimension $n$. We take the second
singular value to be zero when $n=1$. For $i$ maximizing $|\gamma_i^\tau|$,
subtracting the two comparisons gives
\begin{equation}\label{eq:app:trial-gap-comparison}
\begin{aligned}
 &\left|\sigma_1(\tilde M^\tau)-\sigma_2(\tilde M^\tau)
       -\left(|\gamma_i^\tau|-\max_{j\ne i}|\gamma_j^\tau|\right)
       \right|\\
 &\qquad\le2\|\tilde M^\tau-T\times_3\theta^\tau\|
       +C\frac{r\mu}{n}\max_j|\gamma_j^\tau|
 \le\frac{\lambda_{\min}}4.
\end{aligned}
\end{equation}
For $r=1$, the maximum over $j\ne i$ is zero. To see the last
inequalities in \eqref{eq:app:trial-singular-comparison} and
\eqref{eq:app:trial-gap-comparison}, divide both error terms by
$\lambda_{\min}$. Equation~\eqref{eq:app:extraction-matrix-concentration}
bounds the sampling term by $Ca\kappa_{\mathrm{CP}}\mathcal E_{\mathrm{proj}}$.
Equation~\eqref{eq:app:gaussian-trial-maxima} bounds the Gram term by
$C\kappa_{\mathrm{CP}}r\mu\sqrt{a\log n}/n
\le C\sqrt a\,\kappa_{\mathrm{CP}}\mathcal E_{\mathrm{proj}}$.
By \eqref{eq:app:extraction-error-smallness}, increasing the sampling
constant makes their sum at most $1/8$, uniformly over the draws.

Each factor has a draw with the coefficient gap in
\eqref{eq:app:gaussian-trial-gap}. Equation~\eqref{eq:app:trial-gap-comparison}
therefore guarantees at least one candidate per factor with observed gap
at least $\tfrac74\lambda_{\min}\min\{\kappa_{\mathrm{CP}},\sqrt r\}$.
Conversely, for any candidate meeting this threshold, the difference
between the two largest coefficient magnitudes is at least
$\tfrac32\lambda_{\min}\min\{\kappa_{\mathrm{CP}},\sqrt r\}$.
The corresponding factor thus satisfies
\eqref{eq:app:extraction-dominant-coefficient}, and all three estimates
\eqref{eq:retrieval_l2}--\eqref{eq:retrieval_lbda} apply.

For two candidates meeting the observed gap threshold, \eqref{eq:retrieval_l2}
and \eqref{eq:cp_incoherent} give
\[
\begin{aligned}
 |\langle\hat u^\tau,\hat u^{\tau'}\rangle|
 &\ge 1-Ca\kappa_{\mathrm{CP}}^2\mathcal E_{\mathrm{proj}}^2>1/2
 &&\text{for the same factor},\\
 |\langle\hat u^\tau,\hat u^{\tau'}\rangle|
 &\le \mu/n+C\sqrt a\,\kappa_{\mathrm{CP}}\mathcal E_{\mathrm{proj}}<1/2
 &&\text{for different factors}.
\end{aligned}
\]
Initially every factor has a candidate meeting the observed gap threshold. Whenever Algorithm~\ref{alg:factor_extraction} selects
such a candidate, the two inner-product bounds show that pruning removes
every candidate meeting the threshold for its factor and preserves
every such candidate for each other factor. Inductively, the largest
remaining gap meets the threshold at each of the $r$ selections, and
the algorithm selects exactly one candidate for each factor.

For a direction $\hat u_i$ oriented toward $\bar x_i$,
\eqref{eq:retrieval_lbda} and \eqref{eq:app:extraction-error-smallness} give
\[
 \langle q^{-1}\tilde T,\hat u_i^{\otimes3}\rangle
 \ge \lambda_i/2>0 .
\]
The cubic changes sign when the eigenvector is reversed. Hence the sign
rule in Algorithm~\ref{alg:factor_extraction} selects the orientation
used in \eqref{eq:retrieval_l2}--\eqref{eq:retrieval_lbda}, and the
corresponding column of $\hat X^0$ is $\hat\lambda_i^{1/3}\hat u_i$.
\endgroup
\end{proof}

\subsection{Gradient descent}
Let $\Omega$ be the gradient-descent sample, drawn at rate $q$ independently of the initialization data. Condition on those data, so that $\hat X^0$ is fixed, and start each auxiliary trajectory at $\hat X^{(s),0}=\hat X^0$. The full loss is \eqref{eq:gd_objective}. For each $s$, define the auxiliary loss by
\[
 F^{(s)}(\hat X)=\frac1{6q}\left\|
 \mathcal P_{\Omega_{-s}}\left(T-\sum_i\hat x_i^{\otimes3}\right)
 \right\|_F^2+\frac16\left\|
 \mathcal P_s\left(T-\sum_i\hat x_i^{\otimes3}\right)
 \right\|_F^2,
\]
where $\Omega_{-s}=\{(a,b,c)\in\Omega:a,b,c\ne s\}$ and
$\mathcal P_s$ selects all ordered entries having at least one index equal to $s$.
Thus the second term gives every ordered entry containing $s$ the same weight it has in the expected loss. The auxiliary trajectory depends only on the initialization and on $\Omega_{-s}$, so it is independent of the sampling indicators for all triples $(a,b,c)$ containing $s$.

\begin{corollary}\label{cor:app:gd:init}
Fix $\delta\in(0,1)$ and run
Algorithm~\ref{alg:factor_extraction} with target failure probability
$\delta/2$. Fix a sufficiently small $c>0$ and a sufficiently large $C>0$.
There are sufficiently large universal constants $c_0,c_1$, chosen after $c,C$, such that
\[
 p\ge c_0\frac{\kappa_{\mathrm{CP}}^4\mu^3r^2\log n}{n^3},
 \qquad q\ge c_1\frac{\kappa_{\mathrm{CP}}^6\mu^5r^4\log^5 n}{n^2}
\]
imply the following bounds with probability at least $1-\delta/2$
for all sufficiently large $n$ depending on $\delta$.
\[
 \|\hat X^0-X\|_F\le\frac{c}{C\kappa_{\mathrm{CP}}^2\mu^{3/2}\sqrt r}\|X\|_F,\qquad
 \|\hat X^0-X\|_{2,\infty}\le\frac{c}{8\kappa_{\mathrm{CP}}^{2/3}\mu\sqrt r}\|X\|_{2,\infty}.
\]
\end{corollary}
\begin{proof}
\begingroup
Apply Proposition~\ref{prop:app:refinement:main} with failure probability $\delta/2$.
The sampling condition in Corollary~\ref{cor:app:gd:init} gives
\[
 \mathcal E_{\mathrm{proj}}\le
 \frac{c_1^{-1/2}}{\kappa_{\mathrm{CP}}^3\mu^{3/2}r\log^{3/2} n},\qquad
 \mathcal E_{\mathrm{op}}=\sqrt r\,\mathcal E_{\mathrm{proj}}.
\]
By \eqref{eq:retrieval_lbda}, taking $c_1$ sufficiently large gives
$|\hat\lambda_i-\lambda_i|\le\lambda_i/2$ for every $i$. On this interval the
derivative of the cube root is at most $C\lambda_i^{-2/3}$, so
\[
 |\hat\lambda_i^{1/3}-\lambda_i^{1/3}|
 \lesssim \lambda_i^{1/3}\kappa_{\mathrm{CP}}\min\{\kappa_{\mathrm{CP}},\sqrt r\}\mathcal E_{\mathrm{proj}}.
\]
The $i$th factor error is
\[
 \hat\lambda_i^{1/3}(\hat u_i-\bar x_i)
 +(\hat\lambda_i^{1/3}-\lambda_i^{1/3})\bar x_i .
\]
For the Frobenius norm, apply \eqref{eq:retrieval_l2} and sum the squared column bounds
using $\sum_i\lambda_i^{2/3}=\|X\|_F^2$. For the row norm, use
\eqref{eq:retrieval_linf}, $\|\bar x_i\|_\infty\le\sqrt{\mu/n}$,
and $\|X\|_F/\sqrt n\le\|X\|_{2,\infty}$. The resulting relative errors satisfy
\[
\begin{aligned}
 \frac{\|\hat X^0-X\|_F}{\|X\|_F}
 &\lesssim \kappa_{\mathrm{CP}}\min\{\kappa_{\mathrm{CP}},\sqrt r\}\mathcal E_{\mathrm{proj}},\\
 \frac{\|\hat X^0-X\|_{2,\infty}}{\|X\|_{2,\infty}}
 &\lesssim \kappa_{\mathrm{CP}}\min\{\kappa_{\mathrm{CP}},\sqrt r\}\sqrt\mu\,(\mathcal E_{\mathrm{proj}}+\mathcal E_{\mathrm{op}}).
\end{aligned}
\]
For the Frobenius error, $\min\{\kappa_{\mathrm{CP}},\sqrt r\}\le \sqrt r$ gives the required factor
$\kappa_{\mathrm{CP}}^{-2}/(\mu^{3/2}\sqrt r)$. For the row error, $\min\{\kappa_{\mathrm{CP}},\sqrt r\}\le \kappa_{\mathrm{CP}}$
gives $\kappa_{\mathrm{CP}}^{-1}/(\mu\sqrt r)$, which is at most $\kappa_{\mathrm{CP}}^{-2/3}/(\mu\sqrt r)$.
Choosing $c_1$ sufficiently large gives the stated constants.
\endgroup
\end{proof}
\begingroup
\setlength{\emergencystretch}{2em}
\begin{lemma}\label{lem:app:gd:convexity}
There are universal constants $c_1,c_2,c_3>0$ such that if
\[
 q\ge c_1\frac{\kappa_{\mathrm{CP}}^{8/3}\mu^3r^3\log^5 n}{n^2},
 \qquad
 r\le \frac{c_2}{\kappa_{\mathrm{CP}}^{2/3}}\sqrt{\frac n\mu},
 \qquad
 0<\epsilon\le\frac{c_3}{\kappa_{\mathrm{CP}}^{4/3}\mu\sqrt r},
\]
then, with probability at least $1-O(n^{-10})$,
\begin{equation}\label{eq:app:gd-uniform-hessian-bounds}
 \frac12\lambda_{\min}^{4/3}\|V\|_F^2
 \le\operatorname{vec}(V)^\top\nabla^2F(\hat X)\operatorname{vec}(V)
 \le4\lambda_{\max}^{4/3}\|V\|_F^2
\end{equation}
holds simultaneously for every $V\in\mathbb R^{n\times r}$ and every
$\hat X\in\mathbb R^{n\times r}$ satisfying
\begin{equation}\label{eq:app:gd:convexity_region}
 \|\hat X-X\|_F\le\epsilon\|X\|_F,
 \qquad
 \|\hat X-X\|_{2,\infty}\le \kappa_{\mathrm{CP}}^{2/3}\epsilon\|X\|_{2,\infty}.
\end{equation}
\end{lemma}
\begin{proof}
Different permutations of a tensor entry have independent sampling
indicators when their index triples are distinct. We therefore keep all
three positions of $v_i$ when differentiating the loss in
\eqref{eq:gd_objective}. Its second directional derivative is
\begin{equation}\label{eq:app:gd-hessian-decomposition}
\begin{aligned}
 &\operatorname{vec}(V)^\top\nabla^2F(\hat X)\operatorname{vec}(V)\\
 &=\frac1{3q}\left\|\mathcal P_\Omega\left[
   \sum_{i=1}^r\big(
      \hat x_i^{\otimes2}\otimes v_i
     +\hat x_i\otimes v_i\otimes\hat x_i
     +v_i\otimes\hat x_i^{\otimes2}\big)
   \right]\right\|_F^2\\
 &\quad+\frac2{3q}\left\langle
    \mathcal P_\Omega\left(\sum_{j=1}^r\hat x_j^{\otimes3}-T\right),
    \sum_{i=1}^r\big(
       v_i^{\otimes2}\otimes\hat x_i
      +v_i\otimes\hat x_i\otimes v_i
      +\hat x_i\otimes v_i^{\otimes2}\big)
   \right\rangle.
\end{aligned}
\end{equation}
Decompose \eqref{eq:app:gd-hessian-decomposition} into $\alpha_1+\alpha_2+\alpha_3+\alpha_4$.
At $X$, the squared-norm term has mean $\alpha_4$ and sampling error $\alpha_2$.
Its change from $X$ to $\hat X$ is $\alpha_1$. The inner product with the tensor
residual gives $\alpha_3$. Thus
\begin{equation}\label{eq:app:gd-hessian-four-terms}
\begin{aligned}
 \alpha_4
 &=\frac13\left\|\sum_{i=1}^r\big(
      x_i^{\otimes2}\otimes v_i+x_i\otimes v_i\otimes x_i
     +v_i\otimes x_i^{\otimes2}\big)\right\|_F^2,\\
 \alpha_2
 &=\frac1{3q}\left\|\mathcal P_\Omega\left[
   \sum_{i=1}^r\big(x_i^{\otimes2}\otimes v_i
     +x_i\otimes v_i\otimes x_i+v_i\otimes x_i^{\otimes2}\big)
   \right]\right\|_F^2-\alpha_4,\\
 \alpha_1
 &=\frac1{3q}\left\|\mathcal P_\Omega\left[
   \sum_{i=1}^r\big(\hat x_i^{\otimes2}\otimes v_i
     +\hat x_i\otimes v_i\otimes\hat x_i
     +v_i\otimes\hat x_i^{\otimes2}\big)
   \right]\right\|_F^2-\alpha_2-\alpha_4,\\
 \alpha_3
 &=\frac2{3q}\left\langle
    \mathcal P_\Omega\left(\sum_{j=1}^r\hat x_j^{\otimes3}-T\right),
    \sum_{i=1}^r\big(v_i^{\otimes2}\otimes\hat x_i
      +v_i\otimes\hat x_i\otimes v_i
      +\hat x_i\otimes v_i^{\otimes2}\big)\right\rangle.
\end{aligned}
\end{equation}

\medskip
\textbf{Uniform sampling events.}
Apply Theorem~\ref{thm:symmetric_tensor_concentration} to
$\mathbf1^{\otimes3}$ with $k=3$ and with the incoherence parameter
set to $n$, so that $\delta_1=\delta_2=\delta_3=1$.
By homogeneity, with probability $1-O(n^{-12})$,
\begin{equation}\label{eq:app:gd-uniform-box-event}
\begin{aligned}
 &\left|\left\langle
 q^{-1}\mathcal P_\Omega(\mathbf1^{\otimes3})-\mathbf1^{\otimes3},
 u\otimes v\otimes w\right\rangle\right|\le C\left(
   \frac{n\log^2(n)}{\sqrt q}
   +\frac{\log^{5/2}(n)}q\right)
   \|u\|_2\|v\|_2\|w\|_\infty
\end{aligned}
\end{equation}
simultaneously for all $u,v,w\in\mathbb R^n$ and for each choice of the
mode with the infinity-norm constraint. Thus \eqref{eq:app:gd-uniform-box-event}
also applies to factors depending on $\Omega$.

For each fixed coordinate in any of the three modes, the number of
observed entries containing that coordinate in that mode is
$\operatorname{Bin}(n^2,q)$. The sampling hypothesis in
Lemma~\ref{lem:app:gd:convexity} implies
$qn^2\ge 36\log n$ after increasing $c_1$. Chernoff's inequality gives
\begin{equation}\label{eq:app:gd-slice-count-tail}
 \mathbb P\{\operatorname{Bin}(n^2,q)>2qn^2\}
 \le\exp(-qn^2/3).
\end{equation}
A union bound over the $3n$ slices gives failure probability at most
$3n\exp(-qn^2/3)\le3n^{-11}$. Hence every slice contains at most $2qn^2$ observations
outside an event of probability $3n^{-11}$.

\medskip
\textbf{The terms $\alpha_2$ and $\alpha_4$.}
Expanding the three tensor permutations gives the exact population
identity
\begin{align*}
 \alpha_4
 &=\sum_{i,j=1}^r\langle x_i,x_j\rangle^2\langle v_i,v_j\rangle+2\sum_{i,j=1}^r\langle x_i,x_j\rangle
                  \langle v_i,x_j\rangle\langle x_i,v_j\rangle.
\end{align*}
In the first sum, the diagonal terms lie between
$\lambda_{\min}^{4/3}\|V\|_F^2$ and
$\lambda_{\max}^{4/3}\|V\|_F^2$. CP incoherence and the row-sum
bound for a symmetric matrix give
\begin{align*}
 &\left\|
   \big(\langle x_i,x_j\rangle^2\big)_{i,j=1}^r
   -\operatorname{diag}\big(\|x_1\|_2^4,\ldots,\|x_r\|_2^4\big)
 \right\|\le\frac{r\mu^2}{n^2}\lambda_{\max}^{4/3}.
\end{align*}
The off-diagonal contribution to the first sum is consequently at most
$(r\mu^2/n^2)\lambda_{\max}^{4/3}\|V\|_F^2$ in absolute value.
In the second sum, the diagonal contribution is nonnegative and at
most $2\lambda_{\max}^{4/3}\|V\|_F^2$.
For its off-diagonal part, Cauchy--Schwarz gives
\begin{align*}
 &2\left|\sum_{i\ne j}\langle x_i,x_j\rangle
                  \langle v_i,x_j\rangle\langle x_i,v_j\rangle\right|\\
 &\qquad\le\frac{2\mu}{n}\lambda_{\max}^{2/3}
       \sum_{i,j}|\langle v_i,x_j\rangle\langle x_i,v_j\rangle|\\
 &\qquad\le\frac{2\mu}{n}\lambda_{\max}^{2/3}\|X^\top V\|_F^2
 \le\frac{C\mu}{n}\lambda_{\max}^{4/3}\|V\|_F^2.
\end{align*}
Here \eqref{eq:cp_incoherent} and \eqref{eq:app:trial-gram-square-root} give
\begin{equation}\label{eq:app:gd-factor-operator-norm}
 \|X\|^2\le\lambda_{\max}^{2/3}(1+r\mu/n)
 \le C\lambda_{\max}^{2/3}.
\end{equation}
Since $r\ge1$, the rank hypothesis gives
$\mu/n\le c_2^2 \kappa_{\mathrm{CP}}^{-4/3}$ and $r\mu^2/n^2\le c_2^4 \kappa_{\mathrm{CP}}^{-8/3}$.
Both off-diagonal errors are therefore at most
$C(c_2^2+c_2^4)\lambda_{\min}^{4/3}\|V\|_F^2$.
Choosing $c_2$ small gives
\begin{equation}\label{eq:app:gd-population-hessian-bound}
 \frac9{10}\lambda_{\min}^{4/3}\|V\|_F^2
 \le\alpha_4\le\frac72\lambda_{\max}^{4/3}\|V\|_F^2.
\end{equation}

For $\alpha_2$, consider the deterministic linear map
\[
 V\longmapsto\sum_{i=1}^r\big(
    x_i^{\otimes2}\otimes v_i
   +x_i\otimes v_i\otimes x_i
   +v_i\otimes x_i^{\otimes2}\big).
\]
The squared Frobenius norm of its image at $V$ is $3\alpha_4$.
By \eqref{eq:app:gd-population-hessian-bound}, its squared operator norm
is at most $C\lambda_{\max}^{4/3}$ when both spaces carry the Frobenius norm.
For the row indexed by $(a,b,c)$, \eqref{eq:cp_incoherent} bounds the squared Euclidean norm by
\begin{align*}
 3\sum_{i=1}^r\big(
    (x_i)_b^2(x_i)_c^2
   +(x_i)_a^2(x_i)_c^2
   +(x_i)_a^2(x_i)_b^2\big)
 \le\frac{9\mu^2r}{n^2}\lambda_{\max}^{4/3}.
\end{align*}
Its sampled Gram matrix, scaled by $1/(3q)$, is a sum of independent
rank-one matrices, one for each ordered tensor entry. After
centering, each summand has operator norm at most
$C\mu^2r\lambda_{\max}^{4/3}/(qn^2)$.
For each row, the square of its rank-one Gram matrix is the same matrix
multiplied by the squared row norm. Since
$\mathbb E[(q^{-1}\omega_{abc}-1)^2]=(1-q)/q$, the variance parameter
is at most
\[
 \frac{1-q}{9q}\,
 \frac{9\mu^2r\lambda_{\max}^{4/3}}{n^2}
 \sup_{\|V\|_F=1}3\alpha_4
 \le\frac{C\mu^2r\lambda_{\max}^{8/3}}{qn^2}
\]
by \eqref{eq:app:gd-population-hessian-bound}.
Apply Lemma~\ref{thm:(Matrix-Bernstein-inequality)} in dimension
$nr\times nr$ with $a=12$. Since $r\le n$, we have
$\log(2nr)=O(\log n)$, and the failure probability is at most
$(2nr)^{-12}\le n^{-12}$. The resulting bound is
\begin{equation}\label{eq:app:gd-fixed-truth-sampling}
 |\alpha_2|
 \le C\lambda_{\max}^{4/3}
 \left(
  \sqrt{\frac{\mu^2r\log n}{qn^2}}
  +\frac{\mu^2r\log n}{qn^2}
 \right)\|V\|_F^2
 \le\frac1{10}\lambda_{\min}^{4/3}\|V\|_F^2,
\end{equation}
simultaneously for all $V$. Indeed, the lower bound on $q$ gives
$\kappa_{\mathrm{CP}}^{4/3}\sqrt{\mu^2r\log n/(qn^2)}\le Cc_1^{-1/2}$,
so a universal choice of $c_1$ suffices.
The sampled quadratic term at $X$ is $\alpha_4+\alpha_2$ by
\eqref{eq:app:gd-hessian-four-terms}. Taking the square root after
applying \eqref{eq:app:gd-population-hessian-bound} and
\eqref{eq:app:gd-fixed-truth-sampling} gives
\begin{equation}\label{eq:app:gd-sampled-linearization-bound}
 \frac1{\sqrt q}\left\|\mathcal P_\Omega\left[
 \sum_{i=1}^r\big(x_i^{\otimes2}\otimes v_i
      +x_i\otimes v_i\otimes x_i+v_i\otimes x_i^{\otimes2}\big)
 \right]\right\|_F
 \le C\lambda_{\max}^{2/3}\|V\|_F.
\end{equation}

\medskip
\textbf{The term $\alpha_1$.}
The neighborhood constraints \eqref{eq:app:gd:convexity_region} and CP incoherence \eqref{eq:cp_incoherent} imply
\begin{equation}\label{eq:app:gd-basin-factor-bounds}
\begin{aligned}
 \|\hat X-X\|_F
 &\le\epsilon\lambda_{\max}^{1/3}\sqrt r,\qquad
 \|\hat X-X\|_{2,\infty}
 \le \kappa_{\mathrm{CP}}^{2/3}\epsilon\lambda_{\max}^{1/3}\sqrt{\frac{\mu r}{n}},\\
 \max_i\|\hat x_i\|_\infty
 &\le(1+\kappa_{\mathrm{CP}}^{2/3}\epsilon\sqrt r)\lambda_{\max}^{1/3}\sqrt{\frac\mu n},\qquad
 \max_i\|\hat x_i\|_2
 \le(1+\epsilon\sqrt r)\lambda_{\max}^{1/3}.
\end{aligned}
\end{equation}
To bound the change from $x_i$ to $\hat x_i$ in the sampled quadratic
term, first consider a summand with the difference in the first mode.
Cauchy--Schwarz in the factor index gives
\[
 \left|\sum_{i=1}^r
   (\hat x_i-x_i)_{a}(x_i)_b(v_i)_c\right|
 \le\|\hat X-X\|_{2,\infty}
      \max_i\|x_i\|_\infty\,\|V_{c,:}\|_2.
\]
For fixed $c$ there are at most $2qn^2$ observed pairs $(a,b)$ on
the slice-count event from \eqref{eq:app:gd-slice-count-tail}.
Squaring the entry bound and summing first over $(a,b)$ and then over $c$ gives
\begin{equation}\label{eq:app:gd-sampled-difference-term}
\begin{aligned}
 &\frac1q\left\|\mathcal P_\Omega\left[
   \sum_{i=1}^r(\hat x_i-x_i)\otimes x_i\otimes v_i
   \right]\right\|_F^2\\
 &\qquad\le2n^2\|\hat X-X\|_{2,\infty}^2
          \max_i\|x_i\|_\infty^2\,\|V\|_F^2\\
 &\qquad\le C \kappa_{\mathrm{CP}}^{4/3}\epsilon^2\mu^2r\lambda_{\max}^{4/3}\|V\|_F^2.
\end{aligned}
\end{equation}
The calculation in \eqref{eq:app:gd-sampled-difference-term} applies when the two fixed factors are permuted,
and when $x_i$ is replaced by $\hat x_i$, using
\eqref{eq:app:gd-basin-factor-bounds} for
$\max_i\|\hat x_i\|_\infty$. Expand the difference of each pair of
quadratic products using
\[
 \hat x_i\otimes\hat x_i-x_i\otimes x_i
 =(\hat x_i-x_i)\otimes\hat x_i
    +x_i\otimes(\hat x_i-x_i).
\]
For each of the three positions of $v_i$, apply
\eqref{eq:app:gd-sampled-difference-term} to both terms of this identity.
The triangle inequality bounds the change in the sampled linearization
in \eqref{eq:app:gd-sampled-linearization-bound} by
\begin{equation}\label{eq:app:gd-linearization-difference-bound}
 C\lambda_{\max}^{2/3}\kappa_{\mathrm{CP}}^{2/3}\epsilon\mu\sqrt r
 (1+\kappa_{\mathrm{CP}}^{2/3}\epsilon\sqrt r)\|V\|_F
 \le Cc_3\lambda_{\min}^{2/3}\|V\|_F.
\end{equation}
The sampled quadratic terms at $X$ and $\hat X$ are squared norms in
\eqref{eq:app:gd-hessian-four-terms}. The reverse triangle inequality and
\eqref{eq:app:gd-linearization-difference-bound} therefore give
\begin{equation}\label{eq:app:gd-alpha-one-bound}
 \left|\sqrt{\alpha_4+\alpha_2+\alpha_1}-\sqrt{\alpha_4+\alpha_2}\right|
 \le Cc_3\lambda_{\min}^{2/3}\|V\|_F.
\end{equation}
Equations~\eqref{eq:app:gd-population-hessian-bound} and
\eqref{eq:app:gd-fixed-truth-sampling} place $\alpha_4+\alpha_2$
between $0.8\lambda_{\min}^{4/3}\|V\|_F^2$ and
$3.6\lambda_{\max}^{4/3}\|V\|_F^2$.
Choose $c_3$ so the right-hand side of \eqref{eq:app:gd-alpha-one-bound}
is at most $0.04\lambda_{\min}^{2/3}\|V\|_F$. Squaring gives
\begin{equation}\label{eq:app:gd-alpha-one-three-small}
 0.7\lambda_{\min}^{4/3}\|V\|_F^2\le \alpha_4+\alpha_2+\alpha_1
 \le 3.8\lambda_{\max}^{4/3}\|V\|_F^2.
\end{equation}

\medskip
\textbf{The population part of $\alpha_3$.}
The operator-norm bound \eqref{eq:app:gd-factor-operator-norm}
and the neighborhood
condition \eqref{eq:app:gd:convexity_region} give
$\|\hat X\|\le\|X\|+\|\hat X-X\|_F
\le C\lambda_{\max}^{1/3}$.
Contract the third mode of $\sum_j(\hat x_j^{\otimes3}-x_j^{\otimes3})$
against $\hat x_i$. Expanding the difference gives
\begin{align*}
 &\sum_{j=1}^r\left[
   \hat x_j\hat x_j^\top\langle\hat x_j,\hat x_i\rangle
      -x_jx_j^\top\langle x_j,\hat x_i\rangle\right]\\
 &=(\hat X-X)\operatorname{diag}(\hat X^\top\hat x_i)\hat X^\top
   +X\operatorname{diag}(\hat X^\top\hat x_i)(\hat X-X)^\top\\
 &\quad+X\operatorname{diag}\big((\hat X-X)^\top\hat x_i\big)X^\top.
\end{align*}
Apply Cauchy--Schwarz to each diagonal entry and use
\eqref{eq:app:gd-basin-factor-bounds} to obtain
\begin{equation}\label{eq:app:gd-residual-diagonal-bounds}
\begin{aligned}
 \|\operatorname{diag}(\hat X^\top\hat x_i)\|
 &\le\max_j\|\hat x_j\|_2\|\hat x_i\|_2
 \le C\lambda_{\max}^{2/3},\\
 \left\|\operatorname{diag}\big((\hat X-X)^\top\hat x_i\big)\right\|
 &\le\|\hat X-X\|_F\|\hat x_i\|_2
 \le C\lambda_{\max}^{1/3}\|\hat X-X\|_F.
\end{aligned}
\end{equation}
By \eqref{eq:app:gd-residual-diagonal-bounds}, the three matrix
products have total operator norm at most
\[
 C\left[\lambda_{\max}^{2/3}(\|\hat X\|+\|X\|)
       +\lambda_{\max}^{1/3}\|X\|^2\right]\|\hat X-X\|_F.
\]
Using $\|X\|,\|\hat X\|\le C\lambda_{\max}^{1/3}$ gives
\begin{equation}\label{eq:app:gd-population-residual-bound}
 C\lambda_{\max}\|\hat X-X\|_F
 \le C\epsilon\sqrt r\,\lambda_{\max}^{4/3}.
\end{equation}
By symmetry, \eqref{eq:app:gd-population-residual-bound} applies to all three
placements of $\hat x_i$ in \eqref{eq:app:gd-hessian-four-terms}. Each resulting
quadratic form in $v_i$ is at most
$C\epsilon\sqrt r\lambda_{\max}^{4/3}\|v_i\|_2^2$ in absolute value.
Summing over $i$ bounds the population part of $\alpha_3$ by
$C\epsilon\sqrt r\lambda_{\max}^{4/3}\|V\|_F^2$.

\medskip
\textbf{The centered part of $\alpha_3$.}
The exact telescoping identity
\begin{align*}
 \hat x_j^{\otimes3}-x_j^{\otimes3}
 &=(\hat x_j-x_j)\otimes x_j\otimes x_j\\
 &\quad+\hat x_j\otimes(\hat x_j-x_j)\otimes x_j\\
 &\quad+\hat x_j\otimes\hat x_j\otimes(\hat x_j-x_j)
\end{align*}
places the difference $\hat x_j-x_j$ in each mode once. Pairing its
first term with $v_i^{\otimes2}\otimes\hat x_i$ in the centered
sampling term gives a test tensor with factors
$(\hat x_j-x_j)\odot v_i$, $x_j\odot v_i$, and
$x_j\odot\hat x_i$. The second term places the difference in the
other factor paired with $v_i$. For both placements,
\begin{align*}
 \sum_{j=1}^r\|(\hat x_j-x_j)\odot v_i\|_2
 &\le\sqrt r\left(
    \sum_{a=1}^n(v_i)_a^2
       \sum_{j=1}^r(\hat x_j-x_j)_a^2\right)^{1/2}\\
 &\le\sqrt r\,\|\hat X-X\|_{2,\infty}\|v_i\|_2\\
 &\le\kappa_{\mathrm{CP}}^{2/3}\epsilon r\sqrt{\frac\mu n}
      \lambda_{\max}^{1/3}\|v_i\|_2.
\end{align*}
The remaining factor paired with $v_i$ has norm at most
$C\sqrt{\mu/n}\lambda_{\max}^{1/3}\|v_i\|_2$, and the factor
paired with $\hat x_i$ has infinity norm at most
$C(\mu/n)\lambda_{\max}^{2/3}$. Thus the sum over $i,j$ of the
products of the two Euclidean norms and this infinity norm is at most
\begin{equation}\label{eq:app:gd-first-two-product-bounds}
 C\frac{\kappa_{\mathrm{CP}}^{2/3}\epsilon\mu^2r\lambda_{\max}^{4/3}}{n^2}\|V\|_F^2.
\end{equation}
For the third placement, both factors paired with $v_i$ have norm
at most $C\sqrt{\mu/n}\lambda_{\max}^{1/3}\|v_i\|_2$, while
\begin{equation}\label{eq:app:gd-third-factor-bounds}
 \sum_{j=1}^r\|(\hat x_j-x_j)\odot\hat x_i\|_\infty
 \le r\|\hat X-X\|_{2,\infty}\|\hat x_i\|_\infty
 \le C\frac{\kappa_{\mathrm{CP}}^{2/3}\epsilon\mu r^{3/2}}n\lambda_{\max}^{2/3}.
\end{equation}
By \eqref{eq:app:gd-third-factor-bounds}, the product bound for this placement is
\begin{equation}\label{eq:app:gd-third-product-bound}
 C\frac{\kappa_{\mathrm{CP}}^{2/3}\epsilon\mu^2r^{3/2}\lambda_{\max}^{4/3}}{n^2}\|V\|_F^2.
\end{equation}
Apply the uniform box estimate \eqref{eq:app:gd-uniform-box-event}
with these infinity norms in the restricted mode.
For example, the first placement is bounded by
\begin{equation}\label{eq:app:gd-centered-residual-test}
\begin{aligned}
 &\left|\left\langle
 q^{-1}\mathcal P_\Omega(\mathbf1^{\otimes3})-\mathbf1^{\otimes3},
 ((\hat x_j-x_j)\odot v_i)\otimes(x_j\odot v_i)
                        \otimes(x_j\odot\hat x_i)
 \right\rangle\right|\\
 &\quad\le C\left(
   \frac{n\log^2(n)}{\sqrt q}
   +\frac{\log^{5/2}(n)}q\right)
 \|(\hat x_j-x_j)\odot v_i\|_2
 \|x_j\odot v_i\|_2\|x_j\odot\hat x_i\|_\infty.
\end{aligned}
\end{equation}
For the first two placements of $\hat x_j-x_j$, sum
\eqref{eq:app:gd-centered-residual-test} over $i,j$ using
\eqref{eq:app:gd-first-two-product-bounds}. For the third, use
\eqref{eq:app:gd-third-product-bound}. The same estimates hold for the other
positions of $\hat x_i$ by \eqref{eq:app:gd-uniform-box-event}. Since $r\ge1$,
the centered part of $\alpha_3$ is at most
\[
 C\kappa_{\mathrm{CP}}^{2/3}\epsilon\lambda_{\max}^{4/3}
 \left(
   \frac{\mu^2r^{3/2}\log^2(n)}{n\sqrt q}
   +\frac{\mu^2r^{3/2}\log^{5/2}(n)}{qn^2}
 \right)\|V\|_F^2.
\]
Substituting the sampling rate of Lemma~\ref{lem:app:gd:convexity}
bounds the coefficient of $\lambda_{\max}^{4/3}\|V\|_F^2$ by
\[
 C \kappa_{\mathrm{CP}}^{2/3}\epsilon\left(
 \frac{\sqrt\mu}{\kappa_{\mathrm{CP}}^{4/3}\sqrt{c_1\log n}}
 +\frac{1}{c_1\kappa_{\mathrm{CP}}^{8/3}\mu r^{3/2}\log^{5/2}n}\right)
 \le Cc_3 \kappa_{\mathrm{CP}}^{-4/3}(c_1^{-1/2}+c_1^{-1}).
\]
The last inequality uses $\kappa_{\mathrm{CP}}^{2/3}\epsilon\sqrt\mu\le c_3$.
For the population part, \eqref{eq:app:gd-population-residual-bound} uses
the Frobenius radius $\epsilon\sqrt r\le c_3\kappa_{\mathrm{CP}}^{-4/3}$.
Since $\kappa_{\mathrm{CP}}^{-4/3}\lambda_{\max}^{4/3}=\lambda_{\min}^{4/3}$,
choosing $c_3$ sufficiently small gives
\begin{equation}\label{eq:app:gd-alpha-three-bound}
 |\alpha_3|\le Cc_3(1+c_1^{-1/2}+c_1^{-1})\lambda_{\min}^{4/3}\|V\|_F^2
 \le \tfrac1{10}\lambda_{\min}^{4/3}\|V\|_F^2.
\end{equation}
Combining \eqref{eq:app:gd-alpha-three-bound} and
\eqref{eq:app:gd-alpha-one-three-small} in
\eqref{eq:app:gd-hessian-four-terms} proves both Hessian bounds.
The intersection of \eqref{eq:app:gd-uniform-box-event},
\eqref{eq:app:gd-fixed-truth-sampling}, and the slice-count event
from \eqref{eq:app:gd-slice-count-tail} has probability
$1-O(n^{-10})$. All estimates are uniform over
\eqref{eq:app:gd:convexity_region} and all $V$.
\end{proof}
\endgroup

\begin{proposition}\label{prop:app:gd:geometric_recursion}
Choose $\eta=(8\lambda_{\max}^{4/3})^{-1}$. Under the conditions of Corollary~\ref{cor:app:gd:init}, condition on initialization data for which the two bounds in that corollary hold. With conditional probability $1-O(n^{-10})$ over the independent gradient-descent sample, the following estimates hold simultaneously for every $s$ and every integer $0\le t\le\lceil 192\kappa_{\mathrm{CP}}^{4/3}\log n\rceil$:
\begin{subequations}\label{eq:app:gd:recursion_cond}
\begin{align}
 \|\hat X^t-X\|_F
 &\le\frac{c}{C\kappa_{\mathrm{CP}}^2\mu^{3/2}\sqrt r}\|X\|_F
 (1-\eta\lambda_{\min}^{4/3}/3)^t,\label{subeq:app:gd:recursion_cond_1}\\
 \|\hat X^t-X\|_{2,\infty}
 &\le\frac{c}{2\kappa_{\mathrm{CP}}^{2/3}\mu\sqrt r}\|X\|_{2,\infty}
 (1-\eta\lambda_{\min}^{4/3}/3)^t,\label{subeq:app:gd:recursion_cond_2}\\
 \|\hat X^{(s),t}-\hat X^t\|_F
 &\le\frac{c}{C\kappa_{\mathrm{CP}}^2\mu\sqrt r}\|X\|_{2,\infty}
 (1-\eta\lambda_{\min}^{4/3}/3)^t,\label{subeq:app:gd:recursion_cond_3}\\
 \|\hat X^{(s),t}-\hat X^t\|_{2,\infty}
 &\le\frac{c}{C\kappa_{\mathrm{CP}}^2\mu\sqrt r}\|X\|_{2,\infty}
 (1-\eta\lambda_{\min}^{4/3}/3)^t,\label{subeq:app:gd:recursion_cond_4}\\
 \|\bigl(\hat X^{(s),t}-X\bigr)_{s,:}\|_2
 &\le\frac{c}{4\kappa_{\mathrm{CP}}^{2/3}\mu\sqrt r}\|X\|_{2,\infty}
 (1-\eta\lambda_{\min}^{4/3}/3)^t.\label{subeq:app:gd:recursion_cond_5}
\end{align}
\end{subequations}
\end{proposition}

\begin{proof}
\begingroup
It suffices to propagate the full Frobenius error
\eqref{subeq:app:gd:recursion_cond_1}, the Frobenius difference between
iterates \eqref{subeq:app:gd:recursion_cond_3}, and the auxiliary error
in row $s$ \eqref{subeq:app:gd:recursion_cond_5}.
Indeed, $\|\hat X^{(s),t}-\hat X^t\|_{2,\infty}
\le\|\hat X^{(s),t}-\hat X^t\|_F$ gives \eqref{subeq:app:gd:recursion_cond_4}.
For \eqref{subeq:app:gd:recursion_cond_2}, take the maximum over $s$ in
\[
 \|(\hat X^t-X)_{s,:}\|_2
 \le \|\hat X^{(s),t}-\hat X^t\|_F+\|(\hat X^{(s),t}-X)_{s,:}\|_2
\]
and use $1/C+1/4\le1/2$.
At $t=0$ the two iterates coincide, so Corollary~\ref{cor:app:gd:init}
gives all three initial estimates. Assume the estimates hold at time
$t$. Since $q\le1$, the sampling condition in Corollary~\ref{cor:app:gd:init} implies
\[
 r^2\le\frac{n}{\sqrt{c_1}\kappa_{\mathrm{CP}}^3\mu^{5/2}\log^{5/2}n}
 \le\frac{c_2^2 n}{\kappa_{\mathrm{CP}}^{4/3}\mu}
\]
for large enough $c_1$. It also implies the lower bound on $q$ in Lemma~\ref{lem:app:gd:convexity}.
By the triangle inequality, \eqref{eq:app:gd:recursion_cond} thus places both iterates in
\eqref{eq:app:gd:convexity_region}, as does their connecting segment
by convexity.

\paragraph{Full gradient and auxiliary difference.}
Along any segment in \eqref{eq:app:gd:convexity_region}, the average
Hessian has eigenvalues between $\lambda_{\min}^{4/3}/2$ and
$4\lambda_{\max}^{4/3}$ by \eqref{eq:app:gd-uniform-hessian-bounds} in
Lemma~\ref{lem:app:gd:convexity}. Hence
the gradient step with $\eta=(8\lambda_{\max}^{4/3})^{-1}$ contracts
the distance between the endpoints by a factor of at most
$1-\eta\lambda_{\min}^{4/3}/2$. Applying this to $X$ and $\hat X^t$,
with $\nabla F(X)=0$, gives
\begin{equation}\label{eq:app:gd-full-frobenius-step}
 \|\hat X^{t+1}-X\|_F
 \le(1-\eta\lambda_{\min}^{4/3}/2)\|\hat X^t-X\|_F.
\end{equation}
To compare the full and auxiliary gradient steps, condition on the
initialization data and on $\Omega_{-s}$. The auxiliary iterate is then
fixed, and the indicators for entries with an index equal to $s$ are
independent Bernoulli variables. At a fixed $\hat X$, the $i$th column of
$\nabla F(\hat X)-\nabla F^{(s)}(\hat X)$ is
\begin{equation}\label{eq:app:gd-slice-gradient-expansion}
\begin{aligned}
 &\frac13\sum_{a,b,c:\,s\in\{a,b,c\}}
 (q^{-1}\omega_{abc}-1)
 \left(\sum_j\hat x_j(a)\hat x_j(b)\hat x_j(c)-T_{abc}\right)\\
 &\qquad{}\times\left(\hat x_i(b)\hat x_i(c)e_a+
       \hat x_i(a)\hat x_i(c)e_b+\hat x_i(a)\hat x_i(b)e_c\right).
\end{aligned}
\end{equation}
Expanding the difference of the two cubic products gives three terms,
each with one factor difference. In each term, Cauchy--Schwarz over
$j$ bounds the sum by the product of a row error, a maximum factor
entry, and a factor row norm. The last two quantities are at most
$C\lambda_{\max}^{1/3}\sqrt{\mu/n}$ and
$C\lambda_{\max}^{1/3}\sqrt{\mu r/n}$, respectively, by
\eqref{eq:app:gd-basin-factor-bounds}. Thus
\begin{equation}\label{eq:app:gd-slice-residual}
 \left\|\sum_i(\hat x_i^{(s),t})^{\otimes3}-T\right\|_\infty
 \lesssim\lambda_{\max}^{2/3}\frac{\mu\sqrt r}{n}
       \|\hat X^{(s),t}-X\|_{2,\infty}.
\end{equation}
For each $(a,b,c)$, collecting the last factor in
\eqref{eq:app:gd-slice-gradient-expansion} across all $r$ columns
gives a matrix of Frobenius norm at most $C\lambda_{\max}^{2/3}\mu\sqrt r/n$.
There are at most $3n^2$ independent summands. Since
$|q^{-1}\omega_{abc}-1|\le q^{-1}$ and
$\mathbb E(q^{-1}\omega_{abc}-1)^2\le q^{-1}$, \eqref{eq:app:gd-slice-residual}
bounds their maximum Frobenius norm and sum of second moments by
\[
 \frac{C\lambda_{\max}^{4/3}\mu^2r}{qn^2}
       \|\hat X^{(s),t}-X\|_{2,\infty},\qquad
 \frac{C\lambda_{\max}^{8/3}\mu^4r^2}{qn^2}
       \|\hat X^{(s),t}-X\|_{2,\infty}^2,
\]
respectively. Apply Lemma~\ref{thm:(Matrix-Bernstein-inequality)}
after vectorizing each summand, with $a=13$ and dimension $nr\times1$.
With conditional failure probability at most $(nr+1)^{-13}\le n^{-13}$,
\begin{equation}\label{eq:app:gd-slice-gradient-bound}
 \|\nabla F(\hat X^{(s),t})-\nabla F^{(s)}(\hat X^{(s),t})\|_F
 \lesssim\lambda_{\max}^{4/3}\|\hat X^{(s),t}-X\|_{2,\infty}
 \left(\frac{\mu^2r\sqrt{\log n}}{n\sqrt q}
       +\frac{\mu^2r\log n}{qn^2}\right).
\end{equation}
Substituting the lower bound on $q$ from Corollary~\ref{cor:app:gd:init}
into the two coefficients in \eqref{eq:app:gd-slice-gradient-bound} gives
\begin{equation}\label{eq:app:gd-gradient-sampling-coefficients}
 \frac{\mu^2r\sqrt{\log n}}{n\sqrt q}
 \le\frac{1}{\sqrt{c_1}\kappa_{\mathrm{CP}}^3\sqrt\mu\,r\log^2 n},\qquad
 \frac{\mu^2r\log n}{qn^2}
 \le\frac{1}{c_1 \kappa_{\mathrm{CP}}^6\mu^3r^3\log^4 n}.
\end{equation}

\paragraph{The auxiliary update in row $s$.}
Only tensor entries with an index equal to $s$ contribute to row $s$
of the gradient. They have their population weights in $F^{(s)}$, so
its $s$th gradient row equals the $s$th row of the population
gradient. At $\hat X$, the $i$th column of the population gradient is
\[
 \sum_j\langle\hat x_j,\hat x_i\rangle^2\hat x_j
 -\sum_j\langle x_j,\hat x_i\rangle^2x_j.
\]
Separate each summand into its coefficient of $\hat x_j-x_j$ and its
coefficient of $x_j$. For $j=i$, these two terms are
\[
 \|\hat x_i\|^4(\hat x_i-x_i)
 +\bigl(\|\hat x_i\|^4-\langle x_i,\hat x_i\rangle^2\bigr)x_i.
\]
For $j\ne i$, the coefficients of
$(\hat x_j-x_j)_s$ form a symmetric matrix, whose norm is at most
its maximum absolute row sum:
\[
 \max_i\sum_{j\ne i}\langle\hat x_j,\hat x_i\rangle^2
 \le C\lambda_{\max}^{4/3}\frac{\mu^2r}{n^2}
 +C\lambda_{\max}^{2/3}\|\hat X-X\|_F^2
 +C\|\hat X-X\|_F^4.
\]
To obtain this bound, expand each inner product about $x_i,x_j$,
use \eqref{eq:cp_incoherent} for its unperturbed term, and sum the
squares of the other three terms using
\eqref{eq:app:gd-factor-operator-norm}. In the basin,
$\|\hat X-X\|_F\le c_3 \kappa_{\mathrm{CP}}^{-4/3}\lambda_{\max}^{1/3}$ and
$r\mu^2/n^2\le c_2^4 \kappa_{\mathrm{CP}}^{-8/3}$. These inequalities bound the row sum by
$C(c_2^4+c_3^2+c_3^4)\lambda_{\min}^{4/3}$.
For the diagonal, $\|\hat x_i-x_i\|_2\le c_3\lambda_{\min}^{1/3}$
gives $(1-c_3)\|x_i\|_2\le\|\hat x_i\|_2\le(1+c_3)\|x_i\|_2$.
For the remaining off-diagonal terms, Cauchy--Schwarz gives
\[
\begin{aligned}
 &\sum_i\left|\sum_{j\ne i}x_j(s)
       \langle\hat x_j-x_j,\hat x_i\rangle
       \langle\hat x_j+x_j,\hat x_i\rangle\right|^2\\
 &\quad\le\max_j|x_j(s)|^2
       \max_i\sum_j\langle\hat x_j+x_j,\hat x_i\rangle^2
       \sum_{i,j}\langle\hat x_j-x_j,\hat x_i\rangle^2
 \lesssim\lambda_{\max}^{8/3}\frac\mu n\|\hat X-X\|_F^2.
\end{aligned}
\]
Here \eqref{eq:cp_incoherent} bounds $\max_j|x_j(s)|^2$ by
$\lambda_{\max}^{2/3}\mu/n$. For the two sums of squares use
\[
\begin{aligned}
 \max_i\sum_j\langle\hat x_j+x_j,\hat x_i\rangle^2
 &\le\|\hat X+X\|^2\max_i\|\hat x_i\|_2^2\lesssim\lambda_{\max}^{4/3},\\
 \sum_{i,j}\langle\hat x_j-x_j,\hat x_i\rangle^2
 &=\|(\hat X-X)^\top\hat X\|_F^2
 \lesssim\lambda_{\max}^{2/3}\|\hat X-X\|_F^2,
\end{aligned}
\]
by \eqref{eq:app:gd-factor-operator-norm} and \eqref{eq:app:gd-basin-factor-bounds}.
The diagonal remainder has the same bound, since
\[
 |\|\hat x_i\|^4-\langle x_i,\hat x_i\rangle^2|
 \lesssim\lambda_{\max}\|\hat x_i-x_i\|_2,
 \qquad |x_i(s)|\le\lambda_{\max}^{1/3}\sqrt{\mu/n}.
\]
After subtracting the gradient, the coefficient matrix of the row
error is symmetric. For a sufficiently small basin constant its
maximum absolute row sum is bounded by
\[
 \max_i\left\{\big|1-\eta\|\hat x_i\|_2^4\big|
          +\eta\sum_{j\ne i}\langle\hat x_i,\hat x_j\rangle^2\right\}
 \le 1-\tfrac78\eta\lambda_{\min}^{4/3}
        +\tfrac18\eta\lambda_{\min}^{4/3}
 =1-\tfrac34\eta\lambda_{\min}^{4/3}.
\]
Here we used $\|\hat x_i\|_2^4\ge7\lambda_{\min}^{4/3}/8$,
$\eta\|\hat x_i\|_2^4\le1$, and the off-diagonal bound above.
Adding the diagonal and off-diagonal remainder bounds now gives
\begin{equation}\label{eq:app:gd-own-row-step}
 \|\bigl(\hat X^{(s),t+1}-X\bigr)_{s,:}\|_2
 \le(1-3\eta\lambda_{\min}^{4/3}/4)
       \|\bigl(\hat X^{(s),t}-X\bigr)_{s,:}\|_2
 +C\eta\lambda_{\max}^{4/3}\sqrt{\frac\mu n}
       \|\hat X^{(s),t}-X\|_F.
\end{equation}

\paragraph{Closing the induction.}
The inductive bounds \eqref{subeq:app:gd:recursion_cond_2} and
\eqref{subeq:app:gd:recursion_cond_3} give
\begin{equation}\label{eq:app:gd-auxiliary-row-envelope}
 \|\hat X^{(s),t}-X\|_{2,\infty}
 \le\frac{c}{\kappa_{\mathrm{CP}}^{2/3}\mu\sqrt r}\left(\frac12+\frac1C\right)
       \|X\|_{2,\infty}(1-\eta\lambda_{\min}^{4/3}/3)^t.
\end{equation}
Likewise, \eqref{subeq:app:gd:recursion_cond_1} and
\eqref{subeq:app:gd:recursion_cond_3}, together with
$\|X\|_F\le\sqrt n\|X\|_{2,\infty}$ and $\mu\le n$, imply
\begin{equation}\label{eq:app:gd-own-row-forcing}
 \sqrt{\frac\mu n}\|\hat X^{(s),t}-X\|_F
 \le\frac{2c}{C\kappa_{\mathrm{CP}}^2\mu\sqrt r}\|X\|_{2,\infty}
       (1-\eta\lambda_{\min}^{4/3}/3)^t.
\end{equation}
Insert \eqref{eq:app:gd-own-row-forcing} into \eqref{eq:app:gd-own-row-step}.
Since $\lambda_{\max}^{4/3}/\kappa_{\mathrm{CP}}^2
=\lambda_{\min}^{4/3}/\kappa_{\mathrm{CP}}^{2/3}$,
a sufficiently large $C$ makes the row error at time $t+1$ at most
\[
 \left(1-\tfrac34\eta\lambda_{\min}^{4/3}+\tfrac5{12}\eta\lambda_{\min}^{4/3}\right)
 \frac{c\|X\|_{2,\infty}}{4\kappa_{\mathrm{CP}}^{2/3}\mu\sqrt r}
 (1-\eta\lambda_{\min}^{4/3}/3)^t.
\]
Since $3/4-5/12=1/3$, this is \eqref{subeq:app:gd:recursion_cond_5} at $t+1$.
For the difference of the iterates, insert the full gradient step at
$\hat X^{(s),t}$. The contraction used in \eqref{eq:app:gd-full-frobenius-step}
applies to the segment between the two iterates, giving
\[
\begin{aligned}
 \|\hat X^{(s),t+1}-\hat X^{t+1}\|_F
 &\le(1-\eta\lambda_{\min}^{4/3}/2)
       \|\hat X^{(s),t}-\hat X^t\|_F\\
 &\quad+\eta\|\nabla F(\hat X^{(s),t})-
                    \nabla F^{(s)}(\hat X^{(s),t})\|_F.
\end{aligned}
\]
The sum of the two coefficients in \eqref{eq:app:gd-gradient-sampling-coefficients}
is at most $C'(c_1^{-1/2}+c_1^{-1})\kappa_{\mathrm{CP}}^{-3}$.
Substitute this and \eqref{eq:app:gd-auxiliary-row-envelope} into
\eqref{eq:app:gd-slice-gradient-bound}. Since
$\lambda_{\max}^{4/3}\kappa_{\mathrm{CP}}^{-3}
\le\lambda_{\min}^{4/3}\kappa_{\mathrm{CP}}^{-4/3}$,
increasing $c_1$ after fixing $C$ gives
\[
 \|\nabla F(\hat X^{(s),t})-\nabla F^{(s)}(\hat X^{(s),t})\|_F
 \le\frac{\lambda_{\min}^{4/3}}6
 \frac{c\|X\|_{2,\infty}}{C\kappa_{\mathrm{CP}}^2\mu\sqrt r}
 (1-\eta\lambda_{\min}^{4/3}/3)^t .
\]
Substitution into the difference update gives
\eqref{subeq:app:gd:recursion_cond_3} at $t+1$, since $1/2-1/6=1/3$.
Equation~\eqref{eq:app:gd-full-frobenius-step} gives
\eqref{subeq:app:gd:recursion_cond_1} at $t+1$, completing the induction.

For each fixed $s,t$, the auxiliary iterate is determined by the
initialization and $\Omega_{-s}$. Conditional on these data, whenever
the auxiliary iterate satisfies \eqref{eq:app:gd:convexity_region} and
\eqref{eq:app:gd-auxiliary-row-envelope}, the probability that
\eqref{eq:app:gd-slice-gradient-bound} fails is at most $n^{-13}$.
Averaging over $\Omega_{-s}$ and summing over $s\in[n]$ and the update
times $0\le t<\lceil192\kappa_{\mathrm{CP}}^{4/3}\log n\rceil$ bounds
the conditional probability of any such failure by
\begin{equation}\label{eq:app:gd-simultaneous-gradient-tail}
 n\lceil 192\kappa_{\mathrm{CP}}^{4/3}\log n\rceil n^{-13}\le Cn^{-11}.
\end{equation}
Here $q\le1$ and the sampling requirement of Corollary~\ref{cor:app:gd:init}
give $\kappa_{\mathrm{CP}}^6\le Cn^2$, so $\kappa_{\mathrm{CP}}^{4/3}\log n\le Cn$.
Combining \eqref{eq:app:gd-simultaneous-gradient-tail} with the
$O(n^{-10})$ failure probability in Lemma~\ref{lem:app:gd:convexity}
gives conditional probability $1-O(n^{-10})$ for
\eqref{eq:app:gd:recursion_cond} up to and including time
$\lceil192\kappa_{\mathrm{CP}}^{4/3}\log n\rceil$.
\endgroup
\end{proof}

\subsection{Completion of the proof}\label{sec:app:gd:completion}
\begingroup
\paragraph{Contraction after the initial iterations.}
Work on the event used in the proof of
Proposition~\ref{prop:app:gd:geometric_recursion}. On this event the estimates
\eqref{eq:app:gd:recursion_cond} hold through the last coupled iteration,
and the Hessian estimate \eqref{eq:app:gd-uniform-hessian-bounds} in
Lemma~\ref{lem:app:gd:convexity} holds at
\emph{every} point of \eqref{eq:app:gd:convexity_region} for the same
gradient-descent sample. In this paragraph, take
$t=\lceil192\kappa_{\mathrm{CP}}^{4/3}\log n\rceil$ and
$\epsilon=c/(\kappa_{\mathrm{CP}}^{4/3}\mu\sqrt r)$ as in that lemma.

At time $t$, \eqref{subeq:app:gd:recursion_cond_1} and
$\|X\|_F\le\sqrt n\,\|X\|_{2,\infty}$ give
\[
 \|\hat X^t-X\|_F
 \le \epsilon\|X\|_{2,\infty}\,
      \frac{\sqrt n}{C\kappa_{\mathrm{CP}}^{2/3}\sqrt\mu}
      (1-\eta\lambda_{\min}^{4/3}/3)^t
 \le \frac{\epsilon}{4}\|X\|_{2,\infty}.
\]
For the last inequality, $\eta\lambda_{\min}^{4/3}=1/(8\kappa_{\mathrm{CP}}^{4/3})$,
so $(1-\eta\lambda_{\min}^{4/3}/3)^t\le
\exp\{-t/(24\kappa_{\mathrm{CP}}^{4/3})\}\le n^{-8}$.
At this radius, the Frobenius error alone controls both conditions in
\eqref{eq:app:gd:convexity_region}. Indeed, if
$\|\hat X-X\|_F\le\epsilon\|X\|_{2,\infty}/4$, then
\[
 \|\hat X-X\|_F\le\epsilon\|X\|_F,\qquad
 \|\hat X-X\|_{2,\infty}\le\|\hat X-X\|_F
 \le\kappa_{\mathrm{CP}}^{2/3}\epsilon\|X\|_{2,\infty}.
\]

We now induct on the number $j$ of further updates of the full iterate.
The bound at time $t$ gives the base case $j=0$. Suppose $\hat X^{t+j}$
lies in the Frobenius ball of radius $\epsilon\|X\|_{2,\infty}/4$ about $X$.
The entire segment from $X$ to $\hat X^{t+j}$ lies in this ball.
Apply \eqref{eq:app:gd-uniform-hessian-bounds} in Lemma~\ref{lem:app:gd:convexity}
at $\hat X=X+u(\hat X^{t+j}-X)$ for every $u\in[0,1]$. Averaging over $u$ gives
\[
 \tfrac12\lambda_{\min}^{4/3}I_{nr}
 \preceq\int_0^1\nabla^2 F\bigl(X+u(\hat X^{t+j}-X)\bigr)\,du
 \preceq 4\lambda_{\max}^{4/3}I_{nr}.
\]
Since $\nabla F(X)=0$, the gradient update and the fundamental theorem
of calculus along the segment give
\[
\begin{aligned}
 \operatorname{vec}(\hat X^{t+j+1}-X)
 &=\operatorname{vec}(\hat X^{t+j}-X)
   -\eta\,\operatorname{vec}\bigl(\nabla F(\hat X^{t+j})-\nabla F(X)\bigr)\\
 &=\left[I_{nr}-\eta\int_0^1
       \nabla^2F\bigl(X+u(\hat X^{t+j}-X)\bigr)\,du\right]
       \operatorname{vec}(\hat X^{t+j}-X).
\end{aligned}
\]
Since $4\eta\lambda_{\max}^{4/3}=1/2$, the eigenvalues of the matrix
in brackets lie in
$[1/2,\,1-\eta\lambda_{\min}^{4/3}/2]$. Its operator norm therefore gives
\eqref{eq:app:gd-full-frobenius-step} at time $t+j$,
\[
 \|\hat X^{t+j+1}-X\|_F
 \le (1-\eta\lambda_{\min}^{4/3}/2)\|\hat X^{t+j}-X\|_F
 \le \epsilon\|X\|_{2,\infty}/4.
\]
This proves the induction step. Iterating the first inequality gives
\begin{equation}\label{eq:app:gd-deterministic-continuation}
 \|\hat X^{t+j}-X\|_F
 \le (1-\eta\lambda_{\min}^{4/3}/2)^j\|\hat X^t-X\|_F
 \qquad\text{for every integer }j\ge0 .
\end{equation}
After time $t$ the argument uses only the full iterate and the uniform
Hessian bound. The probability estimate in
\eqref{eq:app:gd-simultaneous-gradient-tail} therefore needs only the
finitely many update times in Proposition~\ref{prop:app:gd:geometric_recursion}.

\paragraph{From factor errors to tensor errors.}
The cross-correlation
condition \eqref{eq:cp_incoherent} bounds the off-diagonal row sums
of the squared and cubed Gram matrices by
$\lambda_{\max}^{4/3}r\mu^2/n^2$ and $r(\mu/n)^3$, respectively.
The rank condition in Lemma~\ref{lem:app:gd:convexity} gives
$r\mu^2/n^2\le c_2^4 \kappa_{\mathrm{CP}}^{-8/3}$, so
\begin{equation}\label{eq:app:gd-tensor-gram-bounds}
\begin{aligned}
 \lambda_{\min}\big((X^\top X)\circ(X^\top X)\big)
 &\ge\tfrac12\lambda_{\min}^{4/3},\qquad
 \big\|(X^\top X)\circ(X^\top X)\big\|
 \le2\lambda_{\max}^{4/3},\\
 \lambda_{\min}\big((\langle\bar x_i,\bar x_j\rangle^3)_{i,j=1}^r\big)
 &\ge1-r(\mu/n)^3\ge\tfrac12.
\end{aligned}
\end{equation}
Expand $\hat T-T$ by substituting $\hat x_i=x_i+(\hat x_i-x_i)$.
For the term with a single error in the first mode,
\[
\begin{aligned}
 \left\|\sum_i(\hat x_i-x_i)\otimes x_i\otimes x_i\right\|_F^2
 &=\sum_{i,j}\langle\hat x_i-x_i,\hat x_j-x_j\rangle
                 \langle x_i,x_j\rangle^2\\
 &\le\|\hat X-X\|_F^2\,\|(X^\top X)\circ(X^\top X)\|.
\end{aligned}
\]
Both Gram matrices are positive semidefinite. The trace of their product
is at most $\|\hat X-X\|_F^2$ times the operator norm of the squared Gram
matrix in \eqref{eq:app:gd-tensor-gram-bounds}. By symmetry, the other two
linear terms obey the same Frobenius bound.
Unfolding in the mode containing $X$ bounds each two-error term by
$\|X\|\|\hat X-X\|_F^2$, and the three-error term is at most
$\|\hat X-X\|_F^3$. Hence
\begin{equation}\label{eq:app:gd-tensor-frobenius-conversion}
\begin{aligned}
 \|\hat T-T\|_F
 &\le3\big\|(X^\top X)\circ(X^\top X)\big\|^{1/2}\|\hat X-X\|_F
       +3\|X\|\|\hat X-X\|_F^2+\|\hat X-X\|_F^3\\
 &\lesssim\lambda_{\max}^{2/3}\|\hat X-X\|_F,
\end{aligned}
\end{equation}
where the last step uses \eqref{eq:app:gd-tensor-gram-bounds},
\eqref{eq:app:gd-factor-operator-norm}, and
$\|\hat X-X\|_F\lesssim\lambda_{\max}^{1/3}$ from
\eqref{eq:app:gd:convexity_region}. The cubed Gram bound in
\eqref{eq:app:gd-tensor-gram-bounds} also gives
\begin{equation}\label{eq:app:gd-tensor-frobenius-lower}
 \|T\|_F^2\ge\tfrac12\sum_i\lambda_i^2
 \ge\tfrac12\lambda_{\min}^{4/3}\|X\|_F^2.
\end{equation}

For the entrywise bound, choose $s$ with $\|X_{s,:}\|_2=\|X\|_{2,\infty}$.
The Frobenius norm of the $s$th slice satisfies
\[
 n^2\|T\|_\infty^2\ge\|T_{s,:,:}\|_F^2
 =X_{s,:}\big((X^\top X)\circ(X^\top X)\big)X_{s,:}^\top
 \ge\tfrac12\lambda_{\min}^{4/3}\|X\|_{2,\infty}^2
\]
by \eqref{eq:app:gd-tensor-gram-bounds}. Taking the square root gives
\begin{equation}\label{eq:app:gd-tensor-infinity-lower}
 \|T\|_\infty\ge\frac{\lambda_{\min}^{2/3}}{\sqrt2\,n}\|X\|_{2,\infty}.
\end{equation}
For the numerator, telescope $\hat x_i^{\otimes3}-x_i^{\otimes3}$
one mode at a time. For the first of the three terms, Cauchy--Schwarz over $i$ gives
\[
 \left|\sum_i(\hat x_i(a)-x_i(a))\hat x_i(b)\hat x_i(c)\right|
 \le\|\hat X-X\|_{2,\infty}
       \max_i|\hat x_i(b)|\,\|\hat X_{c,:}\|_2.
\]
In each of the three terms, the nondifference factors belong to $X$
or $\hat X$. Their maximum entries and row norms are bounded by
$C\lambda_{\max}^{1/3}\sqrt{\mu/n}$ and
$C\|X\|_{2,\infty}$, respectively, using
\eqref{eq:app:gd-basin-factor-bounds}. Thus
\begin{equation}\label{eq:app:gd-tensor-infinity-conversion}
 \|\hat T-T\|_\infty
 \lesssim\lambda_{\max}^{1/3}\sqrt{\frac\mu n}
       \|X\|_{2,\infty}\|\hat X-X\|_{2,\infty}.
\end{equation}
Dividing \eqref{eq:app:gd-tensor-frobenius-conversion} by the square
root of \eqref{eq:app:gd-tensor-frobenius-lower}, and
\eqref{eq:app:gd-tensor-infinity-conversion} by
\eqref{eq:app:gd-tensor-infinity-lower}, gives
\begin{equation}\label{eq:app:gd-relative-tensor-errors}
 \frac{\|\hat T-T\|_F}{\|T\|_F}
 \lesssim \kappa_{\mathrm{CP}}^{2/3}\frac{\|\hat X-X\|_F}{\|X\|_F},\qquad
 \frac{\|\hat T-T\|_\infty}{\|T\|_\infty}
 \lesssim \kappa_{\mathrm{CP}}^{2/3}\mu\sqrt r\,
       \frac{\|\hat X-X\|_{2,\infty}}{\|X\|_{2,\infty}}.
\end{equation}
For the second inequality we used
$\|X\|_{2,\infty}\le\lambda_{\max}^{1/3}\sqrt{\mu r/n}$ and $\lambda_{\max}^{2/3}/\lambda_{\min}^{2/3}=\kappa_{\mathrm{CP}}^{2/3}$.

\paragraph{Geometric tensor error for all iterations.}
Recall that $\rho=1-\eta\lambda_{\min}^{4/3}/4$ in Theorem~\ref{thm:cai_sparse}.
For $0\le t\le\lceil192\kappa_{\mathrm{CP}}^{4/3}\log n\rceil$,
substitute \eqref{subeq:app:gd:recursion_cond_1} and
\eqref{subeq:app:gd:recursion_cond_2} into
\eqref{eq:app:gd-relative-tensor-errors}. Both relative tensor errors
are bounded by a constant times $(1-\eta\lambda_{\min}^{4/3}/3)^t\le\rho^t$.

For the remaining iterations, take $t=\lceil192\kappa_{\mathrm{CP}}^{4/3}\log n\rceil$
and let $j\ge0$. Combining \eqref{eq:app:gd-deterministic-continuation}
with \eqref{subeq:app:gd:recursion_cond_1} gives
\[
 \|\hat X^{t+j}-X\|_F
 \le \frac{c\|X\|_F}{C\kappa_{\mathrm{CP}}^2\mu^{3/2}\sqrt r}
       (1-\eta\lambda_{\min}^{4/3}/3)^t
       (1-\eta\lambda_{\min}^{4/3}/2)^j .
\]
Both geometric factors have bases at most $\rho$, so the Frobenius
estimate in \eqref{eq:app:gd-relative-tensor-errors} gives a relative
tensor error of order $\rho^{t+j}$. For the entrywise estimate, use
$\|\hat X^{t+j}-X\|_{2,\infty}\le\|\hat X^{t+j}-X\|_F$.
Since $\|X\|_F/\|X\|_{2,\infty}\le\sqrt n$, the second inequality in
\eqref{eq:app:gd-relative-tensor-errors} gives
\[
 \frac{\|\hat T^{t+j}-T\|_\infty}{\|T\|_\infty}
 \lesssim \frac{\sqrt n}{\kappa_{\mathrm{CP}}^{4/3}\sqrt\mu}
       (1-\eta\lambda_{\min}^{4/3}/3)^t
       (1-\eta\lambda_{\min}^{4/3}/2)^j .
\]
The choice of $t$ makes the decay during the first $t$ updates sufficient to absorb $\sqrt n$, since
\[
 \left(\frac{1-\eta\lambda_{\min}^{4/3}/3}{\rho}\right)^t
 =\left(1-\frac{\eta\lambda_{\min}^{4/3}}{12\rho}\right)^t
 \le \exp(-t\eta\lambda_{\min}^{4/3}/12)\le n^{-2}.
\]
Thus the relative entrywise error is at most
$C\sqrt n\,n^{-2}\rho^t\rho^j\le C\rho^{t+j}$ for every $j\ge0$.
Finally, the initialization failure probability from
Corollary~\ref{cor:app:gd:init} is at most $\delta/2$, and the
conditional failure probability in
Proposition~\ref{prop:app:gd:geometric_recursion}, which includes the
uniform Hessian event, is $O(n^{-10})$. Their sum is at most $\delta$
for all sufficiently large $n$, proving Theorem~\ref{thm:cai_sparse}.
\endgroup

\section[Proof of Theorem~\ref{thm:symmetric_tensor_concentration}]{Proof of Theorem~\ref{thm:symmetric_tensor_concentration}}\label{sec:app:proof_tensor_concentration}

We prove a rectangular extension of
Theorem~\ref{thm:symmetric_tensor_concentration}. The test vectors
have Euclidean norm constraints in two modes and coordinatewise
constraints in the remaining modes. Define
\[
 n_{\max}=\max_j n_j.
\]

\begin{theorem}\label{thm:delocalized_concentration}
Let $k\ge3$, let $T\in\mathbb R^{n_1\times\cdots\times n_k}$ be
deterministic, and let $\Omega$ sample entries independently with
probability $q\in(0,1]$.  For any positive $\delta_1,\ldots,\delta_k$,
with probability at least $1-Ck^2n_{\max}^{-12}$,
\begin{align}
 \big\|q^{-1}\mathcal P_\Omega(T)-T\big\|_\delta
 &\le C\sqrt{k}\,\|T\|_\infty
 \nonumber\\
 &\quad{}\times\max_{j_1\ne j_2}\Bigg\{
 \left(\prod_{j\notin\{j_1,j_2\}}\delta_j\right)
 \left[
 \sqrt{\frac{n_{\max}\prod_{j\notin\{j_1,j_2\}}n_j}{q}}
       \log^2(n_{\max})
 +\frac{\log^{5/2}(n_{\max})}{q}
 \right]\Bigg\}.
 \label{eq:inco_norm}
\end{align}
  In particular, when every dimension is
$n$ and $\delta_j=\sqrt{\mu/n}$,
\[
 \big\|q^{-1}\mathcal P_\Omega(T)-T\big\|_\delta
 \le C\sqrt{k}\left\{
 \mu^{(k-2)/2}\sqrt{\frac nq}\log^2(n)
 +\frac{\mu^{(k-2)/2}}{qn^{(k-2)/2}}\log^{5/2}(n)
 \right\}\|T\|_\infty.
\]
\end{theorem}

\begin{proof}
From homogeneity we may 
assume $\|T\|_\infty=1$. Fix the unrestricted modes to be $1,2$,
so that
\[
 \|u_1\|_2,\|u_2\|_2\le1,\qquad
 \|u_j\|_\infty\le\delta_j\quad(3\le j\le k).
\]
The other choices of unrestricted modes follow by permutation.
Throughout this proof, products over $j\ge3$ concern the restricted
modes.

%The proof has five steps. First, symmetrization splits the desired
%tail into the two terms in \eqref{eq:symmetrized_incoherent}; a
%scalar Bernstein estimate bounds the fixed-test term in
%\eqref{eq:fixed_test_incoherent}. Second,
%\eqref{eq:app:signed-support-decomposition} and
%\eqref{eq:app:restricted-endpoint-decomposition} reduce the remaining
%supremum to vectors specified by supports and signs. Third, the
%uniform occupancy bound \eqref{eq:app:localized_occupancy} controls
%how many sampled entries each support contains. Fourth, conditioning
%on these sampled locations gives the uniform Rademacher estimate
%\eqref{eq:block_conditional_rademacher}. The final step combines
%these estimates to obtain \eqref{eq:app:fixed-pair-incoherent-bound}
%and then takes a union over the unrestricted mode pairs.

\begingroup\color{black}
\paragraph{Step 1. Symmetrization and the fixed-test tail.}
Define $Z=q^{-1}\mathcal P_\Omega(T)-T$, and let $Z'$ be an
independent copy. Compactness and sign symmetry of
$\mathcal U_{1,2}(\delta)$ allow us to choose $U$ with
\[
 \langle Z,U\rangle
 =\sup_{V\in\mathcal U_{1,2}(\delta)}|\langle Z,V\rangle|.
\]
The choice is measurable because $Z$ takes finitely many values.
If this supremum is at least $3t$ and
$\sup_{V\in\mathcal U_{1,2}(\delta)}|\langle Z-Z',V\rangle|<2t$, then
$\langle Z',U\rangle>t$. Conditional on $Z$, the chosen $U$ is
fixed and $Z'$ is independent with the same distribution as $Z$,
so
\[
 \mathbb P\{\langle Z',U\rangle>t\mid Z\}
 \le\sup_{V\in\mathcal U_{1,2}(\delta)}
       \mathbb P\{\langle Z,V\rangle\ge t\}.
\]
Let $R$ have independent Rademacher entries, independently of
$Z,Z'$. Since the entries of $Z-Z'$ are independent and symmetric,
\[
 Z-Z'\overset{d}{=}R\odot(Z-Z')
 =R\odot(Z+T)-R\odot(Z'+T).
\]
Both terms on the right have the distribution of
$q^{-1}R\odot\mathcal P_\Omega(T)$. The triangle inequality and
union bound therefore give, for every $t>0$,
\begin{equation}\label{eq:symmetrized_incoherent}
\begin{aligned}
 &\mathbb P\left\{
 \sup_{U\in\mathcal U_{1,2}(\delta)}|\langle Z,U\rangle|\ge3t\right\}\\
 &\quad\le\sup_{U\in\mathcal U_{1,2}(\delta)}
 \mathbb P\{\langle Z,U\rangle\ge t\}
 +2\mathbb P\left\{
 \sup_{U\in\mathcal U_{1,2}(\delta)}
 |\langle q^{-1}R\odot\mathcal P_\Omega(T),U\rangle|\ge t\right\}.
\end{aligned}
\end{equation}

For a fixed $U=u_1\otimes\cdots\otimes u_k\in\mathcal U_{1,2}(\delta)$,
\[
 \langle Z,U\rangle
 =\sum_{i_1,\ldots,i_k}
 \left(q^{-1}\mathbbm1_{\{(i_1,\ldots,i_k)\in\Omega\}}-1\right)
 T_{i_1\cdots i_k}\prod_{j=1}^k u_j(i_j).
\]
The summands are independent and centered, and each has absolute
value at most $q^{-1}\prod_{j=3}^k\delta_j$. Since
$\|T\|_\infty=1$, their total variance is at most
\[
 \frac{1-q}{q}\sum_{i_1,\ldots,i_k}
 T_{i_1\cdots i_k}^2\prod_{j=1}^k u_j(i_j)^2
 \le q^{-1}\|U\|_F^2
 \le q^{-1}\prod_{j=3}^k(n_j\delta_j^2).
\]
Bernstein's inequality thus gives the uniform fixed-test bound
\begin{equation}\label{eq:fixed_test_incoherent}
 \sup_{U\in\mathcal U_{1,2}(\delta)}
 \mathbb P\{\langle Z,U\rangle\ge t\}
 \le2\exp\left\{-c\min\left(
 \frac{qt^2}{\prod_{j=3}^k(n_j\delta_j^2)},
 \frac{qt}{\prod_{j=3}^k\delta_j}\right)\right\}.
\end{equation}

\paragraph{Step 2. Reduction to signed supports.}
To apply a union bound, we reduce the continuous family of test vectors
to finitely many support and sign patterns. For every $\Omega$ and $R$,
the supremum of the absolute Rademacher form in
\eqref{eq:symmetrized_incoherent} is at most $C\log(n_{\max})$ times
its maximum over these patterns.

If the stated probability lower
bound is nonpositive there is nothing to prove; otherwise
$\log(n_{\max})\ge\log 2$. Let $u\in\mathbb R^d$ with
$1\le d\le n_{\max}$, and let $a_1\ge\cdots\ge a_d\ge0$ be its
coordinate magnitudes in decreasing order, breaking ties once and
for all. Set $a_{d+1}=0$. The ``first $s$'' positions are the
original positions of these $s$ largest coordinates; assign either
sign to zero coordinates. Then
\begin{equation}\label{eq:app:signed-support-decomposition}
 u=\sum_{s=1}^d\sqrt{s}(a_s-a_{s+1})
       \frac{\operatorname{sign}(u)\mathbbm1_{\text{first }s}}{\sqrt{s}}.
\end{equation}
At the original position ranked $i$, the sum contributes the
original sign multiplied by $\sum_{s=i}^d(a_s-a_{s+1})=a_i$,
which proves the identity coordinate by coordinate. Every support
vector has Euclidean norm one, and the coefficients are
nonnegative. Summation by parts and Cauchy--Schwarz give
\begin{equation}\label{eq:app:signed-support-coefficient-sum}
\begin{aligned}
 \sum_{s=1}^d\sqrt{s}(a_s-a_{s+1})
 &=a_1+\sum_{s=2}^d a_s(\sqrt{s}-\sqrt{s-1})\\
 &\le\|u\|_2\left(1+\frac14\sum_{s=1}^{d-1}\frac1s\right)^{1/2}
 \le C\sqrt{\log(n_{\max})}\,\|u\|_2,
\end{aligned}
\end{equation}
since $(\sqrt{s}-\sqrt{s-1})^2\le1/[4(s-1)]$ for $s\ge2$.
Applying \eqref{eq:app:signed-support-decomposition} to $u_1$ and $u_2$,
multilinearity and \eqref{eq:app:signed-support-coefficient-sum}
give nonnegative coefficients whose total is at most
$C\log(n_{\max})\|u_1\|_2\|u_2\|_2\le C\log(n_{\max})$.
The triangle inequality therefore reduces the absolute supremum,
with this factor, to $u_j(i)=\pm|A_j|^{-1/2}$ on nonempty
$A_j\subseteq[n_j]$ and $u_j(i)=0$ for  $i\not\in A_j$, for $j=1,2$.
Here the reduced supremum ranges over \emph{every} pair of nonempty
supports $A_1\subseteq[n_1]$, $A_2\subseteq[n_2]$, and every choice
of signs on those supports. 

For each restricted coordinate, use the identity
\begin{equation}\label{eq:app:restricted-endpoint-decomposition}
 u_j(i)=\frac{\delta_j+u_j(i)}{2\delta_j}\,\delta_j
       +\frac{\delta_j-u_j(i)}{2\delta_j}\,(-\delta_j).
\end{equation}
Its coefficients are nonnegative and sum to one. With every other
coordinate fixed, multilinearity expresses the form as the same
convex combination of its endpoint values. Its absolute value is
therefore at most the larger endpoint absolute value.
Choosing endpoints successively gives $u_j(i)=\pm\delta_j$ for
all $j\ge3$ and $i\in[n_j]$.

\paragraph{Step 3: Control sampled supports uniformly.}
Project $\Omega$ onto the first two modes. In the resulting bipartite
graph, the neighbors of $i_1$ are the distinct indices $i_2$ for which
at least one sampled entry begins with $(i_1,i_2)$.
For a rectangle $Q=A_1\times A_2\times\prod_{j=3}^k[n_j]$,
$|\Omega\cap Q|$ counts all sampled entries whose first two indices
lie in $A_1\times A_2$.  Thus the degree of $i_1$ is at most
\[
 \sum_{i_2=1}^{n_2}\cdots\sum_{i_k=1}^{n_k}
 \mathbbm1_{\{(i_1,\ldots,i_k)\in\Omega\}},
\]
which has distribution $\operatorname{Bin}(n_2\prod_{j=3}^k n_j,q)$.
Choose the deterministic threshold
\begin{equation}\label{eq:app:projected-degree-threshold}
 \nu:=\left\lceil C\left\{
 qn_{\max}\prod_{j=3}^k n_j+\log(n_{\max})\right\}\right\rceil.
\end{equation}
Chernoff's inequality gives
\begin{equation}\label{eq:app:projected-degree-tail}
\begin{aligned}
 \mathbb P\left\{\operatorname{Bin}\left(n_2\prod_{j=3}^k n_j,q\right)>\nu\right\}
 &\le\exp\left\{(e-1)qn_{\max}\prod_{j=3}^k n_j-\nu\right\}\le n_{\max}^{-16}.
\end{aligned}
\end{equation}
The same estimate holds for a fixed right vertex, with $n_1$ in
place of $n_2$. A union bound over at most $2n_{\max}$ vertices
shows that every degree is at most $\nu$, with failure probability
at most $2n_{\max}^{-15}$.

On the simultaneous degree event established by
\eqref{eq:app:projected-degree-tail} and the vertex union bound, take
$Q=A_1\times A_2\times\prod_{j=3}^k[n_j]$. Before deleting any vertices, define
\[
 s:=|A_1|=\#\{i_1\in[n_1]:u_1(i_1)\ne0\},\qquad
 t:=|A_2|=\#\{i_2\in[n_2]:u_2(i_2)\ne0\}.
\]
If $\Omega\cap Q$ is empty, the form is zero. Otherwise, define
\[
 A_j' := \{i_j : (i_1,\ldots,i_k)\in\Omega\cap Q\}\subseteq A_j,
 \qquad j=1,2.
\]
These are the supports that remain after deleting the isolated vertices
from the graph induced by $A_1\times A_2$.
Every remaining vertex has a neighbor. Counting edges from either
side, using the degree bound on the other, gives
\begin{equation}\label{eq:app:trimmed-aspect-ratios}
 |A_1'|\le\nu|A_2'|,\qquad |A_2'|\le\nu|A_1'|.
\end{equation}
Deleting isolated indices removes no term from the sum over $\Omega\cap Q$
and preserves the sign of every term. Renormalizing the two support
vectors therefore changes the denominator from $\sqrt{st}$ to
$\sqrt{|A_1'||A_2'|}$ and multiplies the absolute value of the form by
$\sqrt{st/(|A_1'||A_2'|)}\ge1$. We can therefore restrict to normalized
signed rectangles satisfying \eqref{eq:app:trimmed-aspect-ratios}.

Fix $1\le\ell\le n_1n_2$. To cover every possible trimmed pair
$(A_1',A_2')$ with $|A_1'||A_2'|=\ell$, we count all ordered pairs of supports
whose sizes $s,t$ satisfy $st=\ell$, $s\le\nu t$, and $t\le\nu s$,
counting each assignment of signs separately. This family is determined
by $\ell$ and $\nu$ before sampling. The three constraints give
$s,t\le\sqrt{\nu\ell}$. The sum below counts these signed support pairs.
\begin{equation}\label{eq:app:aspect-ratio-rectangle-count}
\begin{aligned}
 &\sum_{\substack{1\le s\le n_1,\ 1\le t\le n_2\\
                 st=\ell,\ s\le\nu t,\ t\le\nu s}}
       \binom{n_1}{s}\binom{n_2}{t}2^{s+t}\\
 &\qquad\le n_{\max}^2(2n_{\max})^{2\sqrt{\nu\ell}}
 \le\exp\{C\sqrt{\nu\ell}\log(n_{\max})\}.
\end{aligned}
\end{equation}
Here the first inequality uses
$\binom{n_1}{s}\binom{n_2}{t}\le n_{\max}^{s+t}$ and at most
$n_{\max}^2$ possible pairs $(s,t)$. Independently, each coordinate
of a signed support is absent, positive, or negative, so all free
signed supports number at most
\begin{equation}\label{eq:app:all-rectangle-count}
 3^{n_1+n_2}\le\exp\{2n_{\max}\log(3)\}.
\end{equation}
Taking the smaller of \eqref{eq:app:aspect-ratio-rectangle-count}
and \eqref{eq:app:all-rectangle-count} bounds the number of signed
support pairs with $st=\ell$ by $e^{H/4}$, where
\begin{equation}\label{eq:app:trimmed_code_length}
 H:=C\min\left\{n_{\max},\sqrt{\nu\ell}\log(n_{\max})\right\}.
\end{equation}
Choose the constant large enough for both counting bounds and for
$H\ge40\log(n_{\max})$. The latter follows from $\nu,\ell\ge1$
and $n_{\max}\ge\log(n_{\max})$. We will also use
\begin{equation}\label{eq:app:normalized_trimmed_budget}
 H\le Cn_{\max},\qquad
 \frac{H^2}{\ell}\le C\nu\log^2(n_{\max}).
\end{equation}

For each fixed rectangle in this deterministic family,
$|\Omega\cap Q|$ has distribution
$\operatorname{Bin}(\ell\prod_{j=3}^k n_j,q)$. Applying Chernoff's
inequality to this count gives
\begin{align*}
 &\mathbb P\left\{
 |\Omega\cap Q|>C\left(q\ell\prod_{j=3}^k n_j+H\right)\right\}\\
 &\qquad\le\exp\left\{(e-1)q\ell\prod_{j=3}^k n_j
       -C\left(q\ell\prod_{j=3}^k n_j+H\right)\right\}\le e^{-H}.
\end{align*}
For a fixed pair of supports of sizes $s,t$, the factor $2^{s+t}$ in
\eqref{eq:app:aspect-ratio-rectangle-count} counts the two possible signs
of each of the $s$ nonzero coordinates of $u_1$ and each of the $t$
nonzero coordinates of $u_2$. All these assignments give the same
rectangle $Q$ and the same count $|\Omega\cap Q|$. Thus the bound
$e^{H/4}$ for signed support pairs from \eqref{eq:app:trimmed_code_length}
also bounds the number of distinct rectangles with $|A_1||A_2|=\ell$.
The union bound over these rectangles gives failure probability at most
$e^{H/4}e^{-H}=e^{-3H/4}\le n_{\max}^{-30}$. Summing over the integers
$1\le\ell\le n_1n_2\le n_{\max}^2$ shows that, with probability at least
$1-n_{\max}^{-28}$, the bound
\begin{equation}\label{eq:app:localized_occupancy}
 |\Omega\cap Q|\le C\left(q\ell\prod_{j=3}^k n_j+H\right)
\end{equation}
holds simultaneously for every rectangle
$Q=A_1\times A_2\times\prod_{j=3}^k[n_j]$ with nonempty
$A_1\subseteq[n_1]$ and $A_2\subseteq[n_2]$ satisfying
$|A_1|\le\nu|A_2|$ and $|A_2|\le\nu|A_1|$.

For each such $Q$, \eqref{eq:app:trimmed_code_length} defines $H$
with $\ell=|A_1||A_2|$. On the degree event
from \eqref{eq:app:projected-degree-tail}, the trimmed supports
$A_1',A_2'$ satisfy \eqref{eq:app:trimmed-aspect-ratios}, so
\eqref{eq:app:localized_occupancy} also applies to the rectangle
$A_1'\times A_2'\times\prod_{j=3}^k[n_j]$ whenever it is nonempty.
Let $\mathcal E_{\mathrm{ent}}$ be the
intersection of the simultaneous degree event from
\eqref{eq:app:projected-degree-tail} and the uniform occupancy event
\eqref{eq:app:localized_occupancy}. Then
\begin{equation}\label{eq:app:entropy-event-failure}
 \mathbb P(\mathcal E_{\mathrm{ent}}^c)
 \le2n_{\max}^{-15}+n_{\max}^{-28}\le Cn_{\max}^{-12}.
\end{equation}

\paragraph{Step 4. A uniform bound over supports and signs.}
Fix $\Omega\in\mathcal E_{\mathrm{ent}}$ and a signed rectangle $Q$
from the family counted in \eqref{eq:app:aspect-ratio-rectangle-count}.
Let $s=|A_1|$ and $t=|A_2|$. For each choice of signs in modes
$3,\ldots,k$, fix $U=u_1\otimes\cdots\otimes u_k$. Only $R$ is
random in the Rademacher sum
\begin{equation}\label{eq:app:rectangle-rademacher-sum}
 \langle R\odot\mathcal P_\Omega(T),U\rangle
 =\sum_{(i_1,\ldots,i_k)\in\Omega\cap Q}
 R_{i_1\cdots i_k}T_{i_1\cdots i_k}\prod_{j=1}^k u_j(i_j).
\end{equation}
If $\Omega\cap Q=\varnothing$, the sum in
\eqref{eq:app:rectangle-rademacher-sum} is empty and equals zero for
every choice of signs. Hence assume $|\Omega\cap Q|>0$.

In mode $j\ge3$, only coordinates appearing in $\Omega\cap Q$
affect the sum in \eqref{eq:app:rectangle-rademacher-sum}. There are at most
$\min\{n_j,|\Omega\cap Q|\}$ such coordinates. Hence the number
of relevant restricted-mode sign choices is at most
\begin{equation}\label{eq:app:restricted-sign-trace-count}
 \prod_{j=3}^k2^{\min\{n_j,|\Omega\cap Q|\}}
 \le2^{(k-2)\min\{n_{\max},|\Omega\cap Q|\}}.
\end{equation}
Recall that $\ell=|A_1||A_2|=st$. On $A_1\times A_2$, the normalized
vectors satisfy $|u_1(i_1)|=s^{-1/2}$ and $|u_2(i_2)|=t^{-1/2}$, so
$u_1(i_1)^2u_2(i_2)^2=1/(st)=1/\ell$. Since $|u_j(i_j)|=\delta_j$
for $j\ge3$ and $\|T\|_\infty=1$, summing the squared coefficients
in \eqref{eq:app:rectangle-rademacher-sum} gives
\begin{equation}\label{eq:app:conditional-rademacher-variance}
 \sum_{(i_1,\ldots,i_k)\in\Omega\cap Q}
 T_{i_1\cdots i_k}^2\prod_{j=1}^k u_j(i_j)^2
 \le\left(\prod_{j=3}^k\delta_j^2\right)
       \frac{|\Omega\cap Q|}{\ell}.
\end{equation}

To take the union bound over the supports and signs, we use
\begin{equation}\label{eq:observed_block_energy}
 \frac{|\Omega\cap Q|}{\ell}
 \left(H+\min\{n_{\max},|\Omega\cap Q|\}\right)
 \le C\nu\log^2(n_{\max}).
\end{equation}
If $q\ell\prod_{j=3}^k n_j\ge H$, then \eqref{eq:app:localized_occupancy}
gives $|\Omega\cap Q|/\ell\le Cq\prod_{j=3}^k n_j$. Since
$H\le Cn_{\max}$ by \eqref{eq:app:normalized_trimmed_budget}, the left-hand
side of \eqref{eq:observed_block_energy} is at most
$Cqn_{\max}\prod_{j=3}^k n_j\le C\nu$ by \eqref{eq:app:projected-degree-threshold}.
If $q\ell\prod_{j=3}^k n_j<H$, then \eqref{eq:app:localized_occupancy}
gives $|\Omega\cap Q|\le CH$. Bounding the minimum by $|\Omega\cap Q|$
now gives $CH^2/\ell\le C\nu\log^2(n_{\max})$ by
\eqref{eq:app:normalized_trimmed_budget}.

Hoeffding's inequality, \eqref{eq:app:conditional-rademacher-variance},
and \eqref{eq:observed_block_energy} now give
\begin{equation}\label{eq:app:conditional-single-trace-tail}
\begin{aligned}
 &\mathbb P\left\{
 |\langle R\odot\mathcal P_\Omega(T),U\rangle|
 >C\sqrt{k}\left(\prod_{j=3}^k\delta_j\right)
       \sqrt\nu\log(n_{\max})\,\middle|\,\Omega\right\}\\
 &\quad\le2\exp\left\{-cC^2k
       \left(H+\min\{n_{\max},|\Omega\cap Q|\}\right)\right\}
 \le2e^{-2H-2(k-2)\log(2)\min\{n_{\max},|\Omega\cap Q|\}}.
\end{aligned}
\end{equation}
Here the threshold constant is chosen so that $cC^2\ge4$;
then $cC^2k\ge2$ and $cC^2k\ge2(k-2)\log 2$ for every $k\ge3$.
Multiplying \eqref{eq:app:conditional-single-trace-tail} by the
count in \eqref{eq:app:restricted-sign-trace-count} bounds the
failure probability by $2e^{-2H}$ for each signed rectangle.

For fixed $\ell$, \eqref{eq:app:trimmed_code_length} bounds the number
of such signed rectangles with $st=\ell$ by $e^{H/4}$, so their total
failure probability is at most $2e^{-7H/4}$. Since $H\ge40\log(n_{\max})$
for each integer $1\le\ell\le n_1n_2\le n_{\max}^2$, the conditional failure
probability over all supports and signs is therefore at most
\begin{equation}\label{eq:app:conditional-all-traces-tail}
 2n_{\max}^2e^{-70\log(n_{\max})}=2n_{\max}^{-68}.
\end{equation}
Outside the failure event bounded in
\eqref{eq:app:conditional-all-traces-tail}, simultaneously for all
the signed rectangles from \eqref{eq:app:trimmed-aspect-ratios} and all
restricted-mode signs,
\begin{equation}\label{eq:block_conditional_rademacher}
 \left|\left\langle R\odot\mathcal P_\Omega(T),
 u_1\otimes\cdots\otimes u_k\right\rangle\right|
 \le C\sqrt{k}\left(\prod_{j=3}^k\delta_j\right)
       \sqrt\nu\log(n_{\max}).
\end{equation}

\paragraph{Step 5. Return to the original supremum.}
For the unrestricted modes $1,2$, choose
\begin{equation}\label{eq:app:final-incoherent-threshold}
 t=\frac{C\sqrt{k}}{q}\left(\prod_{j=3}^k\delta_j\right)
       \sqrt\nu\log^2(n_{\max}).
\end{equation}
In the signed-support decomposition \eqref{eq:app:signed-support-decomposition},
the coefficients for the first two modes together sum to at most
$C\log(n_{\max})$ by \eqref{eq:app:signed-support-coefficient-sum}. The endpoint
reduction \eqref{eq:app:restricted-endpoint-decomposition} and the
trimming argument at \eqref{eq:app:trimmed-aspect-ratios} allow us
to apply \eqref{eq:block_conditional_rademacher} to each resulting
term. Multiplying its right-hand side by $C\log(n_{\max})/q$ and
choosing $C$ in \eqref{eq:app:final-incoherent-threshold} at least
twice the resulting constant gives
\[
 \sup_{U\in\mathcal U_{1,2}(\delta)}
 |\langle q^{-1}R\odot\mathcal P_\Omega(T),U\rangle|\le t/2.
\]
By \eqref{eq:app:entropy-event-failure} and
\eqref{eq:app:conditional-all-traces-tail}, this estimate fails with
probability at most $2n_{\max}^{-15}+n_{\max}^{-28}+2n_{\max}^{-68}$.

For the fixed-test term in \eqref{eq:symmetrized_incoherent}, the
degree threshold \eqref{eq:app:projected-degree-threshold} gives
$\nu\ge qn_{\max}\prod_{j=3}^k n_j$ and $\nu\ge\log(n_{\max})$.
Thus the two quantities in Bernstein's bound \eqref{eq:fixed_test_incoherent}
satisfy
\[
 \frac{qt^2}{\prod_{j=3}^k(n_j\delta_j^2)}
 \ge C^2 k n_{\max}\log^4(n_{\max}),\qquad
 \frac{qt}{\prod_{j=3}^k\delta_j}
 \ge C\sqrt{k}\log^{5/2}(n_{\max}).
\]
Increasing $C$ so that both lower bounds exceed
$14c^{-1}\log(n_{\max})$, with $c$ from \eqref{eq:fixed_test_incoherent},
makes the fixed-test probability at most $2n_{\max}^{-14}$.
The symmetrization inequality \eqref{eq:symmetrized_incoherent} now gives
\begin{equation}\label{eq:app:fixed-pair-incoherent-bound}
 \mathbb P\!\left\{\sup_{U\in\mathcal U_{1,2}(\delta)}
 |\langle Z,U\rangle|\ge3t\right\}\le Cn_{\max}^{-12}.
\end{equation}

To evaluate $t$, \eqref{eq:app:projected-degree-threshold} gives
\[
 \frac{\sqrt\nu}{q}\le C\left(
 \sqrt{\frac{n_{\max}\prod_{j=3}^k n_j}{q}}
 +\frac{\sqrt{\log(n_{\max})}}{q}\right).
\]
Substituting into \eqref{eq:app:final-incoherent-threshold} and restoring
$\|T\|_\infty$ gives the bound in \eqref{eq:inco_norm} for modes $1,2$.
Taking a union bound in \eqref{eq:app:fixed-pair-incoherent-bound} over
the $\binom{k}{2}$ choices of unrestricted modes gives failure probability
at most $Ck^2n_{\max}^{-12}$ for the incoherent norm.

When $n_j=n$ and $\delta_j=\sqrt{\mu/n}$, the two products over the
$k-2$ restricted modes in \eqref{eq:inco_norm} are
$(\mu/n)^{(k-2)/2}$ and $n^{k-2}$, giving the stated specialization.
\endgroup
\end{proof}

\section[Proof of Theorem~\ref{thm:noisy_spectral}]{Proof of {Theorem~\ref{thm:noisy_spectral}}}\label{sec:app:noisy_spectral}
\subsection{Concentration of \texorpdfstring{$\tilde Z$}{Z tilde}}
Set
\[
        \alpha_{n,k}:=1+\sqrt{\frac{k\log n}{n}}.
\]
We first show $\widetilde{Z}$ is concentrated around  $AA^\top+m\sigma^2 I$.    When $\sigma=0$, this recovers {Theorem~\ref{thm:wedge_concentration}} with $a=3$.
\begin{lemma}\label{lemma:noisy_concentration}
    \edit{For a sufficiently large universal constant \(C_0>0\),} with probability at least $1-O(n^{-3})$, when $p \geq \frac{C_0   \mu_1 \mu_2 r^2 \log(n)}{mn}$,
    \begin{align}
        \|\widetilde{Z}-AA^\top -m \sigma^2 I\| \lesssim \mathcal E \|A\|^2,
    \end{align}
    where $m=n^{k-1}$ and
    \begin{align}\label{eq:def_mathcalE}
        \mathcal E:=&
        \sqrt{\frac{r^2\mu_1\mu_2\log n}{pmn}}
        +\frac{\sigma}{\|A\|} \sqrt{n}
        +\frac{\sigma^2}{\|A\|^2}\sqrt{mn}\\
        &+\sqrt{\frac{\log n}{p} } \left(\frac{\sigma}{\|A\|}(\alpha_{n,k}\sqrt{\mu_1r}+\sqrt{\mu_2 r}) +\frac{\sigma^2}{\|A\|^2}\alpha_{n,k}\sqrt{mn}\right)
        +\frac{k\log^2 n}{p} \cdot\frac{\sigma^2}{\|A\|^2}.
    \end{align}
    In particular, when $\sigma=0$, this recovers the bound in {Theorem~\ref{thm:wedge_concentration}}.
\end{lemma}

\begin{proof}
\begingroup
Condition on $G$, so $\widetilde A=A+G$ is fixed and the wedge
indicators remain independent. For $i<j$, the centered summand is
\[
 (\delta_{i\ell j}/p-1)\widetilde A_{i\ell}\widetilde A_{j\ell}
 (e_ie_j^\top+e_je_i^\top),
\]
and diagonal terms are counted once. The summand norm and variance
norm are bounded, respectively, by
\[
 \frac Cp\|\widetilde A\|_{\max}^2,\qquad
 \frac Cp\max_i\sum_{\ell,j}
       \widetilde A_{i\ell}^2\widetilde A_{j\ell}^2
 \le\frac Cp\|\widetilde A\|_{2,\infty}^2
               \|\widetilde A\|_{\infty,2}^2.
\]
Lemma~\ref{thm:(Matrix-Bernstein-inequality)}, with $a=4$ and
dimension $n\times n$, gives
\begin{equation}\label{eq:app:noisy-wedge-conditional-bernstein}
 \|\widetilde Z-\widetilde A\widetilde A^\top\|
 \le C\left(\frac{\log n}{p}\|\widetilde A\|_{\max}^2
 +\sqrt{\frac{\log n}{p}}\|\widetilde A\|_{2,\infty}
                             \|\widetilde A\|_{\infty,2}\right)
\end{equation}
with conditional failure probability at most $(2n)^{-4}$ for every
$G$, and therefore with the same unconditional failure probability.

Gaussian concentration inequalities give
\begin{equation}\label{eq:Gaussian_norm}
\begin{aligned}
 \|G\|_{\max}&\le C\sigma\sqrt{k\log n},\qquad
 \|G\|_{2,\infty}\le C\sigma\sqrt m,\\
 \|G\|_{\infty,2}&\le C\sigma(\sqrt n+\sqrt{k\log n})
                  =C\sigma\sqrt n\,\alpha_{n,k}
\end{aligned}
\end{equation}
with failure probability $O(n^{-3})$. Since $m\ge n$, these thresholds imply
\eqref{eq:Gaussian_norm}.
Substitute \eqref{eq:Gaussian_norm}, \eqref{eq:2-infty_A}, and
\eqref{eq:infty_norm} into
\eqref{eq:app:noisy-wedge-conditional-bernstein}, bounding the square of the maximum-entry sum by twice the sum of
the squares, to obtain
\begin{align}
 \|\widetilde Z-\widetilde A\widetilde A^\top\|
 &\lesssim\frac{\log n}{p}
     \left(\frac{r^2\mu_1\mu_2}{mn}\|A\|^2+\sigma^2k\log n\right)
 \nonumber\\
 &\quad+\sqrt{\frac{\log n}{p}}
 \left[\sqrt{\frac{r^2\mu_1\mu_2}{mn}}\|A\|^2
 +\sigma(\alpha_{n,k}\sqrt{\mu_1r}+\sqrt{\mu_2r})\|A\|
 +\sigma^2\alpha_{n,k}\sqrt{mn}\right]\\
 &=:\mathcal E_1\|A\|^2.\label{eq:eps_1}
\end{align}

For the conditional mean, use
\[
 \widetilde A\widetilde A^\top-AA^\top-m\sigma^2I
 =AG^\top+GA^\top+(GG^\top-m\sigma^2I).
\]
Using the SVD $A=U\Sigma V^\top$, the matrix $GV\in\mathbb R^{n\times r}$
has independent $N(0,\sigma^2)$ entries. Gaussian matrix concentration
therefore gives
\begin{equation}\label{eq:app:noisy-wedge-gaussian-cross-tail}
 \mathbb P\!\left\{\|GV\|>
 \sigma(\sqrt n+\sqrt r+\sqrt{8\log n})\right\}\le n^{-4}.
\end{equation}
Since $AG^\top=U\Sigma(GV)^\top$ and $r\le n$,
\eqref{eq:app:noisy-wedge-gaussian-cross-tail} implies
\begin{equation}\label{eq:cross-term}
 \|AG^\top+GA^\top\|\le2\|A\|\|GV\|
 \le C\sigma\|A\|\sqrt n.
\end{equation}
The Gaussian singular-value bounds \citep{vershynin2018high}, with
threshold $\sqrt{8\log n}$, also give
\begin{equation}\label{eq:wishart-term}
\begin{aligned}
 \|GG^\top-m\sigma^2I\|
 &\le\sigma^2\left[2\sqrt m(\sqrt n+\sqrt{8\log n})
                      +(\sqrt n+\sqrt{8\log n})^2\right]
 \le C\sigma^2\sqrt{mn}
\end{aligned}
\end{equation}
with failure probability at most $2n^{-4}$; here $m\ge n$ and
$\log n\le n$. Combining \eqref{eq:cross-term} and
\eqref{eq:wishart-term} yields
\begin{equation}\label{eq:eps2}
 \|\widetilde A\widetilde A^\top-AA^\top-m\sigma^2I\|
 \lesssim\sigma\|A\|\sqrt n+\sigma^2\sqrt{mn}
 =:\mathcal E_2\|A\|^2.
\end{equation}
Finally, the sampling assumption gives
\[
 \frac{r^2\mu_1\mu_2\log n}{pmn}\le C_0^{-1}\le1,
 \qquad
 \frac{r^2\mu_1\mu_2\log n}{pmn}
 \le\sqrt{\frac{r^2\mu_1\mu_2\log n}{pmn}}.
\]
Thus the linear signal term in \eqref{eq:eps_1} is absorbed by the
first term of \eqref{eq:def_mathcalE}; the other terms of
$\mathcal E_1+\mathcal E_2$ occur there directly. The triangle
inequality applied to \eqref{eq:eps_1} and \eqref{eq:eps2}, and a
union bound for their concentration events, prove the lemma with
failure probability $O(n^{-3})$.
\endgroup
\end{proof}

\subsection[Comparison with the noiseless wedge matrix]{{Comparison with the noiseless wedge matrix}}
{
We couple the noisy and noiseless wedge estimators using the same wedge
indicators. Thus $Z$ denotes the estimator formed from $A$, whereas
$\widetilde Z$ is formed from $A+G$. This coupling is used only in the
analysis and does not require any additional observations. Put
$\theta=\sigma/\|A\|$ and use
$\alpha_{n,k}=1+\sqrt{k\log n/n}$.

\begin{proposition}[Noisy perturbation bound under wedge sampling]
\label{prop:noisy_op_loo}
Assume that $A\in\mathbb R^{n\times m}$ has rank $r$, is
$(\mu_1,\mu_2)$-incoherent, and $m=n^{k-1}$ for $k\ge3$.
Denote the SVD by $A=U\Sigma V^\top$ and define
$\kappa=\sigma_1(A)/\sigma_r(A)$. Suppose
\[
 p\ge\frac{C_0\kappa^4\mu_1\mu_2r^2\log n}{mn}
\]
for a sufficiently large universal constant $C_0$, and
\begin{align}\label{eq:noise_upperbound}
 \theta\le c_0\kappa^{-2}\min\left\{
 \frac1{\alpha_{n,k}\sqrt n},\,
 \frac{\sqrt p}{\sqrt{r\log n}
       (\alpha_{n,k}\sqrt{\mu_1}+\sqrt{\mu_2})},\,
 \sqrt{\frac{p}{k(\log n)^2}},\,
 \alpha_{n,k}^{-1/2}
 \left(\frac{p}{n^k\log n}\right)^{1/4}\right\}
\end{align}
for a sufficiently small universal constant $c_0$. Let $\widehat U$
contain the leading $r$ eigenvectors of $\widetilde Z$, and let $R$ be
its Procrustes alignment to $U$. Define
\[
 E_0:=\sqrt{\frac{\mu_1\mu_2r^2\log n}{pmn}}
\]
and
\begin{align*}
 E_\sigma:={}&\theta\sqrt n+\theta^2\sqrt{mn}\\
 &+\sqrt{\frac{\log n}{p}}
 \left[\theta(\alpha_{n,k}\sqrt{\mu_1r}+\sqrt{\mu_2r})
       +\theta^2\alpha_{n,k}\sqrt{mn}\right]\\
 &+\frac{\log n}{p}\left[
 \theta\sqrt{\frac{k\mu_1\mu_2r^2\log n}{mn}}
       +\theta^2 k\log n\right].
\end{align*}
Then, with probability at least $1-O(n^{-3})$,
\begin{align}
 \|\widehat UR-U\|&\le C\kappa^2\mathcal E,
 \label{eq:noise_wedge_spectral_op}\\
 \|\widehat UR-U\|_{2,\infty}
 &\le C\kappa^2 E_0\|U\|_{2,\infty}
       +C\kappa^2 E_\sigma.
 \label{eq:noise_wedge_spectral_loo}
\end{align}
Here $\mathcal E$ is defined in \eqref{eq:def_mathcalE}; in particular,
$\mathcal E\le C(E_0+E_\sigma)$. The bounds \eqref{eq:noise_wedge_spectral_op}--\eqref{eq:noise_wedge_spectral_loo} hold simultaneously
for all leave-one-out matrices
\[
 \widetilde Z^{(s)}=AA^\top+m\sigma^2I+\Delta^{(s)},
 \qquad \Delta=\widetilde Z-AA^\top-m\sigma^2I,
\]
where $\Delta^{(s)}$ is obtained by zeroing row and column $s$ of
$\Delta$, with each eigenspace aligned separately to $U$.
\end{proposition}

\begin{lemma}[Noise-only wedge perturbation]\label{lem:noisy-lemma56}
Under the sampling and dimension assumptions of  Proposition~\ref{prop:noisy_op_loo}, with
probability $1-O(n^{-3})$,
\[
 \|\widetilde Z-Z-m\sigma^2I\|
 \le C\|A\|^2E_\sigma.
\]
\end{lemma}
\begin{proof}
\begingroup
Condition on $G$. The coupled estimators use the same independent
wedge indicators, so their centered difference is
\begin{align*}
 &\widetilde Z-Z-AG^\top-GA^\top-GG^\top\\
 &=\sum_{\substack{1\le i<j\le n\\\ell\in[m]}}
 \left(p^{-1}\mathbbm1_{\{(i,\ell,j)\in\widetilde{\mathcal W}\}}-1\right)
 \big(A_{i\ell}G_{j\ell}+G_{i\ell}A_{j\ell}+G_{i\ell}G_{j\ell}\big)
 (E_{ij}+E_{ji})\\
 &\quad+\sum_{i\in[n],\ell\in[m]}
 \left(p^{-1}\mathbbm1_{\{(i,\ell,i)\in\widetilde{\mathcal W}\}}-1\right)
 \big(2A_{i\ell}G_{i\ell}+G_{i\ell}^2\big)E_{ii},
\end{align*}
where $E_{ij}=e_ie_j^\top$ as in the proof of
Theorem~\ref{thm:wedge_concentration}. Apply matrix Bernstein
separately to the terms linear and quadratic in $G$; independence
between these two sums is unnecessary.

For the linear terms, each summand has norm at most
\[
 \frac2p\|A\|_{\max}\|G\|_{\max}.
\]
Since $(E_{ij}+E_{ji})^2=E_{ii}+E_{jj}$ for $i<j$, the variance
matrix is diagonal, and its norm satisfies
\begin{equation}\label{eq:app:noisy-wedge-mixed-variance}
\begin{aligned}
 \frac{1-p}{p}\max_i\sum_{j,\ell}
       (A_{i\ell}G_{j\ell}+G_{i\ell}A_{j\ell})^2
 &\le\frac2p\max_i\sum_\ell\left[
       A_{i\ell}^2\sum_jG_{j\ell}^2
       +G_{i\ell}^2\sum_jA_{j\ell}^2\right]\\
 &\le\frac2p\left(\|A\|_{2,\infty}^2\|G\|_{\infty,2}^2
       +\|G\|_{2,\infty}^2\|A\|_{\infty,2}^2\right).
\end{aligned}
\end{equation}
The terms with $j=i$ include the diagonal wedges.
For the quadratic terms, the summand bound is
$p^{-1}\|G\|_{\max}^2$, and the variance norm is at most
\begin{equation}\label{eq:app:noisy-wedge-quadratic-variance}
 \frac1p\max_i\sum_{j,\ell}G_{i\ell}^2G_{j\ell}^2
 =\frac1p\max_i\sum_\ell G_{i\ell}^2\sum_jG_{j\ell}^2
 \le\frac1p\|G\|_{2,\infty}^2\|G\|_{\infty,2}^2.
\end{equation}
Apply Lemma~\ref{thm:(Matrix-Bernstein-inequality)} to each
$n\times n$ sum with $a=4$. Substitute the incoherence bounds
\eqref{eq:2-infty_A}--\eqref{eq:infty_norm} and the Gaussian bounds
\eqref{eq:Gaussian_norm} into the summand bounds and
\eqref{eq:app:noisy-wedge-mixed-variance}--\eqref{eq:app:noisy-wedge-quadratic-variance}.
After division by $\|A\|^2$, the resulting estimate is
\begin{equation}\label{eq:app:noisy-wedge-centered-difference}
\begin{aligned}
 &\frac{\|\widetilde Z-Z-AG^\top-GA^\top-GG^\top\|}{\|A\|^2}\\
 &\quad\le C\theta\left[
       \sqrt{\frac{\log n}{p}}
       (\alpha_{n,k}\sqrt{\mu_1r}+\sqrt{\mu_2r})
       +\frac{\log n}{p}
        \sqrt{\frac{k\mu_1\mu_2r^2\log n}{mn}}\right]
 \\
 &\qquad+C\theta^2\left[
       \alpha_{n,k}\sqrt{\frac{mn\log n}{p}}
       +\frac{k(\log n)^2}{p}\right].
\end{aligned}
\end{equation}
The two brackets are the linear and quadratic contributions;
within each bracket, the variance term precedes the summand term.
For every $G$ satisfying \eqref{eq:Gaussian_norm}, the conditional
failure probability is at most $2(2n)^{-4}$. The Gaussian event
itself fails with probability $O(n^{-3})$, so
\eqref{eq:app:noisy-wedge-centered-difference} holds unconditionally
with failure probability $O(n^{-3})$.
Finally, \eqref{eq:eps2} bounds the conditional mean after subtracting
$m\sigma^2I$. Adding it to
\eqref{eq:app:noisy-wedge-centered-difference} supplies the first two
terms of $E_\sigma$ and proves the claim, with total failure
probability $O(n^{-3})$.
\endgroup
\end{proof}

\begin{proof}[Proof of Proposition~\ref{prop:noisy_op_loo}]
\begingroup
If $r=n$, every full-dimensional eigenspace equals that of $U$,
and its Procrustes alignment is exact. Assume $r<n$.

\paragraph{Perturbation size and spectral gaps.}
The sampling hypothesis gives $E_0\le C_0^{-1/2}\kappa^{-2}$.
The four noise bounds in \eqref{eq:noise_upperbound} give, respectively,
\begin{equation}\label{eq:app:noisy-wedge-small-components}
\begin{aligned}
 \kappa^2\theta\sqrt n
 &\le c_0\alpha_{n,k}^{-1}\le c_0,\qquad
 \kappa^2\theta\sqrt{\frac{\log n}{p}}
     (\alpha_{n,k}\sqrt{\mu_1r}+\sqrt{\mu_2r})\le c_0,\\
 \kappa^2\theta^2\frac{k(\log n)^2}{p}
 &\le c_0^2\kappa^{-2},\qquad
 \kappa^2\theta^2\alpha_{n,k}
       \sqrt{\frac{mn\log n}{p}}\le c_0^2\kappa^{-2}.
\end{aligned}
\end{equation}
Here $mn=n^k$. The last bound also controls the quadratic mean,
and the third controls the mixed increment:
\[
\begin{aligned}
 \theta^2\sqrt{mn}
 &=\theta^2\alpha_{n,k}\sqrt{\frac{mn\log n}{p}}
       \frac{\sqrt p}{\alpha_{n,k}\sqrt{\log n}}
 \le Cc_0^2\kappa^{-4},\\
 \kappa^2\theta\frac{\log n}{p}
       \sqrt{\frac{k\mu_1\mu_2r^2\log n}{mn}}
 &=\kappa^2E_0\theta\frac{\sqrt k\log n}{\sqrt p}
 \le c_0E_0.
\end{aligned}
\]
Thus $\kappa^2(E_0+E_\sigma)$ is arbitrarily small for large $C_0$
and small $c_0$. The definitions also give
$\mathcal E\le C(E_0+E_\sigma)$.
Theorem~\ref{thm:wedge_concentration} with $a=5$ and
Lemma~\ref{lemma:noisy_concentration} consequently imply
\begin{equation}\label{eq:app:noisy-wedge-gap-control}
 \|Z-AA^\top\|\le\tfrac18\sigma_r^2(A),\qquad
 \|\widetilde Z-m\sigma^2I-AA^\top\|
       \le\tfrac18\sigma_r^2(A).
\end{equation}
By Weyl's inequality, both centered matrices have $r$th eigenvalue
at least $7\sigma_r^2(A)/8$ and $(r+1)$st eigenvalue at most
$\sigma_r^2(A)/8$. Davis--Kahan and the Procrustes principal-angle
bound therefore give
\[
 \|\widehat UR-U\|
 \le C\frac{\|\widetilde Z-m\sigma^2I-AA^\top\|}{\sigma_r^2(A)}
 \le C\kappa^2\mathcal E,
\]
which proves \eqref{eq:noise_wedge_spectral_op}, and
\begin{equation}\label{eq:DK_tildeZ}
 \|\widehat U\widehat U^\top-UU^\top\|\le C\kappa^2\mathcal E.
\end{equation}
For each leave-one-out matrix, the centered perturbation is
$(I-e_se_s^\top)\Delta(I-e_se_s^\top)$, whose norm is at most
$\|\Delta\|$. The gap and operator conclusions therefore hold
simultaneously for all $s$, with no additional probability loss.

\paragraph{Rowwise comparison and alignment.}
Let $F$ be a basis of the leading $r$-dimensional eigenspace of the
noiseless matrix $Z$, Procrustes-aligned to $U$.
Theorem~\ref{thm:wedge_spectral_guarantees} with $a=5$, specifically
\eqref{eq:wedge_spectral_op} and \eqref{eq:wedge_spectral_loo}, gives
\begin{equation}\label{eq:app:noisy-clean-basis}
 \|F-U\|\le C\kappa^2E_0,\qquad
 \|F-U\|_{2,\infty}\le C\kappa^2E_0\|U\|_{2,\infty}.
\end{equation}
Using the gap in \eqref{eq:app:noisy-wedge-gap-control},
Davis--Kahan and Lemma~\ref{lem:noisy-lemma56} give an orthogonal $O$
with
\begin{equation}\label{eq:app:noisy-clean-comparison}
 \|\widehat UO-F\|
 \le C\frac{\|\widetilde Z-Z-m\sigma^2I\|}{\sigma_r^2(A)}
 \le C\kappa^2E_\sigma.
\end{equation}
Every row norm is at most the operator norm, so
\eqref{eq:app:noisy-clean-basis} and
\eqref{eq:app:noisy-clean-comparison} imply
\begin{equation}\label{eq:app:noisy-intermediate-alignment}
\begin{aligned}
 \|\widehat UO-U\|_{2,\infty}
 &\le C\kappa^2E_0\|U\|_{2,\infty}+C\kappa^2E_\sigma,\\
 \|\widehat UO-U\|&\le C\kappa^2(E_0+E_\sigma)\le c<1/4.
\end{aligned}
\end{equation}
To replace $O$ by $R$, note that $(\widehat UO)^\top U$ is within
operator distance $c$ of the identity, and its polar factor is
$O^\top R$. Its distance from that polar factor is the largest
deviation of its singular values from one, bounded by its distance
from the identity. The triangle inequality gives
\begin{equation}\label{eq:app:noisy-polar-alignment}
 \|R-O\|=\|O^\top R-I\|
 \le2\|(\widehat UO)^\top U-I\|
 \le2\|\widehat UO-U\|.
\end{equation}
Right orthogonal multiplication preserves row norms. Consequently,
\eqref{eq:app:noisy-intermediate-alignment} and
\eqref{eq:app:noisy-polar-alignment} yield
\begin{align*}
 \|\widehat UR-U\|_{2,\infty}
 &\le\|\widehat UO-U\|_{2,\infty}
   +(\|U\|_{2,\infty}+\|\widehat UO-U\|_{2,\infty})\|R-O\|\\
 &\le(1+2c)\|\widehat UO-U\|_{2,\infty}
      +2\|U\|_{2,\infty}\|\widehat UO-U\|
 \le C\kappa^2E_0\|U\|_{2,\infty}+C\kappa^2E_\sigma.
\end{align*}
The last step uses $\kappa\ge1$ and $\|U\|_{2,\infty}\le1$,
proving \eqref{eq:noise_wedge_spectral_loo} for the full matrix.

\paragraph{Simultaneous leave-one-out estimates.}
Let $F^{(s)}$ be a basis of the leading $r$-dimensional eigenspace
of the noiseless leave-one-out matrix $Z^{(s)}$.
Conditional on the wedge indicators whose two row indices avoid
$s$, this basis is fixed and the indicators in $Z-Z^{(s)}$ remain
independent. Thus Lemma~\ref{lemma:LOO_concentration}, specifically
\eqref{eq:E55} with $a=5$, applies simultaneously to all $F^{(s)}$.
The spectral gaps of $Z$ and $Z^{(s)}$ follow from \eqref{eq:app:noisy-wedge-gap-control}
by compression. Davis--Kahan gives
\begin{equation}\label{eq:app:noisy-clean-loo-comparison}
 \inf_{O^\top O=I}\|F^{(s)}O-F\|
 \le C\kappa^2E_0\|F^{(s)}\|_{2,\infty}.
\end{equation}
Taking row norms and absorbing the term with $C\kappa^2E_0<1/2$
yields
\begin{equation}\label{eq:app:noisy-clean-loo-delocalization}
\begin{aligned}
 \|F^{(s)}\|_{2,\infty}
 &\le\|F\|_{2,\infty}+C\kappa^2E_0\|F^{(s)}\|_{2,\infty},\\
 \|F^{(s)}\|_{2,\infty}
 &\le2\|F\|_{2,\infty}
 \le C(1+\kappa^2E_0)\|U\|_{2,\infty}.
\end{aligned}
\end{equation}
Combine \eqref{eq:app:noisy-clean-loo-comparison} and
\eqref{eq:app:noisy-clean-loo-delocalization} with
\eqref{eq:app:noisy-clean-basis} to obtain
\begin{equation}\label{eq:app:noisy-clean-loo-alignment}
\begin{aligned}
 \inf_{O^\top O=I}\|F^{(s)}O-U\|_{2,\infty}
 &\le C\big[\kappa^2E_0+\kappa^2E_0(1+\kappa^2E_0)\big]\|U\|_{2,\infty}\le C\kappa^2E_0\|U\|_{2,\infty}.
\end{aligned}
\end{equation}
Here $(\kappa^2E_0)^2\le C\kappa^2E_0$. The alignment through $F$ also has operator error
$C\kappa^2E_0$, by \eqref{eq:app:noisy-clean-loo-comparison},
$\|F^{(s)}\|_{2,\infty}\le1$, and \eqref{eq:app:noisy-clean-basis}.
The polar argument \eqref{eq:app:noisy-polar-alignment} therefore
preserves \eqref{eq:app:noisy-clean-loo-alignment} for the Procrustes
alignment of $F^{(s)}$ to $U$.

For the noisy comparison, the exact identity
\[
 \widetilde Z^{(s)}-m\sigma^2I-Z^{(s)}
 =(I-e_se_s^\top)(\widetilde Z-Z-m\sigma^2I)(I-e_se_s^\top)
\]
bounds every difference by $C\|A\|^2E_\sigma$ on the single event
of Lemma~\ref{lem:noisy-lemma56}. Apply
\eqref{eq:app:noisy-clean-comparison}--\eqref{eq:app:noisy-polar-alignment}
with the noisy leave-one-out basis and the noiseless basis $F^{(s)}$;
\eqref{eq:app:noisy-clean-loo-alignment} supplies the row-norm
bound for $F^{(s)}$, while \eqref{eq:app:noisy-clean-loo-comparison} and
\eqref{eq:app:noisy-clean-basis} give operator distance
$C\kappa^2E_0$ from $U$. This proves \eqref{eq:noise_wedge_spectral_loo} for all $s$.
The concentration and subspace estimates in
Theorems~\ref{thm:wedge_concentration} and~\ref{thm:wedge_spectral_guarantees},
used with $a=5$, and the simultaneous estimate \eqref{eq:E55}
have total failure probability $O(n^{-4})$. The Gaussian and noisy
operator events have total failure $O(n^{-3})$. Hence all conclusions
hold together with probability $1-O(n^{-3})$.
\endgroup
\end{proof}
}

\subsection{Concentration of \texorpdfstring{\(\widetilde Y\)}{Y tilde} after a delocalized projection}

\begin{lemma}\label{lem:delocalization_Q_noisy}
Assume the hypotheses of Proposition~\ref{prop:noisy_op_loo}, and suppose
\begin{align}\label{eq:stronger_noise_condition}
 p\ge\frac{C_0\kappa^4\mu_1\mu_2r^2\log n}{n^k},
 \qquad
 \kappa^2E_\sigma\le c_0\sqrt{\frac{\mu_1r}{n}}.
\end{align}
Then, for $Q=\widehat U\widehat U^\top$, with probability
$1-O(n^{-3})$,
\[
 \|Q\|_{2,\infty}\le2\sqrt{\frac{\mu_1r}{n}}.
\]
\end{lemma}
\begin{proof}
\begingroup
For an orthogonal alignment $R$,
$\|Q\|_{2,\infty}=\|\widehat UR\|_{2,\infty}$.
The two conditions in \eqref{eq:stronger_noise_condition} give
$\kappa^2E_0\le C_0^{-1/2}$ and
$\kappa^2E_\sigma\le c_0\sqrt{\mu_1r/n}$.
On the event giving \eqref{eq:noise_wedge_spectral_loo},
\[
 \|Q\|_{2,\infty}
 \le\|U\|_{2,\infty}+\|\widehat UR-U\|_{2,\infty}
 \le(1+CC_0^{-1/2}+Cc_0)\sqrt{\frac{\mu_1r}{n}}
 \le2\sqrt{\frac{\mu_1r}{n}},
\]
after increasing $C_0$ and decreasing $c_0$.
Proposition~\ref{prop:noisy_op_loo} gives probability $1-O(n^{-3})$.
\endgroup
\end{proof}

Recall that
\[
 \mathcal Q(\widetilde Y)=Q\widetilde Y(Q\otimes\cdots\otimes Q),
 \qquad \mathcal Q(A)=QA(Q\otimes\cdots\otimes Q).
\]
\begin{lemma}\label{lemma:Projection_Y_noisy}
Conditional on the wedge observations, let $Q$ be the estimated rank-$r$
projector. For every $q\in(0,1]$, with conditional probability at least
$1-O(n^{-3})$,
\begin{align*}
 \|\mathcal Q(\widetilde Y-A)\|_F
 &\le C\|Q\|_{2,\infty}^k\|T\|_F\sqrt{\frac{1-q}{q}}
       +C\|T\|_{\max}\sqrt{\frac{(1-q)\log n}{q}}\\
 &\quad+C\|Q\|_{2,\infty}^k\|T\|_{\max}\frac{\log n}{q}\\
 &\quad+C\sigma\left(\frac{\sqrt{r^k}+\sqrt{\log n}}{\sqrt q}
             +\frac{\|Q\|_{2,\infty}^k\sqrt k\log n}{q}\right).
\end{align*}
\end{lemma}
\begin{proof}
The uniform indicators and Gaussian errors are independent of the wedge observations.
For the sampled signal, \eqref{eq:app:projected_vector_bernstein} in
Lemma~\ref{lemma:Projection_Y} with $a=4$ bounds its projected Frobenius error
by the three terms involving $T$, with failure probability at most $n^{-4}$.
The proof uses only that $Q$ is fixed independently of the entry indicators.

For the noise, vectorize after projection by $Q^{\otimes k}$.
Conditional also on the indicators, the result is a centered Gaussian vector
with covariance matrix
\begin{equation}\label{eq:app:noisy-projection-conditional-covariance}
 \frac{\sigma^2}{q^2}\sum_{\bm i}\omega_{\bm i}
 (Q^{\otimes k}e_{\bm i})(Q^{\otimes k}e_{\bm i})^\top .
\end{equation}
Here $e_{\bm i}$ is the coordinate vector indexed by $\bm i\in[n]^k$.
Every squared vector norm is at most $\|Q\|_{2,\infty}^{2k}$.
Summing the outer products over all $\bm i$ gives $Q^{\otimes k}$,
whose trace is $r^k$ and whose operator norm is one.

First apply scalar Bernstein to the trace of the random sum in
\eqref{eq:app:noisy-projection-conditional-covariance}, before multiplication
by $\sigma^2/q^2$. Its mean is $qr^k$, and its variance is bounded by
$q\|Q\|_{2,\infty}^{2k}r^k$. With failure at most $2n^{-5}$,
\begin{equation}\label{eq:app:noisy-projection-trace}
 \sum_{\bm i}\omega_{\bm i}\|Q^{\otimes k}e_{\bm i}\|_2^2
 \le C\big(qr^k+\|Q\|_{2,\infty}^{2k}\log n\big).
\end{equation}
For the operator norm, each centered summand is bounded by
$\|Q\|_{2,\infty}^{2k}$. The square of a rank-one outer product
is its squared vector norm times that outer product, so the sum of the
second moments is bounded above by $q\|Q\|_{2,\infty}^{2k}Q^{\otimes k}$.
Lemma~\ref{thm:(Matrix-Bernstein-inequality)}, with logarithmic parameter
$5\log(2n^k)$, bounds the operator norm of the uncentered sum by
\begin{equation}\label{eq:app:noisy-projection-sampled-operator}
 \left\|\sum_{\bm i}\omega_{\bm i}
 (Q^{\otimes k}e_{\bm i})(Q^{\otimes k}e_{\bm i})^\top\right\|
 \le C\big(q+\|Q\|_{2,\infty}^{2k}k\log n\big)
\end{equation}
except with probability $O(n^{-4})$. Both displays use
$2\sqrt{uv}\le u+v$ to absorb the square-root terms.

For a centered Gaussian vector, the mean Euclidean norm is at most the square
root of the covariance trace, by Jensen's inequality. Its norm as a function
of a standard Gaussian vector is Lipschitz with constant equal to the
square root of the covariance operator norm. Gaussian concentration at
$\sqrt{10\log n}$, followed by
\eqref{eq:app:noisy-projection-trace} and \eqref{eq:app:noisy-projection-sampled-operator}, gives
\begin{equation}\label{eq:app:noisy-projection-noise-bound}
 \|\mathcal Q(\unfold_1(q^{-1}\omega\odot H))\|_F
 \le C\sigma\left(\frac{\sqrt{r^k}+\sqrt{\log n}}{\sqrt q}
 +\frac{\|Q\|_{2,\infty}^k\sqrt k\log n}{q}\right).
\end{equation}
The conditional Gaussian tail is at most $n^{-5}$ for every indicator
realization satisfying the two covariance bounds. Adding the three failure
probabilities and the signal failure probability proves the lemma.
\end{proof}

\subsection[Completion of the proof of Theorem~\ref{thm:noisy_spectral}]{{Completion of the proof of Theorem~\ref{thm:noisy_spectral}}}
\begin{proof}
\begingroup
Use the existing noiseless sampling-error scale and the two terms
in the noise bound:
\[
 E_0=\sqrt{\frac{\mu_1\mu_2r^2\log n}{pn^k}},\qquad
 B_1=\frac{(\mu_1r)^{1/4}\sqrt p}
 {\sqrt k\,n^{k/2+1/4}\sqrt{\log n}},\qquad
 B_2=\frac{(\mu_1\mu_2)^{1/4}\sqrt r\,p^{1/4}}
 {\sqrt k\,n^{3k/4}(\log n)^{1/4}}.
\]
Thus \eqref{eq:thm11-simple-noise} gives
$\theta\le c\kappa^{-2}\min(B_1,B_2)$ and
\begin{equation}\label{eq:app:noisy-spectral-clean-scale}
 E_0\le C_0^{-1/2}\kappa^{-2}.
\end{equation}
We verify the noise conditions in Proposition~\ref{prop:noisy_op_loo}
and Lemma~\ref{lem:delocalization_Q_noisy}. For the linear mean term,
$p\le1$, $k\ge3$, and $\mu_1r,\mu_2r\ge1$ give
\begin{equation}\label{eq:app:noisy-spectral-mean-ratio}
 \frac{B_2\sqrt n}{E_0}
 =\frac{p^{3/4}n^{1/2-k/4}}
 {\sqrt k\,(\mu_1\mu_2r^2)^{1/4}(\log n)^{3/4}}
 \lesssim 1.
\end{equation}
In the six-term definition of $E_\sigma$ in
Proposition~\ref{prop:noisy_op_loo}, use $B_2$ for the linear
mean and variance terms, $B_1$ for the linear increment, and $B_2^2$ for all
quadratic terms. The coherence caps $\mu_1r\le cn$ and
$\mu_2r\le cn^{k-1}$, together with $p\le1$, then give
\begin{equation}\label{eq:app:noisy-spectral-noise-ratio}
\begin{aligned}
 \frac{E_\sigma}{c\kappa^{-2}E_0}
 \lesssim{}&\frac{B_2\sqrt n}{E_0}
 +\frac{\alpha_{n,k}n^{1/4}+n^{(k-1)/4}}
        {\sqrt k\,n^{k/4}(\log n)^{1/4}}
 +\frac{\sqrt{\log n}}{n^{k/2}}\\
 &+c\kappa^{-2}\left[
 \frac1{k n^{k/2}\log n}
 +\frac{\alpha_{n,k}}{k n^{k/2}\sqrt{\log n}}
 +\frac{\log n}{n^k}\right]\lesssim1.
\end{aligned}
\end{equation}
For the row scale, use $B_1$ for all three linear terms and
$B_2^2$ for the quadratic terms. Only the two increments require
the lower bound $pn^k\ge C_0\kappa^4\mu_1\mu_2r^2\log n$:
\[
\begin{aligned}
 \frac{B_1(\log n)/p}{\sqrt{\mu_1r/n}}
       \sqrt{\frac{k\mu_1\mu_2r^2\log n}{n^k}}
 &\le\frac{\sqrt{\log n}}
 {\sqrt{C_0}\kappa^2(\mu_1r)^{1/4}n^{k/2-1/4}},\\
 \frac{B_2^2 k(\log n)^2/p}{\sqrt{\mu_1r/n}}
 &\le\frac{\log n}
 {\sqrt{C_0}\kappa^2\sqrt{\mu_1r}\,n^{k-1/2}}.
\end{aligned}
\]
Substitution in the other four terms and the same coherence caps give
\begin{equation}\label{eq:app:noisy-spectral-row-ratio}
\begin{aligned}
 \frac{E_\sigma}{c\kappa^{-2}\sqrt{\mu_1r/n}}
 \lesssim{}&\frac1{\sqrt k\,n^{k/2-3/4}\sqrt{\log n}}
 +\frac{\alpha_{n,k}}{\sqrt k\,n^{(k-1)/2}}
 +\frac1{\sqrt k\,n^{1/4}}
 +\frac{\sqrt{\log n}}{\sqrt{C_0}\kappa^2n^{k/2-1/4}}\\
 &+c\kappa^{-2}\left[
 \frac1{k n^{k/2}\sqrt{\log n}}
 +\frac{\alpha_{n,k}}{k n^{k/2}}
 +\frac{\log n}{\sqrt{C_0}\kappa^2n^{k-1/2}}
 \right]\lesssim1.
\end{aligned}
\end{equation}
The final bounds in \eqref{eq:app:noisy-spectral-noise-ratio} and
\eqref{eq:app:noisy-spectral-row-ratio} are uniform for $k\ge3$:
$\alpha_{n,k}=1+\sqrt{k\log n/n}$, $\mu_1r,\mu_2r\ge1$, and
$\kappa\ge1$. Hence
\begin{equation}\label{eq:app:noisy-spectral-small-noise}
 \mathcal E\le C(E_0+E_\sigma)\le CE_0,\qquad
 E_\sigma\le Cc\kappa^{-2}\sqrt{\frac{\mu_1r}{n}}.
\end{equation}

For completeness, the four requirements in \eqref{eq:noise_upperbound}
follow from the same substitutions:
\begin{equation}\label{eq:app:noisy-spectral-noise-conditions}
\begin{aligned}
 B_1\alpha_{n,k}\sqrt n
 &\le\frac{C\alpha_{n,k}}{\sqrt k\,n^{(k-1)/2}\sqrt{\log n}},\\
 \frac{B_1\sqrt{r\log n}
       (\alpha_{n,k}\sqrt{\mu_1}+\sqrt{\mu_2})}{\sqrt p}
 &\le C\left(\frac{\alpha_{n,k}}{\sqrt k\,n^{(k-1)/2}}
       +\frac1{\sqrt{kn}}\right),\\
 B_1\sqrt{\frac{k(\log n)^2}{p}}
 &\le\frac{C\sqrt{\log n}}{n^{k/2}},\\
 B_2\alpha_{n,k}^{1/2}(n^k\log n/p)^{1/4}
 &\le\frac{C\alpha_{n,k}^{1/2}}{\sqrt k\,n^{k/4}}.
\end{aligned}
\end{equation}
All four right-hand sides are uniformly bounded for $k\ge3$.
Choosing $c$ sufficiently small therefore verifies
\eqref{eq:noise_upperbound}. The lower sampling bound in
\eqref{eq:thm11-simple-noise} implies the one in
\eqref{eq:stronger_noise_condition}, and
\eqref{eq:app:noisy-spectral-small-noise} gives its row condition:
\[
 \kappa^2E_\sigma\le Cc\sqrt{\frac{\mu_1r}{n}}
 \le c_0\sqrt{\frac{\mu_1r}{n}}.
\]
Proposition~\ref{prop:noisy_op_loo}, its projector estimate
\eqref{eq:DK_tildeZ}, and Lemma~\ref{lem:delocalization_Q_noisy}
therefore give, with probability $1-O(n^{-3})$,
\[
 \|Q\|_{2,\infty}\le2\sqrt{\frac{\mu_1r}{n}},\qquad
 \|(I-Q)U\|\le \|\widehat UR-U\|\le C\kappa^2E_0.
\]
The projection identity \eqref{eq:app:symmetric_projection_identity}
holds for this $Q$. The $k$ orthogonal terms in \eqref{eq:TTstar}
and the SVD of $A$ therefore give
\begin{equation}\label{eq:app:noisy-spectral-projection-error}
 \|(I^{\otimes k}-Q^{\otimes k})\operatorname{vec}(T)\|_2
 \le\sqrt k\,\|(I-Q)A\|_F
 \le\sqrt k\,\|(I-Q)U\|\,\|T\|_F
 \le C\sqrt k\,\kappa^2 E_0\|T\|_F .
\end{equation}

For the refinement error, condition on the wedge observations satisfying these
bounds. Apply Lemma~\ref{lemma:Projection_Y_noisy}, then substitute
$\|Q\|_{2,\infty}\le2\sqrt{\mu_1r/n}$ and the incoherent entry bound
$\|T\|_{\max}\le r\sqrt{\mu_1\mu_2/n^k}\,\|A\|$ from \eqref{eq:infty_norm}.
For the last term of \eqref{eq:app:noisy-projection-noise-bound}, the assumption on $q$ gives
\[
 \frac{\|Q\|_{2,\infty}^k\sqrt k\log n}{\sqrt{q r^k}}
 \le\left(\frac{4^k\mu_1^k k(\log n)^2}{qn^k}\right)^{1/2}\le C.
\]
We obtain, with conditional probability at least $1-O(n^{-3})$,
\begin{equation}\label{eq:app:noisy-spectral-refinement-error}
\begin{aligned}
 \frac{\|\mathcal Q(\widetilde Y-A)\|_F}{\|T\|_F}
 &\le C2^k(\mu_1r)^{k/2}\sqrt{\frac{1-q}{qn^k}}\\
 &\quad+Cr\sqrt{\mu_1\mu_2}\,\frac{\|A\|}{\|T\|_F}
 \left(\sqrt{\frac{(1-q)\log n}{qn^k}}
       +2^k(\mu_1r)^{k/2}\frac{\log n}{qn^k}\right)\\
 &\quad+\frac{C\sigma}{\|T\|_F}\frac{\sqrt{r^k}+\sqrt{\log n}}{\sqrt q}.
\end{aligned}
\end{equation}
In the unfolding of Algorithm~\ref{alg:montanari}, the error splits into
$A-\mathcal Q(A)$ and $\mathcal Q(A-\widetilde Y)$.
The triangle inequality with \eqref{eq:app:noisy-spectral-projection-error}
and \eqref{eq:app:noisy-spectral-refinement-error} proves the stated bound.
Adding the wedge failure probability to the conditional refinement failure
probability gives total failure $O(n^{-3})$.
\endgroup
\end{proof}

\section[Proof of Theorem~\ref{thm:noisy_cai_sparse}]{Proof of Theorem~\ref{thm:noisy_cai_sparse}}\label{sec:app:noisy_cai}
\begingroup

The wedge, factor-extraction, and gradient-descent stages use independent samples and Gaussian errors. Thus $X^0$ is independent of the gradient-descent data $(\Omega,H)$. Once drawn, these data are used at every iteration. The value at $(a,b,c)\in\Omega$ is always $T_{abc}+H_{abc}$, as in \eqref{eq:noisy_uniform}. Define
\[
        \mathcal S(X):=\sum_{i=1}^r x_i^{\otimes3},
        \qquad
        s_\sigma:=\tau\sqrt{\frac{n\log n}{q}},
\]
and
\[
        h_\sigma
        :=
        \tau\left[n\sqrt\mu+
        \sqrt{\frac{\mu n}{q}}\log^{5/2}n+
        \frac{\sqrt\mu}{q\sqrt n}\log^3n\right].
\]
{
Since $q\le1$, the sampling condition in Theorem~\ref{thm:noisy_cai_sparse}
implies
\begin{equation}\label{eq:app:noisy-feasible-rank}
 \kappa_{\mathrm{CP}}^{3/2}\mu^{3/2}r\log^{5/4}n\le c_1^{-1/4}\sqrt n .
\end{equation}
In particular, $r\mu/n\ll1$. The sampling and noise assumptions in
Theorem~\ref{thm:noisy_cai_sparse} also give
\begin{equation}\label{eq:app:noisy-basic-noise-scales}
 s_\sigma\le\frac{c}{\sqrt{c_1}\kappa_{\mathrm{CP}}^3\mu^7r^5\sqrt n\log^2n},
 \qquad
 \tau\sqrt{\frac{rn\log^2n}{q}}
 \le\frac{c}{\sqrt{c_1}\kappa_{\mathrm{CP}}^3\mu^7r^{9/2}\sqrt n\log^{3/2}n}.
\end{equation}
The three contributions to $h_\sigma$ satisfy
\begin{equation}\label{eq:app:noisy-hessian-noise-components}
\begin{aligned}
 \tau n\sqrt\mu&\le\frac{c}{\mu^{7/2}r^3 n},
 &\tau\log^{5/2}n\sqrt{\frac{\mu n}{q}}
 &\le\frac{c}{\sqrt{c_1}\kappa_{\mathrm{CP}}^3\mu^{13/2}r^5\sqrt n},\\
 \tau\log^3 n\frac{\sqrt\mu}{q\sqrt n}
 &\le\frac{c}{c_1\kappa_{\mathrm{CP}}^6\mu^{19/2}r^7\sqrt n\log^2 n}.
\end{aligned}
\end{equation}
For the term $\tau n\sqrt\mu$, use $\kappa_{\mathrm{CP}}^2\lesssim n$
from \eqref{eq:app:noisy-feasible-rank}. Adding
\eqref{eq:app:noisy-basic-noise-scales} and
\eqref{eq:app:noisy-hessian-noise-components} then gives
\begin{equation}\label{eq:app:noisy-small-parameters}
 s_\sigma+\tau\sqrt{\frac{rn\log^2 n}{q}}+h_\sigma
 \le\frac{Cc}{\kappa_{\mathrm{CP}}^2\mu^2 r}.
\end{equation}
}

For \(s\in[n]\), let \(\mathcal P_s\) select all ordered tensor entries
having at least one index equal to \(s\), and set
\[
 \Omega_{-s}:=\{(a,b,c)\in\Omega:a,b,c\ne s\}.
\]
The noisy leave-one-out loss is
\begin{equation}\label{eq:app:noisy-loo-loss}
\begin{aligned}
 F_\sigma^{(s)}(X)
 :=&
 \frac1{6q}
 \left\|
 \mathcal P_{\Omega_{-s}}\bigl(T+H-\mathcal S(X)\bigr)
 \right\|_F^2+
 \frac16
 \left\|
 \mathcal P_s\bigl(T-\mathcal S(X)\bigr)
 \right\|_F^2 .
\end{aligned}
\end{equation}
On the \(s\)-th slice, \eqref{eq:app:noisy-loo-loss} uses every entry of \(T\)
without noise. Set $X^{(s),0}=X^0$ for every $s$.
By \eqref{eq:app:noisy-loo-loss}, the auxiliary updates use random data
only from triples with $a,b,c\ne s$. Since $X^0$ is independent of $(\Omega,H)$, the trajectory
\(\{X^{(s),t}\}_{t\ge0}\) is independent of
\(\{(\omega_{abc},H_{abc}):s\in\{a,b,c\}\}\).
Consequently, fixing $X^{(s),t}$ preserves the Bernoulli and Gaussian
laws on the $s$-th slice in \eqref{eq:app:noisy-independent-contraction}
and \eqref{eq:app:noisy-loo-gradient}.

For \(V=[v_1,\ldots,v_r]\), the derivative of \(\mathcal S\) is
\[
 D\mathcal S(X)[V]
 =
 \sum_{i=1}^r
 \left(
 v_i\otimes x_i\otimes x_i
 +x_i\otimes v_i\otimes x_i
 +x_i\otimes x_i\otimes v_i
 \right).
\]
Define the noise contribution \(\mathcal N_H(X)\) to the gradient by
\begin{equation}\label{eq:app:noisy-gradient-split}
 \langle\mathcal N_H(X),V\rangle
 :=
 \frac1{3q}
 \left\langle
 \mathcal P_\Omega(H),D\mathcal S(X)[V]
 \right\rangle .
\end{equation}
Direct differentiation gives the exact identity
\[
        \nabla F_\sigma(X)=\nabla F_0(X)-\mathcal N_H(X).
\]
Define \(\mathcal N_H^{(s)}\) in the same way with
\(\Omega\) replaced by \(\Omega_{-s}\).  The \(s\)-th row of
\(\mathcal N_H^{(s)}(X)\) is identically zero.

{
\begin{lemma}[Noisy wedge initialization]\label{lem:noisy-cai-subspace}
Under the assumptions of Theorem~\ref{thm:noisy_cai_sparse}, with probability at least
$1-O(n^{-3})$, the noisy wedge eigenspace and all its leave-one-out
versions satisfy, after their Procrustes alignments,
\[
 \|\widehat UR-U\|\le\frac{c}{\mu^2r},\qquad
 \|\widehat UR-U\|_{2,\infty}
 \le\frac{c}{\mu^2r}\|U\|_{2,\infty}.
\]
\end{lemma}
\begin{proof}
{
For the order-three unfolding, $m=n^2$. Lemma~\ref{lem:unfold_incoherence}
gives $\mu_1\lesssim\mu$ and $\mu_2\lesssim\mu^2$.
In the factorization \eqref{eq:unfold_factorization}, the two matrices have
columns $\bar x_i$ and $\bar x_i\otimes\bar x_i$, respectively. Their Gram
matrices have unit diagonal and off-diagonal row sums at most $(r-1)\mu/n$
and $(r-1)(\mu/n)^2$, by \eqref{eq:cp_incoherent}.
Both row sums are less than $1/2$ by \eqref{eq:app:noisy-feasible-rank}. Hence
$\|A\|\asymp\lambda_{\max}$, $\sigma_r(A)\asymp\lambda_{\min}$,
$\kappa\lesssim\kappa_{\mathrm{CP}}$, and
$\theta=\sigma/\|A\|\lesssim\tau/\kappa_{\mathrm{CP}}$.
For $E_0$ in Proposition~\ref{prop:noisy_op_loo}, the lower bound on $p$ gives
\[
 E_0\lesssim\sqrt{\frac{\mu^3r^2\log n}{pn^3}}
 \lesssim\frac{1}{\sqrt{c_0}\kappa_{\mathrm{CP}}^2\mu^{5/2}r^{3/2}\sqrt{\log n}}
 \le\frac{c}{\kappa_{\mathrm{CP}}^2\mu^2r}.
\]
In $E_\sigma$ from Proposition~\ref{prop:noisy_op_loo}, substitute $m=n^2$,
$\alpha_{n,3}\le C$, $\mu_1\lesssim\mu$, and $\mu_2\lesssim\mu^2$ to obtain
\begin{equation}\label{eq:app:noisy-initial-subspace-noise}
\begin{aligned}
 E_\sigma\lesssim{}&
 \theta\sqrt n+\theta\mu\sqrt{\frac{r\log n}{p}}
 +\frac{\theta\mu^{3/2}r(\log n)^{3/2}}{pn^{3/2}}\\
 &+\theta^2n^{3/2}+\theta^2n^{3/2}\sqrt{\frac{\log n}{p}}
 +\frac{\theta^2(\log n)^2}{p}.
\end{aligned}
\end{equation}
Substitute $\theta\lesssim c/(\kappa_{\mathrm{CP}}\mu^4r^3n^2)$ and the
lower bound on $p$ from Theorem~\ref{thm:noisy_cai_sparse}.
In their order in \eqref{eq:app:noisy-initial-subspace-noise}, the terms
linear in $\theta$ are bounded by absolute constants times
\[
 \frac{c}{\kappa_{\mathrm{CP}}\mu^4r^3n^{3/2}},\qquad
 \frac{c}{\sqrt{c_0}\kappa_{\mathrm{CP}}^3\mu^7r^5\sqrt{n\log n}},\qquad
 \frac{c}{c_0\kappa_{\mathrm{CP}}^5\mu^{21/2}r^7\sqrt{n\log n}}.
\]
The terms quadratic in $\theta$ are bounded in the same order by absolute constants times
\[
 \frac{c^2}{\kappa_{\mathrm{CP}}^2\mu^8r^6n^{5/2}},\qquad
 \frac{c^2}{\sqrt{c_0}\kappa_{\mathrm{CP}}^4\mu^{12}r^{17/2}n\sqrt{\log n}},\qquad
 \frac{c^2}{c_0\kappa_{\mathrm{CP}}^6\mu^{16}r^{11}n}.
\]
For $\theta\sqrt n$, use \eqref{eq:app:noisy-feasible-rank} to absorb one
additional factor of $\kappa_{\mathrm{CP}}$ into $n$. Summing the six terms
and using $\|U\|_{2,\infty}\ge\sqrt{r/n}$ gives
\begin{equation}\label{eq:app:noisy-initial-subspace-row-noise}
 E_\sigma\le\frac{Cc}{\kappa_{\mathrm{CP}}^2\mu^2r}\sqrt{\frac rn}
 \le\frac{Cc}{\kappa_{\mathrm{CP}}^2\mu^2r}\|U\|_{2,\infty}.
\end{equation}
To verify the noise hypotheses of Proposition~\ref{prop:noisy_op_loo}, use
$\kappa_{\mathrm{CP}}\lesssim\sqrt n$ in the bounds for
$\theta\sqrt n$ and $\theta\mu\sqrt{r\log(n)/p}$, and
$\kappa_{\mathrm{CP}}^2\lesssim n$ in the bound for $\theta^2n^{3/2}$.
Thus all six terms in \eqref{eq:app:noisy-initial-subspace-noise} are at most
$Cc\kappa^{-4}$. From the definition of $E_\sigma$, the four quantities
$\theta\sqrt n$, $\theta(\sqrt{\mu_1}+\sqrt{\mu_2})\sqrt{r\log(n)/p}$,
$\theta^2\log^2(n)/p$, and $\theta^2\sqrt{n^3\log(n)/p}$ are each at most
$CE_\sigma$. Since $E_\sigma\le Cc\kappa^{-4}$ and $\alpha_{n,3}\le C$,
the first two bounds and the square roots of the other two give
\eqref{eq:noise_upperbound} for sufficiently small $c$.

Insert $E_0\le c/(\kappa_{\mathrm{CP}}^2\mu^2 r)$ and
\eqref{eq:app:noisy-initial-subspace-row-noise} into the two estimates
\eqref{eq:noise_wedge_spectral_op} and \eqref{eq:noise_wedge_spectral_loo} of
Proposition~\ref{prop:noisy_op_loo},
using $\kappa\lesssim\kappa_{\mathrm{CP}}$ and $\|U\|_{2,\infty}\le1$.
This yields
\[
 \|\widehat UR-U\|\lesssim\kappa^2(E_0+E_\sigma)
 \le\frac{c}{\mu^2r},\qquad
 \|\widehat UR-U\|_{2,\infty}\lesssim\kappa^2
 (E_0\|U\|_{2,\infty}+E_\sigma)\le\frac{c}{\mu^2r}\|U\|_{2,\infty}.
\]
The same two bounds in Proposition~\ref{prop:noisy_op_loo} hold simultaneously
for all leave-one-out bases, each with its own Procrustes alignment,
on an event of probability $1-O(n^{-3})$.
}
\end{proof}
}

\begin{lemma}[Gradient and Hessian bounds under Gaussian noise]\label{lem:noisy-cai-gaussian}
There is an event \(\mathcal E_H\), independent of the iteration index and
having probability at least \(1-Cn^{-10}\), on which
\begin{align}
 \|\mathcal N_H(X_\star)\|_F
 &\le
 C\lambda_{\min}s_\sigma\Big(\sum_i\lambda_i^{4/3}\Big)^{1/2},
 \label{eq:app:noisy-full-gradient}\\
 \sup_{X\in\mathcal B}
 \|D\mathcal N_H(X)\|
 &\le
 C\kappa_{\mathrm{CP}}^{1/3}\lambda_{\min}^{4/3}h_\sigma
 \le \frac18\lambda_{\min}^{4/3},
 \label{eq:app:noisy-uniform-hessian}
\end{align}
where \(\mathcal B\) is the local basin in
Lemma~\ref{lem:app:gd:convexity}.  In addition, fix an iteration \(t\) and
suppose the leave-one-out iterates lie in \(\mathcal B\).  Outside an event
of probability at most \(Cn^{-10}\),
\begin{align}
 \max_{s\in[n]}
 \|\mathcal N_H(X^{(s),t})
       -\mathcal N_H^{(s)}(X^{(s),t})\|_F
 &\le
 C\lambda_{\min}\sqrt{\mu/n}\,s_\sigma
 \Big(\sum_i\lambda_i^{4/3}\Big)^{1/2}.
 \label{eq:app:noisy-loo-gradient}
\end{align}
The event $\mathcal E_H$ controls \eqref{eq:app:noisy-full-gradient}
and \eqref{eq:app:noisy-uniform-hessian} simultaneously at every iteration.
For \eqref{eq:app:noisy-loo-gradient}, a union over iterations uses the stated
failure probability for each fixed $t$.
\end{lemma}

\begin{proof}
{
\emph{Gradient at $X_\star$.}
Fix a row $s$, a factor $i$, and a deterministic vector $z$ satisfying
$\|z\|_2\le C\lambda_i^{1/3}$ and
$\|z\|_\infty\le C\lambda_i^{1/3}\sqrt{\mu/n}$.
Since $\omega_{sbc}$ is independent of $H_{sbc}$ and
$\mathbb E[\omega_{sbc}^2]=q$, the variance of the contraction in row $s$ is
\[
 \sum_{b,c}\operatorname{Var}\!\left(q^{-1}\omega_{sbc}H_{sbc}z_bz_c\right)
 =\frac{\sigma^2}{q}\sum_{b,c}z_b^2z_c^2
 =\frac{\sigma^2}{q}\|z\|_2^4\le \frac{C\sigma^2\lambda_i^{4/3}}{q}.
\]
Truncate the Gaussian entries symmetrically at $C\sigma\sqrt{\log n}$.
For $C\ge\sqrt{30}$, the probability that any of the $n^3$ entries
changes is at most $2n^{3-C^2/2}\le2n^{-12}$. The truncated summands
remain independent and centered, with absolute value at most
$C\sigma\lambda_i^{2/3}\mu\sqrt{\log n}/(qn)$.
Apply scalar Bernstein with parameter $16\log n$ to the truncated sum.
Except with probability $2n^{-16}$, its absolute value is bounded by the
right-hand side below. When no entry is truncated, this gives
\[
 \left|q^{-1}\sum_{b,c}\omega_{sbc}H_{sbc}z_bz_c\right|
 \le C\sigma\lambda_i^{2/3}
 \left(\sqrt{\frac{\log n}{q}}+\frac{\mu\log^{3/2}n}{qn}\right).
\]
By the lower bound on $q$ in Theorem~\ref{thm:noisy_cai_sparse},
the ratio of the second term to the first is
$\mu\log n/(n\sqrt q)\le 1/(\sqrt{c_1}\mu^2r^2\log^{3/2}n)$.
Substitution gives
\begin{equation}\label{eq:app:noisy-independent-contraction}
 \left|q^{-1}\sum_{b,c}\omega_{sbc}H_{sbc}z_bz_c\right|
 \le C\sigma\lambda_i^{2/3}\sqrt{\frac{\log n}{q}}.
\end{equation}
If $z$ is independent of the data on the $s$-th slice, conditioning on $z$
gives \eqref{eq:app:noisy-independent-contraction} whenever the same two norm bounds hold.

Apply \eqref{eq:app:noisy-independent-contraction} with $z=x_i^\star$
to each of the three terms in \eqref{eq:app:noisy-gradient-split}.
There are three contractions for each of the $nr$ coordinates, and all use
the same $H$. Thus the probability that a Bernstein bound fails or any
entry of $H$ is truncated is at most $6nr\,n^{-16}+2n^{-12}$.
Summing the squared bound in \eqref{eq:app:noisy-independent-contraction}
over the $n$ coordinates in column $i$ and then over the $r$ columns
gives \eqref{eq:app:noisy-full-gradient}, since
$\sum_i n\sigma^2\lambda_i^{4/3}\log(n)/q
=\lambda_{\min}^2s_\sigma^2\sum_i\lambda_i^{4/3}$.

\smallskip\noindent\emph{Hessian on the local basin.}
A Gaussian tail bound over the $n^3$ entries and the operator-norm bound
for the $n\times n^2$ Gaussian unfolding hold together with probability
$1-O(n^{-10})$ and give
\begin{equation}\label{eq:app:noisy-gd-gaussian-bounds}
 \|H\|_\infty\le C\sigma\sqrt{\log n},\qquad
 \|\unfold_1(H)\|\le C\sigma n.
\end{equation}
Use the coordinate bound $C\sqrt{\mu/n}$ in the restricted mode of
$\|\cdot\|_\delta$, with $C$ large enough for the scaled columns in
\eqref{eq:app:noisy-basin-column-bounds}. In each rank-one test tensor,
the restricted factor has Euclidean norm at most $C\sqrt\mu$ and the
other two have norm at most one. Hence \eqref{eq:app:noisy-gd-gaussian-bounds} gives
\[
 \|H\|_\delta\le C\sqrt\mu\,\|\unfold_1(H)\|
 \le C\sigma n\sqrt\mu.
\]
Conditional on $H$, Theorem~\ref{thm:symmetric_tensor_concentration}
bounds $\|q^{-1}\mathcal P_\Omega(H)-H\|_\delta$ because the sampling
indicators are independent of $H$. Substitute the entry bound in
\eqref{eq:app:noisy-gd-gaussian-bounds} and add $\|H\|_\delta$. With
probability $1-O(n^{-10})$, this gives
\begin{equation}\label{eq:app:noisy-incoherent-H}
 \|q^{-1}\mathcal P_\Omega(H)\|_\delta
 \le C\sigma\left[n\sqrt\mu
 +\sqrt{\frac{\mu n}{q}}\log^{5/2}n
 +\frac{\sqrt\mu}{q\sqrt n}\log^3n\right].
\end{equation}
Using $\sigma=\lambda_{\min}\tau$ in \eqref{eq:app:noisy-incoherent-H}
gives $\|q^{-1}\mathcal P_\Omega(H)\|_\delta\le C\lambda_{\min}h_\sigma$.
For every $X\in\mathcal B$, the basin constraints
\eqref{eq:app:gd:convexity_region} give
\begin{equation}\label{eq:app:noisy-basin-column-bounds}
\begin{aligned}
 \max_i\|x_i-x_i^\star\|_2
 &\le\epsilon\sqrt r\,\lambda_{\max}^{1/3}
 \le c\lambda_{\min}^{1/3}/\mu,\\
 \max_i\|x_i-x_i^\star\|_\infty
 &\le\kappa_{\mathrm{CP}}^{2/3}\epsilon\lambda_{\max}^{1/3}\sqrt{\mu r/n}
 \le c\lambda_{\min}^{1/3}/\sqrt{\mu n}.
\end{aligned}
\end{equation}
Here $\epsilon\le c\kappa_{\mathrm{CP}}^{-4/3}/(\mu\sqrt r)$.
Adding the bounds for $x_i^\star$ from \eqref{eq:cp_incoherent} gives
$\|x_i\|_2\le C\lambda_i^{1/3}$ and
$\|x_i\|_\infty\le C\lambda_i^{1/3}\sqrt{\mu/n}$ throughout $\mathcal B$.
The second derivative of the cubic map is
\[
 D^2\mathcal S(X)[V,V]
 =2\sum_i\bigl(v_i\otimes v_i\otimes x_i
 +v_i\otimes x_i\otimes v_i+x_i\otimes v_i\otimes v_i\bigr).
\]
In each nonzero term, place $x_i/(C\lambda_{\max}^{1/3})$ in the restricted
mode and $v_i/\|v_i\|_2$ in the two Euclidean modes. Differentiating
\eqref{eq:app:noisy-gradient-split}, applying the definition of $\|\cdot\|_\delta$,
and summing $\sum_i\|v_i\|_2^2=\|V\|_F^2$ gives
\begin{equation}\label{eq:app:noisy-hessian-quadratic-form}
 |\langle D\mathcal N_H(X)[V],V\rangle|
 \le C\|q^{-1}\mathcal P_\Omega(H)\|_\delta
             \lambda_{\max}^{1/3}\|V\|_F^2.
\end{equation}
Since $\mathcal N_H$ is the gradient of
$\langle q^{-1}\mathcal P_\Omega(H),\mathcal S(X)\rangle/3$,
$D\mathcal N_H(X)$ is self-adjoint. Taking the supremum in
\eqref{eq:app:noisy-hessian-quadratic-form} gives, for every $X\in\mathcal B$,
\[
 \|D\mathcal N_H(X)\|
 =\sup_{\|V\|_F=1}|\langle D\mathcal N_H(X)[V],V\rangle|
 \le C\kappa_{\mathrm{CP}}^{1/3}\lambda_{\min}^{4/3}h_\sigma\le\tfrac18\lambda_{\min}^{4/3},
\]
where the last step uses \eqref{eq:app:noisy-small-parameters}.
Thus $\mathcal E_H$, the intersection of the events in
\eqref{eq:app:noisy-full-gradient} and \eqref{eq:app:noisy-uniform-hessian},
has probability $1-O(n^{-10})$.

\smallskip\noindent\emph{Gradient difference on the $s$-th slice.}
Fix $s,t$ and condition on $X^0$ and all $(\omega_{abc},H_{abc})$ with
$s\notin\{a,b,c\}$. This fixes $X^{(s),t}$ by \eqref{eq:app:noisy-loo-loss}.
If $X^{(s),t}\in\mathcal B$, its columns satisfy \eqref{eq:app:noisy-basin-column-bounds}.
In the gradient difference in \eqref{eq:app:noisy-loo-gradient}, only entries
on the slice remain. Their scalar multipliers $q^{-1}\omega_{abc}H_{abc}$
are independent, centered, and have variance $\sigma^2/q$.

The sum of the squared Frobenius norms of the corresponding coefficient
matrices is at most
\[
 C\sum_i\left[\|x_i^{(s),t}\|_2^4
       +n|x_i^{(s),t}(s)|^2\|x_i^{(s),t}\|_2^2\right]
 \le C\mu\sum_i\lambda_i^{4/3}.
\]
For example, among entries $(s,b,c)$, differentiating the first or second
tensor factor in column $i$ gives the respective squared-coefficient sums
\[
 \sum_{b,c}|x_i^{(s),t}(b)x_i^{(s),t}(c)|^2=\|x_i^{(s),t}\|_2^4,
 \qquad
 \sum_{b,c}|x_i^{(s),t}(s)x_i^{(s),t}(c)|^2
 =n|x_i^{(s),t}(s)|^2\|x_i^{(s),t}\|_2^2 .
\]
Permuting the three indices covers the remaining positions. If an index
repeats, Cauchy--Schwarz bounds the squared norm of the combined coefficient
by three times the sum of the individual squared norms.
Truncating the slice errors at $C\sigma\sqrt{\log n}$
makes each summand at most
$C\sigma\mu(\sum_i\lambda_i^{4/3})^{1/2}\sqrt{\log n}/(qn)$ in Frobenius norm.
Apply the vector form of Lemma~\ref{thm:(Matrix-Bernstein-inequality)}
to the vectorized coefficient matrices with tail parameter $16\log n$.
The Frobenius norm of the truncated gradient difference is at most
\[
 C\sigma\Big(\sum_i\lambda_i^{4/3}\Big)^{1/2}
 \left(\sqrt{\frac{\mu\log n}{q}}
       +\frac{\mu\log^{3/2}n}{qn}\right).
\]
The ratio of the second term to the first is
$\sqrt\mu\log n/(n\sqrt q)\le1/(\sqrt{c_1}\mu^{5/2}r^2\log^{3/2}n)$
by Theorem~\ref{thm:noisy_cai_sparse}. Hence, whenever the truncation
leaves the slice unchanged,
\begin{equation}\label{eq:app:noisy-slice-gradient-scale}
 \|\mathcal N_H(X^{(s),t})-\mathcal N_H^{(s)}(X^{(s),t})\|_F
 \le C\sigma\Big(\sum_i\lambda_i^{4/3}\Big)^{1/2}\sqrt{\frac{\mu\log n}{q}}.
\end{equation}
For each $s$, the conditional Bernstein failure probability is $O(n^{-14})$.
The chance that any of the at most $3n^2$ slice entries is truncated
is at most $6n^{2-C^2/2}\le6n^{-13}$ for $C\ge\sqrt{30}$.
These bounds hold conditionally for each $s$ whenever $X^{(s),t}\in\mathcal B$.
Averaging over the conditioning data and summing over $s$ gives failure
probability $O(n^{-12})$ for the fixed $t$. Substituting
$\sigma\sqrt{\log(n)/q}=\lambda_{\min}s_\sigma/\sqrt n$ in
\eqref{eq:app:noisy-slice-gradient-scale} proves \eqref{eq:app:noisy-loo-gradient}.
}
\end{proof}

\begingroup\setlength{\emergencystretch}{2em}
\begin{lemma}[Noisy refinement initializer]\label{lem:noisy-cai-refinement}
Apply Algorithm~\ref{alg:factor_extraction} to the noisy wedge eigenspace
and the independent extraction batch. Let $X^0$ be its output and set
$X^{(s),0}=X^0$ for every $s$. Under the assumptions of
Theorem~\ref{thm:noisy_cai_sparse}, with probability at least
$1-\delta-O(n^{-3})$, a single permutation of the estimated columns gives
\begin{align}
 \|X^0-X_\star\|_F
 &\le\frac{c}{C\kappa_{\mathrm{CP}}^2\mu^{3/2}\sqrt r}\|X_\star\|_F,
 \quad \|X^0-X_\star\|_{2,\infty}
 \le\frac{c}{8\kappa_{\mathrm{CP}}^{2/3}\mu\sqrt r}\|X_\star\|_{2,\infty},
 \label{eq:app:noisy-init-full}\\
 \max_s\|X^0-X^{(s),0}\|_F&=0,
 \quad \max_s\|(X^{(s),0}-X_\star)_{s,:}\|_2
 \le\frac{c}{8\kappa_{\mathrm{CP}}^{2/3}\mu\sqrt r}\|X_\star\|_{2,\infty}.
 \label{eq:app:noisy-init-loo}
\end{align}
For any sufficiently large fixed $C$ in \eqref{eq:app:noisy-init-full},
choose the two sampling constants in Theorem~\ref{thm:noisy_cai_sparse}
sufficiently large and the constant in its upper bound on $\tau$ sufficiently small.
\end{lemma}
\begin{proof}
{
\emph{Perturbation of the contractions.}
Condition on wedge observations for which the two subspace estimates in
Lemma~\ref{lem:noisy-cai-subspace} hold. Given this subspace, compare
two applications of Algorithm~\ref{alg:factor_extraction} using the same
sample locations and Gaussian directions. One uses the values $T_{abc}$,
and the other uses $T_{abc}+H_{abc}$. For the first, the proof of
Corollary~\ref{cor:app:gd:init} applies because its only requirements on the
subspace are its accuracy and row bounds. We estimate the additional error
caused by $H$ in the second.

Conditional on the wedge data, $\hat U^\top g^\tau\sim N(0,I_r)$ and
$\|\theta^\tau\|_2=\|\hat U^\top g^\tau\|_2$. By
\eqref{eq:app:extraction-trial-count}, the number of draws is at most
$n^{a-14}$, with the same $a=2C\min\{\kappa_{\mathrm{CP}}^2,r\}+16$.
The Gaussian tail bound at $\sqrt{2a\log n}$, applied to the $r$
coordinates of $\hat U^\top g^\tau$ for each draw, gives
\[
 \|\theta^\tau\|_2\lesssim\min\{\kappa_{\mathrm{CP}},\sqrt r\}\sqrt{r\log n},\qquad
 \|\theta^\tau\|_\infty\le\|\hat U\|_{2,\infty}\|\theta^\tau\|_2
 \le C\sqrt{\mu r/n}\,\|\theta^\tau\|_2
\]
for every draw, except with probability $2rn^{a-14}n^{-a}\le2n^{-13}$.
Condition also on the Gaussian directions. The indicators and errors
in the extraction batch are independent of these directions and of
the wedge observations. Before symmetrization, the change in the contraction
is therefore the following sum of independent centered matrices.
\[
 q^{-1}\mathcal P_\Omega(H)\times_3\theta^\tau
 =\sum_{a,b,c}q^{-1}\omega_{abc}H_{abc}\theta_c^\tau e_ae_b^\top.
\]
Truncating the Gaussian entries at $\pm C\sigma\sqrt{\log n}$ preserves
centering. A Gaussian tail bound and a union bound over the $n^3$ entries
give probability at most $2n^{3-C^2/2}$ that truncation changes any entry.
Each truncated summand has norm at most
$C\sigma\|\theta^\tau\|_\infty\sqrt{\log n}/q$.
The truncated scalar coefficient $q^{-1}\omega_{abc}H_{abc}$ has second
moment at most $\sigma^2/q$. The two second-moment matrices in
Lemma~\ref{thm:(Matrix-Bernstein-inequality)} are therefore bounded above,
respectively, by the two sums below.
\[
 \frac{\sigma^2}{q}\sum_{a,b,c}(\theta_c^\tau)^2 e_ae_a^\top
 =\frac{\sigma^2}{q}\sum_{a,b,c}(\theta_c^\tau)^2 e_be_b^\top
 =\frac{\sigma^2 n\|\theta^\tau\|_2^2}{q}\,I_n .
\]
With tail parameter $a\log(2n)$, Lemma~\ref{thm:(Matrix-Bernstein-inequality)} gives
\begin{equation}\label{eq:app:noisy-refinement-contraction}
 \left\|q^{-1}\mathcal P_\Omega(H)\times_3\theta^\tau\right\|
 \le C\sigma\left(
 \min\{\kappa_{\mathrm{CP}},\sqrt r\}\|\theta^\tau\|_2\sqrt{\frac{n\log n}{q}}
 +\min\{\kappa_{\mathrm{CP}}^2,r\}\|\theta^\tau\|_\infty\frac{\log^{3/2}n}{q}\right).
\end{equation}
By the row bound on $\theta^\tau$ and the lower bound on $q$ in
Theorem~\ref{thm:noisy_cai_sparse}, the ratio of the second term
in \eqref{eq:app:noisy-refinement-contraction} to the first is at most
\[
 C\min\{\kappa_{\mathrm{CP}},\sqrt r\}\frac{\sqrt{\mu r}\log n}{n\sqrt q}
 \le\frac{C}{\sqrt{c_1}\mu^{5/2}r^{3/2}\log^{3/2}n}.
\]
Choose the truncation constant at least $\sqrt{30}$.
By \eqref{eq:app:extraction-trial-count}, the probability that
\eqref{eq:app:noisy-refinement-contraction} fails for any draw is at most
\[
 n^{a-14}n^{-a}
 +2n^{-12}=n^{-14}+2n^{-12}.
\]
Taking the symmetric part in Algorithm~\ref{alg:factor_extraction}
can only decrease the operator norm of the perturbation.

\smallskip\noindent\emph{Factor directions and weights.}
For each factor, \eqref{eq:app:gaussian-trial-gap} provides a draw whose
population coefficient separates from the others. By
\eqref{eq:app:trial-gap-comparison}, the contraction formed with the sampled
values of $T$ has singular-value gap at least
$(7/4)\lambda_{\min}\min\{\kappa_{\mathrm{CP}},\sqrt r\}$ for that draw.
For every draw, \eqref{eq:app:noisy-refinement-contraction} and the bound
on $\|\theta^\tau\|_2$ give
\begin{equation}\label{eq:app:noisy-extraction-direction-error}
 \frac{\|q^{-1}\mathcal P_\Omega(H)\times_3\theta^\tau\|}
      {\lambda_{\min}\min\{\kappa_{\mathrm{CP}},\sqrt r\}}
 \lesssim \tau\|\theta^\tau\|_2\sqrt{\frac{n\log n}{q}}
 \lesssim \min\{\kappa_{\mathrm{CP}},\sqrt r\}\,\tau\sqrt{\frac{rn\log^2n}{q}} .
\end{equation}
By \eqref{eq:app:noisy-basic-noise-scales}, the ratio in
\eqref{eq:app:noisy-extraction-direction-error} is at most $1/32$.
Normalize each singular-value gap by the denominator in
\eqref{eq:app:noisy-extraction-direction-error}. Weyl's inequality bounds the
change in each of the two singular values by $1/32$ on this scale.
Thus every factor has a noisy candidate with normalized gap at least
$7/4-2/32=27/16$. Conversely, a noisy candidate at or above $27/16$ has
normalized gap at least $27/16-2/32=26/16$ before adding noise.
By \eqref{eq:app:trial-gap-comparison}, its population coefficient
gap exceeds $\lambda_{\min}\min\{\kappa_{\mathrm{CP}},\sqrt r\}$,
as required in \eqref{eq:app:extraction-dominant-coefficient}.

For each such candidate, Proposition~\ref{prop:app:refinement:main}
bounds the error of the leading unit eigenvector computed from the sampled
signal. The corresponding eigenvalue is separated in absolute value by at least
$(26/16)\lambda_{\min}\min\{\kappa_{\mathrm{CP}},\sqrt r\}$.
After aligning signs, Lemma~\ref{lem:DK} therefore bounds the distance
between this eigenvector and the noisy one by a constant times the right
side of \eqref{eq:app:noisy-extraction-direction-error}. Each coordinate
difference satisfies the same bound. In particular, \eqref{eq:app:noisy-basic-noise-scales} gives
\begin{equation}\label{eq:app:noisy-extraction-row-error}
 \sqrt n\,\kappa_{\mathrm{CP}}\tau\sqrt{\frac{rn\log^2n}{q}}
 \lesssim\frac{c}{\kappa_{\mathrm{CP}}^2\mu^7r^{9/2}\log^{3/2}n}.
\end{equation}
Together with \eqref{eq:retrieval_linf}, this gives
$\|\hat u_i\|_\infty\le C\sqrt{\mu/n}$ for each qualifying noisy direction.
The pruning rule in Algorithm~\ref{alg:factor_extraction} compares the absolute
inner product of two candidate directions with $1/2$. Replacing two unit
vectors by their noisy counterparts changes that inner product by at most the
sum of their Euclidean errors. By \eqref{eq:app:noisy-extraction-direction-error},
choosing the noise constant sufficiently small preserves the strict inequalities
in the pruning argument of Proposition~\ref{prop:app:refinement:main}.
The noisy algorithm thus selects one direction for each factor.

For each estimated weight, the Gaussian term in the cubic evaluation satisfies
\[
 \bigl|\langle q^{-1}\mathcal P_\Omega(H),\hat u_i^{\otimes3}\rangle\bigr|
 \le C\lambda_{\min}h_\sigma
\]
by \eqref{eq:app:noisy-incoherent-H} for the extraction batch.
Changing one direction at a time in the cubic evaluation gives three
terms. For the signal $T$, each is bounded by
$\|\unfold_1(T)\|\lesssim\lambda_{\max}$ times the Euclidean direction
error in \eqref{eq:app:noisy-extraction-direction-error}. The same bound
holds for the centered sampling error by \eqref{eq:app:extraction-cubic-noise-norm},
since the other two vectors in each term are delocalized unit vectors.
The total change in the weight is therefore bounded by $C\lambda_{\min}h_\sigma$
plus $C\lambda_{\max}$ times the right side of
\eqref{eq:app:noisy-extraction-direction-error}. By \eqref{eq:app:noisy-small-parameters}
this sum is smaller than $\lambda_{\min}/4$. The weight estimate
\eqref{eq:retrieval_lbda} then places both cubic evaluations, oriented toward
the true factor, in $[\lambda_i/4,2\lambda_i]$. They have the same sign.
For $u,v$ in this interval,
\[
 |u^{1/3}-v^{1/3}|
 =\frac{|u-v|}{u^{2/3}+u^{1/3}v^{1/3}+v^{2/3}}
 \le C\lambda_{\min}^{-2/3}|u-v|.
\]
For the contribution of the direction errors, the relative Frobenius and
row bounds use, respectively,
\[
 \frac{\lambda_{\max}^{1/3}\sqrt r}{\|X_\star\|_F}\le\kappa_{\mathrm{CP}}^{1/3},
 \qquad
 \frac{\lambda_{\max}^{1/3}\sqrt r}{\|X_\star\|_{2,\infty}}
 \le\kappa_{\mathrm{CP}}^{1/3}\sqrt n .
\]
For the changes in the cube roots of the weights, use
$\max_a(\sum_i|\hat u_i(a)|^2)^{1/2}\le C\sqrt{\mu r/n}$,
$\|X_\star\|_F\ge\lambda_{\min}^{1/3}\sqrt r$, and
$\|X_\star\|_{2,\infty}\ge\lambda_{\min}^{1/3}\sqrt{r/n}$.
The triangle inequality for $\widehat\lambda_i^{1/3}\hat u_i$ now bounds
the additional relative Frobenius error by
\begin{equation}\label{eq:app:noisy-extraction-frobenius-perturbation}
 C\left(\kappa_{\mathrm{CP}}^2\tau\sqrt{\frac{rn\log^2n}{q}}+h_\sigma\right)
\end{equation}
and an additional relative row error at most
\begin{equation}\label{eq:app:noisy-extraction-row-perturbation}
 C\left(\kappa_{\mathrm{CP}}^2\sqrt n\,\tau\sqrt{\frac{rn\log^2n}{q}}
                      +\sqrt\mu\,h_\sigma\right).
\end{equation}
In \eqref{eq:app:noisy-extraction-row-perturbation}, the error in the
weight caused by the direction perturbation uses $\sqrt\mu\le\sqrt n$.
By \eqref{eq:app:noisy-basic-noise-scales}, the first term in
\eqref{eq:app:noisy-extraction-frobenius-perturbation} is at most
$Cc/(\kappa_{\mathrm{CP}}\mu^7r^{9/2}\sqrt n\log^{3/2}n)$.
Using \eqref{eq:app:noisy-feasible-rank} and \eqref{eq:app:noisy-small-parameters}
bounds the whole Frobenius error by $Cc/(\kappa_{\mathrm{CP}}^2\mu^2r)$.
Equations~\eqref{eq:app:noisy-extraction-row-error} and \eqref{eq:app:noisy-small-parameters}
bound the row error in \eqref{eq:app:noisy-extraction-row-perturbation}
by $Cc/(\kappa_{\mathrm{CP}}\mu^{3/2}r)$. After decreasing the noise
constant, adding these two errors to the corresponding estimates in
Corollary~\ref{cor:app:gd:init} proves \eqref{eq:app:noisy-init-full}.
Since every auxiliary sequence starts at $X^{(s),0}=X^0$,
the initial difference is zero and its $s$th-row error is bounded by
$\|X^0-X_\star\|_{2,\infty}$. This proves \eqref{eq:app:noisy-init-loo}.
}
\end{proof}
\endgroup

\begingroup\setlength{\emergencystretch}{2em}
\begin{lemma}[Noisy leave-one-out GD induction]\label{lem:noisy-cai-gd}
Assume the hypotheses of Theorem~\ref{thm:noisy_cai_sparse}.
Condition on an initializer satisfying \eqref{eq:app:noisy-init-full},
independent of the gradient-descent batch and its noise. With conditional
probability at least $1-O(n^{-5})$, the following bounds hold for every
$0\le t\le n^5$.
\begin{align*}
 \|X^t-X_\star\|_F
 &\lesssim\frac{\rho^t\|X_\star\|_F}{\kappa_{\mathrm{CP}}^2\mu^{3/2}\sqrt r}
 +\frac{s_\sigma}{\lambda_{\min}^{1/3}}\Big(\sum_i\lambda_i^{4/3}\Big)^{1/2},\\
 \|X^t-X_\star\|_{2,\infty}
 &\lesssim\frac{\rho^t\|X_\star\|_{2,\infty}}{\kappa_{\mathrm{CP}}^{2/3}\mu\sqrt r}
 +\frac{\kappa_{\mathrm{CP}}^{4/3}\sqrt{\mu/n}\,s_\sigma}{\lambda_{\min}^{1/3}}
 \Big(\sum_i\lambda_i^{4/3}\Big)^{1/2},
\end{align*}
where $\rho=1-\eta\lambda_{\min}^{4/3}/4$.
\end{lemma}
\begin{proof}
{
The induction controls the full iterate in Frobenius and row norm.
For each $s$, it also controls the distance to the auxiliary iterate
and the error of that iterate in row $s$. More precisely, we prove
\begin{subequations}\label{eq:app:noisy-four-induction}
\begin{align}
 \|X^t-X_\star\|_F
 &\le\left(\frac{C_1\rho^t}{\mu^{3/2}\sqrt r}+C_2s_\sigma\right)\|X_\star\|_F,
 \label{eq:app:noisy-four-induction-a}\\
 \|X^t-X_\star\|_{2,\infty}
 &\le\left(\frac{C_3\rho^t}{\mu\sqrt r}+C_4\sqrt\mu\,s_\sigma\right)\|X_\star\|_{2,\infty},
 \label{eq:app:noisy-four-induction-b}\\
 \max_s\|X^t-X^{(s),t}\|_F
 &\le\left(\frac{C_5\rho^t}{\mu\sqrt r}+C_6\sqrt\mu\,s_\sigma\right)\|X_\star\|_{2,\infty},
 \label{eq:app:noisy-four-induction-c}\\
 \max_s\|(X^{(s),t}-X_\star)_{s,:}\|_2
 &\le\left(\frac{C_7\rho^t}{\mu\sqrt r}+C_8\sqrt\mu\,s_\sigma\right)\|X_\star\|_{2,\infty}.
 \label{eq:app:noisy-four-induction-d}
\end{align}
\end{subequations}
The coefficients multiplying $\rho^t$ have orders
$C_1,C_5\asymp\kappa_{\mathrm{CP}}^{-2}$ and
$C_3,C_7\asymp\kappa_{\mathrm{CP}}^{-2/3}$. The coefficients multiplying
$s_\sigma$ reflect the sums of the individual column variances in
\eqref{eq:app:noisy-full-gradient} and \eqref{eq:app:noisy-loo-gradient}.
\begin{equation}\label{eq:app:noisy-weighted-coefficients}
 C_2\asymp\frac{(\sum_i\lambda_i^{4/3})^{1/2}}
              {\lambda_{\min}^{1/3}\|X_\star\|_F},\quad
 C_6\asymp\frac{C_2\|X_\star\|_F}{\sqrt n\|X_\star\|_{2,\infty}},\quad
 C_4,C_8\asymp\kappa_{\mathrm{CP}}^{4/3}C_6.
\end{equation}
The comparison constants in \eqref{eq:app:noisy-weighted-coefficients}
are absolute. First choose the constant in $C_2$ so that the full gradient
bound \eqref{eq:app:noisy-full-gradient} is at most
$C_2\lambda_{\min}^{4/3}s_\sigma\|X_\star\|_F/8$.
Then choose the constant in $C_6$ so that the slice gradient bound
\eqref{eq:app:noisy-loo-gradient} is at most
$C_6\lambda_{\min}^{4/3}\sqrt{\mu}\,s_\sigma\|X_\star\|_{2,\infty}/8$
and $C_2\|X_\star\|_F\le C_6\sqrt n\|X_\star\|_{2,\infty}$.
The geometric coefficients can all be multiplied by the same small
absolute constant. Choose the initialization constants in
\eqref{eq:app:noisy-init-full} so that
$c/(C\kappa_{\mathrm{CP}}^2)\le C_1$ and
$c/(8\kappa_{\mathrm{CP}}^{2/3})\le \min(C_3,C_7)$.
Then \eqref{eq:app:noisy-init-full} and \eqref{eq:app:noisy-init-loo}
give all four bounds at $t=0$.

Suppose \eqref{eq:app:noisy-four-induction} holds at time $t$.
For the auxiliary iterates the triangle inequality gives
\begin{equation}\label{eq:app:noisy-auxiliary-basin-comparison}
\begin{aligned}
 \|X^{(s),t}-X_\star\|_F
 &\le\|X^t-X_\star\|_F+\|X^t-X^{(s),t}\|_F,\\
 \|X^{(s),t}-X_\star\|_{2,\infty}
 &\le\|X^t-X_\star\|_{2,\infty}+\|X^t-X^{(s),t}\|_F.
\end{aligned}
\end{equation}
By CP incoherence, $\|X_\star\|_{2,\infty}^2
\le(\mu/n)\sum_i\lambda_i^{2/3}=(\mu/n)\|X_\star\|_F^2$.
Since $\sum_i\lambda_i^{4/3}\le\lambda_{\max}^{2/3}\|X_\star\|_F^2$ and
$\|X_\star\|_F\le\sqrt n\|X_\star\|_{2,\infty}$,
\eqref{eq:app:noisy-weighted-coefficients} gives $C_2,C_6\lesssim\kappa_{\mathrm{CP}}^{1/3}$
and $C_4,C_8\lesssim\kappa_{\mathrm{CP}}^{5/3}$.
By \eqref{eq:app:noisy-basic-noise-scales},
\[
 (C_2+C_4+C_6+C_8)\sqrt\mu\,s_\sigma
 \le\frac{Cc}{\kappa_{\mathrm{CP}}^{4/3}\mu^{13/2}r^5\sqrt n\log^2 n}
 \le\frac{Cc}{\kappa_{\mathrm{CP}}^{4/3}\mu\sqrt r}.
\]
With sufficiently small absolute multipliers for $C_1,C_3,C_5,C_7$,
the geometric terms in \eqref{eq:app:noisy-four-induction} satisfy the
two radius bounds in \eqref{eq:app:gd:convexity_region}. Choosing $c$
small enough gives the same bounds for the noise terms, with room for
the sums in \eqref{eq:app:noisy-auxiliary-basin-comparison}.
Thus $X^t$, $X^{(s),t}$ and $X_\star$ lie in $\mathcal B$.
The two norm constraints defining $\mathcal B$ are convex, so the
line segment joining any two of these points also lies in $\mathcal B$.

\smallskip\noindent\emph{Frobenius error of the full iterate.}
Since $\nabla F_0(X_\star)=0$, \eqref{eq:app:noisy-gradient-split}
gives $\nabla F_\sigma(X_\star)=-\mathcal N_H(X_\star)$. Hence
\begin{equation}\label{eq:app:noisy-frob-step}
 X^{t+1}-X_\star
 =X^t-X_\star-\eta[\nabla F_\sigma(X^t)-\nabla F_\sigma(X_\star)]
 +\eta\mathcal N_H(X_\star).
\end{equation}
For every $X\in\mathcal B$, Lemma~\ref{lem:app:gd:convexity} and
\eqref{eq:app:noisy-uniform-hessian} give the matrix bounds
\[
 \tfrac38\lambda_{\min}^{4/3}I
 \preceq\nabla^2 F_\sigma(X)=\nabla^2 F_0(X)-D\mathcal N_H(X)
 \preceq(4\lambda_{\max}^{4/3}+\tfrac18\lambda_{\min}^{4/3})I.
\]
The average Hessian on each of these segments has the same spectral bounds.
The norm of one gradient step is therefore controlled by the maximum of
$|1-\eta\lambda|$ at the two endpoints of this interval. For
$0<\eta\le(4\lambda_{\max}^{4/3})^{-1}$,
{
\begin{equation}\label{eq:app:noisy-gradient-contraction}
 \max\left\{
 1-\frac38\eta\lambda_{\min}^{4/3},\,
 \left|1-\eta\left(4\lambda_{\max}^{4/3}
                         +\frac18\lambda_{\min}^{4/3}\right)\right|
 \right\}
 \le1-\frac38\eta\lambda_{\min}^{4/3}.
\end{equation}
}
Indeed, $4\eta\lambda_{\max}^{4/3}\le1$, so a negative value of
$1-\eta\lambda$ at the upper endpoint has magnitude at most
$\eta\lambda_{\min}^{4/3}/8$. If its value is nonnegative, it is largest
at the lower endpoint.
Applying \eqref{eq:app:noisy-gradient-contraction} to the difference
of gradients in \eqref{eq:app:noisy-frob-step}, and
\eqref{eq:app:noisy-full-gradient} to $\mathcal N_H(X_\star)$, gives
{
\begin{equation}\label{eq:app:noisy-frobenius-envelope-update}
\begin{aligned}
 \frac{\|X^{t+1}-X_\star\|_F}{\|X_\star\|_F}
 &\le\left(1-\frac38\eta\lambda_{\min}^{4/3}\right)
       \left(\frac{C_1\rho^t}{\mu^{3/2}\sqrt r}+C_2s_\sigma\right)
       +\tfrac18 C_2\eta\lambda_{\min}^{4/3}s_\sigma\\
 &\le\frac{C_1\rho^{t+1}}{\mu^{3/2}\sqrt r}
       +\left(1-\tfrac14\eta\lambda_{\min}^{4/3}\right)C_2s_\sigma.
\end{aligned}
\end{equation}
}
This proves \eqref{eq:app:noisy-four-induction-a} at the next iteration.

\smallskip\noindent\emph{Difference between the full and auxiliary iterates.}
Subtracting their updates gives
\begin{equation}\label{eq:app:noisy-loo-step}
\begin{aligned}
 X^{t+1}-X^{(s),t+1}
 ={}&X^t-X^{(s),t}
 -\eta[\nabla F_\sigma(X^t)-\nabla F_\sigma(X^{(s),t})]\\
 &-\eta[\nabla F_\sigma(X^{(s),t})-\nabla F_\sigma^{(s)}(X^{(s),t})].
\end{aligned}
\end{equation}
Equation~\eqref{eq:app:noisy-gradient-contraction} bounds the Frobenius
norm of the first line of \eqref{eq:app:noisy-loo-step} by
$(1-3\eta\lambda_{\min}^{4/3}/8)\|X^t-X^{(s),t}\|_F$.
To estimate the second line, use \eqref{eq:app:noisy-loo-loss} and
\eqref{eq:app:noisy-gradient-split} to separate the two sources of error.
\[
 \nabla F_\sigma(X)-\nabla F_\sigma^{(s)}(X)
 =\bigl(\nabla F_0(X)-\nabla F^{(s)}(X)\bigr)
  -\bigl(\mathcal N_H(X)-\mathcal N_H^{(s)}(X)\bigr),
\]
where $F^{(s)}$ is the noiseless auxiliary loss. Both differences use
only entries with an index equal to $s$. Conditional on the initializer and the observations off this
slice, $X^{(s),t}$ is fixed. The Bernstein estimate
\eqref{eq:app:gd-slice-gradient-bound} therefore bounds the sampling
contribution at $X^{(s),t}$ by
{
\begin{equation}\label{eq:app:noisy-slice-sampling}
 C\lambda_{\max}^{4/3}\|X^{(s),t}-X_\star\|_{2,\infty}
 \left(\frac{\mu^2r\sqrt{\log n}}{n\sqrt q}
 +\frac{\mu^2r\log n}{qn^2}\right).
\end{equation}
}
With $C_6$ chosen as in \eqref{eq:app:noisy-weighted-coefficients},
\eqref{eq:app:noisy-loo-gradient} bounds the noise contribution by
$C_6\lambda_{\min}^{4/3}\sqrt\mu\,s_\sigma\|X_\star\|_{2,\infty}/8$.
For the sampling contribution, the lower bound on $q$ in
Theorem~\ref{thm:noisy_cai_sparse} gives
{
\begin{equation}\label{eq:app:noisy-slice-sampling-coefficient}
 \frac{\mu^2r\sqrt{\log n}}{n\sqrt q}
 +\frac{\mu^2r\log n}{qn^2}
 \le\frac1{\kappa_{\mathrm{CP}}^3\sqrt{c_1}\mu r\log^2n}
       +\frac1{\kappa_{\mathrm{CP}}^6c_1\mu^4r^3\log^4n}.
\end{equation}
}
The second inequality in \eqref{eq:app:noisy-auxiliary-basin-comparison}
and the induction bounds \eqref{eq:app:noisy-four-induction-b}--\eqref{eq:app:noisy-four-induction-c}
bound $\|X^{(s),t}-X_\star\|_{2,\infty}$ by
{
\begin{equation}\label{eq:app:noisy-auxiliary-row-envelope}
 \left[\frac{(C_3+C_5)\rho^t}{\mu\sqrt r}
       +(C_4+C_6)\sqrt\mu\,s_\sigma\right]
       \|X_\star\|_{2,\infty}.
\end{equation}
}
Substitute \eqref{eq:app:noisy-slice-sampling}--\eqref{eq:app:noisy-auxiliary-row-envelope}
and \eqref{eq:app:noisy-loo-gradient} into \eqref{eq:app:noisy-loo-step}.
The resulting upper bound for
$\|X^{t+1}-X^{(s),t+1}\|_F/\|X_\star\|_{2,\infty}$ is
{
\begin{equation}\label{eq:app:noisy-difference-envelope-update}
\begin{aligned}
 &\left(1-\frac38\eta\lambda_{\min}^{4/3}\right)
       \left(\frac{C_5\rho^t}{\mu\sqrt r}
                      +C_6\sqrt\mu\,s_\sigma\right)\\
 &\quad+\frac{C\eta\lambda_{\min}^{4/3}}
                    {\kappa_{\mathrm{CP}}^{5/3}\sqrt{c_1}}
       \left[\frac{(C_3+C_5)\rho^t}{\mu\sqrt r}
                +(C_4+C_6)\sqrt\mu\,s_\sigma\right]\\
 &\quad+\tfrac18 C_6\eta\lambda_{\min}^{4/3}\sqrt\mu\,s_\sigma.
\end{aligned}
\end{equation}
}
The power $\kappa_{\mathrm{CP}}^{-5/3}$ in
\eqref{eq:app:noisy-difference-envelope-update} comes from multiplying
the leading power $\kappa_{\mathrm{CP}}^{-3}$ in
\eqref{eq:app:noisy-slice-sampling-coefficient} by
$\lambda_{\max}^{4/3}/\lambda_{\min}^{4/3}=\kappa_{\mathrm{CP}}^{4/3}$.
The geometric choices give $C_5/(C_3+C_5)\gtrsim\kappa_{\mathrm{CP}}^{-4/3}$,
and \eqref{eq:app:noisy-weighted-coefficients} gives the same lower
bound for $C_6/(C_4+C_6)$. Thus a sufficiently large absolute $c_1$ gives
\[
 \frac{C}{\kappa_{\mathrm{CP}}^{5/3}\sqrt{c_1}}
 \le \min\left\{\frac{C_5}{8(C_3+C_5)},\frac{C_6}{8(C_4+C_6)}\right\}.
\]
Then \eqref{eq:app:noisy-difference-envelope-update} is bounded by
\[
 \left(1-\frac14\eta\lambda_{\min}^{4/3}\right)
       \frac{C_5\rho^t}{\mu\sqrt r}
 +\left(1-\frac18\eta\lambda_{\min}^{4/3}\right)C_6\sqrt\mu\,s_\sigma
 \le\frac{C_5\rho^{t+1}}{\mu\sqrt r}+C_6\sqrt\mu\,s_\sigma,
\]
which closes \eqref{eq:app:noisy-four-induction-c}.

\smallskip\noindent\emph{Row $s$ of the two iterates.}
In this row, \eqref{eq:app:noisy-loo-loss} uses population weights, so
$(\mathcal N_H^{(s)}(X^{(s),t}))_{s,:}=0$.
For a sufficiently small basin constant, \eqref{eq:app:noisy-basin-column-bounds}
and the step-size bound give, for each diagonal coefficient,
\[
 1-\eta\|x_i^{(s),t}\|_2^4
 \ge 1-2\eta\lambda_{\max}^{4/3}\ge\tfrac12.
\]
The row calculation in \eqref{eq:app:gd-own-row-step} therefore gives
{
\begin{equation}\label{eq:app:noisy-own-row-population-update}
\begin{aligned}
 \|(X^{(s),t+1}-X_\star)_{s,:}\|_2
 \le{}&(1-3\eta\lambda_{\min}^{4/3}/4)
 \|(X^{(s),t}-X_\star)_{s,:}\|_2+C\eta\lambda_{\max}^{4/3}\sqrt{\mu/n}\,
 \|X^{(s),t}-X_\star\|_F.
\end{aligned}
\end{equation}
}
Using \eqref{eq:app:noisy-auxiliary-basin-comparison},
$\|X_\star\|_F\le\sqrt n\|X_\star\|_{2,\infty}$, and $\mu\le n$ gives
\[
 \sqrt{\frac\mu n}\,
 \frac{\|X^{(s),t}-X_\star\|_F}{\|X_\star\|_{2,\infty}}
 \le\frac{(C_1+C_5)\rho^t}{\mu\sqrt r}
                 +2C_6\sqrt\mu\,s_\sigma.
\]
For the term from $X^t-X_\star$, use the choice
$C_2\|X_\star\|_F/(\sqrt n\|X_\star\|_{2,\infty})\le C_6$ in
\eqref{eq:app:noisy-weighted-coefficients}. For the term from
$X^t-X^{(s),t}$, use $\sqrt{\mu/n}\le1$. Each noise contribution is
therefore at most $C_6\sqrt\mu\,s_\sigma$.
Substitution into \eqref{eq:app:noisy-own-row-population-update} bounds
$\|(X^{(s),t+1}-X_\star)_{s,:}\|_2/\|X_\star\|_{2,\infty}$ by
{
\begin{equation}\label{eq:app:noisy-own-row-envelope-update}
\begin{aligned}
 &\left(1-\frac34\eta\lambda_{\min}^{4/3}\right)
       \left(\frac{C_7\rho^t}{\mu\sqrt r}
                      +C_8\sqrt\mu\,s_\sigma\right)+C\eta\lambda_{\min}^{4/3}\kappa_{\mathrm{CP}}^{4/3}
       \left[\frac{(C_1+C_5)\rho^t}{\mu\sqrt r}
                      +2C_6\sqrt\mu\,s_\sigma\right].
\end{aligned}
\end{equation}
}
Choose $C_7\ge2C\kappa_{\mathrm{CP}}^{4/3}(C_1+C_5)$ and
$C_8\ge(8C/3)\kappa_{\mathrm{CP}}^{4/3}C_6$.
Then \eqref{eq:app:noisy-own-row-envelope-update} is at most
\[
 \frac{C_7\rho^{t+1}}{\mu\sqrt r}+C_8\sqrt\mu\,s_\sigma,
\]
which proves \eqref{eq:app:noisy-four-induction-d}.
The full row bound follows from
\[
 \|(X^{t+1}-X_\star)_{s,:}\|_2
 \le\|X^{t+1}-X^{(s),t+1}\|_F
                   +\|(X^{(s),t+1}-X_\star)_{s,:}\|_2
\]
by taking $C_3\ge C_5+C_7$ and $C_4\ge C_6+C_8$.
The choices of constants are compatible. For the geometric terms, first fix
the relative sizes of $C_1,C_5$, then $C_7$ for the auxiliary row, and
finally $C_3$ for the full row. Multiplying all four by one small absolute
constant enforces the basin condition, and Lemma~\ref{lem:noisy-cai-refinement}
gives the required initial accuracy. For the noise terms, fix $C_2,C_6$
from \eqref{eq:app:noisy-weighted-coefficients}, then $C_8$ for the
auxiliary row and $C_4$ for the full row. After these choices, increase
$c_1$ to absorb the sampling term in \eqref{eq:app:noisy-difference-envelope-update}
and decrease the absolute constant in the bound on $\tau$ to keep
the noise errors inside \eqref{eq:app:gd:convexity_region}.

\smallskip\noindent\emph{Probability over all iterations.}
Fix $s,t$ and condition only on the
initialization and observations with all indices different from $s$.
These observations determine $X^{(s),t}$ and hence whether it satisfies
\eqref{eq:app:gd:convexity_region}. Whenever it does, the conditional
probability that either the sampling bound \eqref{eq:app:noisy-slice-sampling}
or the noise bound \eqref{eq:app:noisy-loo-gradient} fails at $s$
is at most $Cn^{-11}$. Summing these probabilities over $s$ at fixed $t$ gives
\[
 \mathbb P\!\left\{\exists\,s:\ X^{(s),t}\in\mathcal B,\
   \text{one of these two bounds fails at $s$}
   \,\middle|\, X^0\right\}
 \le\sum_{s=1}^n Cn^{-11}=Cn^{-10}.
\]

On the Hessian event of Lemma~\ref{lem:app:gd:convexity}
and the Gaussian event $\mathcal E_H$ of Lemma~\ref{lem:noisy-cai-gaussian}, an induction step can fail only if
\eqref{eq:app:noisy-slice-sampling} or \eqref{eq:app:noisy-loo-gradient}
fails for an auxiliary iterate in the basin. Consequently,
\[
 \mathbb P\{\text{some induction step fails for }0\le t<n^5\mid X^0\}
 \le O(n^{-10})+\sum_{t=0}^{n^5-1}Cn^{-10}\le Cn^{-5}.
\]
Thus \eqref{eq:app:noisy-four-induction} holds through $t=n^5$ for
the single sample and noise tensor used throughout gradient descent.
Substituting \eqref{eq:app:noisy-weighted-coefficients} into
\eqref{eq:app:noisy-four-induction-a} and \eqref{eq:app:noisy-four-induction-b}
gives the Frobenius and row bounds in the lemma.
}
\end{proof}
\endgroup

\begingroup\setlength{\emergencystretch}{2em}
\begin{lemma}[From factor error to tensor error]\label{lem:noisy-cai-factor-tensor}
Suppose $X$ satisfies the two basin conditions in
\eqref{eq:app:gd:convexity_region} and
\[
 \|X-X_\star\|_F\le\varepsilon_F\|X_\star\|_F,
 \qquad \|X-X_\star\|_{2,\infty}\le\varepsilon_\infty\|X_\star\|_{2,\infty}.
\]
Then
\[
 \|\mathcal S(X)-T\|_F\lesssim\kappa_{\mathrm{CP}}^{2/3}\varepsilon_F\|T\|_F,
 \qquad
 \|\mathcal S(X)-T\|_\infty\lesssim
 \kappa_{\mathrm{CP}}^{2/3}\mu\sqrt r\,\varepsilon_\infty\|T\|_\infty.
\]
\end{lemma}
\begin{proof}
{
The deterministic tensor-error estimates in
Section~\ref{sec:app:gd:completion} apply with $X_\star$ as the true
factor matrix and $X$ as the estimate. Indeed,
\eqref{eq:app:gd:convexity_region} and \eqref{eq:app:noisy-basin-column-bounds} give
\[
 \|X-X_\star\|_F\le c\lambda_{\max}^{1/3}/\mu,\qquad
 \max_i\|x_i\|_\infty\le C\lambda_{\max}^{1/3}\sqrt{\mu/n}.
\]
Equation~\eqref{eq:app:noisy-feasible-rank} implies $r\mu/n\lesssim n^{-1/2}$, as required in
\eqref{eq:app:gd-tensor-gram-bounds}.

For the Frobenius error, divide \eqref{eq:app:gd-tensor-frobenius-conversion}
by the lower bound for $\|T\|_F$ in \eqref{eq:app:gd-tensor-frobenius-lower}.
Indeed, $\|T\|_F^2\gtrsim\sum_i\lambda_i^2
\ge\lambda_{\min}^{4/3}\sum_i\lambda_i^{2/3}
=\lambda_{\min}^{4/3}\|X_\star\|_F^2$. Therefore
\[
 \frac{\|\mathcal S(X)-T\|_F}{\|T\|_F}
 \lesssim\frac{\lambda_{\max}^{2/3}\|X-X_\star\|_F}
                  {\lambda_{\min}^{2/3}\|X_\star\|_F}
 \le\kappa_{\mathrm{CP}}^{2/3}\varepsilon_F .
\]
For the entrywise error, divide \eqref{eq:app:gd-tensor-infinity-conversion}
by \eqref{eq:app:gd-tensor-infinity-lower}. The common factor
$\|X_\star\|_{2,\infty}$ cancels. Using
$\|X_\star\|_{2,\infty}\le\lambda_{\max}^{1/3}\sqrt{\mu r/n}$ gives
\[
 \frac{\|\mathcal S(X)-T\|_\infty}{\|T\|_\infty}
 \lesssim \frac{\lambda_{\max}^{1/3}}{\lambda_{\min}^{2/3}}
          \sqrt{\mu n}\,\|X-X_\star\|_{2,\infty}
 \lesssim \kappa_{\mathrm{CP}}^{2/3}\mu\sqrt r\,\varepsilon_\infty .
\]
}
\end{proof}

\begin{proof}[Proof of Theorem~\ref{thm:noisy_cai_sparse}]
{
Use $\delta/2$ for the factor-extraction failure parameter in
Lemma~\ref{lem:noisy-cai-refinement}. The gradient-descent batch and its
noise are independent of the initialization, so we may apply
Lemma~\ref{lem:noisy-cai-gd} conditional on any initializer satisfying
\eqref{eq:app:noisy-init-full}. Adding the initialization and iteration
failure probabilities gives
\[
 \mathbb P\!\left\{\text{\eqref{eq:app:noisy-four-induction} fails for some }0\le t\le n^5\right\}
 \le\frac{\delta}{2}+O(n^{-3})+O(n^{-5})\le\delta
\]
for all sufficiently large $n$ depending on $\delta$.

Under \eqref{eq:app:noisy-four-induction}, every $X^t$ with
$0\le t\le n^5$ satisfies \eqref{eq:app:gd:convexity_region}.
Applying Lemma~\ref{lem:noisy-cai-factor-tensor} to the terms proportional
to $\rho^t$ in Lemma~\ref{lem:noisy-cai-gd} gives a contribution of at most
$C\rho^t$ to each relative tensor error.

For the terms proportional to $s_\sigma$, keep the sums of powers of the
individual $\lambda_i$ in Lemma~\ref{lem:noisy-cai-gd}. For Frobenius error,
\eqref{eq:app:gd-tensor-frobenius-conversion} and
\eqref{eq:app:gd-tensor-frobenius-lower} bound the relative statistical term by
\[
 C\frac{\lambda_{\max}^{2/3}}{\lambda_{\min}^{1/3}}
 \left(\frac{\sum_i\lambda_i^{4/3}}{\sum_i\lambda_i^2}\right)^{1/2}
 s_\sigma\le C\kappa_{\mathrm{CP}}^{2/3}s_\sigma,
\]
since $\sum_i\lambda_i^2\ge\lambda_{\min}^{2/3}\sum_i\lambda_i^{4/3}$.
For entrywise error, $\|T\|_\infty\ge n^{-3/2}\|T\|_F$ and CP incoherence
gives $\|X_\star\|_{2,\infty}\le\sqrt{(\mu/n)\sum_i\lambda_i^{2/3}}$.
Thus \eqref{eq:app:gd-tensor-infinity-conversion} bounds the relative statistical term by
\[
 C\kappa_{\mathrm{CP}}^{5/3}\mu^{3/2}s_\sigma
 \left(\frac{(\sum_i\lambda_i^{2/3})(\sum_i\lambda_i^{4/3})}
                  {\sum_i\lambda_i^2}\right)^{1/2}
 \le C\kappa_{\mathrm{CP}}^{5/3}\sqrt{\mu^3r}\,s_\sigma.
\]
Indeed, summing
$(\lambda_i^{2/3}-\lambda_j^{2/3})(\lambda_i^{4/3}-\lambda_j^{4/3})\ge0$
over $i,j$ gives $(\sum_i\lambda_i^{2/3})(\sum_i\lambda_i^{4/3})
\le r\sum_i\lambda_i^2$.
Substituting $s_\sigma=\tau\sqrt{n\log n/q}$ proves the theorem.
}
\end{proof}
\endgroup
\endgroup

\end{document}